\documentclass[a4paper, 10pt, reqno]{amsart}
\usepackage[T1]{fontenc}
\usepackage{lmodern}
\usepackage[utf8]{inputenc}

\usepackage{amsmath,amssymb,graphicx,mathtools} 
\usepackage{amsthm}
\usepackage[foot]{amsaddr}
\usepackage[a4paper,left=2.5cm,right=2.5cm,top=3cm,bottom=3cm]{geometry}
\usepackage{enumerate}
\usepackage{enumitem}
\usepackage{tikz} 
\usetikzlibrary{arrows,decorations.markings}
\usepackage[colorlinks, linkcolor=red, urlcolor=blue, citecolor=blue]{hyperref}
\usepackage{url}
\usepackage{comment}

\usepackage[
    giveninits=true,
    url=false,
    doi=false,
    isbn=false,
    eprint=true,
    datamodel=mrnumber,
    sorting=nyt, % Sort by name, year, title
    sortcites=false, % Do not sort citations in the text
    maxbibnames=99,
    maxcitenames=4,
    backref=false,
    block=space,
    backend=biber,
    style=phys,
    biblabel=brackets
]{biblatex}
\AtEveryBibitem{\clearfield{month}}
\AtEveryCitekey{\clearfield{month}} 
\renewbibmacro{in:}{}
\DeclareFieldFormat[article]{title}{\emph{#1}} 
\DeclareFieldFormat{mrnumber}{\ifhyperref{\href{http://www.ams.org/mathscinet-getitem?mr=#1}{\nolinkurl{MR#1}}}{\nolinkurl{#1}}}
\DeclareFieldFormat{pmid}{\ifhyperref{\href{https://www.ncbi.nlm.nih.gov/pubmed/#1}{\nolinkurl{PMID#1}}}{\nolinkurl{#1}}}
\DeclareFieldFormat{eprint}{\ifhyperref{\href{https://arxiv.org/abs/#1}{\nolinkurl{arXiv:#1}}}{\nolinkurl{#1}}}
\renewbibmacro*{doi+eprint+url}{%
  \iftoggle{bbx:doi}{\printfield{doi}}{}
  \newunit\newblock%
  \printfield{mrnumber}%
  \newunit\newblock%
  \printfield{pmid}%
  \newunit\newblock%
  \printfield{eprint}%
  \iftoggle{bbx:url}{\usebibmacro{url+urldate}}{}}
\bibliography{refs}
\usepackage{caption}
\usepackage{subcaption}

\DeclareMathOperator{\Tr}{Tr}

\DeclareMathOperator*{\argmax}{arg\,max}

\newcommand{\norm}[1]{\lVert#1\rVert}
\newcommand{\R}{\mathbb{R}} 
\newcommand{\E}{\mathbb{E}}
\newcommand{\N}{\mathbb{N}} 

\theoremstyle{plain}
\newtheorem{thm}{Theorem}[section]

\newtheorem{lem}[thm]{Lemma}
\newtheorem{prop}[thm]{Proposition}

\theoremstyle{definition}
\newtheorem{defn}[thm]{Definition}

\theoremstyle{remark}
\newtheorem{rmk}[thm]{Remark}

\numberwithin{equation}{section}

\makeatletter
\renewcommand\subsection{\@startsection{subsection}{2}%
  \z@{-.5\linespacing\@plus-.7\linespacing}{.5\linespacing}%
  {\normalfont\scshape}}
\renewcommand\subsubsection{\@startsection{subsubsection}{3}%
  \z@{.5\linespacing\@plus.7\linespacing}{-.5em}%
  {\normalfont\scshape}}
\makeatother

\makeatletter
\@namedef{subjclassname@2020}{%
  \textup{2020} Mathematics Subject Classification}
\makeatother

\usepackage[format=plain,
            labelfont=sc]{caption}
\usepackage[toc,page]{appendix} 
\usepackage{scalerel,stackengine}
\stackMath%
\newcommand\reallywidehat[1]{%
\savestack{\tmpbox}{\stretchto{%
  \scaleto{%
    \scalerel*[\widthof{\ensuremath{#1}}]{\kern.1pt\mathchar"0362\kern.1pt}%
    {\rule{0ex}{\textheight}}
  }{\textheight}% 
}{2.4ex}}%
\stackon[-6.9pt]{#1}{\tmpbox}%
}

\makeatletter
\newsavebox\myboxA
\newsavebox\myboxB
\newlength\mylenA
\newcommand*\xoverline[2][0.75]{%
    \sbox{\myboxA}{$\m@th#2$}%
    \setbox\myboxB\null% Phantom box
    \ht\myboxB=\ht\myboxA%
    \dp\myboxB=\dp\myboxA%
    \wd\myboxB=#1\wd\myboxA% Scale phantom
    \sbox\myboxB{$\m@th\overline{\copy\myboxB}$}%  Overlined phantom
    \setlength\mylenA{\the\wd\myboxA}%   calc width diff
    \addtolength\mylenA{-\the\wd\myboxB}%
    \ifdim\wd\myboxB<\wd\myboxA%
       \rlap{\hskip 0.5\mylenA\usebox\myboxB}{\usebox\myboxA}%
    \else
        \hskip -0.5\mylenA\rlap{\usebox\myboxA}{\hskip 0.5\mylenA\usebox\myboxB}%
    \fi}
\makeatother    

\usepackage{fancyhdr}

\newcommand\shortitle{SGD in high dimensions for multi-spiked tensor PCA}
\newcommand\name{gérard ben arous, cédric gerbelot, and vanessa piccolo}

\begin{document}

\title{Stochastic gradient descent in high dimensions for multi-spiked tensor PCA}
\author{Gérard Ben Arous\(^1\)}
\author{Cédric Gerbelot\(^1\)}
\address{\(^1\)Courant Institute of Mathematical Sciences, New York University}
\email{benarous@cims.nyu.edu}
\email{cedric.gerbelot@cims.nyu.edu}

\author{Vanessa Piccolo\(^2\)}
\address{\(^2\)Unité de Mathématiques Pures et Appliquées (UMPA), ENS Lyon}
\email{vanessa.piccolo@ens-lyon.fr}
\subjclass[2020]{68Q87, 62F30, 60G42} 
\keywords{Multi-index model, Multi-spiked tensor PCA, Tensor estimation, Online stochastic gradient descent, Sequential recovery, Subspace recovery}
\date{\today}

%%%%%%%%%%%%%%%%%%%%%%%%%%%%%%%%%%%%%%%%%%%%%%%%%%%%%%%%%%%%%%%%%%%%%%%%
%%%%%%%%%%%%%%%%%%%%%%%%%%%%%%%%%%%%%%%%%%%%%%%%%%%%%%%%%%%%%%%%%%%%%%%%
\begin{abstract}
We study the high-dimensional dynamics of online stochastic gradient descent (SGD) for the multi-spiked tensor model. This multi-index model arises from the tensor principal component analysis (PCA) problem with multiple spikes, where the goal is to estimate \(r\) unknown signal vectors within the \(N\)-dimensional unit sphere through maximum likelihood estimation from noisy observations of a \(p\)-tensor. We determine the number of samples and the conditions on the signal-to-noise ratios (SNRs) required to efficiently recover the unknown spikes from natural random initializations. We show that full recovery of all spikes is possible provided a number of sample scaling as \(N^{p-2}\), matching the algorithmic threshold identified in the rank-one case~\cite{arous2020algorithmic,arous2021online}. Our results are obtained through a detailed analysis of a low-dimensional system that describes the evolution of the correlations between the estimators and the spikes, while sharply controlling the noise in the dynamics. We find the spikes are recovered sequentially in a process we term ``sequential elimination'': once a correlation exceeds a critical threshold, all correlations sharing a row or column index become sufficiently small, allowing the next correlation to grow and become macroscopic. The order in which correlations become macroscopic depends on their initial values and the corresponding SNRs, leading to either exact recovery or recovery of a permutation of the spikes. In the matrix case, when \(p=2\), if the SNRs are sufficiently separated, we achieve exact recovery of the spikes, whereas equal SNRs lead to recovery of the subspace spanned by them. 
\end{abstract}

\maketitle	
%%%%%%%%%%%%%%%%%%%%%%%%%%%%%%%%%%%%%%%%%%%%%%%%%%%%%%%%%%%%%%%%%%%%%%%%%%%%%%%%%
{
\hypersetup{linkcolor=black}
\tableofcontents
}

\section{Introduction} \label{section:introduction}
Understanding the dynamics of gradient-based algorithms for optimizing high-dimensional, non-convex random functions is crucial for progress in data science. Despite the difficulties of non-convex landscapes, first-order optimization methods---particularly stochastic gradient descent (SGD)---have been remarkably successful in practice~\cite{robbins1951stochastic}, especially in deep learning applications~\cite{lecun1998gradient,bottou2010large}. Theoretical guarantees on the time and sample complexity of SGD are well-established in convex and quasi-convex settings~\cite{bottou2003large, needell2014stochastic, nesterov2018lectures, dieuleveut2020bridging}, but extending these results to non-convex problems remains challenging. A promising strategy is to study specific instances of the optimization problem that are simple enough for closed-form analysis, yet retain the complexities of larger models. In such cases, a key observation is that only a few dominant degrees of freedom---key directions in the parameter space---appear to govern the behavior of gradient-based algorithms, as exemplified by the \emph{neural collapse} phenomenon~\cite{papyan2020traces,papyan2020}. Building on this insight, recent work by the first author, in collaboration with Gheissari and Jagannath~\cite{arous2020algorithmic,arous2021online,benarousneurips, ben2022effective} and with Gheissari, Huang, and Jagannath~\cite{arous2023eigenspace}, shows that the high-dimensional dynamics of (stochastic) gradient descent can often be reduced to low-dimensional, autonomous \emph{effective dynamics} expressed in terms of a small set of \emph{summary statistics}. This reduction provides a clear framework for understanding how factors such as initialization, step size, sample size, and runtime shape the algorithm's trajectory toward fixed points in the optimization landscape.

Motivated by this line of work, the present paper studies the benchmark problem of \emph{multi-spiked tensor estimation}. The goal is to estimate \(r\) unknown vectors on the \(N\)-dimensional unit sphere from noisy observations of a corresponding rank-\(r\) tensor (see the precise formulation in~\eqref{eq: spiked tensor model}). This problem, also known as Tensor Principal Component Analysis (Tensor PCA), was introduced by Johnstone~\cite{Johnstone} for matrices and by Richard and Montanari~\cite{MontanariRichard} for higher-order tensors. Our goal is to determine the number of samples required for online SGD to efficiently recover the \(r\) unknown vectors. In the rank-one (single-spike) case, recovery thresholds have been extensively studied for first-order optimization methods~\cite{MontanariRichard, arous2020algorithmic, arous2021online}, power iteration~\cite{Wu24}, and spectral and Sum-of-Squares methods~\cite{perry2018optimality, bandeira2020}. In the multi-spiked case, recovery thresholds have been established using power iteration~\cite{HuangPCA}. See Subsection~\ref{subsection: related work} for a more detailed discussion of these results. Compared to the single-spike setting, the multi-spiked scenario exhibits fundamentally richer phenomenology: the effective dynamics governing gradient-based algorithms are no longer one-dimensional but genuinely multidimensional.

This multi-spiked problem serves as a benchmark for general multi-index models, which have received significant attention in recent years~\cite{soltanolkotabi2017learning, dudeja2018learning, yehudai2020learning, frei2020agnostic, bietti2022learning, mousavi2022neural, damian2022neural, ba2022high, berthier2023learning, abbe2023sgd, bietti2023learning, dandi2023learning}. A multi-index model consists of a function parameterized by a finite set of unknown directions \(\boldsymbol{v}_1, \ldots, \boldsymbol{v}_r \in \R^N\) and a multivariate link function \(f_\ast \colon \R^r \to \R\), which takes as input the inner products between input vectors in \(\R^N\) and the relevant directions. Here \(N\) denotes the ambient dimension, while \(r \ll  N\) is the number of latent coordinates, with \(r=1\) corresponding to the classical single-index model. From an optimization perspective, multi-index models give rise to low-dimensional yet multidimensional effective dynamics. In particular, when considering a set of \(r\) orthonormal estimators \(\boldsymbol{x}_1, \ldots, \boldsymbol{x}_r \in \R^N\), as we focus on this paper, the effective dynamics are governed by the \(r^2\) inner products \(\langle \boldsymbol{v}_i, \boldsymbol{x}_j \rangle\). In this paper, we provide a precise analysis of this \(r^2\)-dimensional dynamical system in the context of multi-spiked tensor PCA and determine the sample complexity required for the high-dimensional empirical dynamics of SGD to closely track the low-dimensional population dynamics.

For the single-spiked tensor model and single-index models with known link function, previous work~\cite{arous2020algorithmic, arous2021online} showed that the difficulty of recovery using gradient-based methods depends on the curvature of the loss landscape near initialization. This curvature can be quantified by the first non-zero coefficient in the Taylor expansion of the population loss near the equator, also called the \emph{information exponent}~\cite{arous2021online}. Various extensions have addressed the problem of learning both the hidden direction and the unknown link function~\cite{bietti2022learning, berthier2023learning, Mahankali2023}. In contrast, multi-index models introduce new challenges: the effective dynamics are multidimensional, potentially non-monotone, and may have multiple fixed points, making global convergence guarantees significantly harder to establish. Recent works~\cite{abbe2022merged, abbe2023sgd, dandi2023learning, bietti2023learning} have made progress for specific multi-index models by designing tailored two-timescale modifications of SGD. These approaches effectively decouple the dynamics into monotone components (or their superposition), which allows for theoretical guarantees on subspace recovery while bypassing the intricate interactions between correlations. We refer the reader to Subsection~\ref{subsection: related work} for further details on this literature.

These developments raise natural questions that motivate our study: (1) Can we characterize the global convergence of multidimensional effective population dynamics for gradient descent and its variants without relying on a two-timescale algorithm? (2) Can we identify which fixed points are reached and determine their probabilities? (3) Can we establish sample complexity guarantees for the corresponding empirical dynamics of SGD---the most widely used algorithm in practice? Our work addresses these three questions for single-pass online SGD in the multi-spiked tensor PCA setting, under the simplifying assumption that the unknown vectors are orthogonal. In our companion papers~\cite{langevin, gradientflow}, we also analyze continuous high-dimensional dynamics, such as Langevin dynamics and gradient flow, for the finite-rank spiked tensor model. Extensions to more practically relevant algorithms, such as multi-pass or reuse-based SGD, are left for future work.

Our work uncovers a richer phenomenology than in the single-spiked case. We focus on three natural notions of recovery: exact recovery of spikes, recovery up to permutation, and recovery of the hidden subspace up to rotation, the latter being relevant only in the case \(p=2\) with equal signal-to-noise ratios (SNRs). For the first two notions, we characterize a \emph{sequential elimination process}, in which each (possibly permuted) spike is recovered one by one, progressively eliminating correlations with the recovered directions. Intuitively, once one direction (say \(\boldsymbol{v}_1\)) is recovered---meaning \(\langle \boldsymbol{v}_1, \boldsymbol{x}_1 \rangle\) is close to one---all cross-correlations \(\langle \boldsymbol{v}_k, \boldsymbol{x}_1 \rangle\) and \(\langle \boldsymbol{v}_1, \boldsymbol{x}_k \rangle\) for \(k \ge 2\) become negligible, allowing the remaining correlations to grow and subsequent directions to be recovered in turn. The order in which the correlations become macroscopic (or, equivalently, the order in which directions are recovered) depends on their initial values and on the corresponding SNRs, leading to either exact recovery or recovery up to a permutation. The main technical challenge lies in controlling the interactions among the \(r^2\) inner products \(\langle \boldsymbol{v}_i, \boldsymbol{x}_j \rangle\), which arise from the gradient structure and the orthogonality constraints of the model. These interactions can locally dominate the drift and must be handled carefully to ensure that the correlations evolve monotonically and that sequential recovery occurs. In addition, the stochastic gradient noise imposes constraints on the admissible step sizes and time scales, whose contribution must be controlled sharply enough to yield statistically meaningful sample complexity bounds. Our proof relies on systems of discrete-time difference inequalities, derived from the signal-plus-martingale decomposition method of~\cite{arous2021online,tan2023online}. The overall strategy parallels that of our companion works on Langevin and gradient-flow dynamics~\cite{langevin,gradientflow}, where the same sequential recovery phenomenon is established. However, the discrete-time setting introduces additional challenges arising from stochastic gradient updates and other nonlinear effects, which require substantially different analytical techniques. Despite these differences, our analysis achieves the optimal sample complexity---up to logarithmic factors---within the relevant regime studied in the literature.

\subsection{Setting} \label{subsection: SGD}

The multi-spiked tensor model is formalized as follows. Fix integers \(p \ge 2\) and \(r \ge 1\), and suppose that we are given \(M\) i.i.d.\ observations \(\boldsymbol{Y}^\ell\) of a rank \(r\) \(p\)-tensor on \(\R^N\) in the presence of additive noise, i.e.,
\begin{equation}\label{eq: spiked tensor model}
\boldsymbol{Y}^\ell = \boldsymbol{W}^\ell + \sqrt{N} \sum_{i=1}^r \lambda_i \boldsymbol{v}_i^{\otimes p},
\end{equation}
where \((\boldsymbol{W}^\ell)_{\ell=1}^M\) are i.i.d.\ samples of a noise \(p\)-tensor with i.i.d.\ centered sub-Gaussian entries \(W^\ell_{i_1, \ldots, i_p}\), \(\lambda_1 \geq \ldots \geq \lambda_r \geq 0\) are the signal-to-noise ratios (SNRs), assumed to be of order \(1\), and \(\boldsymbol{v}_1, \ldots,\boldsymbol{v}_r\) are \emph{unknown, orthogonal vectors} lying on the \(N\)-dimensional unit sphere \(\mathbb{S}^{N-1}\). Our goal is to estimate the unknown signal vectors based on the sample of i.i.d.\ training data \((\boldsymbol{Y}^\ell)_{\ell=1}^M\). For sub-Gaussian noise, the Gaussian log-likelihood is a natural proxy, leading to solving the optimization problem:
\begin{equation} \label{eq: tensor PCA}
\min_{\boldsymbol{X} \colon \boldsymbol{X}^\top \boldsymbol{X} = \boldsymbol{I}_r} \left \{ - \sum_{i=1}^r \lambda_i \langle \boldsymbol{Y}, \boldsymbol{x}_i^{\otimes p} \rangle \right \},
\end{equation}
where \(\boldsymbol{X} = [\boldsymbol{x}_1, \ldots, \boldsymbol{x}_r] \in \R^{N \times r}\) and the feasible set \( \textnormal{St}(N,r) = \{\boldsymbol{X} \in \R^{N \times r} \colon \boldsymbol{X}^\top \boldsymbol{X} = \boldsymbol{I}_r \} \) is known as the \emph{Stiefel manifold}, once equipped with its natural sub-manifold structure from \(\R^{N \times r}\). The optimization problem~\eqref{eq: tensor PCA} is therefore equivalent to minimizing the loss function \(\mathcal{L}_{N,r} \colon \text{St}(N,r) \times (\R^N)^{\otimes p} \to \R\) given by
\begin{equation} \label{eq: loss function}
\mathcal{L}_{N,r}(\boldsymbol{X}; \boldsymbol{Y}) = H_{N,r}(\boldsymbol{X}) + \Phi_{N,r}(\boldsymbol{X}), .
\end{equation}
where \(H_{N,r}\) denotes the noise part given by
\begin{equation} \label{eq: noise part}
H_{N,r}(\boldsymbol{X}) = - \sum_{i=1}^r \lambda_i \langle \boldsymbol{W}, \boldsymbol{x}_i^{\otimes p} \rangle ,
\end{equation}
and \(\Phi_{N,r}\) denotes the population loss given by
\begin{equation} \label{eq: population loss}
\Phi_{N,r}(\boldsymbol{X}) = \E \left [ \mathcal{L}_{N,r}(\boldsymbol{X}; \boldsymbol{Y})\right ] = - \sum_{1 \leq i,j \leq r} \sqrt{N} \lambda_i  \lambda_j m_{ij}^p(\boldsymbol{X}),
\end{equation}
where \(m_{ij}(\boldsymbol{X}) = \langle \boldsymbol{v}_i, \boldsymbol{x}_j \rangle\).
\begin{defn}[Correlation] \label{def: correlation}
For every \(1\le i,j \le r\), we call \(m_{ij}(\boldsymbol{X})\) the \emph{correlation} of \(\boldsymbol{x}_j\) with \(\boldsymbol{v}_i\). 
\end{defn}
Our work aims to solve the optimization problem~\eqref{eq: tensor PCA} in the large-\(N\) limit while the rank \(r\) is fixed. As previously mentioned, to tackle this tensor estimation problem, we use a first-order optimization method, namely \emph{online stochastic gradient descent (SGD)}. This local algorithm has been studied by the first author, joint with Gheissari and Jagannath in~\cite{arous2021online}, to determine the algorithmic thresholds for the rank-one spiked tensor model. We define the online SGD algorithm as follows.

Let \(\boldsymbol{X}_\ell \in \textnormal{St}(N,r)\) denote the output of the algorithm at time \(\ell\). The online SGD algorithm with initialization \(\boldsymbol{X}_0 \in \textnormal{St}(N,r)\), which is possibly random, and with step size  parameter \(\delta > 0\), will be run using the observations \((\boldsymbol{Y}^\ell)_{\ell=1}^M\) according to the following update rule:
\begin{equation} \label{eq: online SGD}
\boldsymbol{X}_\ell = R_{\boldsymbol{X}_{\ell - 1}}\left (-\frac{\delta}{N} \nabla_{\textnormal{St}} \mathcal{L}_{N,r} (\boldsymbol{X}_{\ell -1}; \boldsymbol{Y}^\ell)\right ),
\end{equation}
where \(R_{\boldsymbol{X}} \colon T_{\boldsymbol{X}} \textnormal{St}(N,r) \to \textnormal{St}(N,r)\) is a retraction map that moves a point from the tangent space \(T_{\boldsymbol{X}} \textnormal{St}(N,r)\) back onto the Stiefel manifold \(\textnormal{St}(N,r)\), while preserving its geometric properties such as orthonormality (see e.g.~\cite[Section 7.3]{boumal2023}). The tangent space \(T_{\boldsymbol{X}} \textnormal{St}(N,r)\) at a point \(\boldsymbol{X} \in \textnormal{St}(N,r)\) is given by 
\[
T_{\boldsymbol{X}} \textnormal{St}(N,r) = \{\boldsymbol{V} \in \R^{N \times r} \colon \boldsymbol{X}^\top \boldsymbol{V} + \boldsymbol{V}^\top \boldsymbol{X} = 0\}.
\]
There are multiple ways to define the retraction map. Here, we use the polar retraction, defined as
\begin{equation} \label{eq: polar retraction} R_{\boldsymbol{X}}(\boldsymbol{U}) = (\boldsymbol{X} + \boldsymbol{U}) \left(\boldsymbol{I}_r + \boldsymbol{U}^\top \boldsymbol{U}\right)^{-1/2}. 
\end{equation}
In the update rule~\eqref{eq: online SGD}, \(\nabla_{\textnormal{St}}\) denotes the Riemannian gradient on \(\textnormal{St}(N,r)\). In particular, for a function \(F \colon \textnormal{St}(N,r) \to \R\), the Riemannian gradient is given by
\begin{equation} \label{eq: Stiefel gradient}
\nabla_{\textnormal{St}} F(\boldsymbol{X}) = \nabla F(\boldsymbol{X}) - \frac{1}{2} \boldsymbol{X}(\boldsymbol{X}^\top \nabla F(\boldsymbol{X}) + \nabla F(\boldsymbol{X})^\top \boldsymbol{X}),
\end{equation}
where \(\nabla\) denotes the Euclidean gradient. The algorithm runs for \(M\) steps, and the final output \(\boldsymbol{X}_M\) is taken as the estimator.

\subsection{Main results} \label{subsection: main asymptotic results}

Our goal is to determine the \emph{sample complexity} (i.e., the number \(M\) of observations) required to recover the unknown orthogonal vectors \(\boldsymbol{v}_1, \ldots, \boldsymbol{v}_r \in \mathbb{S}^{N-1}\) via online SGD~\eqref{eq: online SGD}. We formalize the notion of recovering all spikes as follows.

\begin{defn}[Recovery of all spikes]
We say that online SGD achieves \emph{recovery up to permutation} of the \(r\) unknown vectors \(\boldsymbol{v}_1, \ldots, \boldsymbol{v}_r \in \textnormal{St}(N,r)\) if there exists a permutation \(\pi^\ast \in S_r\) such that, for every \(1 \le i \le r\),
\[
|m_{\pi^\ast(i) i}(\boldsymbol{X})| = 1 - o(1),
\]
with high probability as \(N \to \infty\). If the permutation \(\pi^\ast\) can be taken to be the identity, we say that the algorithm achieves \emph{exact recovery}.
\end{defn}

From this point onward, we consider the sequence of outputs \((\boldsymbol{X}_\ell)_{\ell \in \N}\) defined by~\eqref{eq: online SGD} with step size parameter \(\delta >0\), initialized randomly with \(\boldsymbol{X}_0\) drawn from the uniform distribution \(\mu_{N \times r}\) on \(\textnormal{St}(N,r)\). The measure \(\mu_{N \times r}\) is the unique probability measure on \(\textnormal{St}(N,r)\) that is invariant under both the left and right orthogonal transformations. More precisely, if \(\boldsymbol{X} \sim \mu_{N \times r}\), then for any orthogonal matrices \(\boldsymbol{U} \in \mathcal{O}(N)\) and \(\boldsymbol{V} \in \mathcal{O}(r)\), the matrix \(\boldsymbol{U} \boldsymbol{X} \boldsymbol{V}\) has the same distribution as \(\boldsymbol{X}\). In practice, sampling from \(\mu_{N \times r}\) can be done by drawing a Gaussian matrix \(\boldsymbol{Z} \in \R^{N \times r}\) with i.i.d. standard normal entries and setting \(\boldsymbol{X}_0 = \boldsymbol{Z} (\boldsymbol{Z}^\top \boldsymbol{Z})^{-1/2}\)~\cite[Theorem 2.2.1]{chikuse2012statistics}. 

We consider the probability space \((\Omega,\mathcal{F},\mathbb{P})\) on which the \(p\)-tensors \((\boldsymbol{W}^\ell)_\ell\) are defined. Let \(\mathbb{P}_{\boldsymbol{X}_0}\) denote the law of the process \((\boldsymbol{X}_\ell)_{\ell \ge 0}\) initiated at \(\boldsymbol{X}_0 \sim \mu_{N \times r}\):
\[
\mathbb{P}_{\boldsymbol{X}_0} (A) = \int_{\textnormal{St}(N,r)} \mathbb{P}_{\boldsymbol{X}} (A) \textnormal{d}\mu_{N \times r} (\boldsymbol{X}) ,
\]
for any measurable set \(A\) in the \(\sigma\)-algebra generated by the coordinate mappings from \(\N\) to \(\textnormal{St}(N,r)\). We also define \(\mathbb{P}_{\boldsymbol{X}_0^+}\) as the law of the process with \(\boldsymbol{X}_0 \sim \mu_{N \times r}\) conditioned on \(m_{ij}(\boldsymbol{X}_0) >0\) for all \(1 \leq i,j \leq r\), corresponding to initialization with strictly positive correlations. \\
\textit{Notations.} For \(n \in \N\), we denote by \([n] \coloneqq \{1, \ldots, n\}\). For sequences \(x_N\) and \(y_N\), we write \(x_N \ll y_N\) if \(x_N / y_N \to 0\) as \(N \to \infty\). \\

We are now ready to present our main results. Throughout this section, we assume that the SNRs \(\lambda_1 \ge \cdots \ge \lambda_r \ge 0\) are of order \(1\). While the statements below are presented in asymptotic form, we provide stronger nonasymptotic formulations---including explicit constants and convergence rates---in Section~\ref{section: main results SGD}.

\subsubsection{Main results for \(p \ge 3\)}
Our main result for \(p \ge 3\) is the following. 

\begin{thm}\label{thm: general recovery SGD p>2 asymptotic}
Suppose that \(M = M(N)\) grows at most polynomially in \(N\) and satisfies \(M \gg \log(N) N^{p-2}\). Assume further that \(\delta\) satisfies \(M^{-1}N^{\frac{p-1}{2}} \ll \delta \ll N^{1/2}(\log(N)M)^{-1/2}\). Then, there exists a permutation \(\pi^\ast \in S_r\) such that, for every \(\varepsilon > 0\) and every \(i \in [r]\),
\[
\lim_{N \to \infty} \mathbb{P}_{\boldsymbol{X}_0^+} (m_{\pi^\ast(i) i} (\boldsymbol{X}_M) \ge 1 - \varepsilon) = 1.
\]
Furthermore, if there exists \(\eta > 1\) such that 
\[
\lambda_i > C \left( \eta \sqrt{\log(\eta)} \right)^{p-2} \lambda_{i+1}, \quad \textnormal{for all} \enspace 1 \le i \le r-1,
\]
for some universal constant \(C>0\), then for every \(\varepsilon > 0\) and every \(i \in [r]\),
\[
\lim_{N \to \infty} \mathbb{P}_{\boldsymbol{X}_0^+} \left(  m_{ii}(\boldsymbol{X}_M) \geq 1 - \varepsilon \right) \geq 1 - \frac{1}{\eta}.
\]
\end{thm}

Theorem~\ref{thm: general recovery SGD p>2 asymptotic} shows that, under a positive initialization of the correlations, online SGD successfully recovers all signal directions (up to a permutation) provided the number of samples \(M\) scales as \(\log(N) N^{p-2}\). This scaling matches, up to logarithmic factors, the sharp information-algorithmic threshold known for first-order methods in the rank-one spiked tensor model~\cite{arous2020algorithmic, arous2021online}. The factor \(\log(N)\) in both the step size and sample complexity is due to the control of the retraction term in~\eqref{eq: online SGD}, and can be removed in the case of a single spike. To ensure exact recovery of all spikes with high probability---where the probability depends on the ratio between the SNRs---the SNRs must be sufficiently separated. 

The permutation \(\pi^\ast\) of the recovered spikes can be explicitly determined via the following procedure.

\begin{defn}[Greedy maximum selection] \label{def: greedy operation}
Let \(\boldsymbol{A} \in \R^{r \times r}\) be a matrix whose nonzero entries are all distinct. We define a sequence of index pairs \((i_k^\ast, j_k^\ast) \in [r]^2\) recursively as follows:
\begin{itemize}
\item[1.] Set \(\boldsymbol{A}^{(0)} \coloneqq \boldsymbol{A}\).
\item [2.] For \(k = 1, 2, \ldots\), define
\[
(i_k^\ast, j_k^\ast) \coloneqq \argmax_{1 \le i,j \le r-(k-1)} | \boldsymbol{A}^{(k-1)} |_{ij},
\]
where \(\boldsymbol{A}^{(k-1)} \in \R^{(r-(k-1)) \times (r-(k-1))}\) is obtained from \(\boldsymbol{A}\) by removing the rows \(i_1^\ast, \ldots, i_{k-1}^\ast\) and the columns \(j_1^\ast, \ldots, j_{k-1}^\ast\), and \(|\boldsymbol{A}^{(k-1)}| \) denotes the absolute value of the entries in \(\boldsymbol{A}^{(k-1)}\). 
\item[3.] If at some step \(r_\textnormal{c} \in [r]\) we have
\[
\max_{ij} |\boldsymbol{A}^{(r_\textnormal{c})}|_{ij} = 0 ,
\]
the procedure terminates. 
\end{itemize}
The resulting sequence \((i_1^\ast, j_1^\ast), \ldots, (i_{r_\textnormal{c}}^\ast, j_{r_\textnormal{c}}^\ast)\) is called the \emph{greedy maximum selection} of \(\boldsymbol{A}\).
\end{defn}

The permutation \(\pi^\ast\) in Theorem~\ref{thm: general recovery SGD p>2 asymptotic} can be obtained via the greedy maximum selection procedure. Specifically, applying Definition~\ref{def: greedy operation} to the initialization matrix
\begin{equation} \label{eq: initialization matrix}
\boldsymbol{I}_0 = \left (\lambda_i \lambda_j m_{ij}^{p-2}(\boldsymbol{X}_0) \mathbf{1}_{\{m_{ij}^{p-2}(\boldsymbol{X}_0) \geq 0\}} \right)_{1 \le i, j \le r},
\end{equation}
produces a sequence of index pairs \(\{(i_k^\ast, j_k^\ast)\}_{k=1}^{r_\text{c}}\). These pairs specify the correspondence between recovered and true spikes. Define the permutation \(\pi^\ast\) by \(\pi^\ast (j_k^\ast) = i_k^\ast\) for \(k \in [r_\mathrm{c}]\). In the special case where \(r_\mathrm{c} =r\)---for instance, when all correlations are initialized positively, as in the setting of Theorem~\ref{thm: general recovery SGD p>2 asymptotic}---the pairs \(\{(\pi^\ast(i), i)\}_{i=1}^r\) coincide (as sets) with \(\{(i_k^\ast, j_k^\ast)\}_{k=1}^r\). The matrix \(\boldsymbol{I}_0 \in \R^{r \times r}\) is random. While its entries could potentially be identical due to randomness, Lemma A.3 of our companion paper~\cite{gradientflow} ensures that they are distinct with probability \(1-o(1)\), making the greedy maximum selection of \(\boldsymbol{I}_0\) well-defined with high probability. We now present a more precise formulation of Theorem~\ref{thm: general recovery SGD p>2 asymptotic}.

\begin{thm}[Recovery up to a permutation for \(p \geq 3\)] \label{thm: strong recovery online p>2 asymptotic}
Suppose that \(M = M(N)\) grows at most polynomially in \(N\) and satisfies \(M \gg \log(N) N^{p-2}\). Assume further that \(\delta\) satisfies \(M^{-1}N^{\frac{p-1}{2}} \ll \delta \ll N^{1/2}(\log(N)M)^{-1/2}\). Then, for every \(\varepsilon >0\) and every \(k \in [r_\textnormal{c}]\), 
\[
\lim_{N \to \infty} \mathbb{P}_{\boldsymbol{X}_0} \left( \lvert  m_{i^\ast_k j^\ast_k}(\boldsymbol{X}_M) \rvert \geq 1 - \varepsilon \right) = 1,
\]
where \((i_1^\ast, j_1^\ast), \ldots, (i_{r_\textnormal{c}}^\ast, j_{r_\textnormal{c}}^\ast)\) denote the greedy maximum selection of \(\boldsymbol{I}_0\). 
\end{thm}

Theorem~\ref{thm: strong recovery online p>2 asymptotic} shows that the permutation of the recovered spikes is determined by the greedy maximum selection of the matrix \(\boldsymbol{I}_0\), which in turn depends on the initial correlations and the corresponding signal strengths. We emphasize that Theorem~\ref{thm: strong recovery online p>2 asymptotic} no longer assumes that all initial correlations are strictly positive. Consequently, complete recovery of all spikes may fail in general, and only a subset of size \(r_\textnormal{c}\) is guaranteed to be recovered. As follows from~\eqref{eq: initialization matrix}, full recovery (up to a permutation) is ensured whenever all initial correlations are positive or when \(p\) is even. For odd \(p\), by contrast, negative initial correlations may become trapped near the equator, thus impeding alignment with any signal direction.

\begin{rmk} \label{rmk: sign initialization}
It is important to note that the behavior of online SGD depends on the parity of \(p\) in the following sense. When \(p\) is odd, then each estimator \(\boldsymbol{x}_{j_k^\ast}\) recovers the spike \(\boldsymbol{v}_{i_k^\ast}\) with \(\mathbb{P}\)-probability \(1-o(1)\), since the correlations that are negative at initialization get trapped at the equator. Conversely, when \(p\) is even, we have each estimator \(\boldsymbol{x}_{j_k^\ast}\) recovers \(\textnormal{sgn}(m_{i_k^\ast j_k^\ast}(\boldsymbol{X}_0)) \boldsymbol{v}_{i_k^\ast j_k^\ast}\) with probability \(1-o(1)\). This means that if the correlation at initialization is positive, then \(\boldsymbol{x}_{j_k^\ast}\) recovers \(\boldsymbol{v}_{i_k^\ast}\); otherwise, \(\boldsymbol{x}_{j_k^\ast}\) recovers \(-\boldsymbol{v}_{i_k^\ast}\).
\end{rmk}

The dynamics underlying both Theorems~\ref{thm: general recovery SGD p>2 asymptotic} and~\ref{thm: strong recovery online p>2 asymptotic} reveals a richer structure beyond the stated recovery guarantees. In particular, the correlation \(m_{i_1^\ast j_1^\ast}\) becomes macroscopic first, followed by \(m_{i_2^\ast j_2^\ast}\), and so on. Theorem~\ref{thm: general recovery SGD p>2 asymptotic} includes an additional assumption on the ratio between the signal magnitudes, ensuring that variations due to random initialization are mitigated by differences in the SNRs. Under this condition, the greedy maximum selection of \(\boldsymbol{I}_0\) satisfies \(i_k^\ast = j_k^\ast = k\) for every \(k \in [r_\text{c}]\), so that the correlations \(m_{11}, \ldots, m_{r_\text{c} r_\text{c}}\) become macroscopic sequentially. As illustrated by the numerical simulations in Figures~\ref{fig: p=3,r=2}, ~\ref{fig: p=3, r=2, equal}, and~\ref{fig: p=3 online}, once \(m_{i_1^\ast j_1^\ast}\) exceeds a critical threshold, the correlations \(m_{i_1^\ast j}\) and \(m_{i j_1^\ast}\) for \(i \neq i_1^\ast\) and \(j \neq j_1^\ast\) decrease to sufficiently small levels. This in turn allows \(m_{i_2^\ast j_2^\ast}\) to grow and become macroscopic, and the process continues in the same manner for the remaining correlations \(\{m_{i_k^\ast j_k^\ast}\}\). We refer to this phenomenon as the \emph{sequential elimination process}: the correlations \(\{m_{i_k^\ast j_k^\ast}\}_{k=1}^{r_\text{c}}\) increase one by one, progressively suppressing correlations that share either a row or column index. This notion is formalized below.

\begin{defn}[Sequential elimination] \label{def: sequential elimination}
Let \(S = \{ (i_1, j_1), \ldots, (i_m,j_m)\}\) be a set with distinct indices \(i_1, \ldots, i_m \in [r]\) and \(j_1, \ldots, j_m \in [r]\), where \(m \le r\). We say that the correlations \(\{m_{ij}(t)\}_{1 \leq i,j \leq m}\) exhibit a \emph{sequential elimination with ordering \(S\)} if for every \(\varepsilon, \varepsilon' > 0\), there exist stopping times \(T_1 \leq \cdots \leq T_m\) such that for every \(k \in [m]\) and every \(T \ge T_k\),
\[
| m_{i_k j_k}(\boldsymbol{X}_T)| \geq 1 - \varepsilon  \quad \textnormal{and} \quad |m_{i_k j}(\boldsymbol{X}_T)| \leq \varepsilon', |m_{i j_k}(\boldsymbol{X}_T)| \leq \varepsilon' \enspace \textnormal{for} \: i\neq i_k, j \neq j_k.
\]
\end{defn}

Building on Definition~\ref{def: sequential elimination}, the next result forms the basis of Theorem~\ref{thm: strong recovery online p>2 asymptotic}, and consequently of Theorem~\ref{thm: general recovery SGD p>2 asymptotic}.

\begin{thm}[Theorem~\ref{thm: strong recovery online p>2 asymptotic} revisited] \label{thm: strong recovery online p>2 asymptotic stronger}
Suppose that \(M = M(N)\) grows at most polynomially in \(N\) and satisfies \(M \gg \log(N) N^{p-2}\). Assume further that \(\delta\) satisfies \(M^{-1}N^{\frac{p-1}{2}} \ll \delta \ll N^{1/2}(\log(N)M)^{-1/2}\). Then, the correlations \(\{m_{ij}\}_{1 \le i, j \le r_\textnormal{c}}\) follow a sequential elimination with ordering \(\{(i_k^\ast, j_k^\ast)\}_{k=1}^{r_\textnormal{c}}\) and stopping times of order \(\log(N)N^{p-2}\), with \(\mathbb{P}_{\boldsymbol{X}_0}\)-probability tending to \(1\) as \(N \to \infty\).
\end{thm}

\begin{figure}[h]
\centering
\includegraphics[scale=0.25]{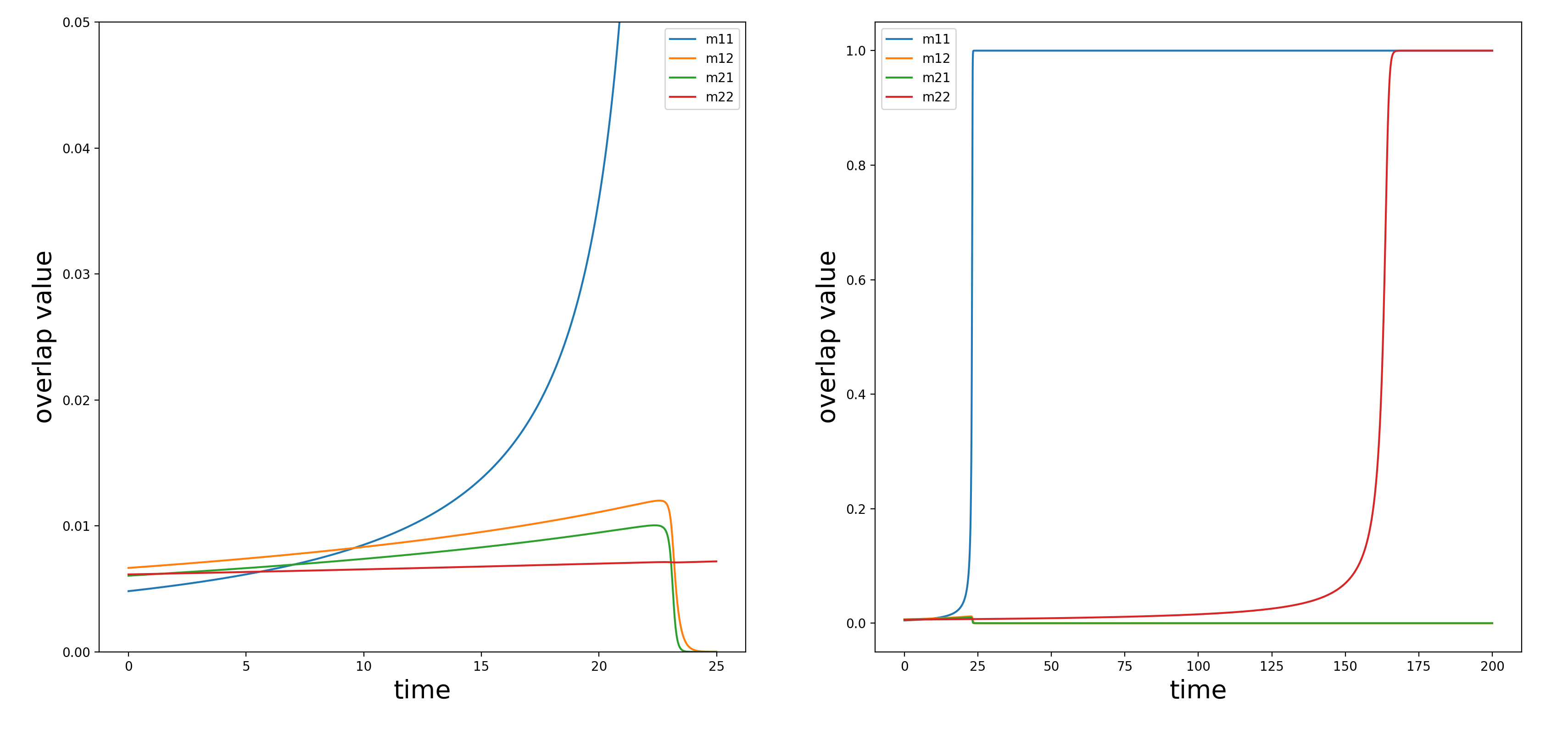}
\caption{Evolution of the correlations \(\{m_{ij}\}_{1 \le i,j \le 2}\) under the population dynamics for \(p=3\), \(r=2\), \(\lambda_1=3\), and \(\lambda_2=1\). The SNRs are sufficiently separated to ensure exact recovery of both spikes \(\boldsymbol{v}_1\) and \(\boldsymbol{v}_2\). Once \(m_{11}\) reaches a sufficiently large microscopic threshold, \(m_{12}\) and \(m_{21}\) begin to decrease, allowing the recovery of \(\boldsymbol{v}_2\) after they become negligible in the evolution of \(m_{22}\).}
\label{fig: p=3,r=2}
\end{figure}

\begin{figure}[h]
\centering
\includegraphics[scale=0.25]{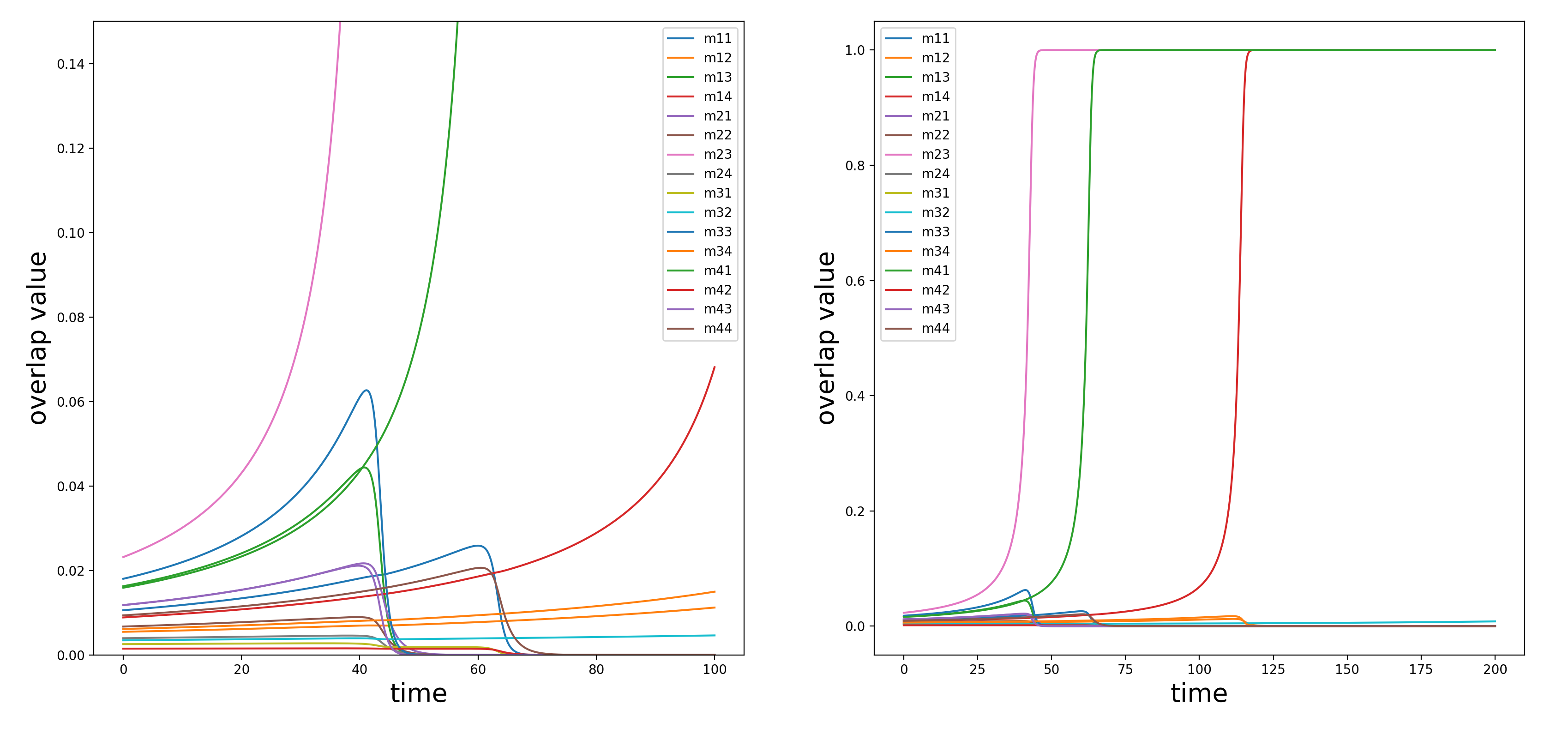}
\caption{Evolution of the correlations \(\{m_{ij}\}_{1 \le i,j \le 4}\) under the population dynamics for \(p=3\), \(r=4\), and equal SNRs \(\lambda_1= \cdots = \lambda_4 =1\). Since the SNRs are identical, the order in which the correlations become macroscopic is determined by their initial values. The simulation illustrates recovery of a permutation of the four spikes \(\boldsymbol{v}_1, \ldots, \boldsymbol{v}_4\) via the sequential elimination phenomenon. The fourth spike is not visible in the plotted time window, as its recovery occurs slightly later.}
\label{fig: p=3, r=2, equal}
\end{figure}

\begin{figure}[h]
\centering
\includegraphics[scale=0.25]{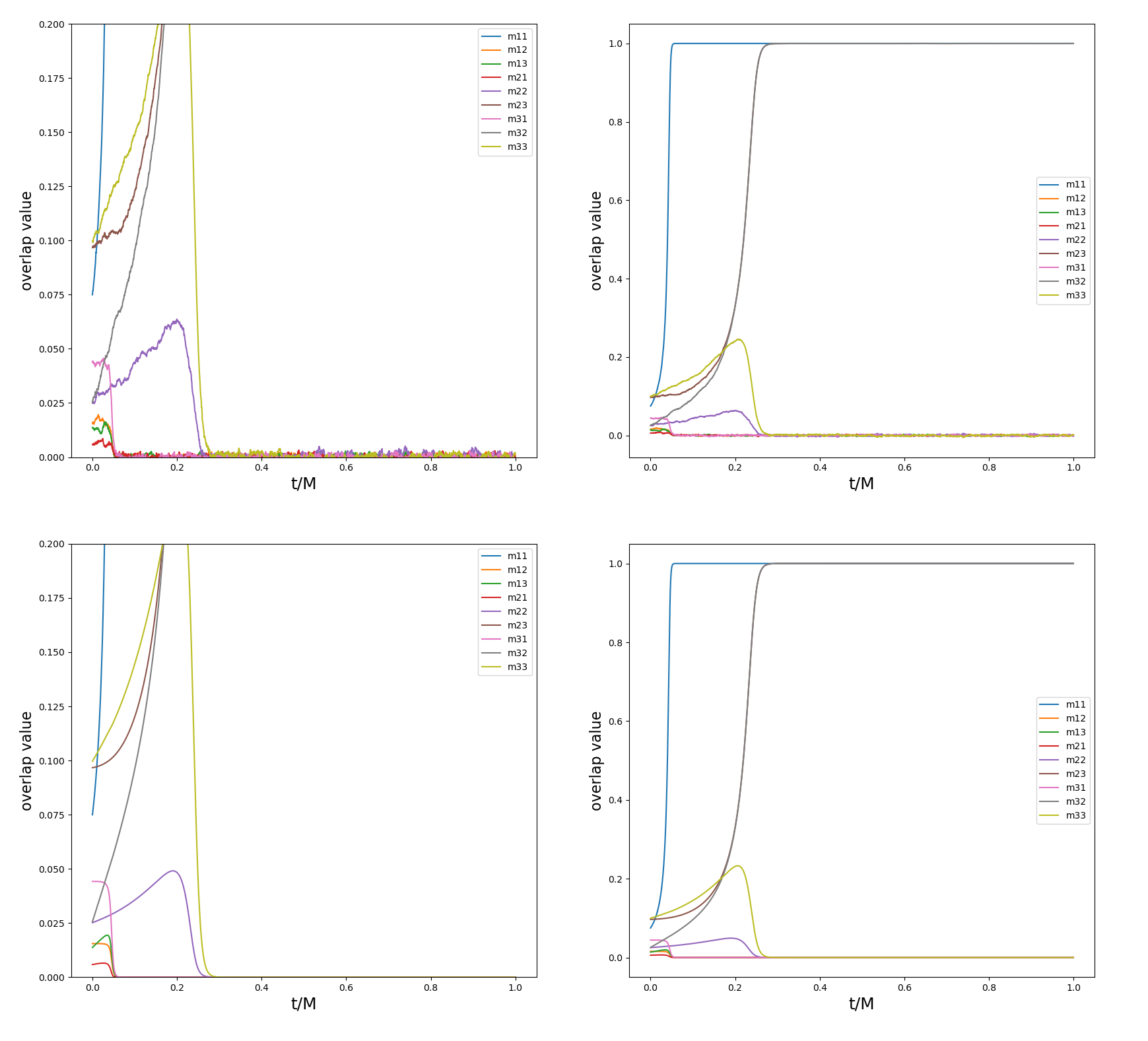}
\caption{Evolution of the correlations \(\{m_{ij}\}_{1 \le i,j \le 3}\) under the online SGD dynamics for \(p=3, r=3\), and SNRs \(\lambda_1 = 3, \lambda_2 =2, \lambda_3 = 1\). The top panels correspond to \(M = 3200\) samples, step size \(\delta/N = 0.0003\), and dimension \(N=500\). The bottom panels show the corresponding noiseless dynamics, obtained with the same step size and number of steps. This simulation illustrates the recovery of a permutation of the three spikes \(\boldsymbol{v}_1, \boldsymbol{v}_2, \boldsymbol{v}_3\) through the sequential elimination phenomenon.}
\label{fig: p=3 online}
\end{figure}

\subsubsection{Main results for \(p =2\)}

We now present our main results for \(p=2\). Our first result shows that exact recovery of all spikes is achievable under sufficiently separated SNRs, with a sample complexity of order \(\log(N)^2 N^{\xi_0 /2}\), where \(\xi_0 = 1 - \frac{\lambda_r^2}{\lambda_1^2} \in (0,1)\) quantifies the relative separation between the smallest and largest signal sizes. We assume the following separation condition: for every \(1 \le i \le r-1\),
\[
\lambda_i = \lambda_{i+1} (1+\kappa_i),
\]
where \(\kappa_i >0\) are constants of order one.

\begin{thm}[Exact recovery for \(p=2\)] \label{thm: strong recovery online p=2 asymptotic}
Assume \(\kappa > \sqrt{2}-1\). Suppose that \(M = M(N)\) grows at most polynomially in \(N\) and satisfies \(M \gg \log(N)^2 N^{\frac{1}{2}\xi_0}\). Assume further that \(\delta\) satisfies \(\sqrt{N}\log(N)M^{-1}\ll \delta \ll N^{\frac{1}{4}(2 - \xi_0)} M^{-1/2}\). Then, for every \(\varepsilon>0\) and every \(i \in [r]\),
\[
\lim_{N \to \infty} \mathbb{P}_{\boldsymbol{X}_0} \left( \lvert m_{ii}(\boldsymbol{X}_M) \rvert \geq 1-\varepsilon \right) = 1.
\]
\end{thm}

The assumption on \(\kappa > \sqrt{2}-1\) serves as a convenient sufficient condition ensuring that the off-diagonal correlations remain controlled throughout the dynamics. In principle, this assumption could be relaxed by refining the comparison estimates in the proof, but we retain it to simplify the analysis and keep the proof concise. As in Theorem~\ref{thm: general recovery SGD p>2 asymptotic}, variations in the initial correlations are compensated by differences in SNRs, leading to exact recovery of all spikes. However, unlike Theorem~\ref{thm: general recovery SGD p>2 asymptotic}, for \(p=2\), exact recovery is already achievable when the SNRs differ by factors of order one. Regarding sample complexity, our scaling differs from the rank-one matrix PCA result of~\cite{arous2021online}, where \(\mathcal{O} ( \log(N)^2)\) samples suffice for recovery of a single spike. Here, this sample size is not enough to ensure stability of the subsequent correlation dynamics. The additional logarithmic and polynomial factors arise from controlling the retraction map, which is intrinsic to the online SGD dynamics on the Stiefel manifold.\\

As with the case of \(p \ge 3\), the correlations exhibit a sequential elimination with ordering \(\{(k,k) \colon k \in [r]\}\). Thus, an analogous version of Theorem~\ref{thm: strong recovery online p>2 asymptotic stronger} can be stated for Theorem~\ref{thm: strong recovery online p=2 asymptotic}. The evolution of the correlations is illustrated by the numerical simulations shown in Figures~\ref{fig: p=2, r=2} and~\ref{fig: p=2 online}. As noted in Remark~\ref{rmk: sign initialization}, the behavior of online SGD depends on the parity of the order \(p\) of the tensor. Specifically, for \(p=2\), the recovery of the unknown vector \(\boldsymbol{v}_i\) or its negative counterpart \(-\boldsymbol{v}_i\) is determined by the sign of the correlation \(m_{ii}\) at initialization.

\begin{figure}[h]
\centering
\includegraphics[scale=0.25]{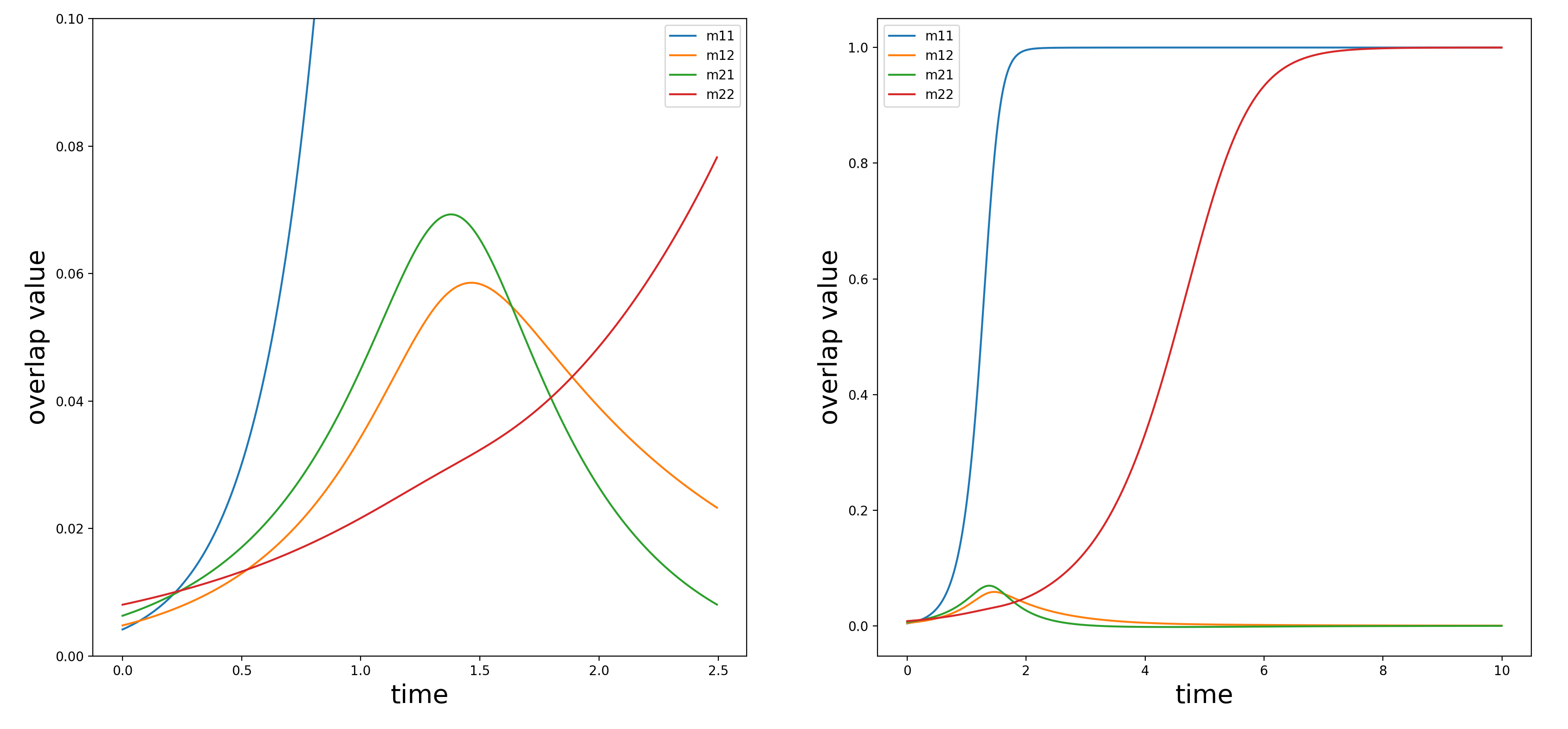}
\caption{Evolution of the correlations \(\{m_{ij}\}_{1 \le i,j \le 2}\) under the population dynamics for \(p=2\), \(r=2\), \(\lambda_1=3\), and \(\lambda_2=1\). The SNRs are sufficiently separated to ensure exact recovery of both spikes.}
\label{fig: p=2, r=2}
\end{figure}

%\begin{figure}[h]
%\centering
%\includegraphics[scale=0.25]{figure1_p=2_r=4_10_5_2_1_rescaled.png}
%\caption{Evolution of the correlations \(m_{ij}\) under the population dynamics in the case where \(p=2\), \(r=4\), \(\lambda_1=10, \lambda_2 = 5, \lambda_3 = 2\) and \(\lambda_1=1\). We observe a sequential elimination phenomenon, although less pronounced than in the case \(p \ge 3\), in the sense that correlations that eventually become negligible can still exhibit significant increases.}
%\label{fig: p=2, r=4}
%\end{figure}

\begin{figure}[h]
\centering
\includegraphics[scale=0.25]{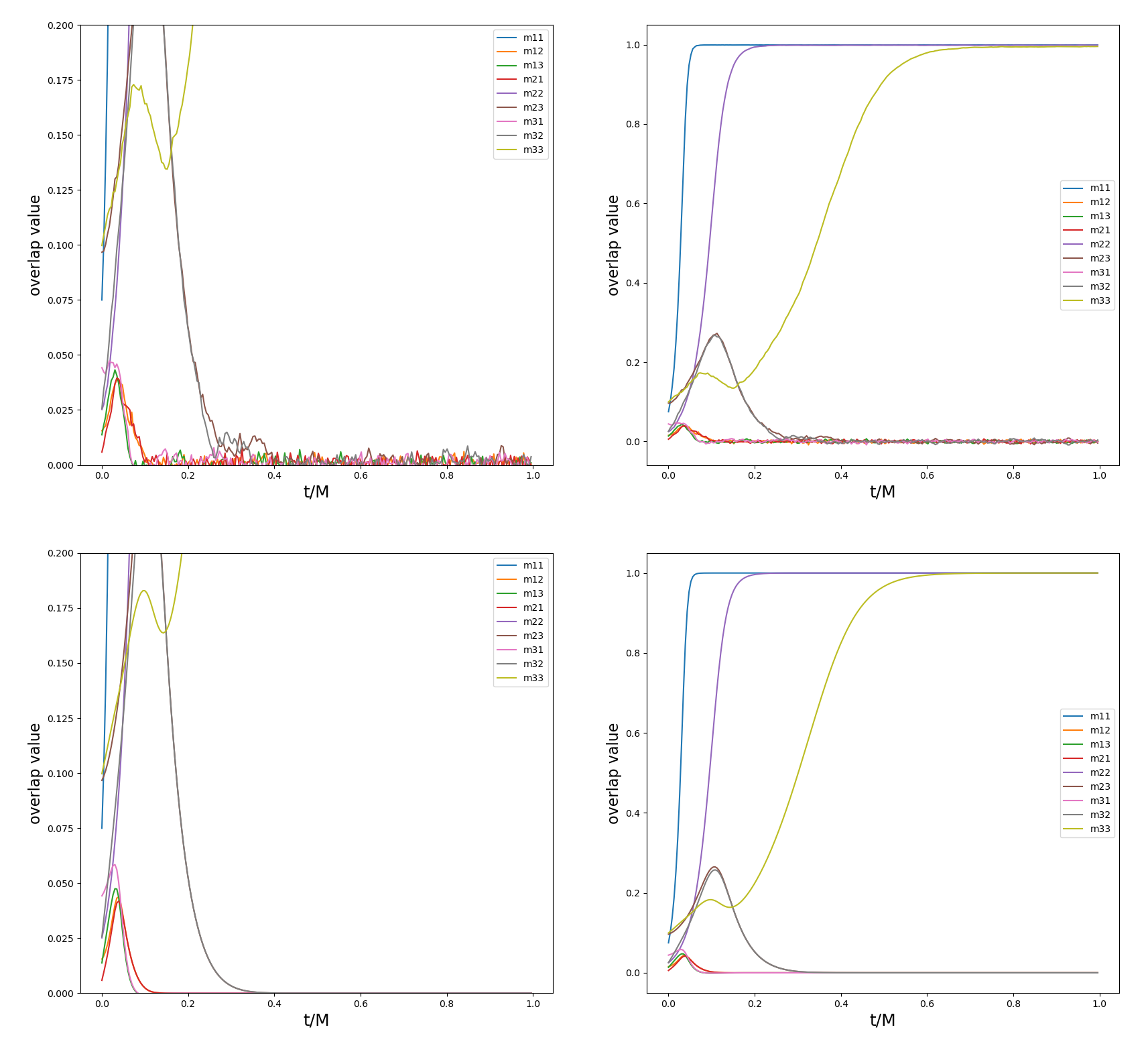}
\caption{Evolution of the correlations \(\{m_{ij}\}_{1 \le i, j \le 3}\) under the online SGD dynamics for \(p=2, r=3\), and SNRs \(\lambda_1 = 3, \lambda_2 =2, \lambda_3 = 1\). The top panels correspond to \(M = 230\) samples, step size \(\delta/N = 0.002\), and dimension \(N=500\). The bottom panels show the corresponding noiseless dynamics, obtained with the same step size and number of steps. The simulation shows exact recovery of the three spikes. The decrease in the yellow curve representing the correlation \(m_{33}\) is quantitatively controlled in the proof.}
\label{fig: p=2 online}
\end{figure}

Our second main result for \(p=2\) addresses the case where the SNRs are all equal. Due to the invariance of the loss function \(\mathcal{L}_{N,r}\) under both right and left rotations, our focus shifts from recovering each individual signal to recovering the subspace spanned by the signal vectors \(\boldsymbol{v}_1,\ldots, \boldsymbol{v}_r\). This notion of recovery is referred to as \emph{subspace recovery} and is defined as follows.

\begin{defn}[Subspace recovery] \label{def: subspace recovery}
We say that we \emph{recover the subspace} spanned by the orthogonal signal vectors \(\boldsymbol{v}_1,\ldots, \boldsymbol{v}_r\) if 
\[
\norm{\boldsymbol{X} \boldsymbol{X}^\top -\boldsymbol{V} \boldsymbol{V}^\top}_{\textnormal{F}} = o(1),
\]
with high probability, where \(\norm{\cdot}_{\textnormal{F}}\) stands for the Frobenius norm and \(\boldsymbol{X} = [\boldsymbol{x}_1, \ldots, \boldsymbol{x}_r], \boldsymbol{V} = [\boldsymbol{v}_1, \ldots, \boldsymbol{v}_r] \in \R^{N \times r}\).
\end{defn}

From Definition~\ref{def: subspace recovery}, we note that subspace recovery is equivalent to recovering the eigenvalues of the matrix \(\boldsymbol{G} \in \R^{r \times r}\) defined by
\[
\boldsymbol{G} = \boldsymbol{M} \boldsymbol{M}^\top, \quad \textnormal{where} \enspace \boldsymbol{M} = \boldsymbol{V}^\top \boldsymbol{X} = (m_{ij}(\boldsymbol{X}))_{1 \le i, j \le r} \in \R^{r \times r}.
\]
Indeed, one can verify that \(\norm{\boldsymbol{X} \boldsymbol{X}^\top - \boldsymbol{V} \boldsymbol{V}^\top}_{\textnormal{F}} = 2 (r - \Tr(\boldsymbol{G}))\). Hereafter, let \(\theta_1(\boldsymbol{X}),\ldots, \theta_r(\boldsymbol{X})\) denote the eigenvalues of the matrix-valued function \(\boldsymbol{G}(\boldsymbol{X})\). Since \(\boldsymbol{G}\) is symmetric and positive definite, its eigenvalues are non-negative. Our main result is as follows.

\begin{thm}[Subspace recovery for \(p=2\)] \label{thm: strong recovery isotropic SGD asymptotic}
Assume that \(\lambda_1 = \cdots = \lambda_r\). Suppose that \(M=M(N)\) grows at most polynomially in \(N\) and satisfies \(M \gg \log(N)^3\). Assume further that \(\delta\) satisfies \( \sqrt{N}\log(N)M^{-1}\ll \delta \ll \sqrt{N}\left(\log(N)M\right)^{-1/2}\). Then, for every \(\varepsilon>0\) and \(i \in [r]\),
\[
\lim_{N \to \infty} \mathbb{P}_{\boldsymbol{X}_0} \left( \theta_i(\boldsymbol{X}_M) \geq 1-\varepsilon \right) = 1.
\]
\end{thm}

The recovery dynamics in this case differs significantly from those in the previous cases: the correlations evolve at similar rates, thereby preventing the sequential elimination phenomenon, as shown in the left-hand side of Figure~\ref{fig: p=2, r=2, equal SNRs}. However, as illustrated on the right-hand side of Figure~\ref{fig: p=2, r=2, equal SNRs}, the eigenvalues of \(\boldsymbol{G}\) establish a natural ordering at initialization, which persists throughout their evolution until recovery, resulting in a monotone evolution. 

Unlike the tensor case, where we analyze the dynamics for any value of the signal sizes, for \(p=2\) we have focused on two cases: when the parameters are separated by constants of order \(1\) (i.e., Theorem~\ref{thm: strong recovery online p=2 asymptotic}) and when they are all equal (i.e., Theorem~\ref{thm: strong recovery isotropic SGD asymptotic}). The intermediate case, where the SNRs are separated by a sufficiently small \(\epsilon\)-factor, possibly depending on \(N\), requires further effort and will be addressed in future work.

\begin{figure}[h]
\centering
\includegraphics[scale=0.25]{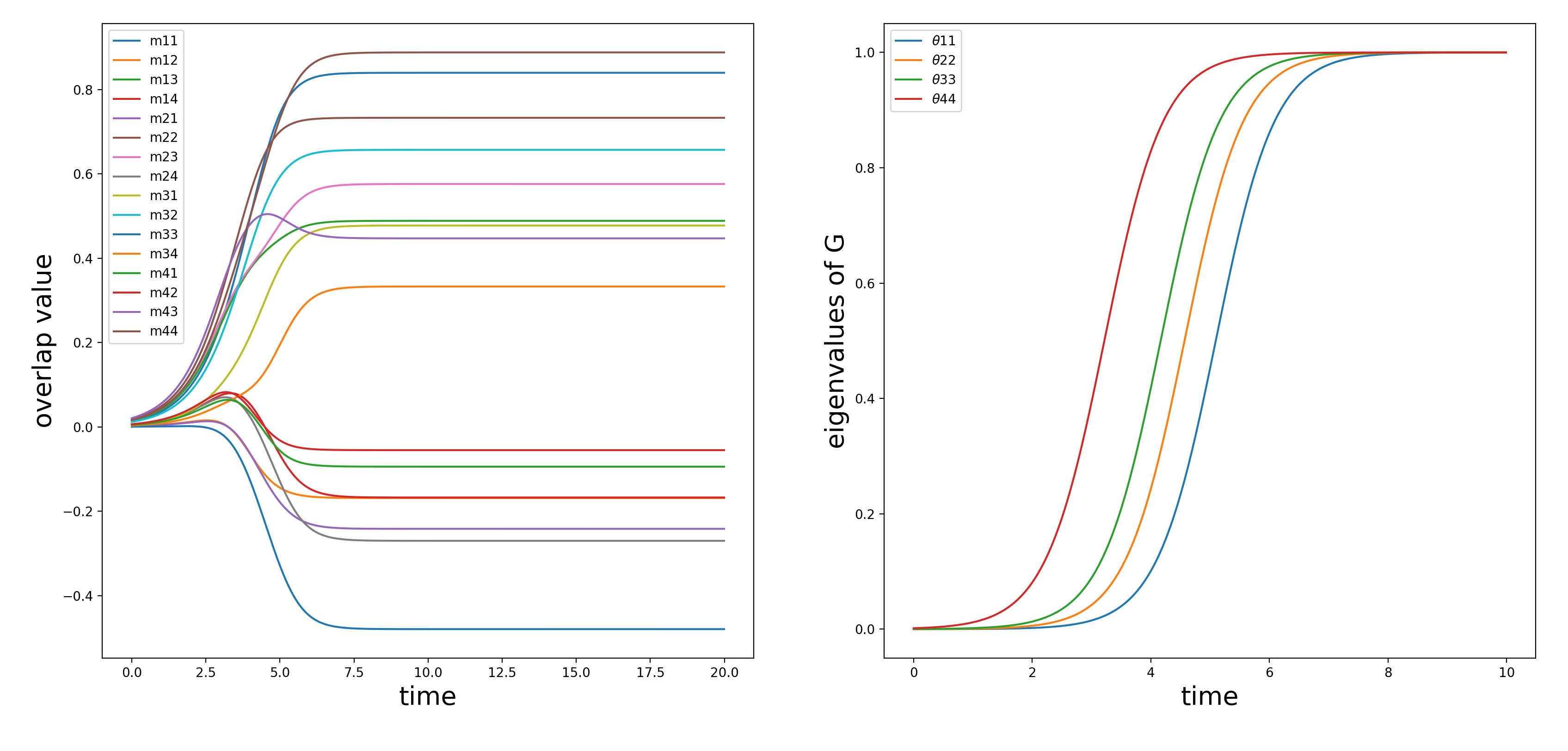}
\caption{Evolution of the correlations \(m_{ij}\) (left) and eigenvalues (right) of \(\boldsymbol{G} = \boldsymbol{M} \boldsymbol{M}^\top\) under the population dynamics. The relative speeds of the correlations are too close to each other for a sequential elimination phenomenon to take place. However, the eigenvalues of \(\boldsymbol{G}\) are monotone and it is possible to characterize subspace recovery.}
\label{fig: p=2, r=2, equal SNRs}
\end{figure}

%\begin{figure}[h]
%\centering
%\includegraphics[scale=0.25]{figure1_p=2_online_instability.png}
%\caption{Evolution of the correlations \(m_{ij}\) under the online SGD algorithm for the case where \(p=2, r=2\), with SNRs \(\lambda_1 = 2, \lambda_2 =1\). The simulation is performed with \(M = 70\) samples, a step size \(\delta/N = 0.006\), and a dimension of \(N=500\). While the correlations \(m_{11}\) and \(m_{22}\) clearly overcome the effect of the noise martingale and become macroscopic, the algorithm becomes unstable. This illustrates the fact that the sample complexity prescribed by Theorem~\ref{thm: strong recovery online p=2 asymptotic} is larger than what is required to just ensure that the martingale term is negligible compared to the signal drift.}
%\label{fig: p=2 online unstable}
%\end{figure}
%%%%%%%%%%%%%%%%%%%%%%%%%%%%%%%%%%%%%%%%%%%%%%%%%%%%%%%%%%%%%%%%%%%%%%%%%%%%%%%%%%%%%%%%%%%%%%%%%%%%%%%%%%%%%%%%%%%%%%%%%%%%%%%%%%%%%%%%%%%%%%%%%%%%%%%%%%%%%%%%%%%%%%%%%%%%%%%%%%%%%%%%%%%%
\subsection{Related works} \label{subsection: related work}

Richard and Montanari~\cite{MontanariRichard} introduced the Tensor PCA problem as a higher-order generalization of matrix PCA, previously studied by Johnstone~\cite{Johnstone}. In their setting, they considered the single-spike case (\(r=1\)) with only one tensor observation:
\begin{equation} \label{eq: single-sample tensor pca}
\boldsymbol{Y} = \sqrt{N} \bar{\lambda} \boldsymbol{v}^{\otimes p} + \boldsymbol{W} \in (\R^N)^{\otimes p}.
\end{equation}
In contrast, in our multi-sample setting~\eqref{eq: spiked tensor model}, we observe \(M\) i.i.d. samples of the same form. Averaging reduces the noise variance by a factor \(1/M\), while the signal term remains unchanged. Thus, for \(r=1\) the multi-sample model with parameter \(\lambda\) is equivalent to the single-sample model~\eqref{eq: single-sample tensor pca} with effective signal strength \(\bar{\lambda} = \sqrt{M} \lambda\). Since \(\lambda = \mathcal{O}(1)\), the detectability and recovery thresholds are determined by the scaling of \(\sqrt{M}\). We next summarize the existing literature on these thresholds for both the single- and multi-spiked tensor models, stated in terms of the number of samples \(M\). In the single-spike model, the algorithmic landscape is by now well understood. Local algorithms with random initialization require a sample complexity scaling as \(N^{p-2}\) to efficiently recover the spike, whereas sharper thresholds scaling as \(N^{(p-2)/2}\) can be achieved using Sum-of-Squares (SoS) and spectral methods. The computational threshold \(N^{p-2}\) has been achieved by gradient flow, Langevin dynamics~\cite{arous2020algorithmic, arous2021online}, and tensor power iteration~\cite{HuangPCA, Wu24}. In particular, Wu and Zhou~\cite{Wu24} showed that the required number of samples scales as \(N^{p-2} \log(N)^{-C}\), with \(C\) depending on \(p\). The sharper threshold \(N^{(p-2)/2}\) is achieved by SoS-based algorithms~\cite{Hopkins15, Hopkins16, Bandeira2017}, spectral methods based on the Kikuchi Hessian~\cite{Weinkikuchi}, and tensor unfolding~\cite{BenArousUnfolding}. Beyond computational thresholds, research has also focused on the information-theoretic threshold for detection~\cite{lesieur2017, perry2018optimality, bandeira2020, chen2019, jagannath2020, dominik2024} and the statistical threshold validating the MLE as a reliable estimator~\cite{benarouscomplexity2019, ros2019, jagannath2020}. For finite-rank tensor PCA, both detection and recovery thresholds have been explored. On the information-theoretic side, there exists an order-\(1\) critical SNR above which the low-rank signal tensor \(\sqrt{N}\sum_{i=1}^r \lambda_i \boldsymbol{v}_i^{\otimes p}\) can be detected~\cite{lelarge2019fundamental, chen2021}. On the algorithmic side, Huang et al.~\cite{HuangPCA} showed that tensor power iteration converges to a single spike, specifically the one with the largest effective correlation at initialization, and thus does not recover the entire rank-\(r\) structure. Furthermore, in our companion papers~\cite{langevin, gradientflow}, we analyze Langevin and gradient flow dynamics. We show that the sample complexity required to efficiently recover the leading spike matches the known rate for the single-spike case (i.e., $N^{p-2}$ for $p\geq 3$), and to recover all spikes (up to a permutation), scales as \(N^{p-1}\) for \(p \ge 3\) and as \(N^{1 + \delta - \lambda_r^2/\lambda_1^2}\) for any \(\delta >0\) when \(p=2\), reflecting dependence on the gap between the SNRs. These thresholds are suboptimal compared to those obtained here via online SGD, due to the difficulty in the control of the additional noise generated under Langevin and gradient flow dynamics.

The single- and multi-spiked tensor models belong to the broader class of single- and multi-index models. In the single-index setting, when the hidden link function \(f_\ast\) is known, the behavior of online SGD exhibits phenomena similar to those observed in the single-spiked tensor PCA model~\cite{arous2021online}. When \(f_\ast\) is unknown, recent works~\cite{bietti2022learning, berthier2023learning,Mahankali2023} have studied how two-layer neural networks can jointly learn the hidden link function and the hidden vector. These approaches leverage gradient flow, either combined with non-parametric regression steps or through a separation of timescales between the inner and outer layers. In multi-index models, several recent works have provided upper bounds on the sample complexity thresholds required for shallow neural networks trained with first-order optimization methods to learn multi-index models. For instance,~\cite{mousavi2022neural, damian2022neural, ba2022high} analyzed the benefits given by a single gradient step coupled with various choices for the second layer weights. hile these works characterize the sample complexity needed for the gradient to achieve macroscopic correlation with the target subspace and demonstrate improvements over kernel methods, they provide limited insight into global convergence when training the first-layer weights. Using an online layer-wise training approach combined with regression steps on the second-layer weights,~\cite{abbe2023sgd} generalized the concept of the information exponent from~\cite{arous2021online} to multi-index functions, introducing the \emph{leap-exponent}. They showed that the dynamics successively visit several saddle points before recovering the unknown subspace. A similar analysis was carried out in~\cite{dandi2023learning}, providing further sample complexity guarantees when using large batches of data over a few gradient steps. Furthermore, in~\cite{bietti2023learning} the authors studied the continuous-time population-limit of a two-step procedure in which each gradient step on the subspace estimator is coupled with a multivariate non-parametric regression step to estimate the link function. They showed that this two-step procedure makes the landscape benign, leading to the recovery of both the hidden link function and the correct subspace via \emph{saddle-to-saddle dynamics} similar to the one exhibited in~\cite{abbe2023sgd}. These approaches rely heavily on the separation of timescales and regression steps, resulting in a superposition of monotone dynamics that allows generality with respect to the link function while focusing on subspace recovery. In contrast, our work provides a complete characterization of the dynamics for online SGD in a specific multi-index model, achieving recovery of permutations and rotations (when \(p=2\) and the SNRs are equal).

In the field of statistical physics of learning and related results in high-dimensional probability, single- and multi-index models can be viewed as instances of the \emph{teacher-student model}, see e.g.~\cite{engel2001statistical,zdeborova2016statistical}, in which a specific target function is learned using a chosen architecture, typically with knowledge of the form of the hidden model. The dynamics of gradient-based methods are usually studied using closed-form, low-dimensional, and asymptotically exact dynamical systems towards which the gradient trajectories are shown to concentrate as the input dimension and number of sample diverge while remaining proportional. In the case of online learning, these dynamics lead to sets of ordinary differential equations (ODEs). For instance, such ODEs have been analyzed for committee machines~\cite{saad1995line} and two-layer neural networks~\cite{goldt2019dynamics,veiga2022phase}. When full-batch learning is considered, correlations along the entire trajectory lead to integro-differential equations, known in statistical physics as \emph{dynamical mean-field theory}. Recently, there has been renewed interest in this framework for classification tasks and learning single- and multi-index functions in the proportional limit of dimension and sample size, using both heuristic methods~\cite{sarao2019,sarao2019bis,mignacco2020dynamical} and rigorous approaches~\cite{celentano2021high,gerbelot2024rigorous}. Despite their theoretical importance, these equations are often difficult to analyze beyond a few steps and are typically limited to linear sample complexity regimes. Consequently, applying these tools to study global convergence, characterize fixed points, or achieve polynomial sample complexity remains challenging.

Recently, the limiting dynamics of online SGD have been investigated in~\cite{benarousneurips, ben2022effective, arous2023eigenspace} for a broad range of estimation and classification problems, including learning an XOR function with a two-layer neural network. In particular,~\cite{benarousneurips, ben2022effective} identifies critical regimes of the step size that introduce an additional term in the population dynamics, called the \emph{corrector}, which also appears in the statistical physics literature\cite{saad1995line,goldt2019dynamics}. By rescaling the step size appropriately, these works also obtain diffusive limits for online SGD. Moreover,~\cite{arous2023eigenspace} employs a spectral analysis of the Hessian evolution to show alignment of SGD trajectories with emerging eigenspaces. We leave the exploration of critical step sizes, diffusive limits, and spectral phenomena to future work, focusing here on deriving nonasymptotic results and deepening the understanding of the population dynamics.

%%%%%%%%%%%%%%%%%%%%%%%%%%%%%%%%%%%%%%%%%%%%%%%%%%%%%%%%%%%%%%%%%%%%%%%%%%%%%%%%%%%%%%%%%%%%%%%%%%%%%%%%%%%%%%%%%%%%%%%%%%%%%%%%%%%%%%%%
\subsection{Overview}

An overview of the paper is given as follows. Section~\ref{section: qualitative picture} provides an outline of the proofs of our main results. In Section~\ref{section: main results SGD}, we present the nonasymptotic versions of our main results, from which we derive the main asymptotic results of Subsection~\ref{subsection: main asymptotic results}. Sections~\ref{section: background}-\ref{section: proof isotropic SGD} are devoted to the proofs of the main results in their nonasymptotic form. Specifically, Section~\ref{section: background} presents comparison inequalities for both the correlations \(m_{ij}(\boldsymbol{X})\) and the eigenvalues of \(\boldsymbol{M} \boldsymbol{M}^\top\) with \(\boldsymbol{M} = (m_{ij}(\boldsymbol{X}))_{i,j \in [r]}\). Section~\ref{section: proof non-isotropic SGD} provides the proofs of the main results related to the full recovery of the spikes, i.e., it addresses the case where \(p \ge 3\) with arbitrary values for the SNRs, as well as the case where \(p=2\) with sufficiently separated SNRs. Finally, Section~\ref{section: proof isotropic SGD} focuses on the case where \(p=2\) and the spikes have equal SNRs, thereby proving subspace recovery. \\ 

\textbf{Acknowledgements.}\ G.B.\ and C.G.\ acknowledge the support of the NSF grant DMS-2134216. V.P.\ acknowledges the support of the ERC Advanced Grant LDRaM No.\ 884584.

%%%%%%%%%%%%%%%%%%%%%%%%%%%%%%%%%%%%%%%%%%%%%%%%%%%%%%%%%%%%%%%%%%%%%%%%%%%%%%%%%%%%%%%%%%%%
\section{Outline of proofs} \label{section: qualitative picture}

In this section, we outline the proofs of our main results. To simplify the discussion, we focus on the spiked tensor model with \(r=2\). First, we address the cases where \(p\ge 3\) with arbitrary values of \(\lambda_1\) and \(\lambda_2\), as well as the case where \(p=2\) with sufficiently separated SNRs. In these cases, we follow the evolution of the correlations under the online SGD algorithm and show full recovery of the spikes. Next, we consider the case where \(p=2\) and the spikes have equal strength, i.e., \(\lambda_1 = \lambda_2\), and show that the algorithm recovers the subspace spanned by the two spikes.

\subsection{Full recovery of spikes}

We begin by analyzing the evolution of the correlations \(m_{ij}(\boldsymbol{X})\) under the online SGD algorithm, following the update rule given by~\eqref{eq: online SGD}. We assume an initial random start with a completely uninformative prior, specifically the invariant distribution on the Stiefel manifold. As a consequence, all correlations \(m_{ij}(\boldsymbol{X}_0)\) have the typical scale of order \(N^{-\frac{1}{2}}\) at initialization. In the following discussion, we assume for simplicity that all correlations are positive at initialization.

According to~\eqref{eq: online SGD}, the evolution of \(m_{ij}(\boldsymbol{X}_t)\) at time \(t\) satisfies
\[
m_{ij}(\boldsymbol{X}_t) = \left \langle \boldsymbol{v}_i, (\boldsymbol{X}_t)_j \right \rangle = \left \langle \boldsymbol{v}_i ,R_{\boldsymbol{X}_{t - 1}}\left(-\frac{\delta}{N} ( \nabla_{\textnormal{St}} \mathcal{L}_{N,r} (\boldsymbol{X}_{t-1}; \boldsymbol{Y}^t) )_j\right) \right \rangle.
\] 
To simplify notation, we write \(m_{ij}(t) = m_{ij}(\boldsymbol{X}_t)\). When the step size \(\delta >0\) is sufficiently small, the matrix inverse appearing in the polar retraction \(R_{\boldsymbol{X}_{t-1}}\) (see~\eqref{eq: polar retraction}) can be expanded using a Neumann series. According to~\cite{horn2012matrix}, for any square matrix \(\boldsymbol{A} \in \R^{r \times r}\) satisfying \(\lVert \boldsymbol{A}\rVert <1\) in some operator norm, 
\[
(\boldsymbol{I}_r - \boldsymbol{A})^{-1} = \sum_{k=0}^\infty \boldsymbol{A}^k,
\]
which is known as the \emph{Neumann series}. Hence, up to controlled error terms, we obtain the following approximate update:
\[ 
m_{ij}(t) \approx m_{ij}(t-1) -\langle \boldsymbol{v}_i, \frac{\delta}{N} (\nabla_{\textnormal{St}} \mathcal{L}_{N,r} (\boldsymbol{X}_{t-1}; \boldsymbol{Y}^t))_j \rangle.
\]
Iterating this expression yields
\begin{equation} \label{eq: evolution correlation SGD informal}
m_{ij}(t) \approx m_{ij}(0) - \frac{\delta}{N} \sum_{\ell=1}^t \langle \boldsymbol{v}_i, (\nabla_{\textnormal{St}} \mathcal{L}_{N,r}(\boldsymbol{X}_{\ell-1}; \boldsymbol{Y}^\ell) )_j \rangle.
\end{equation}
A rigorous version of the approximation~\eqref{eq: evolution correlation SGD informal}, together with explicit error bounds, is provided in Proposition~\ref{prop: inequalities SGD}. The condition on the step size \(\delta\) ensuring convergence of the Neumann series is mild and does not constitute a bottleneck in our analysis. However, the retraction map introduces additional interaction terms among the correlations, which require careful control.

Having derived the evolution equations~\eqref{eq: evolution correlation SGD informal} for the correlations, we now decompose the gradient of the loss function \(\mathcal{L}_{N,r}\) into its \emph{noise} and \emph{population} components:
\[
\mathcal{L}_{N,r}(\boldsymbol{X}_{t-1},\boldsymbol{Y}^t) = H_{N,r}(\boldsymbol{X}_{t-1}) +  \Phi_{N,r}(\boldsymbol{X}_{t-1}),
\]
where \(H_{N,r}\) and \(\Phi_{N,r}\) are defined in~\eqref{eq: noise part} and~\eqref{eq: population loss}, respectively. The main idea of the proof strategy is to control the magnitude of the noise term by appropriately choosing the step size \(\delta\), ensuring that the population term dominates over a sufficiently long time horizon. Balancing these two contributions also determines the sample complexity required for efficient recovery. The main challenge, compared to the single-spiked setting analyzed in~\cite[Proposition 4.1]{arous2021online}, is that in the present multi-spiked case the population component gives rise to a substantially more intricate dynamics, due to nonmonotonic interactions among multiple correlations. This richer dynamics demands a more delicate analysis to maintain statistically and algorithmically meaningful sample-complexity guarantees. In particular, it requires a significantly finer partitioning of the dynamics, obtained by introducing appropriate sequences of stopping times, possibly indexed by the dimension. These sequences are needed to control both the sequential elimination phenomenon and the additional fluctuations inherent to the discrete-time setting.

\subsubsection{Control of the noise term} \label{subsection: control noise term informal}

The evolution equation~\eqref{eq: evolution correlation SGD informal} can be rewritten as
\[
m_{ij}(t) \approx m_{ij}(0) - \frac{\delta}{N} \sum_{\ell=1}^t \langle \boldsymbol{v}_i, (\nabla_{\textnormal{St}} \Phi_{N,r}(\boldsymbol{X}_{\ell-1}) )_j \rangle - \frac{\delta}{N} \sum_{\ell=1}^t \langle \boldsymbol{v}_i, (\nabla_{\textnormal{St}} H_{N,r}(\boldsymbol{X}_{\ell-1}) )_j \rangle,
\]
where  
\[
\begin{split}
- \langle \boldsymbol{v}_i, (\nabla_{\textnormal{St}} \Phi_{N,r}(t) )_j \rangle & = p \sqrt{N} \lambda_i \lambda_j m_{ij}^{p-1}(t) \\
&\quad - \frac{p}{2} \sqrt{N} \sum_{1 \le k, \ell \le 2} \lambda_k m_{i \ell}(t) m_{k j}(t) m_{k \ell}(t) \left (\lambda_j m_{kj}^{p-2}(t) + \lambda_\ell m_{k \ell}^{p-2}(t) \right ) .
\end{split}
\]
The noise term corresponds to the sum of the martingale increments \(- \langle \boldsymbol{v}_i, (\nabla_{\textnormal{St}} H_{N,r}(\boldsymbol{X}_{\ell-1}) )_j \rangle\). Since the noise tensors \((\boldsymbol{W}^\ell)_\ell\) are sub-Gaussian, classical concentration inequalities for sub-Gaussian random variables can be used to control this contribution. In particular, applying exponential versions of Doob’s inequality~\cite{mcdiarmid1998concentration, wainwright2019high} yields faster convergence rates for the stochastic fluctuations compared with the single-spied setting studied in~\cite{arous2021online}. The population part consists of two components: a drift term \( p \sqrt{N} \lambda_i \lambda_j m_{ij}^{p-1}\), which dominates the dynamics, particularly near initialization, and a correction term \(\frac{p}{2} \sqrt{N} \sum_{1 \le k, \ell \le r} \lambda_k m_{i \ell} m_{k j} m_{k \ell} \left (\lambda_j m_{kj}^{p-2} + \lambda_\ell m_{k \ell}^{p-2} \right ) \), which arises from the constraint that the estimator \(\boldsymbol{X}\) lies on the Stiefel manifold. For the correlation \(m_{ij}(t)\) to increase, the drift term must dominate both the noise and correction contributions. Near initialization, the correlations typically scale as \(N^{-\frac{1}{2}}\), in which case the drift term indeed dominates the correction term. Moreover, the noise contribution can be absorbed into the initial condition \(m_{ij}(0)\), leading to the simplified evolution
\begin{equation} \label{eq: simple equation}
m_{ij}(t) \approx m_{ij}(0) + \frac{\delta}{\sqrt{N}} p \lambda_i \lambda_j \sum_{\ell=1}^t m_{ij}^{p-1}(\ell-1).
\end{equation}
At this stage, the key task is to show that the drift term continues to dominate the noise term over a sufficiently long time horizon, allowing the correlations \(m_{ij}\) to \emph{escape mediocrity}. Once the leading correlation reaches a critical threshold, the previously negligible correction term in the population loss becomes significant, and the dynamics transitions to a more delicate regime where the correlations start to interact with one another, as discussed below.

\subsubsection{Analysis of the population dynamics}

We now analyze the evolution of the correlations by studying the population dynamics, first for the case \(p \ge 3\), and then for \(p=2\).\\

\textbf{Recovery for \(p \ge 3\).} The solution of~\eqref{eq: simple equation} is given by
\begin{equation} \label{eq: simple sol p>3}
m_{ij}(t) \approx m_{ij}(0) \left(1 - \frac{\delta}{\sqrt{N}} p \lambda_i \lambda_j m_{ij}(0)^{p-2} t \right)^{-\frac{1}{p-2}},
\end{equation}
where \(m_{ij}(0) = \gamma_{ij} N^{-1/2}\). The typical time for \(m_{ij}\) to reach a macroscopic threshold \(\varepsilon > 0\), denoted by \(T_\varepsilon^{(ij)}\), is therefore given by
\[
T_\varepsilon^{(ij)} \approx \frac{1- \left (\frac{\gamma_{ij}}{\varepsilon \sqrt{N}} \right)^{p-2}}{\delta \lambda_i \lambda_j p  \gamma_{ij}^{p-2}} N^{\frac{p-1}{2}}.
\]
In particular, the first correlation to become macroscopic is the one associated with the largest value of \(\lambda_i \lambda_j \gamma_{ij}^{p-2}\), where \((\gamma_{ij})_{1 \le i,j \le 2}\) are approximately independent standard normal random variables. According to Definition~\ref{def: greedy operation}, this corresponds to the pair \((i_1^\ast,j_1^\ast)\) obtained via the greedy maximum selection of the matrix \(\boldsymbol{I}_0 = (\lambda_i \lambda_j m_{ij}^{p-2} (0) \mathbf{1}_{\{m_{ij}^{p-2}(0) \ge 0\}})_{ij}\). Once the first correlation \(m_{i_1^\ast j_1^\ast}\) reaches a macroscopic value \(\varepsilon >0\), one can verify that all remaining correlations \(\{m_{ij}\}_{(i,j) \neq (i_1^\ast, j_1^\ast)}\) stay at their initial microscopic scale. This follows directly by evaluating~\eqref{eq: simple sol p>3} at time \(T_\varepsilon^{(i_1^\ast j_1^\ast)}\). Note that if the SNRs are sufficiently separated to compensate for initialization differences, then with high probability \(m_{11}\) is the first correlation to escape mediocrity (see also Theorem~\ref{thm: general recovery SGD p>2 asymptotic}). In the following, without loss of generality we may assume that \(i_1^\ast = j_1^\ast = 1\), i.e., \(m_{11}\) is the first correlation to reach the macroscopic threshold \(\varepsilon\). 

We now describe the behavior of \(m_{12}, m_{21}\), and \(m_{22}\) as \(t \ge T_\varepsilon^{(11)}\). The correction part of the population term,
\[
\sum_{1 \le k, \ell \le 2} \lambda_k m_{kj}(t) m_{k \ell}(t) m_{i \ell}(t) (\lambda_j m_{kj}^{p-2}(t) + \lambda_\ell m_{k \ell}^{p-2}(t) ),
\]
becomes dominant in the evolution of the correlations \(m_{12}\) and \(m_{21}\) once \(m_{11}\) reaches the microscopic threshold \(N^{-{\frac{p-2}{2p}}}\). In other words, \(m_{12}\) and \(m_{21}\) initially grow but begin to decrease once \(m_{11}\) exceeds this scale. Similarly, the correction part of the population term may influence the evolution of \(m_{22}\) once \(m_{11}\) reaches another microscopic threshold of order \(N^{-{\frac{p-3}{2(p-1)}}}\), potentially inducing a slight decrease in \(m_{22}\). A detailed analysis shows that this decrease is at most of order \(\frac{\log(N)}{N}\), so that \(m_{22}\) remains stable at the scale \(\frac{1}{\sqrt{N}}\) during the growth phase of \(m_{11}\). This behavior is illustrated in Figure~\ref{fig: zoom p=3}. As \(m_{12}\) and \(m_{21}\) become sufficiently small, the evolution of \(m_{22}\) again follows~\eqref{eq: simple sol p>3}, thus ensuring the recovery of the second direction. This mechanism is referred to as the sequential elimination phenomenon (see Definition~\ref{def: sequential elimination}). The correlations increase sequentially: when one correlation (e.g., \(m_{11}\)) crosses a critical threshold, the correlations sharing its row or column indices (e.g., \(m_{12}, m_{21}\)) decrease until they become negligible, allowing the subsequent correlations (e.g., \(m_{22}\)) to grow. This phenomenon is illustrated in Figure~\ref{fig: p=3,r=2}. Finally, we remark that if the SNRs are sufficiently separated, exact recovery of the unknown signal vectors is achieved with high probability, which depends on the ratio between the signal sizes, as stated in Theorem~\ref{thm: general recovery SGD p>2 asymptotic}. Otherwise, recovery occurs up to a permutation of the spikes, determined by the greedy maximum selection of \(\boldsymbol{I}_0\), as illustrated in Figure~\ref{fig: p=3, r=2, equal}.

\begin{figure}[h]
\centering
\includegraphics[scale=0.25]{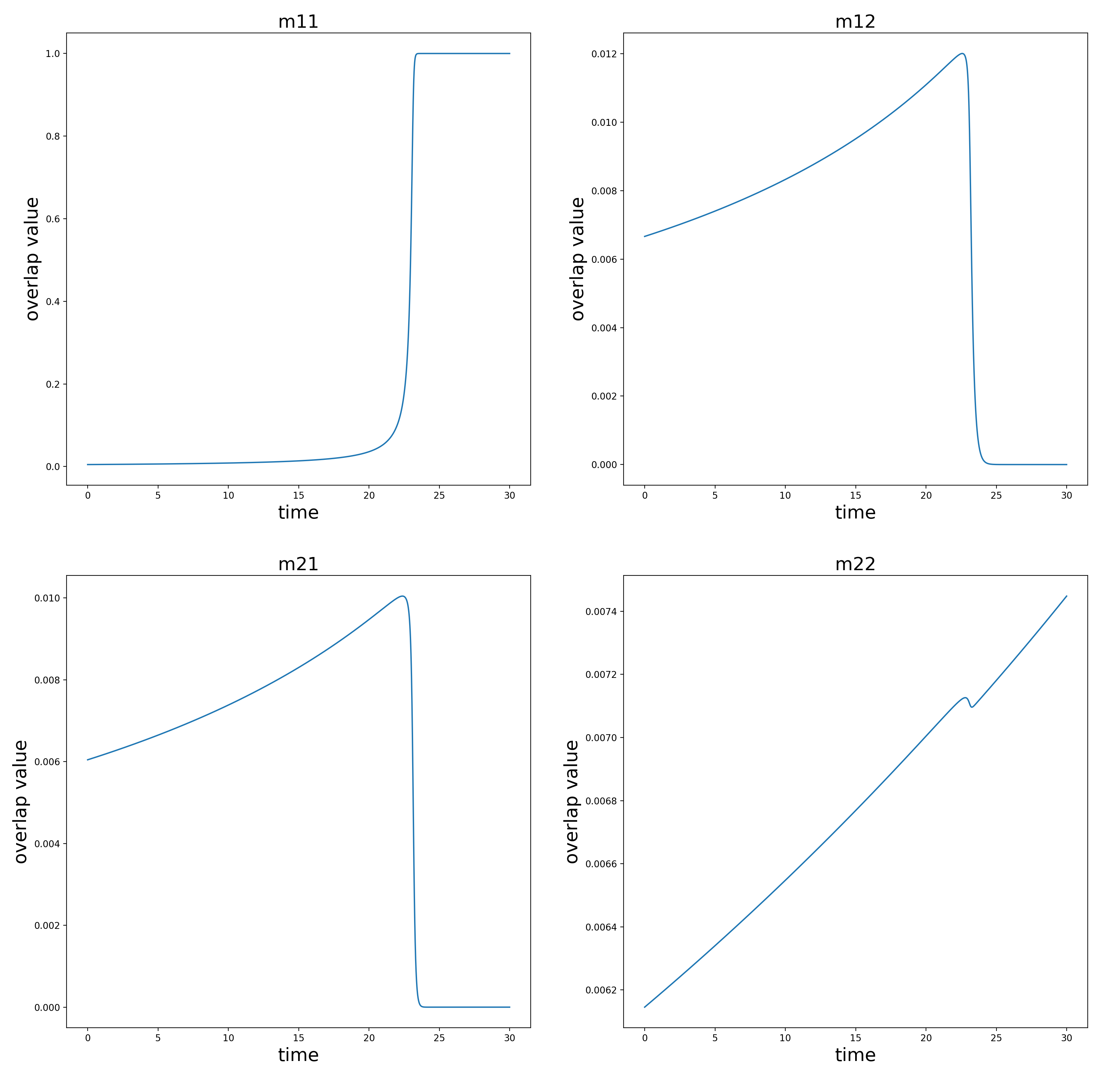}
\caption{Evolution of the correlations \(m_{ij}\) under the population dynamics in the case where \(p=3\), \(r=2\), \(\lambda_1=3\) and \(\lambda_2=1\). The SNRs are sufficiently separated to ensure exact recovery of both spikes. The figure consists of four subfigures, each zooming in on the evolution of one of the four correlations from initialization to the moment when the first correlation \(m_{11}\) becomes macroscopic. In particular, \(m_{12}\) and \(m_{21}\) decrease rapidly once \(m_{11}\) becomes macroscopic, whereas \(m_{22}\) remains stable overall. The short decrease in \(m_{22}\) is bounded by \(\frac{\log(N)}{N}\) with high probability.}
\label{fig: zoom p=3}
\end{figure}

To determine the sample complexity threshold for the online SGD algorithm, we control both the retraction and martingale terms to ensure that the discrete-time dynamics remains dominated by the drift component. We show that this requires a step size parameter of order \(\log(N)^{-1} N^{-\frac{p-3}{2}}\), which in turn leads to a sample complexity of order \(\log(N) N^{p-2}\). This threshold guarantees efficient recovery of both spikes.\\

\textbf{Recovery for \(p =2\).} In this case, the solution of~\eqref{eq: simple equation} is given by
\begin{equation} \label{eq: simple sol p=2}
m_{ij} \approx m_{ij}(0) \exp \left (2 \delta \lambda_i  \lambda_j N^{-\frac{1}{2}} t \right) ,
\end{equation}
with \(m_{ij}(0) = \gamma_{ij}N^{-1/2}\). Hence, the typical time for \(m_{ij}\) to reach a macroscopic threshold \(\varepsilon > 0\) is given by 
\[
T_\varepsilon^{(ij)} \approx \frac{\sqrt{N}  }{2 \delta \lambda_i \lambda_j} \log \left (\frac{\varepsilon \sqrt{N}}{\gamma_{ij}} \right).
\]
In contrast to the case \(p \ge 3\), where the product \(\lambda_i \lambda_j \gamma_{ij}^{p-2}\) determines which correlation first escapes mediocrity, the initialization \(\gamma_{ij}\) has a weaker influence when \(p=2\). Here, the relative separation between the SNRs \(\{\lambda_i\}_{i=1}^r\) governs the order in which the correlations reach macroscopic values. When \(\lambda_1 > \lambda_2\), a sequential elimination phenomenon similar to that observed for \(p \geq 3\) arises (see Figures~\ref{fig: p=2, r=2} and~\ref{fig: p=2 online}). However, there is a key difference compared to the tensor case. Once \(m_{11}\) reaches the macroscopic threshold \(\varepsilon >0\), the other correlations scale as
\[
m_{12}, m_{21} = \mathcal{O}(N^{-\delta_1/2}), \quad  m_{22} = \mathcal{O} (N^{-\delta_2/2}),
\]
where \(\delta_1 = 1 - \lambda_2 /\lambda_1\) and \(\delta_2 = 1 - \lambda_2^2 / \lambda_1^2\). We then show that \(m_{11}\) must reach approximately the critical value \((\lambda_2/\lambda_2)^{1/2}\) before for \(m_{12}\) and \(m_{21}\) begin to decrease. This value is attained in \(\mathcal{O}(1)\) time once \(m_{11}\) becomes macroscopic, while \(m_{12}\) and \(m_{21}\) require a time of order \(\log(N)\) to exit their scale of \(N^{-\frac{\delta_1}{2}}\). Moreover, we show that \(m_{22}\) remains stable during this first phase: it continues to increase throughout the ascent of \(m_{11}\) and up to the point where \(m_{12}\) and \(m_{21}\) start to decrease. This stability holds under an order-one lower bound on \(\log(N)\), involving the inverse of the gap between the SNRs. Consequently, when the SNRs are close, finite-size effects become significant. Once \(m_{12}\)and \(m_{21}\) decay below the scale \(N^{-1/2}\), the correlation \(m_{22}\) is again increasing and can reach a macroscopic threshold, as is illustrated in Figure~\ref{fig:p2_r2_subplots}.

\begin{figure}
\centering
\includegraphics[scale = 0.25]{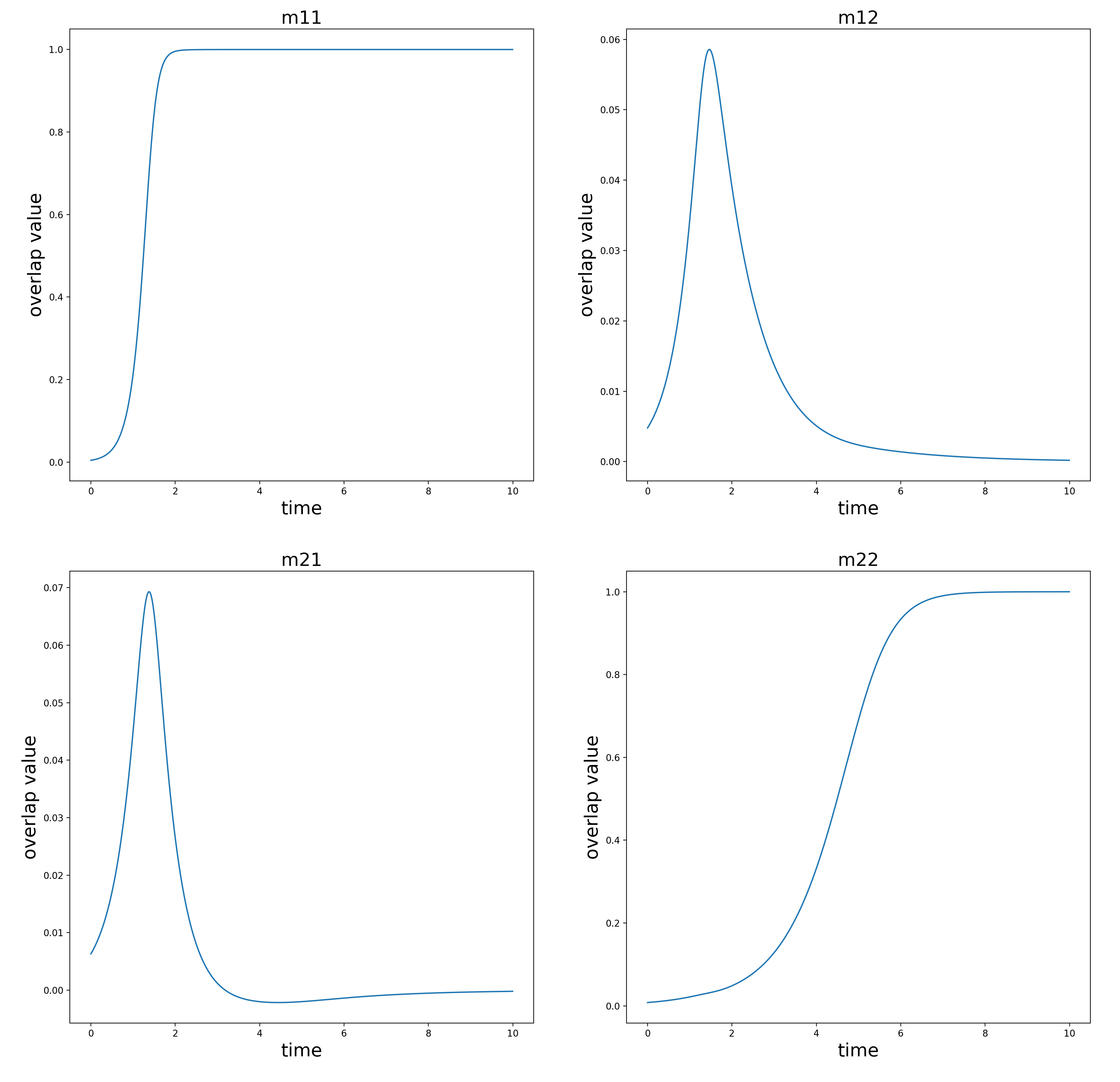}
\caption{Evolution of the correlations \(m_{ij}\) under the population dynamics in the case where \( p=2, r=2, \lambda_1 = 2\) and \(\lambda_2 = 1\). We observe that as \(m_{11}\) reaches a certain threshold, the correlations \(m_{12}\) and \(m_{21}\) start to decrease, despite rising higher than in the case where \(p\geq 3\).}
\label{fig:p2_r2_subplots}
\end{figure}

Regarding the sample complexity, as in the case \(p \ge 3\), we control the martingale and retraction terms so that the discrete-time dynamics remains dominated by the drift component. We show that this requires a step size \(\delta\) of order \(\log(N)^{-1} N^{ \lambda_r^2/(2 \lambda_1^2)}\), which in turn leads to a sample complexity of order \(\log(N)^2 N^{\frac{1}{2}(1 - \lambda_r^2/\lambda_1^2)}\). The factor \(N^{\lambda_r^2/(2\lambda_1^2)}\) arises from the need to control the stability of the retraction map, which must hold uniformly across all directions, including the first.

\subsection{Subspace recovery}

As discussed in Subsection~\ref{subsection: main asymptotic results}, when \(\lambda_1 = \lambda_2 \equiv \lambda\), the spiked covariance model becomes isotropic. In this case, we focus on the evolution of the eigenvalues of the matrix \(\boldsymbol{M} \boldsymbol{M}^\top\), denote by \(\boldsymbol{G}\). To analyze their evolution, we employ eigenvalue perturbation bounds together with matrix versions of martingale concentration inequalities. Once the noise is appropriately controlled, the population dynamics satisfy
\[
\boldsymbol{G} (t) \approx \boldsymbol{G}(0)+\frac{4\delta \lambda^2}{\sqrt{N}} \sum_{\ell=1}^{t}\boldsymbol{G}(\ell-1)(\boldsymbol{I} - \boldsymbol{G}(\ell-1)),
\]
which leads to the following approximate evolution for the eigenvalues \(\theta_1,\ldots, \theta_r\):
\[
\theta_i(t) \approx \theta_i(0)+\frac{4\delta \lambda^2}{\sqrt{N}}\sum_{\ell=1}^{t} \theta_i (\ell-1) \left ( 1-\theta_i (\ell-1)\right ).
\]
This equation then prescribes a time horizon of order \(\log(N)\) for maintaining the necessary bounds on the noise process, leading to the sample complexity stated in Theorem~\ref{thm: strong recovery isotropic SGD asymptotic}. A key difference from the full recovery case is that, at initialization, the eigenvalues of \(\boldsymbol{G}\) are of order \(N^{-1}\). Simply adapting the bounds derived for the correlations \(m_{ij}\) is not sufficient here. To handle this regime, we establish time-dependent bounds for the martingale terms using a matrix generalization of Freedman’s inequality~\cite{tropp2011freedman}.

%%%%%%%%%%%%%%%%%%%%%%%%%%%%%%%%%%%%%%%%%%%%%%%%%%%%%%%%%%%%%%%%%%%%%%%%
\section{Main results in nonasymptotic form} \label{section: main results SGD}
%\addtocontents{toc}{\protect\setcounter{tocdepth}{1}}

%%%%%%%%%%%%%%%%%%%%%%%%%%%%%%%%%%%%%%%%%%%%%%%%%%%%%%%%%%%%%%%%%%%%%%%%
This section presents the nonasymptotic versions of our main results stated in Subsection~\ref{subsection: main asymptotic results}.

As previously discussed, we consider the online SGD algorithm starting from a random initialization, precisely \(\boldsymbol{X}_0 \sim \mu_{N \times r}\), where \(\mu_{N \times r}\) denotes the uniform measure on the Stiefel manifold \(\textnormal{St}(N,r)\). Our recovery guarantees, however, applies to a broader class of initial data that meets two natural conditions, that we introduce hereafter. The first condition ensures that the initial correlation is on the typical scale \(\Theta(N^{-\frac{1}{2}})\).

\begin{defn}[Condition 1] \label{def: condition 1 SGD}
For every \(\gamma_1 > \gamma_2 >0\), let \(\mathcal{C}_1(\gamma_1,\gamma_2)\) denote the sequence of events given by
\[
\mathcal{C}_1 (\gamma_1,\gamma_2) = \left \{ \boldsymbol{X} \in \textnormal{St}(N,r) \colon \frac{\gamma_2}{\sqrt{N}} \leq m_{ij}(\boldsymbol{X}) < \frac{\gamma_1}{\sqrt{N}}\enspace \text{for every} \enspace 1 \leq i,j \leq r\right \}.
\]
We say that a sequence of probability measures \(\mu_N \in \mathcal{M}_1 (\textnormal{St}(N,r))\) satisfies \emph{Condition 1} if for every \(N \in \N\) and \(\gamma_1 > \gamma_2 > 0\),
\[
\mu_N \left (\mathcal{C}_1(\gamma_1,\gamma_2)^\textnormal{c} \right ) \leq C_1 e^{- c_1 \gamma_1^2} + C_2 e^{- c_2 \gamma_2 \sqrt{N}} + C_3 \gamma_2 ,
\]
for absolute constants \(C_1, c_1, C_2, c_2, C_3 >0\).
\end{defn} 

In the special case \(p=2\) and \(\lambda_1 = \cdots = \lambda_r\), we analyze the evolution of the eigenvalues of the matrix \(\boldsymbol{G} = \boldsymbol{M} \boldsymbol{M}^\top\), where \(\boldsymbol{M} = (m_{ij})_{1 \le i,j \le r}\). To ensure that the eigenvalues at initialization are on the typical scale of \(\Theta(N^{-1})\), we require a modified version of Condition \(1\) from Definition~\ref{def: condition 1 SGD}.

\begin{defn}[Condition 1'] \label{def: condition 1' SGD}
For every \(\gamma_1 > \gamma_2 >0\), let \(\mathcal{C}_1'(\gamma_1,\gamma_2)\) denote the sequence of events given by
\[
\mathcal{C}_1'(\gamma_1,\gamma_2) = \left \{ \boldsymbol{X} \in \textnormal{St}(N,r) \colon \frac{\gamma_2}{N} \le \theta_i (\boldsymbol{X}) < \frac{\gamma_1}{N} \enspace \textnormal{for every} \, i \in [r] \right \}.
\]
We say that a sequence of probability measures \(\mu_N \in \mathcal{M}_1 (\textnormal{St}(N,r)\) satisfies \emph{Condition \(1'\)} if for every \(N \in \N\) and \(\gamma_1 > \gamma_2 > 0\),
\[
\mu_N \left ( \mathcal{C}_1'(\gamma_1,\gamma_2)^\textnormal{c} \right) \le C_1 e^{-c_1 \gamma_1^2} +  C_2  e^{- c_2 \gamma_2 \sqrt{N}} + C_3 \gamma_2,
\]
for absolute constants \(C_1, c_1, C_2, c_2, C_3 >0\).
\end{defn} 

The second condition concerns the separation of the initial correlations, which is needed for the recovery when \(p \ge 3\).

\begin{defn}[Condition 2] \label{def: condition 2 SGD}
For every \(\gamma_1 > \gamma > 0\), let \(\mathcal{C}_2(\gamma_1, \gamma)\) denote the sequence of events given by
\[
\mathcal{C}_2(\gamma_1, \gamma) = \left \{\boldsymbol{X} \in \textnormal{St}(N,r) \colon \left | \frac{\lambda_i \lambda_j m_{ij}^{p-2}(\boldsymbol{X})}{\lambda_k \lambda_\ell m_{k \ell}^{p-2}(\boldsymbol{X})} - 1 \right | > \frac{\gamma}{\gamma_1} \enspace \text{for every} \enspace 1 \leq i,j,k,\ell \leq r , (i,j) \neq (k,\ell)\right \}.
\]
We say that a sequence of probability measures \(\mu_N \in \mathcal{M}_1 (\textnormal{St}(N,r)\) satisfies \emph{Condition 2} if for every \(N \in \N\) and \(\gamma_1 > \gamma >0\), 
\[
\mu_N \left ( \mathcal{C}_2(\gamma_1,\gamma)^\textnormal{c} \right ) \leq C_1 e^{- c_1 \gamma_1^2} + C_2 e^{- c_2 \gamma \sqrt{N}} + C_3 \sup_{i,j,k,\ell} \left ( 1 + \left( \frac{\lambda_i \lambda_j}{\lambda_k \lambda_\ell} \right)^{\frac{2}{p-2}} \right)^{-\frac{1}{2}} \gamma,
\]
for absolute constants \(C_1, c_1, C_2, c_2, C_3 >0\).
\end{defn}

The uniform measure \(\mu_{N \times r}\) satisfies the above conditions, as stated in the following lemma.

\begin{lem} \label{lem: invariant measure SGD}
The uniform measure \(\mu_{N \times r}\) on \(\textnormal{St}(N,r)\) satisfies Condition \(1\), Condition \(1'\), and Condition \(2\).  
\end{lem}

The proof of Lemma~\ref{lem: invariant measure SGD} is deferred to our companion papers~\cite{langevin, gradientflow}. In particular, Conditions \(1\) and \(2\) are established in~\cite[Appendix A]{gradientflow}, while Condition \(1'\) is proved in Appendix A of~\cite[Appendix A]{langevin}. In both works, we consider the normalized Stiefel manifold \(\mathcal{M}_{N,r} = \{\boldsymbol{X} \in \R^{N \times r} \colon \boldsymbol{X}^\top \boldsymbol{X} = N\boldsymbol{I}_r\}\) and define correlations by \(m_{ij}^{(N)}(\boldsymbol{X}) = \frac{\langle \boldsymbol{v}_i, \boldsymbol{x}_j \rangle}{N}\) with \(\boldsymbol{V}= [\boldsymbol{v}_1, \ldots, \boldsymbol{v}_r] \in \mathcal{M}_{N,r}\). We show that the invariant measure on \(\mathcal{M}_{N,r}\) satisfies the three conditions above. Since this framework is equivalent to the one considered here, the proofs in~\cite{langevin, gradientflow} generalize directly for \(\mu_{N \times r}\).\\

We are now ready to present our main results in a nonasymptotic form. Following Remark~\ref{rmk: sign initialization}, we assume that the correlations are initialized positively, i.e., \(m_{ij}(\boldsymbol{X}_0)>0\) for every \(i,j \in [r]\). This allows us to drop the absolute values in the correlation terms in the statements below and ensures that all \(r\) spikes can be recovered up to a permutation (i.e., \(r_\textnormal{c}=r\) in Definition~\ref{def: greedy operation}). By definition, \(W_{i_1, \ldots, i_p}\) are i.i.d.\ centered sub-Gaussian random variables with parameter \(\sigma > 0\), meaning that for all \(t \in \R\), 
\[
\E \left [ e^{t W_{i_1, \ldots, i_p}} \right ] \le e^{K t^2 \sigma^2},
\]
where \(K>0\) is an absolute constant. A more precise definition is given in~\eqref{eq: MGF sub-Gaussian}. Let \(T\) denote the number of iterations after which we terminate the online SGD algorithm. For any set \(\mathcal{E}\), we define the hitting time \(\mathcal{T}_{\mathcal{E}}\) of the set \(\mathcal{E}\) as
\[
\mathcal{T}_{\mathcal{E}} = \inf \{t \in \N_0 \colon \boldsymbol{X}_t \in \mathcal{E}\}. 
\]
Moreover, in the following theorems, the step size \(\delta>0\) will depend on a sequence \(d_0 = d_0(N)\) which tends to zero as \(N \to \infty\), with \(d_0^{-1}\) growing at most polynomially in \(N\). That is, there exists constants \(k \ge 0\) and \(C>0\) such that
\begin{equation} \label{eq: sequence d_0}
d_0  \ge \frac{1}{C N^k}.
\end{equation}
Throughout, we use the notation \(\lesssim\) to denote inequality up to an absolute constant. Finally, the constants \(K_1, K_2, \ldots\) and \(c_1, c_2, \ldots\) appearing in the theorems below may depend on the  model parameters \(\sigma^2, p, r\), and \(\{\lambda_i\}_{i=1}^r\).\\

Let us first consider the case \(p \ge 3\). Our first main result establishes that the sample complexity threshold for efficiently recovering all spikes (up to a permutation) scales as \(\log(N) N^{p-2}\). Recall that the pairs \(\{(i_k^\ast, j_k^\ast)\}_{k=1}^r\) denote the greedy maximum selection of \(\boldsymbol{I}_0\) (see~\eqref{eq: initialization matrix} and Definition~\ref{def: greedy operation}). Furthermore, for every \(\varepsilon > 0\), let \(R(\varepsilon)\) denote 
\[
R(\varepsilon) = \left \{ \boldsymbol{X} \colon m_{i_k^\ast j_k^\ast}(\boldsymbol{X}) \geq 1 - \varepsilon  \enspace \textnormal{for all} \: k \in [r] \enspace \textnormal{and} \enspace m_{ij}(\boldsymbol{X}) \lesssim \log(N)^{-\frac{1}{2}}N^{-\frac{p-1}{4}} \enspace \textnormal{otherwise} \right \}.
\]

\begin{thm}[Recovery of all spikes up to a permutation for \(p \ge 3\)] \label{thm: strong recovery online p>2 nonasymptotic}
For every \(\gamma_1 > \gamma_2 \vee \gamma >0\) and every \(\varepsilon >0\), there exist a constant \(c_0 = c_0 (\gamma_1, \gamma, p) \in (0,1)\) and a sequence \(d_0 = d_0 (N)\) satisfying~\eqref{eq: sequence d_0} such that the following holds. If the initialization satisfies \(\boldsymbol{X}_0 \in \mathcal{C}_1(\gamma_1,\gamma_2) \cap \mathcal{C}_2(\gamma_1,\gamma)\), and the step size \(\delta >0\) is chosen as
\[
\delta = C_\delta d_0 \log(N)^{-1} N^{-\frac{p-3}{2}},
\]
for a constant \(C_\delta>0\), then for \(N\) sufficiently large,
\[
\mathbb{P}_{\boldsymbol{X}_0^+} \left(  \mathcal{T}_{R(\varepsilon)} \lesssim T \right) \geq 1- \eta,
\]
where the runtime \(T\) and error probability \(\eta\) are given by
\[
T = \frac{\log(N) N^{p-2} }{ C_\delta d_0 p \lambda_r^2 \gamma_2^{p-2}}+ \frac{\log(N)^2 N^{\frac{p-2}{2}}}{C_\delta d_0\varepsilon^p p \lambda_r^2} ,
\]
and 
\[
\eta  = K_1 \frac{ \log(N)^2N^{p-2}}{d_0} e^{-c_1 N} + K_2 \log(N)e^{-c_2 c_0^2\log(N)/d_0} + K_3 \frac{d_0^2}{\log(N) N^{\frac{p-2}{2}}}  e^{-c_3/d_0^2} .
\]
Finally, the correlations $\{m_{ij}\}_{1 \leq i,j \leq r}$ follow a sequential elimination with ordering $S = \{(i^\ast_{k},j^\ast_{k})\}_{k=1}^r$.
\end{thm}

Next, we consider the case \(p=2\) with sufficiently separated SNRs, assuming
\[
\lambda_i = \lambda_{i+1} (1 + \kappa_i) \quad \textnormal{for} \enspace 1 \le i \le r-1,
\]
where \(\kappa_i >0\) are constants of order \(1\). Let \(\kappa\) denote \(\kappa = \min_{1 \le i \le r-1} \kappa_i\). The following theorem shows that, under this separation condition on the SNRs, exact recovery of all spikes is possible with a sample complexity of order \(\log(N)^2 N^{\frac{1}{2} \left ( 1 - \lambda_r^2 / \lambda_1^2\right)}\). To rigorously state our result, define for every \(\varepsilon > 0\) and \(c_0 \in (0, \frac{1}{2} \wedge \frac{\kappa}{2+\kappa})\), 
\[
R(\varepsilon, c_0) = \left \{ \boldsymbol{X} \colon m_{ii}(\boldsymbol{X}) \geq 1 - \varepsilon  \enspace \textnormal{for all} \: i \in [r] \enspace \textnormal{and} \enspace m_{k \ell}(\boldsymbol{X}) \lesssim N^{- \xi / 2} \enspace \textnormal{for all} \: k \neq \ell \right \},
\]
where \(\xi = \xi(c_0)\coloneqq  1 - \frac{1-c_0}{1+c_0} \frac{\lambda_r^2}{\lambda_1^2}\).

\begin{thm}[Exact recovery of all spikes for \(p = 2\)] \label{thm: strong recovery online p=2 nonasymptotic}
Assume \(\kappa > \sqrt{2} - 1\). Fix \(\gamma_1 > \gamma_2 >0\). Then, there exists \(\bar{\varepsilon} (\kappa) \in (0, \varepsilon_0)\) with \(\varepsilon_0^2 = \frac{\kappa}{1+\kappa}\) such that, for every \(\varepsilon \in (0, \bar{\varepsilon} (\kappa))\), there exist \(c_0 = c_0 (\kappa,\varepsilon) \in \left (0, \frac{1}{2} \wedge \frac{1+\kappa}{2+\kappa} (\varepsilon_0^2 - \varepsilon^2) \right)\) and a sequence \(d_0 = d_0(N)\) satisfying~\eqref{eq: sequence d_0} such that the following holds. If the initialization satisfies \(\boldsymbol{X}_0 \in \mathcal{C}_1(\gamma_1,\gamma_2)\), and the step size \(\delta >0\) is chosen as 
\[
\delta = \frac{C_\delta d_0 \gamma_2^2 N^{(1-\xi)/2}}{\log \left( \frac{2 \varepsilon \sqrt{N}}{\gamma_2}\right)},
\]
for a constant \(C_\delta>0\), then for \(N\) sufficiently large, 
\[
\mathbb{P}_{\boldsymbol{X}_0^+} \left( \mathcal{T}_{R(\varepsilon, c_0)} \lesssim T \right) \geq 1 - \eta,
\]
where \(T\) and \(\eta\) are given by
\[
T = \frac{ N^{\xi / 2} \log \left( \frac{2\varepsilon \sqrt{N}}{\gamma_2} \right)^2}{C_\delta d_0 (1 - c_0 - \varepsilon^2)\gamma_2^2 \lambda_r^2},
\]
and  
\[
\eta = K_1 \frac{\log(N)^2 N^{\xi/2}}{d_0} e^{- c_1 N} + K_2 e^{ - c_2 c_0^2 N^{\xi/2} /d_0} + K_3 \frac{d_0^2}{c_0 \log(N) N^{\xi}} e^{- c_3 / d_0^2}.
\]
Finally, the correlations $\{m_{ij}\}_{1 \leq i,j \leq r}$ follow a sequential elimination with ordering $S = \{(k,k)\}_{k=1}^r$.
\end{thm}

\begin{rmk}
The assumption \(\kappa > \sqrt{2}-1\) and the restriction \(\varepsilon < \varepsilon_0\) in Theorem~\ref{thm: strong recovery online p=2 nonasymptotic} are both required to ensure that the growth of all cross-correlations remains controlled. For \(p=2\), the off-diagonal correlations \(m_{ij}\) (\(i \neq j\)) can temporarily exceed their initial microscopic scale \(N^{-1/2}\), and these conditions guarantee that such effects do not destabilize the macroscopic dynamics. In particular, \(\kappa > \sqrt{2}-1\) ensures sufficient separation between consecutive SNRs, preventing interference between recovered and unrecovered components, while \(\varepsilon < \varepsilon_0\) guarantees that all quantities in the finite-sample estimates remain positive and that the contraction arguments used in the proof hold uniformly for large but finite \(N\). In the asymptotic regime of Theorem~\ref{thm: strong recovery online p=2 asymptotic}, however, this small-\(\varepsilon\) constraint can be removed: for every fixed \(\varepsilon > 0\), one may choose \(c_0 = c_0(\kappa, \varepsilon) \in (0,1)\) such that the same estimates hold for all sufficiently large \(N\). Consequently, the asymptotic statement remains valid for arbitrary \(\varepsilon > 0\).
\end{rmk}

We conclude by considering the case of identical SNRs, i.e., \(\lambda_1 =\cdots = \lambda_r \equiv \lambda\). In this regime, the dynamics is naturally described in terms of the eigenvalues \(\{\theta_i\}_{i=1}^r\) of the symmetric matrix \(\boldsymbol{M} \boldsymbol{M}^\top \in \R^{r \times r}\), where \(\boldsymbol{M} = (m_{ij})_{1 \le i,j \le r}\) denotes the correlation matrix. 

\begin{thm}[Subspace recovery for \(p=2\)] \label{thm: strong recovery isotropic SGD nonasymptotic}
For every \(\gamma_1 > \gamma_2 >0\) and every \(\varepsilon,\varepsilon' >0\), there exists a sequence \(d_0 = d_0(N)\) satisfying~\eqref{eq: sequence d_0} such that the following holds. If the initialization satisfies \(\boldsymbol{X}_0 \in \mathcal{C}_1'(\gamma_1,\gamma_2)\), and the step size \(\delta > 0\) is chosen as
\[
\delta = C_\delta \frac{d_0 \sqrt{N}}{\log(N)^2},
\]
for a constant \(C_\delta >0\), then for \(N\) sufficiently large,
\[
\mathbb{P}_{\boldsymbol{X}_0} \left( \theta_{\min} \left(\boldsymbol{X}_T \right)\geq 1 - \varepsilon' \right) \geq 1 - \eta,
\]
where \(T\) and \(\eta\) are given by
\[
T = \frac{\log(N)^2\log(\frac{4\varepsilon N}{\gamma_2 })}{C_\delta d_0 \lambda^2}+\frac{\log(N)^2 \log(\frac{1-\varepsilon'}{\varepsilon})}{C_\delta d_0 \lambda^2 \varepsilon'},
\]
and 
\[
\eta = K_1 \frac{\log(N)^3}{d_0} e^{-c_1 N} + K_2 e^{-c_2/d_0} + K_3  \frac{\log(N)^3}{N d_0}.
\]
\end{thm}

Before addressing the proof of the three main nonasymptotic results, we first provide the proofs of their asymptotic counterparts of Subsection~\ref{subsection: main asymptotic results}.

\begin{proof}[\textbf{Proof of Theorems~\ref{thm: strong recovery online p>2 asymptotic stronger},~\ref{thm: strong recovery online p=2 asymptotic}, and~\ref{thm: strong recovery isotropic SGD asymptotic}}]
We begin with the proof of Theorem~\ref{thm: strong recovery online p>2 asymptotic stronger} under the assumption that all the correlations are positively initialized. Theorems~\ref{thm: general recovery SGD p>2 asymptotic} and~\ref{thm: strong recovery online p>2 asymptotic} are straightforward consequences of Theorem~\ref{thm: strong recovery online p>2 asymptotic stronger}. Let \(\{(i^\ast_k, j^\ast_k)\}_{k=1}^r\) denote the greedy maximum selection, as per Definition~\ref{def: greedy operation}, of the initial matrix \(\boldsymbol{I}_0\), as defined~\eqref{eq: initialization matrix}. Moreover, let \(\pi^\ast \in S_r\) denote the permutation of the recovered spikes such that the pairs \(\{(\pi^\ast(i),i)\}_{i=1}^r\) coincide (as sets) with \(\{(i^\ast_k, j^\ast_k)\}_{k=1}^r\). Then, for every \(\varepsilon > 0\),
\[
\mathbb{P}_{\boldsymbol{X}^{+}_0}\left( \forall k \in [r] \colon m_{i^\ast_{k}j^\ast_{k}}(\boldsymbol{X}_M)  \geq 1-\varepsilon \right)\geq \mathbb{P}_{\boldsymbol{X}^{+}_0}\left(\boldsymbol{X}_M \in \mathcal{S}_{\pi^\ast}(\varepsilon)\right),
\]
where the set \(\mathcal{S}_{\pi^\ast}(\varepsilon)\) is defined in Appendix~\ref{app:stability_results}. We begin by conditioning the initialization to lie in the set \(\mathcal{C}_1(\gamma_1,\gamma_2) \cap \mathcal{C}_2(\gamma_1,\gamma)\). We then obtain
\[
\begin{split}
\mathbb{P}_{\boldsymbol{X}^{+}_0} \left(\boldsymbol{X}_M \in \mathcal{S}_{\pi^\ast}(\varepsilon)\right) &  = \E_{\boldsymbol{X}_0^+} \left [ \mathbf{1}_{\{(\boldsymbol{X}_M \in \mathcal{S}_{\pi^\ast}(\varepsilon)\}} \right ] \\
& \geq \E_{\boldsymbol{X}_0^+} \left [ \mathbf{1}_{\{(\boldsymbol{X}_M \in \mathcal{S}_{\pi^\ast}(\varepsilon)\}} \Big | \boldsymbol{X}_0 \in \mathcal{C}_1(\gamma_1,\gamma_2) \cap \mathcal{C}_2(\gamma_1,\gamma) \right ] \mu_{N \times r} \left ( \mathcal{C}_1(\gamma_1,\gamma_2) \cap \mathcal{C}_2(\gamma_1,\gamma) \right) \\
& \ge \inf_{\boldsymbol{X}_0 \in \mathcal{C}_1(\gamma_1,\gamma_2) \cap \mathcal{C}_2(\gamma_1,\gamma)} \mathbb{P}_{\boldsymbol{X}_0^+}\left(\boldsymbol{X}_M \in \mathcal{S}_{\pi^\ast}(\varepsilon)\right) - \left ( \mu_{N \times r} \left ( \mathcal{C}_1(\gamma_1,\gamma_2) ^\mathrm{c} \right ) +  \mu_{N \times r} \left ( \mathcal{C}_2(\gamma_1,\gamma) ^\mathrm{c} \right ) \right ) \\
& \ge \inf_{\boldsymbol{X}_0 \in \mathcal{C}_1(\gamma_1,\gamma_2) \cap \mathcal{C}_2(\gamma_1,\gamma)} \mathbb{P}_{\boldsymbol{X}_0^+}\left(\boldsymbol{X}_M \in \mathcal{S}_{\pi^\ast}(\varepsilon)\right)  - \eta_1,
\end{split}
\]
where Lemma \ref{lem: invariant measure SGD} provides the explicit estimate for \(\eta_1\):
\[
\eta_1 = C_1 e^{-c_1\gamma_1^2}+C_2e^{-c_2(\gamma+\gamma_2)\sqrt{N}}+C_3\gamma_2+C_4 \sup_{i,j,k,\ell} \left ( 1 + \left( \frac{\lambda_i \lambda_j}{\lambda_k \lambda_\ell} \right)^{\frac{2}{p-2}} \right)^{-\frac{1}{2}} \gamma.
\]
Now, fix \(\boldsymbol{X}_0 \in \mathcal{C}(\gamma_1,\gamma_2)\cap \mathcal{C}(\gamma_1,\gamma)\). Then,
\[
\mathbb{P}_{\boldsymbol{X}_0^+} \left(\boldsymbol{X}_M \in \mathcal{S}_{\pi^\ast}(\varepsilon)\right) \geq \mathbb{P}_{\boldsymbol{X}_0^+} \left(\boldsymbol{X}_M \in \mathcal{S}_{\pi^\ast} (\varepsilon) \,  \Big \vert \, \mathcal{T}_{R(\varepsilon/2)} \leq \frac{M}{2}, \boldsymbol{X}_{\mathcal{T}_{R(\varepsilon/2)}}\right)-\mathbb{P}_{\boldsymbol{X}_0^+} \left(\mathcal{T}_{R(\varepsilon/2)}>\frac{M}{2}\right),
\]
where \(R(\varepsilon/2)\) is the recovery event of Theorem~\ref{thm: strong recovery online p>2 nonasymptotic} and \(\mathcal{T}_{R(\varepsilon/2)}\) is the stopping time corresponding to the event \(R(\varepsilon/2)\). We introduce sequences \(d_1(N)\) and \(d_2(N)\) such that
\[
\begin{split}
M & = d_1(N)\log(N)N^{p-2},\\
\delta & = d_2(N)N^{1/2}\left(\log(N)M\right)^{-1/2}.
\end{split}
\]
By assumption of Theorem~\ref{thm: strong recovery online p>2 asymptotic stronger}, \(d_1(N) \to \infty\) and grows at most polynomially in \(N\) and \(d_2(N) \to 0\) as \(N \to \infty\). Moreover, since \(\delta^{-1}N^{\frac{p-1}{2}} = (d_2(N)\sqrt{d_1(N)})^{-1}M\), we require 
\[
d_2 (N) \gg d_1(N) ^{-1/2}. 
\]
Define 
\[
d_0(N) \coloneqq \frac{1}{C_{\delta}}\frac{d_2(N)}{\sqrt{d_1(N)}}, 
\]
where \(C_\delta >0\) is the constant from Theorem~\ref{thm: strong recovery online p>2 nonasymptotic}. Then,  
\[
\delta = C_{\delta}d_0(N)\log(N)^{-1}N^{-\frac{p-3}{2}},
\]
and \(d_0(N)\) tends to \(0\) as \(N \to \infty\) and \(\frac{1}{d_0(N)}\) grows at most polynomially in \(N\). By Theorem~\ref{thm: strong recovery online p>2 nonasymptotic}, for \(N\) sufficiently large, we have
\begin{equation} \label{eq1}
\inf_{\boldsymbol{X}_0 \in \mathcal{C}_1(\gamma_1,\gamma_2) \cap \mathcal{C}_2(\gamma_1,\gamma)} \mathbb{P}_{\boldsymbol{X}_0^+} \left(  \mathcal{T}_{R(\varepsilon)} \lesssim T \right) \geq 1 - \eta_2 ,
\end{equation}
where  
\[
\eta_2 = K_1 \frac{ \log(N)^2N^{p-2}}{d_0} e^{-c_1 N} + K_2 \log(N)e^{-c_2 c_0^2\log(N)/d_0} + K_3 \frac{d_0^2}{\log(N) N^{\frac{p-2}{2}}}  e^{-c_3/d_0^2},
\]
and the runtime \(T\) satisfies
\begin{equation} \label{eq2}
T \gtrsim \gamma_2^{-(p-2)}\delta^{-1}N^{\frac{p-1}{2}}.
\end{equation}
We now choose the sequences 
\[
\gamma_1(N) = -\log (d_0(N)), \enspace \gamma_2(N) = \log \left( \left (d_2(N)\sqrt{d_1(N)} \right)^{-1} \right), \enspace \gamma(N) = -\log(d_0(N))^{-1},
\]
and set 
\[
c_0 = \frac{\left(1+\frac{\gamma}{\gamma_1}\right)^{\frac{1}{p-1}}-1}{2\left(1+\frac{\gamma}{\gamma_1}\right)^{\frac{1}{p-1}}+2}.
\]
For \(N\) sufficiently large, these choices ensure that \(d_0 (N) \ll c_0^2\) and \(M/2 > T\) deterministically. Combining~\eqref{eq1} and~\eqref{eq2} gives
\[
\inf_{\boldsymbol{X}_0 \in \mathcal{C}_1(\gamma_1,\gamma_2) \cap \mathcal{C}_2(\gamma_1,\gamma)} \mathbb{P}_{\boldsymbol{X}_0^+}\left(  \mathcal{T}_{R(\varepsilon)} \lesssim \frac{M}{2} \right) \ge 1 - \eta_2.
\]
Under the event \(\{\mathcal{T}_{R(\varepsilon)} \lesssim \frac{M}{2}\}\), the strong Markov property of the online SGD process yields 
\[
\mathbb{P}_{\boldsymbol{X}_0^+} \left(\boldsymbol{X}_M \in \mathcal{S}_{\pi^\ast}(\varepsilon) \,  \Big \vert \, \mathcal{T}_{R(\varepsilon/2)} \leq \frac{M}{2}, \boldsymbol{X}_{\mathcal{T}_{R(\varepsilon/2)}}\right) \geq \inf_{t \in [\frac{M}{2},M]}\mathbb{P}_{\boldsymbol{X}_{\mathcal{T}_{R(\varepsilon/2)}}}\left(\boldsymbol{X}_{t} \in \mathcal{S}_{\pi^\ast}(\varepsilon)\right).
\]
Applying Lemma~\ref{lem: stability separated} with $p \ge 3$ provides the lower bound 
\[
\inf_{t \in [\frac{M}{2},M]} \mathbb{P}_{\boldsymbol{X}_{\mathcal{T}_{R(\varepsilon/2)}}}\left(\boldsymbol{X}_{t} \in \mathcal{S}_{\pi^\ast}(\varepsilon)\right) \ge 1 - \eta_3,
\]
where \(\eta_3 = \eta_3(N)\) tends to \(0\) in the large-\(N\) limit. Finally, combining the bounds above, we obtain
\[
\mathbb{P}_{\boldsymbol{X}^{+}_0}\left( \forall k \in [r] \colon m_{i^\ast_{k}j^\ast_{k}}(\boldsymbol{X}_M) \ge  1-\varepsilon\right) \geq 1-\eta_1-\eta_2-\eta_3.
\]
Using the prescribed sequences for $\gamma,\gamma_1,\gamma_2$ in the expressions for $\eta_1$ and $\eta_2$ clearly show that both terms vanish as \(N \to \infty\). This completes the proof. 

The proofs of Theorems~\ref{thm: strong recovery online p=2 asymptotic} and~\ref{thm: strong recovery isotropic SGD asymptotic} follow by identical arguments, replacing Theorem~\ref{thm: strong recovery online p>2 nonasymptotic} with Theorem~\ref{thm: strong recovery online p=2 nonasymptotic} and Theorem~\ref{thm: strong recovery isotropic SGD nonasymptotic}, respectively, and substituting Lemma~\ref{lem: stability separated} with Lemma~\ref{lem:stability_isotropic}.
\end{proof}

%%%%%%%%%%%%%%%%%%%%%%%%%%%%%%%%%%%%%%%%%%%%%%%%%%%%%%%%%%%%%%%%%%%%%%%%%%%%%%%%%%%%%%%%%%%%%%%%%%%%%%%%%%%%%%%%%%%%%%%%%%%%%%%%%%%%%%%%%%%%%%%%%%%%%%%%%%%%%%%%%%%%%%%%%%%%%%%%%%%%%%%%%%
\section{Comparison inequalities} \label{section: background}

This section presents the comparison inequalities for both the correlations and the eigenvalues of \(\boldsymbol{G} = \boldsymbol{M} \boldsymbol{M}^\top\). Before delving into these, we provide some preliminary results.\\

\textit{Notations.} Throughout, we write \(\norm{\cdot}\) for the operator norm on elements of \(\R^{r \times r}\) induced by the \(L^2\)-distance on \(\R^r\). In particular, for \(\boldsymbol{A} \in \R^{r \times r}\),
\[
\norm{\boldsymbol{A}} = \max_{\boldsymbol{x} \in \R^r\backslash \{0\}} \frac{\norm{\boldsymbol{A x}}_2}{\norm{\boldsymbol{x}}_2}.
\]
In terms of its spectrum, the operator norm of \(\boldsymbol{A} \in \R^{r \times r}\) corresponds to the largest singular value of \(\boldsymbol{A}\), i.e., \(\norm{\boldsymbol{A}} = s_{\max}(\boldsymbol{A}) = \sqrt{\lambda_{\max}(\boldsymbol{A}^\top \boldsymbol{A})}\). Moreover, let \(\norm{\boldsymbol{A}}_{\textnormal{F}}\) denote the Frobenius norm of \(\boldsymbol{A}\):
\[
\norm{\boldsymbol{A}}_{\textnormal{F}} = \left ( \sum_{1 \le i,j \le r} |\boldsymbol{A}_{ij}|^2\right)^{1/2}.
\]
For any symmetric matrices \(\boldsymbol{A}\) and \(\boldsymbol{B}\), we write \(\boldsymbol{A} \succeq \boldsymbol{B}\) if \(\boldsymbol{A} - \boldsymbol{B} \succeq 0\), meaning that \(\boldsymbol{A} - \boldsymbol{B}\) is positive semi-definite. \\
We assume that \(\boldsymbol{W} \in (\R^N)^{\otimes p}\) is an order-\(p\) tensor with i.i.d.\ centered sub-Gaussian entries \(W_{i_1, \ldots, i_p}\): 
\begin{equation} \label{eq: MGF sub-Gaussian}
\E \left [ e^{t W_{i_1,\ldots, i_p}}\right ] \leq e^{K t^2 \sigma^2}  \enspace \textnormal{for all} \enspace t \in \R \quad \textnormal{and} \quad \E \left [ |W_{i_1,\ldots, i_p}|^k \right ] \leq C k^{k/2}\sigma^k \enspace \textnormal{for all} \enspace k \geq 1,
\end{equation}
where \(K,C>0\) are absolute constants and \(\sigma >0\) is the parameter. Equivalently, \((W_{i_1, \ldots, i_p})_{i_1, \ldots, i_p}\) are i.i.d.\ random variables with sub-Gaussian norm \(\norm{W_{i_1, \ldots, i_p}}_{\psi_2} = \sigma\), where for any sub-Gaussian random variable \(Z\), the sub-Gaussian norm \(\norm{Z}_{\psi_2}\) is defined by \(\norm{Z}_{\psi_2} = \inf \{t > 0 \colon \E \left [ \exp (Z^2 / t^2) \right ] \le 2 \}\) (see e.g.~\cite[Definition 2.5.6]{vershynin2018high}). 

\subsection{Preliminary results} \label{subsection: preliminaries}

We first recall that the loss function \(\mathcal{L}_{N,r} \colon \textnormal{\textnormal{St}}(N,r) \times (\R^N)^{\otimes p} \to \R\) is given by
\[
\mathcal{L}_{N,r}(\boldsymbol{X}; \boldsymbol{Y}) = H_{N,r}(\boldsymbol{X}) + \Phi_{N,r}(\boldsymbol{X}),
\]
where \(H_{N,r}\) and \(\Phi_{N,r}\) denote the noise and deterministic parts, respectively, which are defined by~\eqref{eq: noise part} and~\eqref{eq: population loss}. To simplify notation slightly, in the remainder of the article we will write \(\mathcal{L}, H, \Phi\) in place of \(\mathcal{L}_{N,r}, H_{N,r}, \Phi_{N,r}\), respectively. The following lemma shows that the random variable \(\langle \boldsymbol{v}_i, (\nabla_{\textnormal{St}} H(\boldsymbol{X}))_j \rangle \) is sub-Gaussian.

\begin{lem} \label{lem: condition of noise SGD} 
For every \(\boldsymbol{v} \in \mathbb{S}^{N-1}\) and \(i \in [r]\), the random variable \(\langle \boldsymbol{v},( \nabla_{\textnormal{St}}H(\boldsymbol{X}))_i \rangle\) is sub-Gaussian with zero mean and sub-Gaussian norm satisfying 
\[
\norm{\langle \boldsymbol{v},( \nabla_{\textnormal{St}}H(\boldsymbol{X}))_i \rangle}_{\psi_2} \leq \alpha_1,
\]
for a constant \(\alpha_1 = \alpha_1(p, \sigma, \{\lambda_i\}_{i=1}^r) >0\).
\end{lem}
\begin{proof}
From~\eqref{eq: noise part} and~\eqref{eq: Stiefel gradient}, the \(i\)th column vector of \(\nabla_{\textnormal{St}} H(\boldsymbol{X})\) is given by
\[
(\nabla_{\textnormal{St}} H(\boldsymbol{X}))_i = - \lambda_i \nabla_{\boldsymbol{x}_i} \langle \boldsymbol{W}, \boldsymbol{x}_i^{\otimes p} \rangle + \frac{1}{2} \sum_{j=1}^r \left (\lambda_i \boldsymbol{x}_j^\top \nabla_{\boldsymbol{x}_i} \langle \boldsymbol{W}, \boldsymbol{x}_i^{\otimes p} \rangle + \lambda_j (\nabla_{\boldsymbol{x}_j} \langle \boldsymbol{W}, \boldsymbol{x}_j^{\otimes p} \rangle)^\top \boldsymbol{x}_i \right)\boldsymbol{x}_j,
\]
where \(\nabla_{\boldsymbol{x}_i} \langle \boldsymbol{W}, \boldsymbol{x}_i^{\otimes p} \rangle \in \R^N\) is given by 
\[
(\nabla_{\boldsymbol{x}_i} \langle \boldsymbol{W}, \boldsymbol{x}_i^{\otimes p} \rangle)_k = p \sum_{1 \leq k_1, \ldots, k_{p-1} \leq N} W_{k, k_1, \ldots, k_{p-1}} (x_i)_{k_1} \cdots (x_i)_{k_{p-1}}.
\]
Now, for every \(\boldsymbol{v} \in \mathbb{S}^{N-1}\) the inner product \(\langle \boldsymbol{v}, (\nabla_{\textnormal{St}} H(\boldsymbol{X}))_i \rangle\) can be expanded as
\[
\langle \boldsymbol{v}, (\nabla_{\textnormal{St}} H(\boldsymbol{X}))_i \rangle  = p \sum_{k=1}^N \sum_{1 \le k_1, \ldots, k_{p-1} \le N} W_{k, k_1,\ldots, k_{p-1}} f_i(\boldsymbol{v}, \boldsymbol{X}, \{\lambda_i\}_{i=1}^r),
\]
where 
\[
\begin{split}
f_i(\boldsymbol{v},\boldsymbol{X}, \{\lambda_i\}_{i=1}^r) &= - \lambda_i v_k (x_i)_{k_1} \cdots (x_i)_{k_{p-1}} \\
& \quad + \frac{1}{2} \sum_{j=1}^r \left ( \lambda_i (x_j)_k (x_i)_{k_1} \cdots (x_i)_{k_{p-1}} + \lambda_j (x_i)_k (x_j)_{k_1} \cdots (x_j)_{k_{p-1}} \right ) \langle \boldsymbol{v}, \boldsymbol{x}_j \rangle.
\end{split}
\]
This implies that the moment generating function of \(\langle \boldsymbol{v}, (\nabla_{\textnormal{St}} H(\boldsymbol{X}))_i \rangle\) is given by
\[
\begin{split}
\E \left [ e^{t \langle \boldsymbol{v}, (\nabla_{\textnormal{St}} H(\boldsymbol{X}))_i \rangle}\right] &= \prod_{k, k_1, \ldots, k_{p-1}}  \E \left [ e^{t p W_{k, k_1,\ldots, k_{p-1}} f_i(\boldsymbol{v},\boldsymbol{X}, \{\lambda_i\}_{i=1}^r)}\right] \\
&\leq e^{C t^2 p^2 \sigma^2 \sum_{1 \le k, k_1, \ldots, k_{p-1} \le N} f_i(\boldsymbol{v},\boldsymbol{X}, \{\lambda_i\}_{i=1}^r)^2},
\end{split}
\]
where we used the independence of the random variables \(W_{i_1,\ldots, i_p}\) and the sub-Gaussian property~\eqref{eq: MGF sub-Gaussian}. Since the bound on the moment generating function characterizes the sub-Gaussian distribution (see e.g.~\cite[Proposition 2.5.2]{vershynin2018high}), it follows that the random variable \(\langle \boldsymbol{v}, (\nabla_{\textnormal{St}} H(\boldsymbol{X}))_i \rangle\) is sub-Gaussian with zero mean and sub-Gaussian norm bounded above by
\[
\begin{split}
\norm{\langle \boldsymbol{v}, (\nabla_{\textnormal{St}} H(\boldsymbol{X}))_i \rangle}_{\psi_2}^2 &\leq p^2 \sigma^2 \sum_{1 \le k, k_1, \ldots, k_{p-1} \le N}  f_i (\boldsymbol{v},\boldsymbol{X}, \{\lambda_i\}_{i=1}^r)^2 \\
&= p^2 \sigma^2 \left (\frac{1}{4} \sum_{j=1}^r \lambda_j^2 \langle \boldsymbol{v},\boldsymbol{x}_j \rangle^2 - \frac{1}{2} \lambda_i^2 \langle \boldsymbol{v},\boldsymbol{x}_i \rangle^2 - \frac{3}{4} \lambda_i^2 \sum_{j=1}^r \langle \boldsymbol{v},\boldsymbol{x}_j \rangle^2\right ) \\
& \leq \frac{1}{4} p^2 \sigma^2 \sum_{i=1}^r \lambda_j^2,
\end{split}
\]
where the equality follows by the fact that \(\boldsymbol{v} \in \mathbb{S}^{N-1}\) and \(\boldsymbol{X} = [\boldsymbol{x}_1, \ldots, \boldsymbol{x}_r] \in \textnormal{St}(N,r)\). Therefore, \(\norm{\langle \boldsymbol{v}, (\nabla_{\textnormal{St}} H(\boldsymbol{X}))_i \rangle}_{\psi_2} \le \frac{p \sigma}{2} \sqrt{\sum_{i=1}^r \lambda_i^2} \), which completes the proof.
\end{proof}

According to Subsection~\ref{subsection: SGD}, the output \(\boldsymbol{X}_t\) of the online SGD algorithm can be expanded as 
\begin{equation} \label{eq: SGD expansion}
\boldsymbol{X}_t = \left ( \boldsymbol{X}_{t-1} - \frac{\delta}{N}\nabla_{\textnormal{St}}\mathcal{L}(\boldsymbol{X}_{t-1};\boldsymbol{Y}^t) \right ) \boldsymbol{P}_t.
\end{equation}
Here, \(\boldsymbol{P}_t\) denotes the projection term given by
\begin{equation} \label{eq: P_t}
\boldsymbol{P}_t = \left( \boldsymbol{I}_r  + \frac{\delta^2}{N^2} \boldsymbol{\mathcal{G}}(\boldsymbol{X}_{t-1};\boldsymbol{Y}^t) \right)^{-1/2},
\end{equation}
where \(\boldsymbol{\mathcal{G}}(\boldsymbol{X}_{t-1};\boldsymbol{Y}^t) \in \R^{r \times r}\) is the Gram matrix given by
\begin{equation} \label{eq: Gram matrix}
\boldsymbol{\mathcal{G}}(\boldsymbol{X}_{t-1};\boldsymbol{Y}^t) = \nabla_{\textnormal{St}}\mathcal{L}(\boldsymbol{X}_{t-1};\boldsymbol{Y}^t)^\top \nabla_{\textnormal{St}}\mathcal{L}(\boldsymbol{X}_{t-1};\boldsymbol{Y}^t).
\end{equation}
We have the following estimate for the Frobenius norm of \(\boldsymbol{\mathcal{G}}(\boldsymbol{X} ;\boldsymbol{Y})\).
\begin{lem} \label{lem: bound op norm of Gram matrix}
There exists a constant \(\alpha_2 = \alpha_2(p, \sigma, \{\lambda_i\}_{i=1}^r)\) such that \(\sup_{\boldsymbol{X} \in \textnormal{\textnormal{St}}(N,r)}  \E \left [ \norm{\boldsymbol{\mathcal{G}}(\boldsymbol{X};\boldsymbol{Y})}_{\textnormal{F}} \right ] \leq \alpha_2 N\). In particular, if \(\delta >0\) satisfies \(\delta < \left( \frac{N}{\alpha_2}\right)^{1/2}\), 
\[
\mathbb{P} \left( \frac{\delta^2}{N^2} \norm{\boldsymbol{\mathcal{G}}(\boldsymbol{X} ; \boldsymbol{Y})}_{\textnormal{F}} <1 \right) \geq 1-\exp(-c(p)N).
\]
\end{lem}
\begin{proof}
We first notice that
\[
\norm{\boldsymbol{G}(\boldsymbol{X};\boldsymbol{Y})}_{\textnormal{F}} =  \norm{\nabla_{\textnormal{St}} \mathcal{L} (\boldsymbol{X};\boldsymbol{Y})^{\top} \nabla_{\textnormal{St}} \mathcal{L}(\boldsymbol{X};\boldsymbol{Y})}_{\textnormal{F}} \leq \norm{\nabla_{\textnormal{St}} \mathcal{L}(\boldsymbol{X};\boldsymbol{Y})}_{\textnormal{F}}^2.
\]
Since by definition \(\nabla_{\textnormal{St}} \mathcal{L}(\boldsymbol{X};\boldsymbol{Y}) = \nabla \mathcal{L}(\boldsymbol{X};\boldsymbol{Y}) - \frac{1}{2} \boldsymbol{X} (\nabla \mathcal{L}(\boldsymbol{X};\boldsymbol{Y})^\top \boldsymbol{X} + \boldsymbol{X}^\top \nabla \mathcal{L}(\boldsymbol{X};\boldsymbol{Y}))\), it then follows that
\[
\norm{\nabla_{\textnormal{St}}\mathcal{L} (\boldsymbol{X};\boldsymbol{Y})}_{\textnormal{F}} \leq  \norm{\nabla \mathcal{L}(\boldsymbol{X};\boldsymbol{Y})}_{\textnormal{F}} + \norm{\boldsymbol{X}^\top \nabla \mathcal{L}(\boldsymbol{X};\boldsymbol{Y})}_{\textnormal{F}} \leq 2  \norm{\nabla \mathcal{L}(\boldsymbol{X};\boldsymbol{Y})}_{\textnormal{F}},
\]
where for the first inequality we used the triangle inequality and the fact that \(\norm{\boldsymbol{X} \nabla \mathcal{L}(\boldsymbol{X};\boldsymbol{Y})^\top \boldsymbol{X}}_{\textnormal{F}} = \norm{\boldsymbol{X}^\top \nabla \mathcal{L}(\boldsymbol{X};\boldsymbol{Y})}_{\textnormal{F}}\) and similarly \(\norm{\boldsymbol{X} \boldsymbol{X}^\top \nabla \mathcal{L}(\boldsymbol{X};\boldsymbol{Y})}_{\textnormal{F}} = \norm{\boldsymbol{X}^\top \nabla \mathcal{L}(\boldsymbol{X};\boldsymbol{Y})}_{\textnormal{F}}\), while the second inequality follows by the property \(\norm{\boldsymbol{X}^\top \boldsymbol{A}}_{\textnormal{F}} \leq \norm{\boldsymbol{A}}_{\textnormal{F}}\) for \(\boldsymbol{X} \in \textnormal{St}(N,r)\) and \(\boldsymbol{A} \in \R^{N \times r}\). We can write the Frobenius norm of the Euclidean gradient of \(\mathcal{L}\) as
\[
\norm{\nabla \mathcal{L}(\boldsymbol{X};\boldsymbol{Y})}_{\textnormal{F}}^2 = \sum_{i=1}^r \norm{\nabla_{\boldsymbol{x}_i} \mathcal{L}(\boldsymbol{X} ; \boldsymbol{Y})}_2^2, 
\]
where \(\nabla_{\boldsymbol{x}_i}\mathcal{L}(\boldsymbol{X} ; \boldsymbol{Y})\) denotes the \(i\)th column of the Euclidean gradient \(\nabla \mathcal{L}(\boldsymbol{X};\boldsymbol{Y})\) given by 
\[
\nabla_{\boldsymbol{x}_i} \mathcal{L}(\boldsymbol{X} ; \boldsymbol{Y}) = \lambda_i \nabla_{\boldsymbol{x}_i} \langle \boldsymbol{W}, \boldsymbol{x}_i^{\otimes p} \rangle -  \sum_{j=1}^r \sqrt{N} p \lambda_i \lambda_j \langle \boldsymbol{v}_j,\boldsymbol{x}_i \rangle^{p-1} \boldsymbol{v}_j.
\]
Therefore, taking expectations we obtain
\[
\E \left [ \norm{ \nabla \mathcal{L} (\boldsymbol{X}; \boldsymbol{Y})}_{\textnormal{F}}^2 \right ] \leq 2 \sum_{i=1}^r \E \left [ \lambda_i^2 \norm{\nabla_{\boldsymbol{x}_i} \langle \boldsymbol{W}, \boldsymbol{x}_i^{\otimes p} \rangle}_2^2 + \norm{\sum_{j=1}^r \sqrt{N} p \lambda_i \lambda_j  \langle \boldsymbol{v}_j, \boldsymbol{x}_i \rangle^{p-1} \boldsymbol{v}_j}_2^2 \right ]  .
\]
Using \(\norm{a+b}_2^2 \le 2 (\norm{a}_2^2 + \norm{b}_2^2)\), the orthogonality of the \(\boldsymbol{v}_i\)'s, and the bound \(\max_{i,j} |\langle \boldsymbol{v}_j, \boldsymbol{x}_i \rangle | \leq 1\), we further obtain
\[
\E \left [ \norm{ \nabla \mathcal{L} (\boldsymbol{X}; \boldsymbol{Y})}_{\textnormal{F}}^2 \right ] \leq 2 \sum_{i=1}^r \left(p^2 \lambda_i^2 \E \left [ \norm{\boldsymbol{W} \{\boldsymbol{x}_i \}}_2^2 \right ] + N \sum_{j=1}^r p^2 \lambda_i^2 \lambda_j^2 \right),
\]
where \(\boldsymbol{W} \{\boldsymbol{x}_i \}\) denotes the contraction of \(\boldsymbol{W}\) with \(\boldsymbol{x}_i^{\otimes (p-1)}\), where each coordinate \(k \in [N]\) is given by \((\boldsymbol{W} \{\boldsymbol{x}_i\})_k = \sum_{1 \le k_1, \ldots, k_{p-1} \le N} W_{k, k_1, \ldots, k_{p-1}} (x_i)_{k_1} \cdots  (x_i)_{k_{p-1}} \). A direct computation shows that \(\E \left [ \norm{\boldsymbol{W} \{\boldsymbol{x}_i \}}_2^2 \right ]\) gives the bound \(\E \left [ \norm{\boldsymbol{W} \{\boldsymbol{x}_i \}}_2^2 \right ] \leq 2K\sigma^2N\). This shows the first part of the statement. For the second part, we note that 
\[
\norm{\boldsymbol{G}(\boldsymbol{X}; \boldsymbol{Y})}_{\textnormal{F}} \leq 8 \sum_{i=1}^r \left(p^2 \lambda_i^2 \norm{\boldsymbol{W} \{\boldsymbol{x}_i \}}_2^2 + N \sum_{j=1}^r p^2 \lambda_i^2 \lambda_j^2 \right) \leq 8 \sum_{i=1}^r \left(p^2 \lambda_i^2 \norm{\boldsymbol{W}}^2 + N \sum_{j=1}^r p^2 \lambda_i^2 \lambda_j^2 \right).
\]
From Lemma 4.7 of~\cite{ben2020bounding} we then have for every \(\gamma >0\) there exists an absolute constant \(c>0\) such that
\[
\mathbb{P} \left(\norm{\boldsymbol{W}} \geq \gamma \sqrt{N}\right) \leq 2 (8p)^{Np} e^{-c\gamma^2N/8}.
\]
Choosing \(\gamma > \sqrt{2K}\sigma\), we then obtain there is \(c(p) >0\) such that 
\[
\norm{\boldsymbol{G}(\boldsymbol{X} ; \boldsymbol{Y})}_{\textnormal{F}} \leq 4 \norm{ \nabla \mathcal{L} (\boldsymbol{X}; \boldsymbol{Y})}_{\textnormal{F}}^2 \le 8 p^2 N \sum_{i=1}^r \lambda_i^2 \left(  2K\sigma^2  +  \sum_{j=1}^r \lambda_j^2 \right) = N \alpha_2(p, \sigma, \{\lambda_i\}_{i=1}^r). 
\]
with \(\mathbb{P}\)-probability at least \(1-e^{-c(p)N}\). This completes the proof of Lemma~\ref{lem: bound op norm of Gram matrix}. 
\end{proof}

%%%%%%%%%%%%%%%%%%%%%%%%%%%%%%%%%%%%%%%%%%%%%%%%%%%%%%%%%%%%%%%%%%%%%%%%%%%%%%%%%%%%%%%%%%%%%%%%%%%%%%%%%%%%%%%%%%%%%%%%%%%%%%%%%%%%%%%%%%%%%%%%%%%%%%%%%%%%%%%%%%%%%%%%%%%%%%%%%%%%%%%%%%
\subsection{Comparison inequalities for correlations} \label{subsection: comparison inequality correlations}

The goal of this subsection is to obtain a two-sided difference inequality for the evolution of the correlations \(m_{ij}(t) = m_{ij}(\boldsymbol{X}_t)\). This will indeed be needed to prove both Theorems~\ref{thm: strong recovery online p>2 nonasymptotic} and~\ref{thm: strong recovery online p=2 nonasymptotic}. 

\begin{prop} \label{prop: inequalities SGD}
Let \(\alpha_1 = \alpha_1(p, \sigma, \{\lambda_i\}_{i=1}^r)\) and \(\alpha_2 = \alpha_2(p, \sigma, \{\lambda_i\}_{i=1}^r)\) denote the constants given by Lemmas~\ref{lem: condition of noise SGD} and~\ref{lem: bound op norm of Gram matrix}, respectively. Consider the sequence of outputs \((\boldsymbol{X}_t)_{t \in \N}\) given by~\eqref{eq: online SGD} with step size \(\delta >0\) satisfying \(\delta < (N/\alpha_2)^{1/2}\), and let \(T > 0\) denote a fixed time horizon. For a fixed constant \(c_0 \in (0,1)\) and a truncation sequence \((K_N)_{N \ge 1}\) with \(K_N >0\), define for every \(i,j \in [r]\) the comparison functions \(\ell_{ij}\) and \(u_{ij}\) by
\[
\begin{split}
\ell_{ij}(t) &= (1 - c_0) m_{ij}(0) + \frac{\delta}{N} \sum_{\ell=1}^t  \left ( - \langle \boldsymbol{v}_i, \left (\nabla_{\textnormal{St}}\Phi(\boldsymbol{X}_{\ell-1}) \right )_j \rangle - h_i(\ell-1) \right ) , \\
u_{ij}(t) & = (1 + c_0) m_{ij}(0) + \frac{\delta}{N} \sum_{\ell=1}^t  \left ( - \langle \boldsymbol{v}_i, \left (\nabla_{\textnormal{St}}\Phi(\boldsymbol{X}_{\ell-1}) \right )_j \rangle + h_i(\ell-1)  \right ) ,
\end{split}
\]
where 
\[
h_i(\ell-1) = \frac{\delta \alpha_2}{2} \sum_{k=1}^r \left( \vert m_{ik}(\ell-1) \vert + \frac{\delta}{N}\vert \langle \boldsymbol{v}_i, (\nabla_{\textnormal{St}}\Phi \left(\boldsymbol{X}_{\ell-1}\right))_k \rangle \vert +\frac{K_N \delta}{N} \right),
\]
and
\[
\langle \boldsymbol{v}_i ,(\nabla_{\textnormal{St}}\Phi(\boldsymbol{X}))_j \rangle =  -\sqrt{N}p \lambda_i \lambda_j m_{ij}^{p-1} + \sqrt{N}\frac{p}{2}  \sum_{1 \le k, \ell \le r} \lambda_k m_{i \ell} m_{kj} m_{k\ell} \left ( \lambda_j m_{k j}^{p-2} + \lambda_\ell m_{k \ell}^{p-2} \right).
\] 
Then, for every \(\gamma_1 > \gamma_2 > 0\) and every \(c_0 \in (0,1)\), there exist \(c_1, c_2, c_3, C>0\), and a truncation sequence \((K_N)_{N \ge 1}\) such that for every \(i,j \in [r]\) and every \(t \leq T\), 
\[
\inf_{\boldsymbol{X}_0 \in \mathcal{C}_1 (\gamma_1,\gamma_2)} \mathbb{P}_{\boldsymbol{X}_0^+} \left(\ell_{ij}(t) \leq m_{ij}(t) \leq u_{ij}(t) \right) \ge 1 - \eta (\delta, c_0, T, K_N) ,
\]
where \(\eta = \eta (\delta, c_0, T, K_N)\) is given by
\begin{equation} \label{eq: eta}
\eta = T \exp(-c_1 N) + 2 \exp \left(- c_2 \frac{c_0^2 \gamma_2^2 N}{8\alpha_1^2 \delta^2T} \right) + C \frac{2  r \alpha_1 \alpha_2 \delta^3 T}{c_0 \gamma_2 N^{3/2}}  \exp\left(- c_3 \frac{K_N^2}{2 \alpha_1^2}\right).
\end{equation}
\end{prop}

The proof of Proposition~\ref{prop: inequalities SGD} will be the main focus of this subsection. Our strategy involves first expanding the polar retraction term \(R_{\boldsymbol{X}_{t-1}}\) from the algorithm (see~\eqref{eq: polar retraction}) as a Neumann series and then decompose the evolution of \(m_{ij}(t)\) in three parts: the drift induced by the population loss, a martingale induced by the gradient of the noise Hamiltonian \(H\), and high-order terms in \(\delta\). By controlling the latter two parts, the proposition will follow.

\begin{lem} \label{lem: expansion correlations}
Let \(\alpha_2 = \alpha_2(p, \sigma, \{\lambda_i\}_{i=1}^r)\) be the constant of Lemma~\ref{lem: bound op norm of Gram matrix}.  Consider the sequence of outputs \((\boldsymbol{X}_t)_{t \in \N}\) given by~\eqref{eq: online SGD} with step size \(\delta >0\) satisfying \(\delta < (N/\alpha_2)^{1/2}\). For every \(T >0\), let \(\mathcal{E}_{\boldsymbol{\mathcal{G}}}= \mathcal{E}_{\boldsymbol{\mathcal{G}}}(T)\) denote the event  
\begin{equation} \label{eq: event E_G}
\mathcal{E}_{\boldsymbol{\mathcal{G}}}(T) = \bigcap_{1 \leq t \leq T} \left\{ \norm{\frac{\delta^2}{N^2} \boldsymbol{\mathcal{G}}(\boldsymbol{X}_{t-1},\boldsymbol{Y}^t )}_{\textnormal{F}} < 1\right\}.
\end{equation}
Then, there exists a constant \(c(p)\) such that \(\mathbb{P}_{\boldsymbol{X}_0 }\left( \mathcal{E}_{\boldsymbol{\mathcal{G}}}\right) \geq 1 - T \exp(-c(p) N)\). Moreover, on the event \(\mathcal{E}_{\boldsymbol{\mathcal{G}}}\), for every \(i,j \in [r]\) and every \(t \leq T\), it holds that
\begin{equation} \label{eq: comparison ineq for m_ij}
\left | m_{ij}(t) - \left ( m_{ij}(0) - \frac{\delta}{N} \sum_{\ell=1}^t  \langle \boldsymbol{v}_i,(\nabla_{\textnormal{St}}\Phi(\boldsymbol{X}_{\ell-1}))_j \rangle - \frac{\delta}{N} \sum_{\ell=1}^t  \langle \boldsymbol{v}_i,(\nabla_{\textnormal{St}}H^\ell(\boldsymbol{X}_{\ell-1}))_j \rangle \right ) \right | \leq \frac{\delta}{N} \sum_{\ell=1}^t  \vert a_i(\ell-1) \vert,
\end{equation}
where 
\begin{equation} \label{eq: bound A_l}
\begin{split}
\vert a_i(\ell-1) \vert & \leq \frac{\delta \alpha_2}{2} \sum_{k=1}^r \left( \vert m_{ik}(\ell-1) \vert + \frac{\delta}{N}\vert \langle \boldsymbol{v}_i, (\nabla_{\textnormal{St}}\Phi \left(\boldsymbol{X}_{\ell-1} \right))_k \rangle \vert + \frac{\delta}{N} \vert \langle \boldsymbol{v}_i, (\nabla_{\textnormal{St}}H^\ell \left( \boldsymbol{X}_{\ell-1})_k \right) \rangle \vert \right).
\end{split}
\end{equation}
\end{lem}

\begin{proof}
We begin by expanding the projection term \(\boldsymbol{P}_t\) in~\eqref{eq: SGD expansion} as a Neumann series. To do so, we require that \(\norm{\frac{\delta^2}{N^2} \boldsymbol{\mathcal{G}}(\boldsymbol{X}_{t-1},\boldsymbol{Y}^t )}_{\textnormal{F}} < 1 \) for every \(t \le T\). According to Lemma~\ref{lem: bound op norm of Gram matrix}, this condition holds if \(\delta\) satisfies \(\delta< \left ( \frac{N}{\alpha_2} \right)^{1/2}\), leading to
\[ 
\mathbb{P}_{\boldsymbol{X}_0} \left(\mathcal{E}_{\boldsymbol{\mathcal{G}}}^{\textnormal{c}}\right) \leq T \, \mathbb{P} \left(  \norm{\frac{\delta^2}{N^2} \boldsymbol{\mathcal{G}}(\boldsymbol{X}_{t-1},\boldsymbol{Y}^t )}_{\textnormal{F}} \ge 1\right) \leq T e^{-c(p)N}.
\]
We therefore place ourselves on the event \(\mathcal{E}_{\boldsymbol{\mathcal{G}}}\) for the rest of the proof. Expanding \(\boldsymbol{P}_t\) as a Neumann series~\cite{horn2012matrix} gives
\[
\boldsymbol{P}_t = \sum_{k=0}^\infty (-1)^k {\frac{1}{2} \choose k} \frac{\delta^{2k}}{N^{2k}} \boldsymbol{\mathcal{G}}(\boldsymbol{X}_{t-1}; \boldsymbol{Y}^t)^k = \boldsymbol{I}_r + \sum_{k=1}^\infty (-1)^k {\frac{1}{2} \choose k} \frac{\delta^{2k}}{N^{2k}} \boldsymbol{\mathcal{G}}(\boldsymbol{X}_{t-1}; \boldsymbol{Y}^t)^k,
\]
where, for any nonnegative integer \(k\), the generalized binomial coefficient is defined by 
\[
{\frac{1}{2} \choose 0}=1, \quad {\frac{1}{2} \choose k} = {2k \choose k} \frac{(-1)^{k+1}}{2^{2k} (2k-1)} \enspace \textnormal{for} \enspace k \ge 1. 
\]
Plugging this expansion into~\eqref{eq: SGD expansion}, the online SGD output \(\boldsymbol{X}_t\) at time \(t\) can rewritten as
\[
\begin{split}
\boldsymbol{X}_t &=\boldsymbol{X}_{t-1} - \frac{\delta}{N} \nabla_{\textnormal{St}}\mathcal{L}(\boldsymbol{X}_{t-1};\boldsymbol{Y}^t) \\
& \quad + \sum_{k=1}^\infty (-1)^k {\frac{1}{2} \choose k} \frac{\delta^{2k}}{N^{2k}} \left ( \boldsymbol{X}_{t-1} - \frac{\delta}{N}\nabla_{\textnormal{St}}\mathcal{L}(\boldsymbol{X}_{t-1};\boldsymbol{Y}^t) \right ) \boldsymbol{\mathcal{G}}(\boldsymbol{X}_{t-1};\boldsymbol{Y}^t)^k.
\end{split}
\]
For every \(i,j \in [r]\) and every \(1 \leq t \leq T\), the evolution of the correlation \(m_{ij}(\boldsymbol{X}_t)=m_{ij}(t)\) under~\eqref{eq: online SGD} is therefore given by
\[
\begin{split}
m_{ij}(t) & = \langle \boldsymbol{v}_i, (\boldsymbol{X}_t)_j \rangle   \\
& = m_{ij}(t-1) - \frac{\delta}{N} \langle \boldsymbol{v}_i, \left (\nabla_{\textnormal{St}}\mathcal{L} (\boldsymbol{X}_{t-1}; \boldsymbol{Y}^t) \right )_j \rangle \\
& \quad + \sum_{k=1}^\infty (-1)^k {\frac{1}{2} \choose k} \frac{\delta^{2k}}{N^{2k}} \Tr \left ( [\boldsymbol{v}_i]_j^\top \left ( \boldsymbol{X}_{t-1}  - \frac{\delta}{N}\nabla_{\textnormal{St}}\mathcal{L} (\boldsymbol{X}_{t-1};\boldsymbol{Y}^t) \right ) \boldsymbol{\mathcal{G}}(\boldsymbol{X}_{t-1}; \boldsymbol{Y}^t )^k\right ),
\end{split}
\]
where \((\boldsymbol{X}_t)_j\) denotes the \(j\)th column of \(\boldsymbol{X}_t\), \(\left (\nabla_{\textnormal{St}}\mathcal{L}(\boldsymbol{X}_{t-1}; \boldsymbol{Y}_t) \right )_j \) the \(j\)th column of \(\nabla_{\textnormal{St}}\mathcal{L}(\boldsymbol{X}_{t-1}; \boldsymbol{Y}_t)\), and \([\boldsymbol{v}_i]_j = [\boldsymbol{0}, \ldots, \boldsymbol{v}_i, \ldots, \boldsymbol{0}] \in \R^{N \times r}\) the matrix in which all columns are the zero vector except column \(j\) which is given by \(\boldsymbol{v}_i\). Iterating the above expression, we then obtain
\[
\begin{split} 
m_{ij}(t) & = m_{ij}(0) - \frac{\delta}{N} \sum_{\ell=1}^t \langle \boldsymbol{v}_i, \left (\nabla_{\textnormal{St}}\mathcal{L} (\boldsymbol{X}_{\ell-1}; \boldsymbol{Y}^\ell) \right )_j \rangle  \\
& \quad + \sum_{\ell=1}^t \sum_{k=1}^\infty (-1)^k {\frac{1}{2} \choose k} \frac{\delta^{2k}}{N^{2k}} \Tr \left ( [\boldsymbol{v}_i]_j^\top \left ( \boldsymbol{X}_{\ell-1}  - \frac{\delta}{N}\nabla_{\textnormal{St}}\mathcal{L} (\boldsymbol{X}_{\ell-1};\boldsymbol{Y}^\ell) \right ) \boldsymbol{\mathcal{G}}(\boldsymbol{X}_{\ell-1}; \boldsymbol{Y}^\ell )^k\right ).
\end{split}
\]
In the following, for every \(i,j \in [r]\) and every \(1 \leq \ell \leq t\) let \(a_{ij}(\ell-1)\) denote
\[
a_{ij}(\ell-1) = \sum_{k=1}^\infty (-1)^k {\frac{1}{2} \choose k} \frac{\delta^{2k}}{N^{2k}} \Tr \left ( [\boldsymbol{v}_i]_j^\top \left ( \boldsymbol{X}_{\ell-1}  - \frac{\delta}{N}\nabla_{\textnormal{St}}\mathcal{L}(\boldsymbol{X}_{\ell-1};\boldsymbol{Y}^\ell) \right ) \boldsymbol{\mathcal{G}}(\boldsymbol{X}_{\ell-1}; \boldsymbol{Y}^\ell )^k\right ).
\]
By Cauchy-Schwarz we then have
\begin{equation} \label{eq: bound 1}
\begin{split}
\vert a_{ij}(\ell-1)  \vert & \leq \sum_{k=1}^{\infty}\binom{\frac{1}{2}}{k}\frac{\delta^{2k}}{N^{2k}} \left | \Tr \left ( [\boldsymbol{v}_i]_j^\top \left ( \boldsymbol{X}_{\ell-1}  - \frac{\delta}{N} \nabla_{\textnormal{St}}\mathcal{L} (\boldsymbol{X}_{\ell-1};\boldsymbol{Y}^\ell) \right ) \boldsymbol{\mathcal{G}}(\boldsymbol{X}_{\ell-1}; \boldsymbol{Y}^\ell )^k\right ) \right |  \\
&\leq \sum_{k=1}^{\infty}\binom{\frac{1}{2}}{k} \norm{\frac{\delta^{2k}}{N^{2k}} \boldsymbol{\mathcal{G}} ( \boldsymbol{X}_{\ell-1}; \boldsymbol{Y}^\ell)^{k}}_{\textnormal{F}} \norm{[\boldsymbol{v}_i ]_j^\top  \left(\boldsymbol{X}_{\ell - 1}-\frac{\delta}{N} \nabla_{\textnormal{St}}\mathcal{L} (\boldsymbol{X}_{\ell-1}; \boldsymbol{Y}^\ell) \right)}_{\textnormal{F}}\\
& \leq \sum_{k=1}^{\infty}\binom{\frac{1}{2}}{k} C_{\boldsymbol{\mathcal{G}}}^{k-1} \norm{\frac{\delta^2}{N^2} \boldsymbol{\mathcal{G}} ( \boldsymbol{X}_{\ell-1}; \boldsymbol{Y}^\ell)}_{\textnormal{F}} \norm{[\boldsymbol{v}_i ]_j^\top  \left(\boldsymbol{X}_{\ell - 1}-\frac{\delta}{N} \nabla_{\textnormal{St}}\mathcal{L}(\boldsymbol{X}_{\ell-1}; \boldsymbol{Y}^\ell) \right)}_{\textnormal{F}},
\end{split}
\end{equation}
where \(C_{\boldsymbol{\mathcal{G}}} = \max \norm{\frac{\delta^2}{N^2} \boldsymbol{\mathcal{G}}(\boldsymbol{X}; \boldsymbol{Y})} \in (0,1)\) under the event \(\mathcal{E}_{\boldsymbol{\mathcal{G}}}\) for \(\delta < (N/\alpha_2)^{1/2}\). From the generalized binomial series (Newton’s binomial theorem for non-integer exponents), we find
\begin{equation} \label{eq: bound 2}
\sum_{k=1}^{\infty}\binom{\frac{1}{2}}{k} C_{\boldsymbol{\mathcal{G}}}^{k-1}  = \frac{1}{C_{\boldsymbol{\mathcal{G}}}} \left(\sqrt{1 + C_{\boldsymbol{\mathcal{G}}}} -1\right) \leq \frac{1}{C_{\boldsymbol{\mathcal{G}}}}\frac{C_{\boldsymbol{\mathcal{G}}}}{2} = \frac{1}{2}.
\end{equation}
Moreover, a simple computation shows that
\begin{equation} \label{eq: bound 3}
\begin{split}
\left \lVert [\boldsymbol{v}_i ]_j^\top  \left(\boldsymbol{X}_{\ell - 1}-\frac{\delta}{N} \nabla_{\textnormal{St}}\mathcal{L}(\boldsymbol{X}_{\ell-1}; \boldsymbol{Y}^\ell) \right) \right \rVert_{\textnormal{F}} & = \sqrt{\sum_{k=1}^r \left (m_{ik}(\ell-1) - \frac{\delta}{N} \langle \boldsymbol{v}_i, (\nabla_{\textnormal{St}}\mathcal{L}(\boldsymbol{X}_{\ell-1};\boldsymbol{Y}^\ell))_k \rangle \right)^2}\\
& \leq \sum_{k=1}^r \left(m_{ik}(\ell-1) - \frac{\delta}{N} \langle \boldsymbol{v}_i, (\nabla_{\textnormal{St}}\mathcal{L}(\boldsymbol{X}_{\ell-1};\boldsymbol{Y}^\ell))_k \rangle \right).
\end{split}
\end{equation}
Combining~\eqref{eq: bound 1}-\eqref{eq: bound 3} yields
\[
|a_{ij}(\ell-1)| \leq \frac{\delta^2}{2N^2} \norm{ \boldsymbol{\mathcal{G}} ( \boldsymbol{X}_{\ell-1}; \boldsymbol{Y}^\ell)}_{\textnormal{F}} \sum_{k=1}^r \left (m_{ik}(\ell-1) - \frac{\delta}{N} \langle \boldsymbol{v}_i, (\nabla_{\textnormal{St}}\mathcal{L}(\boldsymbol{X}_{\ell-1};\boldsymbol{Y}^\ell))_k \rangle \right),
\]
which gives the desired bound since \(\norm{ \boldsymbol{\mathcal{G}} ( \boldsymbol{X}_{\ell-1}; \boldsymbol{Y}^\ell)}_{\textnormal{F}} \leq \alpha_2 N\) according to Lemma~\ref{lem: bound op norm of Gram matrix}.
\end{proof}

According to~\eqref{eq: comparison ineq for m_ij} from Lemma~\ref{lem: expansion correlations}, we now provide an estimate for the first-order (in \(\delta\)) noise Hamiltonian \(H\), i.e., \(\frac{\delta}{N} \sum_{\ell=1}^t \langle \boldsymbol{v}_i, \left (\nabla_{\textnormal{St}}H^\ell (\boldsymbol{X}_{\ell-1}) \right )_j  \rangle\).

\begin{lem} \label{lem: noise martingale}
Let \(\alpha_1 = \alpha_1(p, \sigma, \{\lambda_i\}_{i=1}^r)\) be the constant from Lemma~\ref{lem: condition of noise SGD}. Then, for every \(a > 0\), there is an absolute constant \(c>0\) such that
\[
\sup_{\boldsymbol{X}_0 \in \, \textnormal{\textnormal{St}}(N,r)} \mathbb{P}_{\boldsymbol{X}_0} \left (\max_{t \leq T} \frac{\delta}{N}  \left |  \sum_{\ell=1}^t \langle \boldsymbol{v}_i, \left (\nabla_{\textnormal{St}}H^\ell(\boldsymbol{X}_{\ell-1}) \right )_j \rangle\right | \geq a \right ) \leq 2\exp\left(- c\frac{a^2 N^2}{2 \delta^2 \alpha_1^2  T}\right).
\]
\end{lem}
\begin{proof}
Let \(\mathcal{F}_t\) denote the natural filtration associated to \(\boldsymbol{X}_1, \ldots, \boldsymbol{X}_t\), namely \(\mathcal{F}_t = \sigma (\boldsymbol{X}_\ell, \ell \leq t)\). Moreover, let \(M_t\) denote the process
\[
M_t = \sum_{\ell=1}^t \langle \boldsymbol{v}_i, \left (\nabla_{\textnormal{St}}H^\ell(\boldsymbol{X}_{\ell-1}) \right )_j \rangle.
\]
Since 
\[
\E \left [ M_t - M_{t-1} | \mathcal{F}_{t-1} \right ] = \E \left [ \langle \boldsymbol{v}_i, \left (\nabla_{\textnormal{St}}H^t(\boldsymbol{X}_{t-1}) \right )_j | \mathcal{F}_{t-1}  \right ]=0,
\]
we have the process \(M_t\) is adapted to the filtration \((\mathcal{F}_t)_{t \ge 0}\). Moreover, since \(M_s = \E \left [ M_t | \mathcal{F}_s \right ]\), we have \((M_t)_{t \ge 0}\) is a martingale with respect to \((\mathcal{F}_t)_{t \ge 0}\). Since by Lemma~\ref{lem: condition of noise SGD} the distribution of the increments \(M_t - M_{t-1} = \langle \boldsymbol{v}_i ,( \nabla_{\textnormal{St}}H^t (\boldsymbol{X}_{t-1}))_j \rangle\) given \(\mathcal{F}_{t-1}\) is sub-Gaussian with sub-Gaussian norm bounded by \(\alpha_1\), we have for every \(b >0\), there exists an absolute constants \(C>0\) such that
\[
\E \left[ e^{b (M_t - M_{t-1})} \vert \mathcal{F}_{t-1}\right] \leq \exp \left (C b^2 \alpha_1^2 \right ).
\]
The corresponding Chernoff bound applied to \(M_t  = \sum_{\ell=1}^t (M_\ell-M_{\ell-1})\) (see e.g.\ Theorem 2.19 of~\cite{wainwright2019high}) gives the tail bound: for every \(t \leq T\) and every \(a >0\),
\[
\sup_{\boldsymbol{X}_0 \in \textnormal{\textnormal{St}}(N,r)} \mathbb{P}_{\boldsymbol{X}_0} \left (\frac{\delta}{N}  \left |  M_t \right | \geq a \right ) \leq 2 \exp \left(- c \frac{a^2 N^2}{2 \delta^2 t  \alpha_1^2} \right).
\]
To improve this bound to the maximum \(\max_t \vert M_t  \vert\), we follow the remark at the end of section 3.5 from~\cite{mcdiarmid1998concentration}, noting that the random process \(e^{M_t}\) is a submartingale. The other side of the tail bound is obtained by noting that \(-M_t\) is also a martingale. 
\end{proof}

It remains to estimate the third-order noise term in~\eqref{eq: bound A_l} from Lemma~\ref{lem: expansion correlations}. To this end, we introduce a truncation level \(K_N>0\) to be chosen later and write
\[
\left | \langle \boldsymbol{v}_i, \left (\nabla_{\textnormal{St}}H^\ell(\boldsymbol{X}_{\ell-1}) \right )_k \rangle \right | \leq  K_N + \left | \langle \boldsymbol{v}_i, \left (\nabla_{\textnormal{St}}H^\ell(\boldsymbol{X}_{\ell-1}) \right )_k \rangle \right | \mathbf{1}_{\left \{ | \langle \boldsymbol{v}_i, \left (\nabla_{\textnormal{St}}H^\ell(\boldsymbol{X}_{\ell-1}) \right )_k \rangle | > K_N \right \}}.
\]

\begin{lem} \label{lem: higher truncation}
Let \(\alpha_1 = \alpha_1 (p, \sigma, \{\lambda_i\}_{i=1}^r)\) and \(\alpha_2 = \alpha_2 (p, \sigma, \{\lambda_i\}_{i=1}^r)\) be the constants from Lemmas~\ref{lem: condition of noise SGD} and~\ref{lem: bound op norm of Gram matrix}, respectively. Then, for every \(a >0\), there are absolute constants \(C,c >0\) such that
\[
\begin{split}
& \sup_{\boldsymbol{X}_0 \in \textnormal{\textnormal{St}}(N,r)} \mathbb{P}_{\boldsymbol{X}_0} \left(\max_{t \le T} \frac{\delta^3 \alpha_2}{2N^2} \sum_{\ell=1}^t \sum_{k=1}^r \vert \langle \boldsymbol{v}_i , (\nabla_{\textnormal{St}}H^\ell\left( \boldsymbol{X}_{\ell-1}\right))_k \rangle \vert \mathbf{1}_{\{\vert \langle \boldsymbol{v}_i , (\nabla_{\textnormal{St}}H^\ell \left( \boldsymbol{X}_{\ell-1}\right))_k \rangle \vert > K_N\}} \geq a \right) \\
& \quad \leq C \frac{ \delta^3 r \alpha_1 \alpha_2}{N^2a} T \exp \left(- c \frac{K_N^2}{2\alpha_1^2}\right).
\end{split}
\]
\end{lem}
\begin{proof}
Since the summand is positive, the maximum over \(t \le T\) is attained at \(t=T\), so it suffices to consider that case. Moreover, by Lemma~\ref{lem: condition of noise SGD} we have the distribution of \( \langle \boldsymbol{v}_i , (\nabla_{\textnormal{St}}H(\boldsymbol{X}))_j \rangle\) is sub-Gaussian with sub-Gaussian norm bounded by \(\alpha_1\). It therefore follows that 
\[
\sup_{\boldsymbol{X} \in \textnormal{\textnormal{St}}(N,r)} \E \left[\vert \langle \boldsymbol{v}_i , (\nabla_{\textnormal{St}}H \left( \boldsymbol{X})_j \right) \rangle \vert^2 \right] \leq 2 C \alpha_1^2,
\]
and that
\[
\sup_{\boldsymbol{X} \in \textnormal{\textnormal{St}}(N,r)} \mathbb{P} \left( \vert \langle \boldsymbol{v}_i , (\nabla_{\textnormal{St}}H \left(\boldsymbol{X}\right))_j \rangle \vert > K_N \right) \leq 2 \exp\left(- c \frac{K_N^2}{\alpha_1^2}\right),
\]
where we used the equivalence of the sub-Gaussian properties (see e.g.~\cite[Proposition 2.5.2]{vershynin2018high}) for both inequalities. Fix \(\boldsymbol{X}_0 \in \textnormal{\textnormal{St}}(N,r)\). We then obtain for every \(a >0\),
\[
\begin{split}
&\mathbb{P}_{\boldsymbol{X}_0} \left( \frac{\delta^3 \alpha_2}{2N^2} \sum_{\ell=1}^T \sum_{k=1}^r \vert \langle \boldsymbol{v}_i , (\nabla_{\textnormal{St}}H^\ell \left( \boldsymbol{X}_{\ell-1}\right))_k \rangle \vert \mathbf{1}_{\{\vert \langle \boldsymbol{v}_i , (\nabla_{\textnormal{St}}H^\ell\left( \boldsymbol{X}_{\ell-1}\right))_k \rangle \vert > K_N \}} \geq a \right) \\
& \leq \frac{\delta^3 r \alpha_2}{2N^2 a} T \left ( \sup_{\ell \leq T} \sup_{k \in [r]} \E_{\boldsymbol{X}_0} \left[\vert \langle \boldsymbol{v}_i , (\nabla_{\textnormal{St}}H^\ell \left( \boldsymbol{X}_{\ell-1}\right))_j \rangle \vert \mathbf{1}_{\{\vert \langle \boldsymbol{v}_i , (\nabla_{\textnormal{St}}H^\ell\left( \boldsymbol{X}_{\ell-1}\right))_k \rangle \vert > K_N \}}\right] \right ) \\
&\leq \frac{\delta^3 r \alpha_2}{2N^2a} T \, \sqrt{\sup_{\boldsymbol{X}} \E \left[\vert \langle \boldsymbol{v}_i, \left(\nabla_{\textnormal{St}}H ( \boldsymbol{X})_k \right) \rangle \vert^2 \right]} \sqrt{\sup_{\boldsymbol{X}} \E \left[\mathbf{1}_{\{\vert \langle \boldsymbol{v}_i , (\nabla_{\textnormal{St}}H\left( \boldsymbol{X}\right))_k \rangle \vert > K_N \}}\right]} \\
&\leq \frac{C \delta^3 r \alpha_1 \alpha_2}{N^2 a} T \exp \left(- c \frac{K_N^2}{2\alpha_1^2}\right),
\end{split}
\]
where we used Markov for the first inequality and Cauchy-Schwarz for the second one.
\end{proof}

Proposition~\ref{prop: inequalities SGD} now follows straightforwardly by combining the above three lemmas.

\begin{proof}[\textbf{Proof of Proposition~\ref{prop: inequalities SGD}}]
Assume that the step size \(\delta\) satisfies \(\delta < \sqrt{\frac{N}{\alpha_2}}\), where \(\alpha_2 = \alpha_2(p, \sigma, \{\lambda_i\}_{i=1}^r)\) is the constant from Lemma~\ref{lem: bound op norm of Gram matrix}. From Lemma~\ref{lem: expansion correlations}, for every \(i,j \in [r]\) and every \(t \le T\) we have the following comparison inequalities for the correlations \(m_{ij}\):
\[
\begin{split}
m_{ij}(t) & \leq m_{ij}(0) - \frac{\delta}{N} \sum_{\ell=1}^t \langle \boldsymbol{v}_i, \left (\nabla_{\textnormal{St}}\Phi(\boldsymbol{X}_{\ell-1}) \right )_j \rangle + \frac{\delta}{N} \left | \sum_{\ell=1}^t \langle \boldsymbol{v}_i , (\nabla_{\textnormal{St}}H^\ell(\boldsymbol{X}_{\ell-1}))_j \rangle \right | + \frac{\delta}{N} \sum_{\ell=1}^t \vert a_i(\ell-1) \vert, \\
m_{ij}(t) & \geq m_{ij}(0) - \frac{\delta}{N} \sum_{\ell=1}^t \langle \boldsymbol{v}_i, \left (\nabla_{\textnormal{St}}\Phi(\boldsymbol{X}_{\ell-1}) \right )_j \rangle - \frac{\delta}{N} \left | \sum_{\ell=1}^t  \langle \boldsymbol{v}_i , (\nabla_{\textnormal{St}}H^\ell(\boldsymbol{X}_{\ell-1})_j) \rangle \right | - \frac{\delta}{N} \sum_{\ell=1}^t \vert a_i(\ell-1)\vert,
\end{split}
\]
with \(\mathbb{P}_{\boldsymbol{X}_0}\)-probability at least \(1 - T \exp \left(- c_1 N\right)\), where \(|a_i(\ell-1)|\) is given by~\eqref{eq: bound A_l}. Next, we bound the noise terms by \(\frac{c_0}{2}m_{ij}(0)\) for some constant \(c_0 \in (0,1)\) using Lemmas~\ref{lem: noise martingale} and~\ref{lem: higher truncation}. We first note that under the event \(\mathcal{C}_1(\gamma_1,\gamma_2)\), for every \(i,j \in [r]\) there exists \(\gamma_{ij} \in (\gamma_2,\gamma_1)\) such that \(m_{ij}(0) = \gamma_{ij}N^{-\frac{1}{2}}\). It then follows from Lemma~\ref{lem: noise martingale} that for all \(t \le T\),
\[
\frac{\delta}{N} \left | \sum_{\ell=1}^t \langle \boldsymbol{v}_i , (\nabla_{\textnormal{St}}H^\ell(\boldsymbol{X}_{\ell-1}))_j \rangle \right | \leq \frac{c_0 \gamma_{ij}}{2 \sqrt{N}},
\]
with \(\mathbb{P}\)-probability at least \(1 - 2 \exp \left ( - c_2 \frac{c_0^2 \gamma_2^2 N}{8 \delta^2  \alpha_1^2 T}\right )\). Moreover, we have from Lemma~\ref{lem: higher truncation} that for all \(t \le T\),
\[
\frac{\delta^3 \alpha_2}{2 N^2} \sum_{\ell=1}^t \sum_{k=1}^r \left | \langle \boldsymbol{v}_i , (\nabla_{\textnormal{St}}H^\ell(\boldsymbol{X}_{\ell-1}))_k \rangle \right | \mathbf{1}_{ \{| \langle \boldsymbol{v}_i , (\nabla_{\textnormal{St}}H^\ell(\boldsymbol{X}_{\ell-1}))_k \rangle | > K_N\}} \le \frac{c_0 \gamma_{ij}}{2 \sqrt{N}},
\]
with \(\mathbb{P}\)-probability at least \(1 - \frac{2 C_1 \delta^3 r \alpha_1 \alpha_2}{N^{\frac{3}{2}} c_0 \gamma_2} T \exp \left ( - \frac{c_3 K_N^2}{2 \alpha_1^2}\right )\). The result then follows by standard properties of conditional probabilities. 

It remains to compute \(\langle \boldsymbol{v}_i, \left (\nabla_{\textnormal{St}}\Phi(\boldsymbol{X}) \right )_j \rangle\) for every \(\boldsymbol{X} \in \textnormal{St}(N,r)\). According to~\eqref{eq: population loss}, the \(j\)th column of the Euclidean gradient of the population loss \(\Phi(\boldsymbol{X})\) is given by
\[
(\nabla  \Phi(\boldsymbol{X}))_j = - \sum_{k=1}^r p \sqrt{N} \lambda_k \lambda_j m_{kj}^{p-1}(\boldsymbol{X}) \boldsymbol{v}_k,
\]
and the \(j\)th column of the Riemannian gradient is given by
\[
\begin{split}
(\nabla_{\textnormal{St}} \Phi(\boldsymbol{X}))_j & = (\nabla  \Phi(\boldsymbol{X}))_j - \frac{1}{2} \sum_{\ell = 1}^r \left( \boldsymbol{X}^\top \nabla \Phi(\boldsymbol{X}) + (\nabla \Phi(\boldsymbol{X}))^\top \boldsymbol{X} \right)_{\ell j} \boldsymbol{x}_\ell \\
& = - \sum_{k=1}^r p \sqrt{N} \lambda_k \lambda_j m_{kj}^{p-1} \boldsymbol{v}_k + \frac{1}{2} \sum_{1 \le k, \ell \le r} p \sqrt{N} \lambda_k  m_{kj} m_{k \ell} \left ( \lambda_j m_{k j}^{p-2} + \lambda_\ell m_{k \ell}^{p-2} \right) \boldsymbol{x}_\ell.
\end{split}
\]
We therefore have for every \(i,j \in [r]\),
\[
\langle \boldsymbol{v}_i ,(\nabla_{\textnormal{St}}\Phi(\boldsymbol{X}))_j \rangle =  -\sqrt{N}p \lambda_i \lambda_j m_{ij}^{p-1} + \sqrt{N}\frac{p}{2}  \sum_{1 \le k, \ell \le r} \lambda_k m_{i \ell} m_{kj} m_{k\ell} \left ( \lambda_j m_{k j}^{p-2} + \lambda_\ell m_{k \ell}^{p-2} \right),
\]
which completes the proof.
\end{proof}

%%%%%%%%%%%%%%%%%%%%%%%%%%%%%%%%%%%%%%%%%%%%%%%%%%%%%%%%%%%%%%%%%%%%%%%%%%%%%%%%%%%%%%%%%%%%%%%%%%%%%%%%%%%%%%%%%%%%%%%%%%%%%%%%%%%%%%%%%%%%%%%%%%%%%%%%%%%%%%%%%%%%%%%%%%%%%%%%%%%%%%%%%%
\subsection{Comparison inequalities for eigenvalues} \label{subsection: comparison inequality eigenvalues}

In this subsection, we consider \(p=2\) and assume that \(\lambda_1 = \ldots = \lambda_r \equiv \lambda\). We will derive a two-sided difference inequality for the eigenvalues \(\theta_1(t), \ldots, \theta_r(t)\) of \(\boldsymbol{G}_t = \boldsymbol{G}(\boldsymbol{X}_t)\), where \(\boldsymbol{X}_t\) denotes the online SGD output at time \(t\). Let \(\Pi_{\boldsymbol{X}}^{\textnormal{St}}\) denote the orthogonal projection on the tangent space \(T_{\boldsymbol{X}} \textnormal{St}(N,r)\), i.e., 
\[
\Pi_{\boldsymbol{X}}^{\textnormal{St}} (\boldsymbol{A} )= \boldsymbol{A} - \boldsymbol{X} \textnormal{sym}(\boldsymbol{X}^\top \boldsymbol{A}), \quad \textnormal{with} \enspace \textnormal{sym}(\boldsymbol{A}) = \frac{1}{2} \left ( \boldsymbol{A} + \boldsymbol{A}^\top \right ).
\]
In particular, according to~\eqref{eq: Stiefel gradient}, for a function \(F \colon \textnormal{St}(N,r) \to \R\), we can write \(\nabla_{\textnormal{St}} F(\boldsymbol{X}) = \Pi_{\boldsymbol{X}}^{\textnormal{St}} (\nabla F(\boldsymbol{X}) ) \).
To avoid cumbersome phrasing, we will often omit stating explicitly that bounds on conditional expectations hold almost surely. 

\begin{lem} \label{lem: iso_partial_inter}
Let \(\alpha_2 = \alpha_2(p, \sigma,\{\lambda_i\}_{i=1}^r)\) be the constant of Lemma~\ref{lem: bound op norm of Gram matrix}. Consider the sequence of outputs \((\boldsymbol{X}_t)_{t \in \N}\) given by~\eqref{eq: online SGD} with step size \(\delta\) satisfying \(\delta < (N/\alpha_2)^{1/2}\). Let \(T>0\) be a fixed time horizon. Then, there exist constants \(c_1, c_2 >0\) and a truncation sequence \((K_N)_{N \ge 1}\) with \(K_N >0\), such that for every \(t \leq T\),
\[
- \boldsymbol{A}_t \preceq \boldsymbol{G}_t - \left( \boldsymbol{G}_0  + \frac{4 \lambda^2 \delta}{\sqrt{N}}  \sum_{\ell=1}^t   \boldsymbol{G}_{\ell-1} \left( \boldsymbol{I}_r - \boldsymbol{G}_{\ell-1} \right) + \boldsymbol{S}_t \right) \preceq  \boldsymbol{A}_t,
\]
with \(\mathbb{P}\)-probability at least \(1 - T \exp \left(-c_1 N\right) - 4 r^2 T \exp \left(-c_2 \frac{K_N^2}{ \sigma^2}\right)\). Here, \(\boldsymbol{S}_t \in \R^{r \times r}\) is defined by
\begin{equation} \label{eq: S_t}
\boldsymbol{S}_t = \frac{4\lambda \delta}{N} \sum_{\ell=1}^t \textnormal{sym}\left(\boldsymbol{V}^\top \Pi^{\textnormal{St}}_{\boldsymbol{X}_{\ell-1}} \left(\boldsymbol{W}^\ell  \boldsymbol{X}_{\ell-1}\right) \boldsymbol{M}^\top _{\ell-1} \right),
\end{equation}
and \(\boldsymbol{A}_t \in \R^{r \times r}\) is given by
\begin{equation}\label{eq: A_t}
\begin{split}
\boldsymbol{A}_t  = \sum_{\ell=1}^t & \Big  [ \frac{8 \lambda^2 \delta^2}{N^2} \left( 4 r^2 K_N^2 + \lambda^2 N \norm{\boldsymbol{G}_{\ell-1}} \right) + \frac{\alpha_2 \delta^2}{N} \norm{\boldsymbol{G}_{\ell-1}} \\
&  + \frac{ 4 \lambda \alpha_2\delta^3}{N^2}  \left( 2 rK_N + \lambda \sqrt{N} \norm{\boldsymbol{G}_{\ell-1}}^{1/2} \right) \norm{\boldsymbol{G}_{\ell-1}}^{1/2}  \Big ] \boldsymbol{I}_r.
\end{split}
\end{equation}
\end{lem} 

\begin{proof}
From~\eqref{eq: SGD expansion} and~\eqref{eq: P_t}, the iteration for the matrix \(\boldsymbol{G}_t = \boldsymbol{V}^\top \boldsymbol{X}_t \boldsymbol{X}_t^\top \boldsymbol{V}\) results in 
\[
\boldsymbol{G}_t  = \boldsymbol{V}^\top \left(\boldsymbol{X}_{t-1} - \frac{\delta}{N} \nabla_{\textnormal{St}} \mathcal{L}(\boldsymbol{X}_{t-1}; \boldsymbol{Y}^t)\right) \boldsymbol{P}_t ^2 \left(\boldsymbol{X}_{t-1} - \frac{\delta}{N} \nabla_{\textnormal{St}} \mathcal{L}(\boldsymbol{X}_{t-1}; \boldsymbol{Y}^t)\right)^\top \boldsymbol{V},
\]
where we used the fact that \(\boldsymbol{P}_t^\top = \boldsymbol{P}_t\). Let \(\boldsymbol{B}_t\) denote
\[
\boldsymbol{B}_t = \left (\boldsymbol{X}_{t-1}^\top - \frac{\delta}{N} \left ( \nabla_{\textnormal{St}} \mathcal{L}(\boldsymbol{X}_{t-1}; \boldsymbol{Y}^t)\right)^\top \right) \boldsymbol{V}.
\]
In the following, we place ourselves on the event \( \mathcal{E}_{\boldsymbol{\mathcal{G}}}(T) = \bigcap_{1 \leq t \leq T} \left\{ \norm{ \frac{\delta^2}{N^2} \boldsymbol{\mathcal{G}}(\boldsymbol{X}_{t-1} ; \boldsymbol{Y}^t )}_{\textnormal{F}} < 1 \right\}\). According to Lemma~\ref{lem: expansion correlations}, we have \(\mathbb{P} \left( \mathcal{E}_{\boldsymbol{\mathcal{G}}}\right) \geq 1 - T \exp(-c_1 N)\) for every \(p \ge 2\). As in the proof of Lemma~\ref{lem: expansion correlations}, we can therefore expand the projection term \(\boldsymbol{P}_t^2\) as a Neumann series:
\[
\boldsymbol{P}_t^2 = \left( \boldsymbol{I}_r  + \frac{\delta^2}{N^2} \boldsymbol{\mathcal{G}}(\boldsymbol{X}_{t-1};\boldsymbol{Y}^t) \right)^{-1} = \sum_{k=0}^{\infty}(-1)^k \frac{\delta^{2k}}{N^{2k}}\boldsymbol{\mathcal{G}}(\boldsymbol{X}_{t-1};\boldsymbol{Y}^t)^k.
\]
We write the output \(\boldsymbol{G}_t \) at time \(t\) as
\[
\boldsymbol{G}_t =  \boldsymbol{B}_t^\top \boldsymbol{P}_t^2 \boldsymbol{B}_t  = \boldsymbol{B}_t^\top  \boldsymbol{B}_t + \sum_{k=1}^\infty (-1)^k \frac{\delta^{2k}}{N^{2k}} \boldsymbol{B}_t^\top \boldsymbol{\mathcal{G}}(\boldsymbol{X}_{t-1};\boldsymbol{Y}^t)^k \boldsymbol{B}_t.
\]
We now expand \(\boldsymbol{B}_t\) by explicitly computing the Riemannian gradient. Since \(p=2\) and \(\lambda_1 = \cdots = \lambda_r \equiv \lambda\), the loss function \(\mathcal{L}(\boldsymbol{X} ; \boldsymbol{Y})\) given by~\eqref{eq: loss function} reduces to
\[
\mathcal{L}(\boldsymbol{X} ; \boldsymbol{Y}) = - \lambda \sum_{i=1}^r \boldsymbol{x}_i^\top \boldsymbol{W} \boldsymbol{x}_i - \sqrt{N}\lambda^2 \sum_{i,j=1}^r m_{ij}^2(\boldsymbol{X}).
\]
The Euclidean gradient \(\nabla \mathcal{L}(\boldsymbol{X} ; \boldsymbol{Y}) \) is then given by
\[
\nabla \mathcal{L}(\boldsymbol{X} ; \boldsymbol{Y}) = - 2 \lambda \boldsymbol{W} \boldsymbol{X} - 2 \sqrt{N}\lambda^2\boldsymbol{V}\boldsymbol{M},
\]
where \(\boldsymbol{M} = (m_{ij})_{1 \le i,j \le r} = \boldsymbol{V}^\top \boldsymbol{X}\). The Riemannian gradient is then given by
\begin{equation} \label{eq: stiefel gradient explicit}
\nabla_{\textnormal{St}} \mathcal{L} (\boldsymbol{X} ;\boldsymbol{Y}) 
= \Pi_{\boldsymbol{X}}^{\textnormal{St}} (\nabla \mathcal{L} (\boldsymbol{X}; \boldsymbol{Y})) = -2 \lambda \Pi^{\textnormal{St}}_{\boldsymbol{X}}(\boldsymbol{W} \boldsymbol{X}) - 2 \sqrt{N} \lambda^2 \boldsymbol{V} \boldsymbol{M} + 2 \sqrt{N} \lambda^2 \boldsymbol{X} \boldsymbol{M}^\top \boldsymbol{M},
\end{equation}
where we recall that \(\Pi^{\textnormal{St}}_{\boldsymbol{X}}\) denotes the projection onto \(T_{\boldsymbol{X}} \textnormal{St}(N,r)\). This implies that 
\begin{equation}\label{eq: expansion 1}
\begin{split}
\boldsymbol{G}_t  & = \boldsymbol{G}_{t-1}+\frac{4 \delta \lambda^2}{\sqrt{N}} \boldsymbol{G}_{t-1} \left(\boldsymbol{I}_r -\boldsymbol{G}_{t-1}\right) + \frac{4 \delta \lambda}{N} \textnormal{sym} \left(\boldsymbol{V}^\top \Pi^{\textnormal{St}}_{\boldsymbol{X}_{t-1}} \left(\boldsymbol{W}^t \boldsymbol{X}_{t-1} \right)\boldsymbol{M}_{t-1}^\top \right) \\
& \quad +\frac{\delta^2}{N^2} \boldsymbol{V}^\top \nabla_{\textnormal{St}} \mathcal{L} \left(\boldsymbol{X}_{t-1};\boldsymbol{Y}^t \right) \left (\nabla_{\textnormal{St}} \mathcal{L} \left(\boldsymbol{X}_{t-1};\boldsymbol{Y}^t \right) \right )^\top \boldsymbol{V} \\
& \quad + \sum_{k=1}^{\infty}(-1)^{k}\frac{\delta^{2k}}{N^{2k}} \boldsymbol{B}_t^\top \boldsymbol{\mathcal{G}}(\boldsymbol{X}_{t-1} ; \boldsymbol{Y}^t )^k \boldsymbol{B}_t.
\end{split}
\end{equation}
The remainder of the proof focuses on estimating the terms in~\eqref{eq: expansion 1} that are of order \(2\) or higher in \(\delta\). 

First, we consider the summand
\[
\frac{\delta^2}{N^2} \boldsymbol{V}^\top \nabla_{\textnormal{St}} \mathcal{L} (\boldsymbol{X}_{t-1}; \boldsymbol{Y}^t) \left ( \nabla_{\textnormal{St}} \mathcal{L} (\boldsymbol{X}_{t-1}; \boldsymbol{Y}^t) \right )^\top \boldsymbol{V} ,
\]
and observe that 
\[
\begin{split}
\frac{\delta^2}{N^2} \boldsymbol{V}^\top \nabla_{\textnormal{St}} \mathcal{L} (\boldsymbol{X}_{t-1}; \boldsymbol{Y}^t) \left ( \nabla_{\textnormal{St}} \mathcal{L} (\boldsymbol{X}_{t-1}; \boldsymbol{Y}^t) \right )^\top \boldsymbol{V} & \preceq \frac{\delta^2}{N^2} \norm{ \boldsymbol{V}^\top \nabla_{\textnormal{St}} \mathcal{L} (\boldsymbol{X}_{t-1}; \boldsymbol{Y}^t)}^2 \boldsymbol{I}_r,\\
\frac{\delta^2}{N^2} \boldsymbol{V}^\top \nabla_{\textnormal{St}} \mathcal{L} (\boldsymbol{X}_{t-1}; \boldsymbol{Y}^t) \left ( \nabla_{\textnormal{St}} \mathcal{L} (\boldsymbol{X}_{t-1}; \boldsymbol{Y}^t) \right )^\top \boldsymbol{V} & \succeq - \frac{\delta^2}{N^2} \norm{ \boldsymbol{V}^\top \nabla_{\textnormal{St}} \mathcal{L} (\boldsymbol{X}_{t-1}; \boldsymbol{Y}^t)}^2 \boldsymbol{I}_r.
\end{split}
\]
We therefore want to estimate the operator norm of \(\boldsymbol{V}^\top \nabla_{\textnormal{St}} \mathcal{L} \left(\boldsymbol{X}_{t-1} ; \boldsymbol{Y}_t \right)\). From~\eqref{eq: stiefel gradient explicit}, we have
\[
\boldsymbol{V}^\top \nabla_{\textnormal{St}} \mathcal{L} \left(\boldsymbol{X}_{t-1} ; \boldsymbol{Y}_t \right) = - 2 \lambda \boldsymbol{V}^\top \Pi^{\textnormal{St}}_{\boldsymbol{X}_{t-1}} \left( \boldsymbol{W}^t \boldsymbol{X}_{t-1} \right) - 2 \sqrt{N} \lambda^2 \boldsymbol{M}_{t-1} \left ( \boldsymbol{I}_r - \boldsymbol{M}_{t-1}^\top \boldsymbol{M}_{t-1}\right),
\]
which gives the bound 
\begin{equation} \label{eq: op bound stiefel 2}
\norm{\boldsymbol{V}^\top \nabla_{\textnormal{St}} \mathcal{L} \left(\boldsymbol{X}_{t-1} ; \boldsymbol{Y}_t \right)} \le 2 \lambda \norm{ \boldsymbol{V}^\top \Pi^{\textnormal{St}}_{\boldsymbol{X}_{t-1}} \left( \boldsymbol{W}^t \boldsymbol{X}_{t-1} \right)} + 2 \sqrt{N} \lambda^2 \norm{\boldsymbol{M}_{t-1}},
\end{equation}
where we used the fact that \(\norm{\boldsymbol{I}_r - \boldsymbol{M}_{t-1}^\top \boldsymbol{M}_{t-1}} \leq 1\), since \(\boldsymbol{M}_{t-1}^\top \boldsymbol{M}_{t-1}\) is positive semi-definite and its operator norm is bounded above by \(1\). Additionally, we can bound the operator norm of \(\boldsymbol{V}^\top \Pi^{\textnormal{St}}_{\boldsymbol{X}_{t-1}} \left( \boldsymbol{W}^t \boldsymbol{X}_{t-1}\right)\) as  follows:
\begin{equation} \label{eq: op bound stiefel}
\begin{split}
\norm{\boldsymbol{V}^\top \Pi^{\textnormal{St}}_{\boldsymbol{X}_{t-1}} \left( \boldsymbol{W}^t \boldsymbol{X}_{t-1}\right)}
& = \norm{ \boldsymbol{V}^\top \boldsymbol{W}^t \boldsymbol{X}_{t-1} -\boldsymbol{V}^\top  \boldsymbol{X}_{t-1} \textnormal{sym} \left( \boldsymbol{X}_{t-1}^\top \boldsymbol{W}^t \boldsymbol{X}_{t-1} \right)} \\
& \leq \norm{\boldsymbol{V}^\top \boldsymbol{W}^t \boldsymbol{X}_{t-1}} +   \norm{\boldsymbol{X}_{t-1}^\top \boldsymbol{W}^t \boldsymbol{X}_{t-1}},
\end{split}
\end{equation}
where we used the fact that \(\norm{\boldsymbol{V}^\top \boldsymbol{X}_{t-1}} = \norm{\boldsymbol{M}_{t-1}} \leq 1\) and the fact that the matrix \(\boldsymbol{W}^t\) is symmetric. Combining~\eqref{eq: op bound stiefel 2} and~\eqref{eq: op bound stiefel} yields
\begin{equation} \label{eq: first bound order 2}
\begin{split}
\norm{\boldsymbol{V}^\top \nabla_{\textnormal{St}} \mathcal{L} \left(\boldsymbol{X}_{t-1} ; \boldsymbol{Y}_t \right)}^2 
& \leq 2 \left( 4 \lambda^2 \norm{\boldsymbol{V}^\top \Pi^{\textnormal{St}}_{\boldsymbol{X}_{t-1}} \left( \boldsymbol{W}^t \boldsymbol{X}_{t-1}\right)}^2+ 4 N \lambda^4 \norm{\boldsymbol{M}_{t-1}}^2 \right) \\
& \leq 8 \lambda^2 \left( 2 \norm{\boldsymbol{V}^\top \boldsymbol{W}^t \boldsymbol{X}_{t-1}}^2 + 2 \norm{\boldsymbol{X}_{t-1}^\top \boldsymbol{W}^t \boldsymbol{X}_{t-1}}^2 + N \lambda^2 \norm{\boldsymbol{G}_{t-1}} \right),
\end{split}
\end{equation}
where we used the fact that \((a+b)^2 \leq 2 (a^2 + b^2)\). 

Now, consider the higher order terms in \(\delta\) from~\eqref{eq: expansion 1}, i.e., 
\[
\sum_{k=1}^{\infty}(-1)^{k}\frac{\delta^{2k}}{N^{2k}} \boldsymbol{B}_t^\top \boldsymbol{\mathcal{G}}(\boldsymbol{X}_{t-1} ; \boldsymbol{Y}^t )^k \boldsymbol{B}_t. 
\]
Note first that for every \(k \ge 1\),
\[
- \norm{\boldsymbol{\mathcal{G}}}^{k-1} \boldsymbol{B}_t^\top  \boldsymbol{\mathcal{G}} \boldsymbol{B}_t \preceq \boldsymbol{B}_t^\top \boldsymbol{\mathcal{G}}^k \boldsymbol{B}_t \preceq \norm{\boldsymbol{\mathcal{G}}}^{k-1} \boldsymbol{B}_t^\top  \boldsymbol{\mathcal{G}} \boldsymbol{B}_t.     
\]
Indeed, since \(\boldsymbol{\mathcal{G}}\) is positive semi-definite, for every \(\boldsymbol{y} \in \R^r\) we can write
\[
\begin{split}
\boldsymbol{y}^\top \boldsymbol{B}_t^\top \left ( \boldsymbol{\mathcal{G}}^k - \norm{\boldsymbol{\mathcal{G}}}^{k-1} \boldsymbol{\mathcal{G}} \right ) \boldsymbol{B}_t \boldsymbol{y} & = \boldsymbol{y}^\top \boldsymbol{B}_t^\top \boldsymbol{\mathcal{G}}^{1/2} \left ( \boldsymbol{\mathcal{G}}^{k-1}  - \norm{\boldsymbol{\mathcal{G}}}^{k-1} \boldsymbol{I}_r\right ) \boldsymbol{\mathcal{G}}^{1/2} \boldsymbol{B}_t \boldsymbol{y}\\
&\le \boldsymbol{z}^\top   \left ( \boldsymbol{\mathcal{G}}^{k-1}  - \norm{\boldsymbol{\mathcal{G}}^{k-1}} \boldsymbol{I}_r\right )\boldsymbol{z} \leq 0,
\end{split}
\]
where \(\boldsymbol{z} = \boldsymbol{\mathcal{G}}^{1/2} \boldsymbol{B}_t \boldsymbol{y}\). The other inequality follows similarly. This implies that 
\[
- \sum_{k=1}^{\infty}(-1)^k \frac{\delta^{2k}}{N^{2k}} \norm{\boldsymbol{\mathcal{G}}}^{k-1} \boldsymbol{B}_t^\top  \boldsymbol{\mathcal{G}} \boldsymbol{B}_t  \preceq \sum_{k=1}^{\infty}(-1)^{k}\frac{\delta^{2k}}{N^{2k}} \boldsymbol{B}_t^\top \boldsymbol{\mathcal{G}}(\boldsymbol{X}_{t-1} ; \boldsymbol{Y}^t )^k \boldsymbol{B}_t  \preceq \sum_{k=1}^{\infty}(-1)^k \frac{\delta^{2k}}{N^{2k}} \norm{\boldsymbol{\mathcal{G}}}^{k-1} \boldsymbol{B}_t^\top  \boldsymbol{\mathcal{G}} \boldsymbol{B}_t , 
\]
where
\[
\sum_{k=1}^{\infty}(-1)^k \frac{\delta^{2k}}{N^{2k}} \norm{\boldsymbol{\mathcal{G}}}^{k-1} \boldsymbol{B}_t^\top  \boldsymbol{\mathcal{G}} \boldsymbol{B}_t = - \left( \sum_{k=0}^\infty (-1)^k  \norm{\frac{\delta^2}{N^2} \boldsymbol{\mathcal{G}}}^k \right) \frac{\delta^2}{N^2} \boldsymbol{B}_t^\top  \boldsymbol{\mathcal{G}} \boldsymbol{B}_t = - \frac{1}{1 + \alpha_2 \frac{\delta^2}{N}} \frac{\delta^2}{N^2} \boldsymbol{B}_t^\top  \boldsymbol{\mathcal{G}} \boldsymbol{B}_t,
\]
where we used the fact that \(\max \norm{\frac{\delta^2}{N^2} \boldsymbol{\mathcal{G}}} \le \alpha_2 \frac{\delta^2}{N}\) on the event \(\mathcal{E}_{\boldsymbol{\mathcal{G}}}\). Therefore, we have
\begin{equation} \label{eq: expansion 2}
\begin{split}
\sum_{k=1}^{\infty}(-1)^{k}\frac{\delta^{2k}}{N^{2k}} \boldsymbol{B}_t^\top \boldsymbol{\mathcal{G}}(\boldsymbol{X}_{t-1} ; \boldsymbol{Y}^t )^k \boldsymbol{B}_t  & \preceq - \frac{1}{1 + \alpha_2 \frac{\delta^2}{N}} \frac{\delta^2}{N^2} \boldsymbol{B}_t^\top  \boldsymbol{\mathcal{G}} \boldsymbol{B}_t ,\\
\sum_{k=1}^{\infty}(-1)^{k}\frac{\delta^{2k}}{N^{2k}} \boldsymbol{B}_t^\top \boldsymbol{\mathcal{G}}(\boldsymbol{X}_{t-1} ; \boldsymbol{Y}^t )^k \boldsymbol{B}_t  & \succeq  \frac{1}{1 + \alpha_2 \frac{\delta^2}{N}} \frac{\delta^2}{N^2} \boldsymbol{B}_t^\top  \boldsymbol{\mathcal{G}} \boldsymbol{B}_t ,\\
\end{split}
\end{equation}
It remains to bound the operator norm of \(\boldsymbol{B}_t^\top \boldsymbol{\mathcal{G}} \boldsymbol{B}_t\):
\[
\begin{split}
&  \frac{\delta^2}{N^2} \boldsymbol{B}_t^\top \boldsymbol{\mathcal{G}}(\boldsymbol{X}_{t-1} ; \boldsymbol{Y}^t ) \boldsymbol{B}_t \\
& = \frac{\delta^2}{N^2} \boldsymbol{M}_{t-1} \boldsymbol{\mathcal{G}}(\boldsymbol{X}_{t-1} ; \boldsymbol{Y}^t) \boldsymbol{M}_{t-1}^\top - \frac{2\delta^3}{N^3} \textnormal{sym} \left( \boldsymbol{V}^\top \nabla_{\textnormal{St}} \mathcal{L} \left(\boldsymbol{X}_{t-1} ; \boldsymbol{Y}^t \right) \boldsymbol{\mathcal{G}}(\boldsymbol{X}_{t-1} ;\boldsymbol{Y}^t) \boldsymbol{M}_{t-1}^\top \right)\\
& \quad + \frac{\delta^{4}}{N^{4}}  \boldsymbol{V}^\top \nabla_{\textnormal{St}} \mathcal{L} \left(\boldsymbol{X}_{t-1} ; \boldsymbol{Y}^t \right) \boldsymbol{\mathcal{G}}(\boldsymbol{X}_{t-1} ; \boldsymbol{Y}^t ) \left( \nabla_{\textnormal{St}} \mathcal{L}(\boldsymbol{X}_{t-1} ; \boldsymbol{Y}^t ) \right )^\top \boldsymbol{V}.
\end{split}
\]
The operator norm of the first summand is bounded above by  
\[
\begin{split}
\frac{\delta^2}{N^2} \norm{\boldsymbol{M}_{t-1} \boldsymbol{\mathcal{G}}(\boldsymbol{X}_{t-1} ;\boldsymbol{Y}^t) \boldsymbol{M}_{t-1}^\top} & =\frac{\delta^2}{N^2} \norm{\boldsymbol{M}_{t-1} \nabla_{\textnormal{St}} \mathcal{L}(\boldsymbol{X}_{t-1};\boldsymbol{Y}^t)}^2 \\
& \le \frac{\delta^2}{N^2} \norm{\boldsymbol{M}_{t-1}}^2 \norm{\nabla_{\textnormal{St}} \mathcal{L}(\boldsymbol{X}_{t-1};\boldsymbol{Y}^t)}^2 \\
& \le \alpha_2 \frac{\delta^2}{N}\norm{\boldsymbol{G}_{t-1}},
\end{split}
\]
where we used the fact that \(\norm{\nabla_{\textnormal{St}} \mathcal{L}(\boldsymbol{X}_{t-1};\boldsymbol{Y}^t)}^2 = \norm{\boldsymbol{\mathcal{G}}(\boldsymbol{X}_{t-1};\boldsymbol{Y}^t)}\). For the second summand, we have
\[
\begin{split}
& \frac{2 \delta^3}{N^3} \norm{\mbox{sym} \left(\boldsymbol{V}^\top \nabla_{\textnormal{St}} \mathcal{L} \left(\boldsymbol{X}_{t-1} ; \boldsymbol{Y}^t \right) \boldsymbol{\mathcal{G}}(\boldsymbol{X}_{t-1} ; \boldsymbol{Y}^t) \boldsymbol{M}_{t-1}^\top\right)} \\
& \leq \frac{2 \delta^3}{N^3}  \norm{\boldsymbol{\mathcal{G}} (\boldsymbol{X}_{t-1};\boldsymbol{Y}^t)} \norm{\boldsymbol{V}^\top \nabla_{\textnormal{St}} \mathcal{L} (\boldsymbol{X}_{t-1} ; \boldsymbol{Y}^t)} \norm{\boldsymbol{M}_{t-1}} \\
& \leq \frac{2 \alpha_2 \delta^3}{N^2} \left ( 2 \lambda \norm{\boldsymbol{V}^\top \boldsymbol{W}^t \boldsymbol{X}_{t-1}} + 2 \lambda \norm{\boldsymbol{X}_{t-1}^\top \boldsymbol{W}^t \boldsymbol{X}_{t-1}} + 2 \sqrt{N} \lambda^2 \norm{\boldsymbol{M}_{t-1}} \right)  \norm{\boldsymbol{M}_{t-1}},
\end{split}
\]
where the last inequality follows by~\eqref{eq: op bound stiefel 2} and~\eqref{eq: op bound stiefel}. Finally, the third term is bounded in operator norm as follows:
\[
\begin{split}
& \frac{\delta^4}{N^4} \norm{\boldsymbol{V}^\top \nabla_{\textnormal{St}} \mathcal{L} \left(\boldsymbol{X}_{t-1} ;\boldsymbol{Y}^t \right) \boldsymbol{\mathcal{G}}(\boldsymbol{X}_{t-1} ; \boldsymbol{Y}^t ) \left ( \nabla_{\textnormal{St}} \mathcal{L} (\boldsymbol{X}_{t-1} ; \boldsymbol{Y}^t ) \right )^\top \boldsymbol{V} } \\
& = \frac{\delta^4}{N^4} \norm{\boldsymbol{V}^\top \nabla_{\textnormal{St}} \mathcal{L} \left(\boldsymbol{X}_{t-1} ; \boldsymbol{Y}^t \right) \left (\nabla_{\textnormal{St}} \mathcal{L} (\boldsymbol{X}_{t-1} ; \boldsymbol{Y}^t )\right )^\top}^2 \\
& \le \frac{\delta^4}{N^4} \norm{\boldsymbol{V}^\top \nabla_{\textnormal{St}} \mathcal{L} (\boldsymbol{X}_{t-1} ; \boldsymbol{Y}^t )}^2 \norm{\boldsymbol{\mathcal{G}}(\boldsymbol{X}_{t-1};\boldsymbol{Y}^t)}\\
& \le 8 \lambda^2 \frac{\alpha_2 \delta^4}{N^3} \left( 2 \norm{\boldsymbol{V}^\top \boldsymbol{W}^t \boldsymbol{X}_{t-1}}^2 + 2 \norm{\boldsymbol{X}_{t-1}^\top \boldsymbol{W}^t \boldsymbol{X}_{t-1}}^2 + N \lambda^2 \norm{\boldsymbol{G}_{t-1}} \right),
\end{split}
\]
where we used~\eqref{eq: first bound order 2} for the last inequality. 

Combining the last four estimates with~\eqref{eq: expansion 2} and from~\eqref{eq: expansion 1} and~\eqref{eq: first bound order 2}, we finally obtain the partial ordering
\[
\begin{split}
\boldsymbol{G}_t  &\preceq \boldsymbol{G}_{t-1} + \frac{4 \lambda^2 \delta}{\sqrt{N}}\boldsymbol{G}_{t-1} \left(\boldsymbol{I}_r-\boldsymbol{G}_{t-1}\right) + \frac{4 \lambda \delta}{N} \mbox{sym}\left(\boldsymbol{V}^\top \Pi^{\textnormal{St}}_{\boldsymbol{X}_{t-1}} \left(\boldsymbol{W}^t \boldsymbol{X}_{t-1}\right)\boldsymbol{M}_{t-1}^\top \right) \\
& \enspace + \frac{8 \lambda^2 \delta^2}{N^2} \left( 1 + \alpha_2 \frac{\delta^2}{N} \frac{1}{1 + \alpha_2 \frac{\delta^2}{N}} \right) \left( 2 \norm{\boldsymbol{V}^\top \boldsymbol{W}^t \boldsymbol{X}_{t-1}}^2 + 2 \norm{\boldsymbol{X}_{t-1}^\top \boldsymbol{W}^t \boldsymbol{X}_{t-1}}^2 + N \lambda^2 \norm{\boldsymbol{G}_{t-1}} \right)\boldsymbol{I}_r \\
& \enspace +\frac{1}{1 + \alpha_2\frac{\delta^2}{N}} \alpha_2 \frac{\delta^2}{N} \norm{\boldsymbol{G}_{t-1}} \boldsymbol{I}_r\\
& \enspace + \frac{4 \lambda \delta^3}{N^2} \frac{\alpha_2}{1 + \alpha_2 \frac{\delta^2}{N}} \left ( \norm{\boldsymbol{V}^\top \boldsymbol{W}^t \boldsymbol{X}_{t-1} } + \norm{\boldsymbol{X}_{t-1}^\top \boldsymbol{W}^t \boldsymbol{X}_{t-1} } + \sqrt{N} \lambda \norm{\boldsymbol{G}_{t-1}}^{1/2} \right) \norm{\boldsymbol{G}_{t-1}}^{1/2} \boldsymbol{I}_r.
\end{split}
\]
Iterating the above expression yields 
\[
\boldsymbol{G}_t  \preceq \boldsymbol{G}_0 + \frac{4 \lambda^2 \delta}{\sqrt{N}}  \sum_{\ell=1}^t \left ( \boldsymbol{G}_{\ell-1} \left(\boldsymbol{I}_r -\boldsymbol{G}_{\ell-1}\right)\right ) + \frac{4 \lambda \delta}{N} \sum_{\ell=1}^t \textnormal{sym} \left(\boldsymbol{V}^\top \Pi^{\textnormal{St}}_{\boldsymbol{X}_{\ell-1}} \left(\boldsymbol{W}^\ell \boldsymbol{X}_{\ell-1} \right)\boldsymbol{M}_{\ell-1}^\top \right) + \sum_{\ell=1}^t\boldsymbol{A}'_\ell,
\]
where \(\boldsymbol{A}'_\ell\) is given by
\[
\begin{split}
\boldsymbol{A}'_\ell & = \frac{8 \lambda^2 \delta^2}{N^2} \left( 1 + \alpha_2 \frac{\delta^2}{N} \frac{1}{1 + \alpha_2 \frac{\delta^2}{N}} \right) \left(2 \norm{\boldsymbol{V}^\top \boldsymbol{W}^\ell \boldsymbol{X}_{\ell-1}}^2+ 2 \norm{\boldsymbol{X}_{\ell-1}^\top \boldsymbol{W}^\ell \boldsymbol{X}_{\ell-1}}^2 + N \lambda^2 \norm{\boldsymbol{G}_{\ell-1}}\right)\boldsymbol{I}_r \\
& \quad + \frac{1}{1 + \alpha_2 \frac{\delta^2}{N}} \alpha_2 \frac{\delta^2}{N} \norm{\boldsymbol{G}_{\ell-1}} \boldsymbol{I}_r \\
& \quad + \frac{4 \lambda \delta^3}{N^2} \frac{\alpha_2}{1 + \alpha_2 \frac{\delta^2}{N}} \left ( \norm{\boldsymbol{V}^\top \boldsymbol{W}^\ell  \boldsymbol{X}_{\ell-1}} +  \norm{\boldsymbol{X}_{\ell-1}^\top \boldsymbol{W}^\ell  \boldsymbol{X}_{\ell-1}} + \sqrt{N} \lambda \norm{\boldsymbol{G}_{\ell-1}}^{1/2}\right ) 
 \norm{\boldsymbol{G}_{\ell-1}}^{1/2} \boldsymbol{I}_r.
\end{split}
\]
The lower bound follows similarly. 

The remaining of the proof focuses on deriving a high-probability estimate for \(\norm{\boldsymbol{V}^\top \boldsymbol{W}^\ell \boldsymbol{X}_{\ell-1}}\) and \(\norm{\boldsymbol{X}_{\ell-1}^\top \boldsymbol{W}^\ell \boldsymbol{X}_{\ell-1}}\) for every \(\ell \leq t\). Let \(\mathcal{F}_t = \sigma (\boldsymbol{X}_\ell, \ell \leq t)\) denote the canonical filtration. Then, for every \(\eta >0\) it holds that
\[
\begin{split}
\mathbb{E}\left[\exp\left(\eta\left(\boldsymbol{V}^\top \boldsymbol{W}^t \boldsymbol{X}_{t-1}\right)_{ij}\right) \vert \mathcal{F}_{t-1}\right] & = \mathbb{E} \left[\exp\left(\eta \sum_{k=1}^r \sum_{\ell=1}^r (\boldsymbol{v}_i)_k \boldsymbol{W}^t _{k\ell} (\boldsymbol{x}_j^{t-1})_\ell \right) \vert \mathcal{F}_{t-1}\right]  \\
& = \prod_{k,\ell} \mathbb{E}\left[\exp\left(\eta (\boldsymbol{v}_i)_k \boldsymbol{W}^t_{k \ell} (\boldsymbol{x}_j^{t-1})_\ell \right) \vert \mathcal{F}_{t-1} \right] \\
& \leq \prod_{k,\ell} \exp \left (C \eta^2 \sigma^2 (\boldsymbol{v}_i)_k^2 (\boldsymbol{x}_j^{t-1})_\ell^2\right) \quad \mbox{for some absolute constant \(C\)} \\
& \leq \exp\left(C\eta^2\sigma^2\right),
\end{split}
\]
where we used the fact that the entries of \(\boldsymbol{W}^{t-1}\) are i.i.d.\ sub-Gaussian satisfying~\eqref{eq: MGF sub-Gaussian} and independent from \(\mathcal{F}_{t-1}\), and the fact that \(\boldsymbol{v}_i, \boldsymbol{x}_j \in \textnormal{St}(N,r)\) for every \(i,j \in [r]\).
The entries of the matrix \(\boldsymbol{V}^\top \boldsymbol{W}^t \boldsymbol{X}_{t-1}\) given \(\mathcal{F}_{t-1}\) are therefore sub-Gaussian with sub-Gaussian norm \(\sigma^2\). We have the following sub-Gaussian tail bound: for every \(\gamma>0\),
\[
\mathbb{P}\left(\max_{1 \leq i,j \leq r} \vert (\boldsymbol{V}^\top \boldsymbol{W}^t \boldsymbol{X}_{t-1})_{ij} \vert \geq \gamma | \mathcal{F}_{t-1}\right) \leq 2 r^2 \exp \left(-\frac{\gamma^2}{C\sigma^2} \right).
\]
Since \(\norm{\boldsymbol{V}^\top \boldsymbol{W}^t \boldsymbol{X}_{t-1}} \leq r\max_{1 \leq i,j \leq r}\vert(\boldsymbol{V}^\top \boldsymbol{W}^t \boldsymbol{X}_{t-1})_{ij}\vert\), for every \(K_N\ge 0\) it holds that 
\[
\mathbb{P} \left ( \norm{\boldsymbol{V}^\top \boldsymbol{W}^t \boldsymbol{X}_{t-1}} \geq rK_N | \mathcal{F}_{t-1} \right ) \leq 2 r^2 \exp \left( -\frac{K_N ^2}{C\sigma^2}\right).
\]
Finally, an union bound over the time horizon \(T\) yields 
\begin{equation} \label{eq:union_V}
\mathbb{P} \left(\bigcup_{1 \leq t \leq T} \norm{\boldsymbol{V}^\top \boldsymbol{W}^t \boldsymbol{X}_{t-1}} \geq rK_N | \mathcal{F}_{t-1} \right) \leq 2r^2T\exp\left(-\frac{K_N ^2}{C\sigma^2}\right).
\end{equation}
We derive the same bound for \(\norm{\boldsymbol{X}_{t-1}^\top \boldsymbol{W}^t \boldsymbol{X}_{t-1}} \) using a similar argument. Thus, on the event \(\mathcal{E}_{\mathcal{G}}\), we have for every \(\ell \le t\),
\[
\begin{split}
\boldsymbol{A}'_\ell & \le \frac{8 \lambda^2 \delta^2}{N^2} \left(4r^2 K_N^2 + N \lambda^2 \norm{\boldsymbol{G}_{\ell-1}}\right)\boldsymbol{I}_r + \alpha_2 \frac{\delta^2}{N} \norm{\boldsymbol{G}_{\ell-1}} \boldsymbol{I}_r  \\
& \quad  + \frac{4 \lambda \delta^3}{N^2} \alpha_2 \left ( 2r K_N + \sqrt{N} \lambda \norm{\boldsymbol{G}_{\ell-1}}^{1/2} \right ) \norm{\boldsymbol{G}_{\ell-1}}^{1/2} \boldsymbol{I}_r,
\end{split}
\]
with \(\mathbb{P}\)-probability at least \(1 - 4r^2 T \exp(-K_N^2/(C \sigma^2))\). This completes the proof of Lemma~\ref{lem: iso_partial_inter}.
\end{proof}

The following lemma gives a high-probability estimate for \(\boldsymbol{S}_t\), defined in~\eqref{eq: S_t}.

\begin{lem} \label{lem: matrix_martin_incr}
There exists a constant \(c = c (r, \sigma)\) such that for every scalar, positive, continuous functions \(a,b\) and any fixed time horizon \(T\), 
\[
\mathbb{P} \left(\exists \, t \in [0,T] \colon \norm{\boldsymbol{S}_t} \geq a(t) \enspace \textnormal{and} \enspace \lambda_{\max} \left(\boldsymbol{Q}_t \right) \leq b(t) \right) \leq \inf_{\theta >0} \sup_{t \in [0,T]} 2r\exp\left(-\theta a(t) + c \frac{\theta^2}{2} b(t)\right),
\]
where \(\boldsymbol{Q}_t \) is given by \(\boldsymbol{Q}_t  = \left(\frac{4\lambda\delta}{N}\right)^2\sum_{\ell=1}^t \norm{\boldsymbol{G}_{\ell -1}} \boldsymbol{I}_r\), thus \(\lambda_{\max}(\boldsymbol{Q}_t) = \left(\frac{4\lambda\delta}{N}\right)^2\sum_{\ell=1}^t \norm{\boldsymbol{G}_{\ell -1}}\). In particular, if \(a\) is continuously differentiable with non-vanishing derivative and \(b(t) = K a^2(t)\) for some constant \(K\), then it holds that 
\[
\mathbb{P} \left( \exists \, t \in [0,T] \colon \norm{\boldsymbol{S}_t} \geq a(t) \enspace \textnormal{and} \enspace \lambda_{\max} \left(\boldsymbol{Q}_t \right) \leq b(t) \right) \leq  2r\exp\left(-\frac{c}{2K}\right).
\]
\end{lem}

\begin{proof}
Let \(\mathcal{F}_t\) denote the natural filtration associated to the online SGD outputs \(\boldsymbol{X}_1, \ldots, \boldsymbol{X}_t\). Since for every \(t \ge 0\),
\[
\mathbb{E} \left [\boldsymbol{S}_t - \boldsymbol{S}_{t-1} | \mathcal{F}_{t-1} \right ] = \frac{4 \lambda \delta}{N} \mathbb{E} \left[ \textnormal{sym} \left( \boldsymbol{V}^\top \Pi^{\textnormal{St}}_{\boldsymbol{X}_{t-1}} \left( \boldsymbol{W}^t \boldsymbol{X}_{t-1}\right)\boldsymbol{M}_{t-1}^\top \right) \vert \mathcal{F}_{t-1} \right] = 0,
\]
the process \(\boldsymbol{S}_t\) is an \(\mathcal{F}_t\)-adapted matrix martingale. Moreover, from~\eqref{eq: op bound stiefel} it follows that
\[
\norm{\boldsymbol{V}^\top \Pi^{\textnormal{St}}_{\boldsymbol{X}_{t-1}} \left(\boldsymbol{W}^t \boldsymbol{X}_{t-1}\right) \boldsymbol{M}_{t-1}^\top } \leq \left ( \norm{\boldsymbol{V}^\top \boldsymbol{W}^t \boldsymbol{X}_{t-1}} + \norm{\boldsymbol{X}_{t-1}^\top \boldsymbol{W}^t \boldsymbol{X}_{t-1}} \right ) \norm{\boldsymbol{M}_{t-1}^\top }.
\]
Using the sub-Gaussian tail for the entries of \(\boldsymbol{V}^\top \boldsymbol{W}^t \boldsymbol{X}_{t-1}\) established in the proof of Lemma~\ref{lem: iso_partial_inter}, we have for every \(k \ge 1\),
\[
\begin{split}
\mathbb{E} \left[\norm{\boldsymbol{V}^\top \boldsymbol{W}^t \boldsymbol{X}_{t-1}}^k \vert \mathcal{F}_{t-1}\right] & = \int_0^\infty k y^{k-1} \mathbb{P} \left(\norm{\boldsymbol{V}^\top \boldsymbol{W}^t \boldsymbol{X}_{t-1}} \geq y \vert \mathcal{F}_{t-1} \right) d y \\
& \leq  2 r^2 k \int_0^\infty y^{k-1} \exp\left(-\frac{y^2}{C \sigma^2 r^2} \right) dy \\
& = 2 r^2 k \left ( \frac{C r^2 \sigma^2}{2} \right)^{k/2} 2^{k/2-1} \Gamma \left( \frac{k}{2}\right) \\
& \leq 3 r^2  ( C r^2 \sigma^2)^{k/2}  k \left( \frac{k}{2}\right)^{k/2},
\end{split}
\]
where we used the fact that \(\Gamma(x) \leq 3 x^x\) for every \(x \ge 1/2\). We obtain the same upper bound for \(\E \left[\norm{\boldsymbol{X}_{t-1}^\top \boldsymbol{W}^t \boldsymbol{X}_{t-1}}^k \vert \mathcal{F}_{t-1}\right] \). Using the fact that \((a+b)^k \leq 2^{k-1}(a^k + b^k) \le 2^k (a^k + b^k)\), it therefore follows that
\[
\begin{split}
\mathbb{E} \left[\norm{\left(\boldsymbol{S}_t -\boldsymbol{S}_{t-1}\right)^k} \vert \mathcal{F}_{t-1}\right]  & = \left ( \frac{4 \lambda \delta}{N}\right)^k \mathbb{E} \left[\norm{\mbox{sym} \left(\boldsymbol{V}^\top \Pi^{\textnormal{St}}_{\boldsymbol{X}_{t-1}} \left(\boldsymbol{W}^t \boldsymbol{X}_{t-1}\right) \boldsymbol{M}_{t-1}^\top \right) }^k \vert \mathcal{F}_{t-1}\right] \\
& \leq \left ( \frac{8 \lambda \delta}{N}\right)^k 3 r^2 (Cr^2\sigma^2)^{k/2} k \left(\frac{k}{2} \right)^{k/2}  \norm{\boldsymbol{G}_{t-1}}^{k/2}.
\end{split}
\]
According to the Taylor series expansion of the exponential function, we obtain for every \(z \geq 0\) that
\[
\begin{split}
\mathbb{E}\left[\exp\left(z^2\left(\boldsymbol{S}_t -\boldsymbol{S}_{t-1}\right)^2\right) \vert \mathcal{F}_{t-1}\right] &= \boldsymbol{I}_r + \sum_{k=1}^\infty \frac{z^{2k}}{k!} \mathbb{E} \left[(\boldsymbol{S}_t -\boldsymbol{S}_{t-1})^{2k} \vert \mathcal{F}_{t-1} \right] \\
& \preceq \boldsymbol{I}_r + \sum_{k=1}^\infty \frac{z^{2k}}{k!} \mathbb{E} \left[\norm{\left(\boldsymbol{S}_t -\boldsymbol{S}_{t-1}\right)^{2k}} \boldsymbol{I}_r  \vert \mathcal{F}_{t-1}\right]  \\
& \preceq \left( 1 + \sum_{k=1}^\infty \frac{z^{2k}}{k!} \left (\frac{8\lambda \delta}{N} \right)^{2k} 12 r^2 \left( Cr^2\sigma^2\right)^k k^{k+1}  \norm{\boldsymbol{G}_{t-1}}^k \right) \boldsymbol{I}_r.
\end{split}
\]
We then apply similar arguments to that used for the proof of~\cite[Proposition 2.5.2]{vershynin2018high}, in particular \(k! \geq \left(\frac{k}{e}\right)^k\), to obtain
\[
\begin{split}
1 + \sum_{k=1}^\infty \frac{z^{2k}}{k!} \left (\frac{8\lambda \delta}{N} \right)^{2k} 12 r^2 \left( Cr^2\sigma^2\right)^k k^{k+1} \norm{\boldsymbol{G}_{t-1}}^k & \le \sum_{k=0}^\infty \left( z^2 \tilde{C}(r,\sigma^2) \left(\frac{4 \lambda \delta}{N}\right)^2 \norm{\boldsymbol{G}_{t-1}} \right)^k \\
& = \frac{1}{1 - z^2 \tilde{C}(r,\sigma^2) \left(\frac{8\lambda \delta}{N}\right)^2 \norm{\boldsymbol{G}_{t-1}}}
\end{split}
\]
where the last equality is valid for every \(z^2 < \frac{1}{\tilde{C}(r,\sigma^2) \left(\frac{8\lambda \delta}{N}\right)^2}\) since \(\norm{\boldsymbol{G}_{t-1}} \leq 1\). Then, since \(\frac{1}{1-x} \leq e^{2x}\) for every \(x \in [0,\frac{1}{2}]\), we have for every 
\( z^2 \leq \frac{1}{2\tilde{C}(r,\sigma^2) \left(\frac{8\lambda \delta}{N}\right)^2}\),  
\[
\mathbb{E} \left[ \exp \left(z^2\left(\boldsymbol{S}_t -\boldsymbol{S}_{t-1}\right)^2\right) \vert \mathcal{F}_{t-1}\right] \preceq \exp \left(2 z^2 \tilde{C}(r,\sigma^2) \left(\frac{8 \lambda \delta}{N}\right)^2\norm{\boldsymbol{G}_{t-1}} \boldsymbol{I}_r\right).
\]
Using again a similar argument in the proof of Propostion 2.5.2 in~\cite{vershynin2018high} along with a use of the transfer rule for matrix functions to obtain the numeric inequality \(e^x \leq x+e^{2x}\) for all \(x \in \R\) (see e.g.\ Section 2 of~\cite{tropp2012user}) gives
\[
\begin{split}
\mathbb{E} \left[\exp\left(z\left(\boldsymbol{S}_t -\boldsymbol{S}_{t-1}\right)\right) \vert \mathcal{F}_{t - 1}\right] & \preceq \mathbb{E}\left[z(\boldsymbol{S}_t -\boldsymbol{S}_{t-1})\vert \mathcal{F}_{t-1}\right]+\mathbb{E}\left[\exp\left(z^2\left(\boldsymbol{S}_t -\boldsymbol{S}_{t-1}\right)^2\right) \vert \mathcal{F}_{t-1}\right] \\
& \preceq \exp \left(2 z^2 \tilde{C}(r,\sigma^2) \left(\frac{8 \lambda \delta}{N}\right)^2\norm{\boldsymbol{G}_{t-1}} \boldsymbol{I}_r\right),
\end{split}
\]
for every \(\vert z \vert \leq \left(\frac{1}{2\tilde{C}(r,\sigma^2)(\frac{8\lambda \delta}{N})^2}\right)^{1/2}\), where we used the fact that \(\boldsymbol{S}_t\) is a \(\mathcal{F}_t\)-adapted matrix martingale. The validity of the bound for \(\vert z \vert > \left(\frac{1}{2\tilde{C}(r,\sigma^2)(\frac{8\lambda \delta}{N})^2}\right)^{1/2}\) is obtained by adapting similarly the argument of ~\cite{vershynin2018high} mentioned above, up to an absolute constant that we absorb in \(\tilde{C}(r,\sigma^2)\). Finally, for every \(t \ge 0\), we have
\[
\mathbb{E}\left[\exp\left(z(\boldsymbol{S}_t -\boldsymbol{S}_{t-1})\right)\vert \mathcal{F}_{t-1}\right] \preceq \exp\left(\tilde{C}(r,\sigma^2) z^2 \left(\frac{4\lambda\delta}{N}\right)^2 \norm{\boldsymbol{G}_{t-1}} \boldsymbol{I}_r\right),
\]
where we absorbed a factor \(4\) in the constant \(\tilde{C}(r,\sigma^2)\) for convenience. According to Lemma~\ref{lem: tropp time}, for every positive scalar-valued, continuous functions \(a\) and \(b\), we obtain
\[
\mathbb{P} \left(\exists \, t \geq 0 \colon \lambda_{\max}\left(\boldsymbol{S}_t \right) \geq a(t) \enspace \textnormal{and} \enspace \lambda_{\max} \left(\boldsymbol{Q}_t \right) \leq b(t) \right) \leq r \inf_{\theta > 0} \sup_{t \in [0,T]} \exp \left(-\theta a(t) + \tilde{C}(r,\sigma^2)\frac{\theta^2}{2} b(t)\right).
\]
The other side of the inequality is obtained in a similar fashion. To prove the second statement, we consider the optimization problem 
\[
\sup_{t \in [0,T]} -\theta a(t) + \tilde{C}(r,\sigma^2) \frac{\theta^2}{2} K a^2(t).
\]
Setting the derivative with respect to \(t\) to zero and using the assumption that \(a\) has non-vanishing derivative, we obtain 
\[
a(t) = \frac{1}{\tilde{C}(r,\sigma^2)\theta K}.
\]
By the regularity assumptions on \(a\), the above equation has a unique solution given by 
\[
t_\ast = a^{-1} \left (\frac{1}{\tilde{C}(r,\sigma^2)\theta K}\right ).
\]
We therefore deduce that for every value of \(\theta >0\), the function \(t \mapsto -\theta a(t) + \tilde{C}(r,\sigma^2) \frac{\theta^2}{2} Ka^2(t)\) attains its maximum at either \(t_{\max} = t_\ast\) or at the boundaries \(t_{\max} =0\) or \(t_{\max}=T\). In both cases, plugging the corresponding value of \(t_{\max}\) and solving in \(\theta\) yields the desired result.
\end{proof}

%%%%%%%%%%%%%%%%%%%%%%%%%%%%%%%%%%%%%%%%%%%%%%%%%%%%%%%%%%%%%%%%%%%%%%%%%%%%%%%%%%%%%%%%%%%%%%%%%%%%%%%%%%%%%%%%%%%%%%%%%%%
\section{Proofs of full recovery of spikes} \label{section: proof non-isotropic SGD}

In this section, we present the proofs of Theorems~\ref{thm: strong recovery online p>2 nonasymptotic} and~\ref{thm: strong recovery online p=2 nonasymptotic}. We first address the case \(p \ge 3\), followed by \(p=2\) with SNRs separated by constants of order \(1\). The proof strategy for both cases is similar: we partition the overall recovery event of Theorems~\ref{thm: strong recovery online p>2 nonasymptotic} and~\ref{thm: strong recovery online p=2 nonasymptotic} into a union of recovery events corresponding to each direction, introducing suitable stopping times and applying the strong Markov property to combine them. At each stage, the procedure recovers one spike while ensuring the stability of the remaining ones, thereby establishing the sequential elimination phenomenon introduced in Definition~\ref{def: sequential elimination}.

%%%%%%%%%%%%%%%%%%%%%%%%%%%%%%%%%%%%%%%%%%%%%%%%%%%%%%%%%%%%%%%%%%%%%%%%%%%%%
\subsection{Proofs for \(p \ge 3\)}

In this subsection, we focus on the case \(p \geq 3\). We assume that all correlations are positive at initialization, that is, \(m_{ij}(\boldsymbol{X}_0) >0\). Let \((i_1^\ast, j_1^\ast), \ldots, (i_r^\ast, j_r^\ast)\) denote the greedy maximum selection of \(\boldsymbol{I}_0 = \left (\lambda_i \lambda_j m_{ij}(\boldsymbol{X}_0)^{p-2} \right)_{1 \le i, j \le r}\), as given in Definition~\ref{def: greedy operation}. To establish recovery of multiple directions (Theorem~\ref{thm: strong recovery online p>2 nonasymptotic}), we decompose the argument into a sequence of intermediate events. For every \(\varepsilon > 0\), define the sets
\[
\begin{split}
E_1 (\varepsilon) & = R_1 (\varepsilon) \cap \left \{ \boldsymbol{X} \colon m_{ij}(\boldsymbol{X}) \in \Theta(N^{-\frac{1}{2}}) \enspace \forall \, i \neq i_1^\ast , j\neq j_1^\ast \right \},\\
E_2 (\varepsilon) & = R_1(\varepsilon) \cap R_2(\varepsilon) \cap \left \{ \boldsymbol{X} \colon m_{ij}(\boldsymbol{X}) \in \Theta(N^{-\frac{1}{2}}) \enspace \text{for} \enspace i \neq i_1^\ast, i_2^\ast \, \textnormal{and} \, j\neq j_1^\ast, j_2^\ast \right \},\\
& \vdots \\
E_{r-1} (\varepsilon) & = \cap_{1 \leq i \leq r-1} R_i(\varepsilon) \cap \left \{ \boldsymbol{X} \colon m_{i_r^\ast j_r^\ast}(\boldsymbol{X}) \in \Theta(N^{-\frac{1}{2}})\right \},\\
E_r (\varepsilon) & = \cap_{1 \leq i \leq r-1} R_i(\varepsilon) \cap \left \{ \boldsymbol{X} \colon m_{i_r^\ast j_r^\ast}(\boldsymbol{X}) \ge 1-\varepsilon \right \},
\end{split}
\]
where \(R_k(\varepsilon)\) is the strong recovery event for the spike \(\boldsymbol{v}_{i_k^\ast}\):
\[
R_k(\varepsilon) = \left \{ \boldsymbol{X} \colon m_{i_k^\ast j_k^\ast}(\boldsymbol{X}) \geq 1-\varepsilon \enspace \textnormal{and} \enspace m_{i_k^\ast j}(\boldsymbol{X}),m_{i j_k^\ast}(\boldsymbol{X}) \lesssim \log(N)^{-\frac{1}{2}}N^{-\frac{p-1}{4}} \enspace \forall \, i\neq i_k^\ast, j \neq j_k^\ast \right \}.
\]
By construction, reaching \(E_r(\varepsilon)\) corresponds to recovering a permutation of all spikes. We remind that we denote by \(\mathcal{T}_A = \inf \{ t \in \N_0 \colon \boldsymbol{X}_t \in A\}\) the hitting time of a set \(A\).

\begin{lem} \label{lem: recovery first spike SGD}
For every \(\gamma_1, \gamma_2, \gamma >0\) with \(\gamma_1 > \gamma_2, \gamma_1 > \gamma\), and for every \(\varepsilon >0\), there exists a constant \(0 < c_0 < \frac{\left( 1 + \frac{\gamma}{\gamma_1}\right)^{\frac{1}{p-1}}-1}{\left( 1 + \frac{\gamma}{\gamma_1}\right)^{\frac{1}{p-1}}+1}\) and a sequence \(d_0 = d_0(N)\) satisfying~\eqref{eq: sequence d_0}, such that if
\begin{equation} \label{eq: step size p>2 SGD}
\delta = C_\delta d_0 \log(N)^{-1} N^{-\frac{p-3}{2}},
\end{equation}
for some constant \(C_\delta > 0\), then for \(N\) sufficiently large,
\[
\inf_{\boldsymbol{X}_0 \in \mathcal{C}_1(\gamma_1,\gamma_2) \cap \mathcal{C}_2 (\gamma_1,\gamma)} \mathbb{P}_{\boldsymbol{X}_0^+} \left( \mathcal{T}_{E_1(\varepsilon)} \lesssim T_1 \right) \geq 1 - \eta,
\]
where \(T_1\) and \(\eta\) are given by 
\[
T_1 = \frac{ 1 - \left(\frac{(1-c_0) \gamma_1}{\varepsilon \sqrt{N}} \right)^{p-2} }{C_\delta (1-c_0)^{p-2} (1-c_0-\varepsilon^2) d_0 p \lambda_r^2 \gamma_2^{p-2}} \log(N) N^{p-2} +\frac{1}{C_\delta d_0 \varepsilon^p p \lambda_r^2} \log(N)^2 N^{\frac{p-2}{2}},
\]
and 
\[
\begin{split}
\eta &= K_1 \frac{\log(N)^2 N^{p-2}}{d_0} e^{-c_1 N} + K_2 \log(N) e^{-c_2 c_0^2 \log(N)/d_0} + K_3 \frac{d_0^2}{\log(N) N^{\frac{p-2}{2}}} e^{-c_3/d_0^2} .
\end{split}
\]
\end{lem}

Once the first spike (up to a permutation) has been recovered, subsequent spikes can be recovered sequentially, one after another.

\begin{lem} \label{lem: recovery subsequent spikes SGD}
Consider the step size \(\delta = C_\delta d_0 \log(N)^{-1} N^{- \frac{p-3}{2}} \) from Lemma~\ref{lem: recovery first spike SGD}, with \(d_0 = d_0(N)\) satisfying~\eqref{eq: sequence d_0}. Then, for every \(\varepsilon >0\) and every \(2 \le k \le r-1\), there exist \(c_0 \in (0,\frac{1}{2})\) and \(T_k = \Theta (d_0^{-1} \log(N)N^{p-2})\) such that, for \(N\) sufficiently large,
\[
\inf_{\boldsymbol{X} \in E_{k-1}(\varepsilon)} \mathbb{P}_{\boldsymbol{X}} \left( \mathcal{T}_{E_k(\varepsilon)} \lesssim T_k \right) \geq 1 - K \eta,
\]
where \(\eta\) is as in Lemma~\ref{lem: recovery first spike SGD}.
\end{lem}

The proof of Theorem~\ref{thm: strong recovery online p>2 nonasymptotic} follows by combining Lemmas~\ref{lem: recovery first spike SGD} and~\ref{lem: recovery subsequent spikes SGD} through an application of the strong Markov property.

\begin{proof} [\textbf{Proof of Theorem~\ref{thm: strong recovery online p>2 nonasymptotic}}]
For every \(1 \le k \le r\), denote the hitting times \(\tau_k \coloneqq \mathcal{T}_{E_k (\varepsilon)}\) and \(\tau_0=0\), and introduce the events
\[
A_k \coloneqq \left \{ \tau_k - \tau_{k-1} \lesssim T_k \right \},
\]
where \(T_k\) are given in Lemmas~\ref{lem: recovery first spike SGD} and~\ref{lem: recovery subsequent spikes SGD}. Our goal is to bound from below \(\mathbb{P}_{\boldsymbol{X}_0^+} \left ( \tau_r \lesssim T \right)\) with \(T = \sum_{k=1}^r T_k\), for any initialization \(\boldsymbol{X}_0 \in \mathcal{C}_1 (\gamma_1, \gamma_2) \cap \mathcal{C}_2 (\gamma_1, \gamma)\). By construction, the event \(\{\tau_r \lesssim T\}\) is equivalent to the joint occurrence of all \(A_k\), hence
\[
\mathbb{P}_{\boldsymbol{X}_0^+} \left ( \tau_r \lesssim T \right) = \mathbb{P}_{\boldsymbol{X}_0^+} \left ( \cap_{k=1}^r A_k \right ).
\]
Applying the strong Markov property at time \(\tau_1\), we obtain
\[
\begin{split}
\mathbb{P}_{\boldsymbol{X}_0^+} \left ( \cap_{k=1}^r A_k \right ) &= \mathbb{P}_{\boldsymbol{X}_0^+} \left ( A_1 \right ) \E \left [ \mathbb{P}_{\boldsymbol{X}_{\tau_1}} \left (\cap_{k=2}^r A_k \right) | A_1 \right ] \ge \mathbb{P}_{\boldsymbol{X}_0^+}  (A_1) \inf_{\boldsymbol{X} \in E_1 (\varepsilon)} \mathbb{P}_{\boldsymbol{X}} \left ( \cap_{k=2}^r A_k \right ),
\end{split}
\]
since conditioning on \(A_1\), \(\boldsymbol{X}_{\tau_1} \in E_1 (\varepsilon)\). Repeating this argument inductively for \(k=2, \ldots, r\) yields
\[
\mathbb{P}_{\boldsymbol{X}_0^+} \left ( \cap_{k=1}^r A_k \right ) \ge \mathbb{P}_{\boldsymbol{X}_0^+} \left ( A_1 \right) \prod_{k=2}^r \inf_{\boldsymbol{X} \in E_{k-1}(\varepsilon)} \mathbb{P}_{\boldsymbol{X}} \left (A_k  \right ).
\]
By Lemma~\ref{lem: recovery first spike SGD} we have, for the chosen \(T_1\), 
\[
\inf_{\boldsymbol{X}_0 \in \mathcal{C}_1 (\gamma_1, \gamma_2) \cap \mathcal{C}_2 (\gamma_1, \gamma)} \mathbb{P}_{\boldsymbol{X}_0^+} \left (A_1 \right) \ge 1 - \eta,
\]
where \(\eta\) is given in the statement of Lemma~\ref{lem: recovery first spike SGD}. By Lemma~\ref{lem: recovery subsequent spikes SGD}, we have for every \(k \ge 2\),
\[
\inf_{\boldsymbol{X} \in E_{k-1} (\varepsilon) } \mathbb{P}_{\boldsymbol{X}} \left (A_k \right) \ge 1 -  K \eta.
\]
Combining these inequalities gives 
\[
\inf_{\boldsymbol{X}_0 \in \mathcal{C}_1 (\gamma_1, \gamma_2) \cap \mathcal{C}_2 (\gamma_1, \gamma)} \mathbb{P}_{\boldsymbol{X}_0^+} \left ( \tau_r \lesssim \sum_{k=1}^r T_k \right) \ge (1-\eta) (1-K \eta)^{r-1} \ge 1 - (1 + (r-1)K) \eta,
\]
where the last inequality follows from \((1-x)^n \ge 1-nx\) for \(x \in [0,1]\). This completes the proof of Theorem~\ref{thm: strong recovery online p>2 nonasymptotic}.
\end{proof}

In the remainder of this subsection, we provide the proof of Lemma~\ref{lem: recovery first spike SGD}. Lemma~\ref{lem: recovery subsequent spikes SGD} can be proved by following the same strategy as in the proof of Lemma~\ref{lem: recovery first spike SGD}, with the dominant correlation given by \(m_{i_k^\ast j_k^\ast}\) instead of \(m_{i_1^\ast j_1^\ast}\). Since it is a straightforward adaptation, we omit the details here and focus on the proof of Lemma~\ref{lem: recovery first spike SGD}. We also refer the reader to the proof of~\cite[Lemma 5.4]{langevin}, which provides the analogous argument under Langevin dynamics.

\begin{proof}[\textbf{Proof of Lemma~\ref{lem: recovery first spike SGD}}]
By assumption, \(\boldsymbol{X}_0\) satisfies Condition \(1\) and Condition \(2\), and \(m_{ij}(\boldsymbol{X}_0) >0\) for every \(i,j \in [r]\). Condition \(1\) then implies that for every \(i,j \in [r]\), there exists \(\gamma_{ij} \in (\gamma_2,\gamma_1)\) such that \(m_{ij}(0) = \gamma_{ij} N^{-\frac{1}{2}}\) (here \(t=0\) corresponds to \(\boldsymbol{X}_0\)). Condition \(2\) gives that for every \((i,j) \neq (i_1^\ast, j_1^\ast)\),
\begin{equation} \label{eq: condition 2 SGD}
\lambda_{i_1^\ast} \lambda_{j_1^\ast} \gamma_{i_1^\ast j_1^\ast}^{p-2} = \max_{1 \le i,j \le r} \{\lambda_i \lambda_j \gamma_{ij}^{p-2}\}> \left ( 1 + \frac{\gamma}{\gamma_1} \right ) \lambda_i \lambda_j \gamma_{ij}^{p-2}.
\end{equation}
Let \(T_0 >0\) be a time horizon to be chosen later. Applying Proposition~\ref{prop: inequalities SGD} with the step size \(\delta\) given in~\eqref{eq: step size p>2 SGD} and the time horizon \(T_0\), we obtain for every \(t \le T_0\),
\[
\mathbb{P}_{\boldsymbol{X}_0^+} \left( \ell_{ij}(t) \leq m_{ij}(t) \leq u_{ij}(t) \right) \geq 1 - \eta (\delta, c_0, T_0, K_N),
\]
where \(\eta (\delta, c_0, T_0, K_N)\) is given in~\eqref{eq: eta}. The comparison functions \(\ell_{ij}\) and \(u_{ij}\) are given by
\[
\begin{split}
\ell_{ij}(t) & = (1-c_0) \frac{\gamma_{ij}}{\sqrt{N}} + \frac{\delta}{N} \sum_{\ell=1}^t \left (- \langle \boldsymbol{v}_i, (\nabla_{\textnormal{St}} \Phi (\boldsymbol{X}_{\ell-1}))_j\rangle  - h_i(\ell-1)\right )\\
u_{ij}(t) & = (1+c_0) \frac{\gamma_{ij}}{\sqrt{N}} + \frac{\delta}{N} \sum_{\ell=1}^t \left (- \langle \boldsymbol{v}_i, (\nabla_{\textnormal{St}} \Phi (\boldsymbol{X}_{\ell-1}))_j\rangle  + h_i(\ell-1)\right )\\
\end{split}
\]
where 
\begin{equation} \label{eq: correlation exact computation}
- \langle \boldsymbol{v}_i ,(\nabla_{\textnormal{St}}\Phi(\boldsymbol{X}_t))_j \rangle =  \sqrt{N} p \lambda_i \lambda_j m_{ij}^{p-1}(t) - \sqrt{N}\frac{p}{2}  \sum_{1 \le k, \ell \le r} \lambda_k m_{i \ell}(t) m_{kj} (t) m_{k\ell} (t) \left ( \lambda_j m_{k j}^{p-2} (t) + \lambda_\ell m_{k \ell}^{p-2}(t) \right),
\end{equation}
and
\begin{equation} \label{eq: h_i exact computation}
h_i(t) = \frac{\delta \alpha_2}{2} \sum_{k=1}^r |m_{ik}(t)| + \frac{\delta^2 \alpha_2}{2N} \sum_{k=1}^r \left | \langle \boldsymbol{v}_i ,(\nabla_{\textnormal{St}}\Phi(\boldsymbol{X}_t))_k \rangle \right |+ \frac{\delta^2 \alpha_2 r K_N}{2N}.
\end{equation}

We introduce the microscopic thresholds \(\varepsilon_N^+ = \gamma_+ / \sqrt{N}\) and \(\varepsilon_N^- = \gamma_{-} / \sqrt{N}\), for constants \(\gamma_+ > \gamma_1\) and \(\gamma_- < \gamma_2\). Our proof proceeds in three steps, via the intermediate events
\[
E_0 = \left \{ \boldsymbol{X} \colon \ m_{i_1^\ast j_1^\ast}(\boldsymbol{X}) \ge \varepsilon_N^+ \enspace \textnormal{and} \enspace
m_{i_1^\ast j_1^\ast}(\boldsymbol{X}) \ge 2\max_{(i,j)\ne(i_1^\ast,j_1^\ast)} m_{ij}(\boldsymbol{X}) \right \},
\]
and
\[
\begin{split}
E (\varepsilon)  = \Big \{\boldsymbol{X} &\colon m_{i_1^\ast j_1^\ast}(\boldsymbol{X}) \geq \varepsilon \enspace \textnormal{and} \\
& \enspace \varepsilon_N^- \lesssim m_{ij}(\boldsymbol{X}) \lesssim \varepsilon_N^+, m_{i j_1^\ast} (\boldsymbol{X}) \lesssim \varepsilon_N^+, m_{i_1^\ast j} (\boldsymbol{X}) \lesssim \varepsilon_N^+ \enspace \textnormal{for all} \enspace  i \neq i_1^\ast, j\neq j_1^\ast \Big\}.
\end{split}
\]
Let \(\mathcal{T}_{E_0}\) and \(\mathcal{T}_{E(\varepsilon)}\) denote the corresponding hitting times. We claim that 
\begin{equation} \label{eq: step 1 p>2}
\inf_{\boldsymbol{X}_0 \in \mathcal{C}_1(\gamma_1,\gamma_2) \cap \mathcal{C}_2 (\gamma_1,\gamma)} \mathbb{P}_{\boldsymbol{X}_0^+} \left ( \mathcal{T}_{E_0} \leq \frac{1 - \left ( \frac{(1-c_0) \gamma_{i_1^\ast j_1^\ast}}{\gamma_+}\right)^{p-2}}{C_\delta d_0 (1-c_0)^{p-1}  p \lambda_{i_1^\ast}  \lambda_{j_1^\ast} \gamma_{i_1^\ast j_1^\ast}^{p-2}} \log(N) N^{p-2} \right ) \geq 1 - \tilde{\eta}_1,
\end{equation}
\begin{equation} \label{eq: step 2 p>2}
\inf_{\boldsymbol{X} \in E_0} \mathbb{P}_{\boldsymbol{X}} \left ( \mathcal{T}_{E(\varepsilon)} \leq \frac{1 - \left ( \frac{(1-c_0) \gamma_+}{\varepsilon \sqrt{N}}\right)^{p-2}}{C_\delta d_0 (1-c_0)^{p-2} (1-c_0-\varepsilon^2) p \lambda_{i_1^\ast} \lambda_{j_1^\ast} \gamma_+^{p-2}} \log(N) N^{p-2} \right ) \geq 1 - \log(N) \tilde{\eta}_1,
\end{equation}
where \(\tilde{\eta}_1 = K_1 \log(N) d_0^{-1} N^{p-2} e^{-c_1N} + K_2 e^{-c_2 c_0^2 \log(N)/d_0} + K_3 d_0^2 \log(N)^{-2} N^{- \frac{p-2}{2}} e^{-c_3/d_0^2}\), and
\begin{equation} \label{eq: step 3 p>2}
\inf_{\boldsymbol{X} \in E(\varepsilon)} \mathbb{P}_{\boldsymbol{X}} \left ( \mathcal{T}_{E_1(\varepsilon)} \lesssim \frac{\log(N)^2 N^{\frac{p-2}{2}}}{C_\delta d_0 \varepsilon^p p \lambda_r^2} \right ) \ge 1 - \tilde{\eta}_2,
\end{equation}
with \(\tilde{\eta}_2 = K_4 \frac{\log(N)^2N^{p-2}}{\varepsilon^pd_0}e^{-c_4 N} + K_5 e^{-c_5 c_0^2 \varepsilon^{p+2} N^{p/2}/d_0} + K_6 d_0 \log(N)^{-1} N^{- \frac{p-1}{2}} e^{- c_6 \log(N)^2 N^{p-1}/d_0^2}\). Combining~\eqref{eq: step 1 p>2},~\eqref{eq: step 2 p>2}, and~\eqref{eq: step 3 p>2} via the strong Markov property at the stopping times \(\mathcal{T}_{E_0}\) and \(\mathcal{T}_{E(\varepsilon)}\) (cf.\ the proof of Theorem~\ref{thm: strong recovery online p>2 nonasymptotic}) yields the desired result. Note that the first term in \(\tilde{\eta}_2\) is absorbed by the first term in \(\log(N) \tilde{\eta}_1\), while the second and third terms in \(\tilde{\eta}_2\) are dominated by the second term of \(\log(N) \tilde{\eta}_1\).

We now prove the three claims. Throughout, for every \(\epsilon >0\) and \(i,j \in [r]\), we define the hitting times \(\mathcal{T}_\epsilon^{(ij,+)}\) and \(\mathcal{T}_\epsilon^{(ij,-)}\) by
\begin{equation} \label{eq: hitting times}
\begin{split}
\mathcal{T}_\epsilon^{(ij,+)} & = \inf \{t \in \N_0 \colon m_{ij}(t) \ge \epsilon \},\\
\mathcal{T}_\epsilon^{(ij,-)} & = \inf \{t \in \N_0 \colon m_{ij}(t) \le \epsilon \}.
\end{split}
\end{equation}

\begin{proof}[Proof of~\eqref{eq: step 1 p>2}]
According to~\eqref{eq: correlation exact computation}, for every \(t\leq \min_{1 \leq i,j \leq r} \mathcal{T}_{\varepsilon_N^+}^{(ij,+)} \wedge \min_{1 \leq i,j \leq r} \mathcal{T}_{\varepsilon_N^-}^{(ij,-)}\) and every \(i,j \in [r]\),
\[
\begin{split}
- \langle \boldsymbol{v}_i ,(\nabla_{\textnormal{St}}\Phi(\boldsymbol{X}_t))_j \rangle  & \le  \sqrt{N} p \lambda_i \lambda_j m_{ij}^{p-1}(t) - \sqrt{N} p \lambda_r^2 r^2 \left(\frac{\gamma_{-}}{\sqrt{N}} \right)^{p+1}, \\
- \langle \boldsymbol{v}_i ,(\nabla_{\textnormal{St}}\Phi(\boldsymbol{X}_t))_j \rangle & \geq \sqrt{N} p\lambda_i\lambda_j m_{ij}^{p-1}(t) - \sqrt{N} p \lambda_1^2 r^2 \left(\frac{\gamma_+}{\sqrt{N}}\right)^{p+1},
\end{split}
\]
and 
\[
h_i(t) \le \frac{\delta \alpha_2}{2} \sum_{k=1}^r |m_{ik}(t)| + \frac{\delta^2 \alpha_2}{2N} \left ( \sqrt{N} p r \lambda_i \lambda_1 \left ( \frac{\gamma_+}{\sqrt{N}}\right)^{p-1} + \sqrt{N} p r^3 \lambda_1^2 \left ( \frac{\gamma_+}{\sqrt{N}}\right)^{p+1} \right) + \frac{\delta^2 r \alpha_2K_N}{2N}.
\]
This implies the following estimates for \(u_{ij}(t)\) and \(\ell_{ij}(t)\):
\begin{equation} \label{eq: upper bound u_ij p>2 SGD}
\begin{split}
u_{ij}(t) & \leq (1+c_0) \frac{\gamma_{ij}}{\sqrt{N}} + \frac{\delta}{N} \sum_{\ell=1}^t \Big \{ \sqrt{N} p\lambda_i\lambda_j m_{ij}^{p-1}(\ell-1) - \sqrt{N} p \lambda_r^2 r^2 \left(\frac{\gamma_{-}}{\sqrt{N}}\right)^{p+1}  \\
& \enspace + \frac{\delta \alpha_2}{2} \sum_{k=1}^r |m_{ik}(\ell-1)| + \frac{\delta^2 \alpha_2}{2N} \left ( \sqrt{N} p r \lambda_i \lambda_1 \left ( \frac{\gamma_+}{\sqrt{N}}\right)^{p-1} + \sqrt{N} p r^3 \lambda_1^2 \left ( \frac{\gamma_+}{\sqrt{N}}\right)^{p+1} \right) + \frac{\delta^2 r \alpha_2 K_N}{2N} \Big \},
\end{split}
\end{equation}
and
\begin{equation} \label{eq: lower bound l_ij p>2 SGD}
\begin{split}
\ell_{ij}(t) & \geq (1-c_0) \frac{\gamma_{ij}}{\sqrt{N}} + \frac{\delta}{N} \sum_{\ell=1}^t \Big \{ \sqrt{N} p\lambda_i\lambda_j m_{ij}^{p-1}(\ell-1) - \sqrt{N} p \lambda_1^2 r^2 \left(\frac{\gamma_+}{\sqrt{N}}\right)^{p+1}  \\
& \enspace - \frac{\delta \alpha_2}{2} \sum_{k=1}^r |m_{ik}(\ell-1)| - \frac{\delta^2 \alpha_2}{2N} \left ( \sqrt{N} p r\lambda_i \lambda_1 \left ( \frac{\gamma_+}{\sqrt{N}}\right)^{p-1} + \sqrt{N} p r^3 \lambda_1^2 \left ( \frac{\gamma_+}{\sqrt{N}}\right)^{p+1} \right) - \frac{\delta^2 r \alpha_2 K_N}{2N} \Big \},
\end{split}
\end{equation}
for every \(t \leq \min_{1 \leq i,j \leq r} \mathcal{T}_{\varepsilon_N^+}^{(ij,+)} \wedge \min_{1 \leq i,j \leq r} \mathcal{T}_{\varepsilon_N^-}^{(ij,-)} \wedge T_0\), with \(\mathbb{P}_{\boldsymbol{X}_0^+}\)-probability at least \(1 -\eta(\delta, c_0, T_0,K_N)\). We now consider~\eqref{eq: lower bound l_ij p>2 SGD} and claim that  
\begin{equation} \label{eq: claim l_ij p>2 SGD}
\ell_{ij}(t) \geq (1-c_0) \frac{\gamma_{ij}}{\sqrt{N}} + \frac{\delta}{\sqrt{N}} (1-c_0) p \lambda_i \lambda_j \sum_{\ell=1}^t m_{ij}^{p-1}(\ell-1),
\end{equation}
for every \(t \leq \min_{1 \leq i,j \leq r} \mathcal{T}_{\varepsilon^{+}_N}^{(ij,+)} \wedge \min_{1 \leq i,j \leq r}\mathcal{T}_{\varepsilon^{-}_N}^{(ij,-)} \wedge T_0\). To see~\eqref{eq: claim l_ij p>2 SGD}, we first note that, over the considered time interval, 
\[
\ell_{ij}(t) \geq (1-c_0) \frac{\gamma_{ij}}{\sqrt{N}} + \frac{\delta}{N} \sum_{\ell=1}^t \left (\sqrt{N} p \lambda_i \lambda_j m_{ij}^{p-1}(\ell-1) - 3 \frac{\delta \alpha_2}{2} \sum_{k=1}^r \vert m_{ik}(\ell-1) \vert - \frac{\delta^2 r\alpha_2  K_N}{2N} \right ),
\]
where we used the fact that \(\delta = C_\delta d_0 N^{-\frac{p-3}{2}} \log(N)^{-1}\) by assumption~\eqref{eq: step size p>2 SGD}. Claim~\eqref{eq: claim l_ij p>2 SGD} then follows by the following two sufficient conditions. First, 
\[
\frac{c_0}{2} \sqrt{N} p \lambda_i \lambda_j m_{ij}^{p-1}(t) \ge \frac{c_0}{2} p \lambda_r^2 \frac{\gamma_-^{p-1}}{N^{\frac{p-2}{2}}} \ge \frac{3 \delta r \alpha_2}{2} \frac{\gamma_+}{\sqrt{N}} \ge \frac{3 \delta \alpha_2}{2} \sum_{k=1}^r |m_{ik}(t)|,
\]
for every \(t \leq \min_{1 \leq i,j \leq r} \mathcal{T}_{\varepsilon^{+}_N}^{(ij,+)} \wedge \min_{1 \leq i,j \leq r}\mathcal{T}_{\varepsilon^{-}_N}^{(ij,-)}\), provided \(d_0 \le \frac{c_0 p \lambda_r^2 \gamma_-^{p-1}}{3 \alpha_2 C_\delta r \gamma_+} \log(N)\) which certainly holds since \(d_0 \ll c_0 \log(N)\). Second, 
\[
\frac{c_0}{2} \sqrt{N} p \lambda_i \lambda_j m_{ij}^{p-1}(t) \ge \frac{c_0}{2} p \lambda_r^2 \frac{\gamma_-^{p-1}}{N^{\frac{p-2}{2}}} \ge \frac{\delta^2 r \alpha_2 K_N}{2N} ,
\]
for \(t \leq \min_{1 \leq i,j \leq r} \mathcal{T}_{\varepsilon^{+}_N}^{(ij,+)} \wedge \min_{1 \leq i,j \leq r}\mathcal{T}_{\varepsilon^{-}_N}^{(ij,-)}\), provided \(d_0^2 K_N \le \frac{c_0 p \lambda_r^2 \gamma_-^{p-1}}{C_\delta^2 \alpha_2 r} \log(N)^2 N^{\frac{p-2}{2}}\). We therefore can choose any value \(K_N  = o\left(N^{\frac{p-2}{2}}\right)\). Since the convergence speed from Proposition~\ref{prop: inequalities SGD} given in~\eqref{eq: eta} is primarily determined by the ratio \(\frac{c_0^2}{d_0}\), we can simply pick \(K_N  = \frac{1}{d_0}\) where \(d_0^{-1}\) grows at most polynomially in \(N\). This proves claim~\eqref{eq: claim l_ij p>2 SGD}. A similar reasoning for~\eqref{eq: upper bound u_ij p>2 SGD} yields 
\begin{equation} \label{eq: claim u_ij p>2 SGD}
u_{ij}(t) \leq (1+c_0) \frac{\gamma_{ij}}{\sqrt{N}}  + \frac{\delta}{\sqrt{N}}  (1+c_0) p \lambda_i \lambda_j \sum_{\ell=1}^t m_{ij}^{p-1}(\ell-1),
\end{equation}
on the same time interval. Since both comparison functions are increasing on the stopping window, it follows that \(\min_{1 \le i,j \le r} \mathcal{T}_{\varepsilon_N^+}^{(ij,+)} \le \min_{1 \le i,j \le r} \mathcal{T}_{\varepsilon_N^-}^{(ij,-)}\). Applying item (b) of Lemma~\ref{lem: discrete growth} to~\eqref{eq: claim l_ij p>2 SGD} and~\eqref{eq: claim u_ij p>2 SGD}, we obtain 
\[
\frac{(1-c_0)\gamma_{ij}}{\sqrt{N}\left(1-\delta(1-c_0)^{p-1}p\lambda_i\lambda_j\gamma_{ij}^{p-2}N^{-\frac{p-1}{2}}t\right)^{\frac{1}{p-2}}} \leq m_{ij}(t) \leq \frac{(1+c_0)\gamma_{ij}}{\sqrt{N}\left(1-\delta(1+c_0)^{p-1}p\lambda_i\lambda_j\gamma_{ij}^{p-2}N^{-\frac{p-1}{2}}t\right)^{\frac{1}{p-2}}},
\]
for all \(t \leq \min_{1 \leq i,j \leq r} \mathcal{T}_{\varepsilon^{+}_N}^{(ij,+)} \wedge T_0\). From these inequalities we can deduce explicit bounds on the hitting time \(\mathcal{T}_\varepsilon ^{(ij,+)}\) for every \(\varepsilon \leq \varepsilon_N^+\):
\[
T_{\ell,\varepsilon}^{(ij)} \leq \mathcal{T}_\varepsilon ^{(ij,+)} \leq T_{u,\varepsilon}^{(ij)} ,
\]
where we define
\[
\begin{split}
T_{u,\varepsilon}^{(ij)} &= \frac{1-\left(\frac{(1-c_0)\gamma_{ij}}{\varepsilon\sqrt{N}}\right)^{p-2}}{\delta(1-c_0)^{p-1}p\lambda_i\lambda_j\gamma_{ij}^{p-2}}N^{\frac{p-1}{2}} ,\\
T_{\ell,\varepsilon}^{(ij)} &= \frac{1-\left(\frac{(1+c_0)\gamma_{ij}}{\varepsilon\sqrt{N}}\right)^{p-2}}{\delta(1+c_0)^{p-1}p\lambda_i\lambda_j\gamma_{ij}^{p-2}}N^{\frac{p-1}{2}}.
\end{split}
\]

Let \(c_0 < \frac{\left( 1 + \frac{\gamma}{\gamma_1}\right)^{\frac{1}{p-1}} -1}{\left( 1 + \frac{\gamma}{\gamma_1}\right)^{\frac{1}{p-1}}+1}\) and \(\gamma_+ > \gamma_1\) be a constant of order \(1\) satisfying
\[
\gamma_+^{p-2} \ge (1+c_0)^{p-2} \gamma_1^{p-2} + \frac{(1+c_0)^{p-2} \gamma_1^{p-2} - (1-c_0)^{p-2} \gamma_2^{p-2}}{\frac{(1-c_0)^{p-1}}{(1+c_0)^{p-1}} \left( 1 + \frac{\gamma}{\gamma_1}\right) -1}.
\]
Then, we see that \(T_{u,\varepsilon_N^+}^{(i_1^\ast j_1^\ast)} \wedge T_0 \leq \min_{(i,j) \neq (i_1^\ast,j_1^\ast)} T_{\ell,\varepsilon_N^+}^{(ij)}\). That is, the correlation \(m_{i_1^\ast j_1^\ast}\) exceeds the microscopic threshold \(\varepsilon_N^+\) first, i.e., \(\min_{1 \le i,j \le r} \mathcal{T}_{\varepsilon_N^+}^{(ij,+)} = \mathcal{T}_{\varepsilon_N^+}^{(i_1^\ast j_1^\ast,+)}\). We can choose \(\gamma_+\) such that \(m_{ij}( \mathcal{T}_{\varepsilon_N^+}^{(i_1^\ast j_1^\ast,+)}) \le \frac{1}{2} m_{i_1^\ast j_1^\ast}( \mathcal{T}_{\varepsilon_N^+}^{(i_1^\ast j_1^\ast,+)} ) = \frac{1}{2} \varepsilon_N^+\) for every \((i,j) \neq (i_1^\ast, j_1^\ast)\). This shows that \(\mathcal{T}_{E_0} = \mathcal{T}_{\varepsilon_N^+}^{(i_1^\ast j_1^\ast, +)}\) and 
\[
\mathcal{T}_{\varepsilon_N^+}^{(i_1^\ast j_1^\ast, +)} \le T_{u, \varepsilon_N^+}^{(i_1^\ast j_1^\ast)} \le \frac{1-\left(\frac{(1-c_0)\gamma_{i_1^\ast j_1^\ast}}{\gamma_+}\right)^{p-2}}{C_\delta d_0 (1-c_0)^{p-1} p  \lambda_{i_1^\ast}  \lambda_{j_1^\ast} \gamma_{i_1^\ast j_1^\ast}^{p-2}} \log(N) N^{p-2},
\]
with \(\mathbb{P}_{\boldsymbol{X}_0^+}\)-probability at least \(1 - \tilde{\eta}\), where we substituted the prescribed value~\eqref{eq: step size p>2 SGD} for \(\delta\), as well as \(K_N\) and \(T_0 =  T_{\ell, \varepsilon_N^+}^{(i_1^\ast j_1^\ast)}\), into \(\eta(\delta, c_0, T_0,K_N)\).
\end{proof}

\begin{proof}[Proof of~\eqref{eq: step 2 p>2}] 
Assume that \(\boldsymbol{X}_0 \in E_0\). We now study the evolution of the correlations over the time interval 
\[
I = [0, \mathcal{T}_\varepsilon^{(i_1^\ast j_1^\ast, +)} \wedge \mathcal{T}_{\varepsilon_N^{-}}^{(i_1^\ast j_1^\ast,-)} \wedge \min_{(i,j) \neq (i_1^\ast, j_1^\ast)} \mathcal{T}_{\varepsilon_N^+}^{(ij,+)} \wedge \min_{i \neq i_1^\ast, j \neq j_1^\ast} \mathcal{T}_{\varepsilon_N^{-}}^{(ij,-)} \wedge \mathcal{T}_{\textnormal{bad}} \wedge T_0], 
\]
where \(T_0 > 0\) is a time horizon to be fixed later, and
\[
\mathcal{T}_{\textnormal{bad}} = \inf \left \{ t \in \N_0 \colon \max_{1 \leq i,j \leq r \atop (i,j) \neq (i_1^\ast, j_1^\ast)} m_{ij}(t) > \frac{1}{2} m_{i_1^\ast j_1^\ast }(t) \right\}.
\]
Since \(\boldsymbol X_0\in E_0\) implies that \(m_{ij}(0)\le \frac{1}{2} m_{i_1^\ast j_1^\ast}(0)\) for all \((i,j)\neq(i_1^\ast,j_1^\ast)\), it follows that \(\mathcal{T}_{\textnormal{bad}}>0\). By Proposition~\ref{prop: inequalities SGD}, for all \(t \in I\), we have
\[
\begin{split}
m_{ij}(t) & \geq (1 - c_0) m_{ij}(0) + \frac{\delta}{N} \sum_{\ell=1}^t  \left ( - \langle \boldsymbol{v}_i, \left (\nabla_{\textnormal{St}}\Phi(\boldsymbol{X}_{\ell-1}) \right )_j \rangle - h_i(\ell-1) \right ) ,\\
m_{ij}(t) &\leq (1 +c_0) m_{ij}(0) + \frac{\delta}{N} \sum_{\ell=1}^t  \left ( - \langle \boldsymbol{v}_i, \left (\nabla_{\textnormal{St}}\Phi(\boldsymbol{X}_{\ell-1}) \right )_j \rangle +  h_i(\ell-1)  \right ),
\end{split}
\] 
with probability at least \(1-\eta (\delta, c_0, T_0, K_N)\), where \(\eta\) is defined in~\eqref{eq: eta}. The constant \(c_0 \in (0, \frac{1}{2})\) is the same as in the previous step and is independent of \(N\), and chosen sufficiently small to satisfy all preceding constraints. Recall that \(- \langle \boldsymbol{v}_i, \left (\nabla_{\textnormal{St}}\Phi(\boldsymbol{X}_t) \right )_j \rangle\) and \(h_i(t)\) are given in~\eqref{eq: correlation exact computation} and~\eqref{eq: h_i exact computation}, respectively. On the interval \(I\), and for sufficiently large \(N\), we have
\[
h_i(t) \le \delta r \alpha_2 \max_{1 \le k \le r} |m_{ik}(t)| \le 2 \delta r \alpha_2 m_{i_1^\ast j_1^\ast} (t).
\]
where the second inequality holds since \(t \le \mathcal{T}_{\text{bad}}\).\\

We now focus on the growth of the dominant correlation \(m_{i_1^\ast j_1^\ast}\). For the upper bound, we have
\[
-\langle \boldsymbol{v}_{i_1^\ast}, (\nabla_{\text{St}} \Phi(\boldsymbol{X}_t))_{j_1^\ast} \rangle  \le \sqrt{N} p \lambda_{i_1^\ast} \lambda_{j_1^\ast} m_{i_1^\ast j_1^\ast}^{p-1}(t).
\]
For the lower bound, 
\[
\begin{split}
-\langle \boldsymbol{v}_{i_1^\ast}, (\nabla_{\text{St}} \Phi(\boldsymbol{X}_t))_{j_1^\ast} \rangle & \ge  \sqrt{N} p \left [ \lambda_{i_1^\ast} \lambda_{j_1^\ast} m_{i_1^\ast j_1^\ast}^{p-1}(t) (1 - m_{i_1^\ast j_1^\ast}^2(t)) - \sum_{(k,\ell) \neq (i_1^\ast, j_1^\ast)} \lambda_k \lambda_{j_1^\ast} m_{k j_1^\ast}^{p-1}(t)  m_{i_1^\ast \ell} (t) m_{k \ell}(t) \right]\\
& \ge \sqrt{N} p \lambda_{i_1^\ast} \lambda_{j_1^\ast} m_{i_1^\ast j_1^\ast}^{p-1}(t) \left (1 - m_{i_1^\ast j_1^\ast}^2(t) - C (r^2-1) \frac{\gamma_+^2}{N}\right) \\
&\ge \left (1 - \frac{c_0}{2} - \varepsilon^2 \right)\sqrt{N} p \lambda_{i_1^\ast} \lambda_{j_1^\ast} m_{i_1^\ast j_1^\ast}^{p-1}(t),
\end{split}
\]
where the last step uses \(N \ge 2 C(r^2-1) \gamma_+^2 c_0^{-1}\) for some constant \(C >0\). Furthermore, on \(I\), we have 
\[
h_{i_1^\ast}(t) \le 2 \delta r \alpha_2  m_{i_1^\ast j_1^\ast}(t) \le \frac{2C_\delta d_0 \alpha_2 r}{\gamma_-^{p-2} \log(N)} \sqrt{N} m_{i_1^\ast j_1^\ast}^{p-1}(t) \le \frac{c_0}{2} \sqrt{N} p \lambda_{i_1^\ast} \lambda_{j_1^\ast} m_{i_1^\ast j_1^\ast}^{p-1}(t),
\]
where the second inequality uses \(m_{i_1^\ast j_1^\ast}(t) \ge \gamma_- N^{-1/2}\) for all \(t \in I\) and the last step uses \(c_0\log(N) \gg d_0\). Combining these bounds yields
\[
\begin{split}
m_{i_1^\ast j_1^\ast }(t)  & \ge (1-c_0)\varepsilon_N^{+}+(1-c_0-\varepsilon^2)\frac{\delta}{\sqrt{N}}p\lambda_{i_1^\ast }\lambda_{j_1^\ast } \sum_{\ell=1}^t m_{i_1^\ast j_1^\ast }^{p-1}(\ell-1)  ,\\
m_{i_1^\ast j_1^\ast }(t)  & \le (1 + c_0)\varepsilon_N^{+} + (1 + c_0)\frac{\delta}{\sqrt{N}}p\lambda_{i_1^\ast }\lambda_{j_1^\ast } \sum_{\ell=1}^t m_{i_1^\ast j_1^\ast }^{p-1}(\ell-1)  .
\end{split}
\]
Applying Lemma~\ref{lem: discrete growth} gives
\begin{equation} \label{eq: lower bounding star}
m_{i_1^\ast j_1^\ast }(t) \ge (1-c_0) \varepsilon_N^+ \left ( 1 -  \frac{\delta}{\sqrt{N}} (1-c_0-\varepsilon^2) (1-c_0)^{p-2} p \lambda_{i_1^\ast} \lambda_{j_1^\ast} (\varepsilon_N^+)^{p-2} t \right)^{- \frac{1}{p-2}},
\end{equation}
and 
\begin{equation} \label{eq: upper bounding star}
m_{i_1^\ast j_1^\ast }(t) \le (1+c_0) \varepsilon_N^+ \left ( 1 -  \frac{\delta}{\sqrt{N}} (1+c_0)^{p-1} p \lambda_{i_1^\ast} \lambda_{j_1^\ast} (\varepsilon_N^+)^{p-2} t \right)^{- \frac{1}{p-2}},
\end{equation}
The right-hand side of~\eqref{eq: lower bounding star} is strictly increasing in \(t\). Therefore, 
\[
\mathcal{T}_\varepsilon^{(i_1^\ast j_1^\ast, +)}  \wedge \min_{(i,j) \neq (i_1^\ast, j_1^\ast)} \mathcal{T}_{\varepsilon_N^+}^{(ij,+)} \wedge \min_{i \neq i_1^\ast, j \neq j_1^\ast} \mathcal{T}_{\varepsilon_N^{-}}^{(ij,-)} \wedge \mathcal{T}_{\textnormal{bad}} \wedge T_0 \le \mathcal{T}_{\varepsilon_N^{-}}^{(i_1^\ast j_1^\ast,-)}.\\
\]

We next focus on bounding all non-dominant correlations \(m_{ij}(t)\) for every \((i,j) \neq (i_1^\ast ,j_1^\ast )\) from above. For \((i,j)\neq(i_1^\ast,j_1^\ast)\) we have
\begin{equation} \label{eq: upper bounding not star}
\begin{split}
m_{ij}(t) & \leq  (1+c_0) \frac{\varepsilon_N^+}{2} + \frac{\delta}{N} \sum_{\ell=1}^t \left ( \sqrt{N}p \lambda_i \lambda_j m_{ij}^{p-1} (\ell-1) + 2 \delta  r \alpha_2 m_{i_1^\ast j_1^\ast }(\ell-1) \right ),
\end{split}
\end{equation}
since from \(E_0\) we know that \(m_{ij}(0) \le \frac{1}{2} m_{i_1^\ast j_1^\ast}(0) \le \frac{\varepsilon_N^+}{2}\). We claim that the term \(\frac{2 \delta^2 r \alpha_2}{N} \sum_{\ell = 1}^t m_{i_1^\ast j_1^\ast }(\ell-1)\) on the right-hand side of~\eqref{eq: upper bounding not star} is sufficiently small compared to the drift term \(\frac{\delta p \lambda_i \lambda_j}{\sqrt{N}} \sum_{\ell=1}^t m_{ij}^{p-1}(\ell-1)\). To control the small additive term, we partition \([0,\mathcal{T}_\varepsilon^{(i_1^\ast j_1^\ast,+)}]\) into subintervals \([\mathcal{T}_{\varepsilon_k}^{(i_1^\ast j_1^\ast,+)}, \mathcal{T}_{\varepsilon_{k+1}}^{(i_1^\ast j_1^\ast,+)}]\) for \(1 \le k \le n_\varepsilon\) and \(\varepsilon_k \coloneqq \frac{\gamma_+^k}{\sqrt{N}}\). We notice that \(\mathcal{T}_{\varepsilon_1}^{(i_1^\ast j_1^\ast, +)}=0\) since at \(t=0\), \(m_{i_1^\ast j_1^\ast}(0) \ge \varepsilon_N^+\). By definition, \(\varepsilon_{n_\varepsilon+1} = \varepsilon\), so that 
\[
n_\varepsilon = \left \lceil \frac{\log(\varepsilon)+\frac{1}{2} \log(N)}{\log(\gamma_+)} -1 \right \rceil. 
\]
Using this partition, we have
\[
\sum_{\ell=1}^{\mathcal{T}_\varepsilon^{(i_1^\ast j_1^\ast, +)}} m_{i_1^\ast j_1^\ast} (\ell-1) = \sum_{k=1}^{n_\varepsilon} \sum_{\ell= \mathcal{T}_{\varepsilon_k}^{(i_1^\ast j_1^\ast,+)}}^{\mathcal{T}_{\varepsilon_{k+1}}^{(i_1^\ast j_1^\ast,+)}} m_{i_1^\ast j_1^\ast} (\ell-1) \le n_\varepsilon \max_{1 \le k \le n_\varepsilon} \left\{\varepsilon_{k+1} \left ( \mathcal{T}_{\varepsilon_{k+1}}^{(i_1^\ast j_1^\ast,+)} - \mathcal{T}_{\varepsilon_k}^{(i_1^\ast j_1^\ast,+)} \right ) \right\}.
\]
To bound the time difference in the previous display, we use~\eqref{eq: lower bounding star}. Indeed, since \(m_{i_1^\ast j_1^\ast}\) is lower-bounded by a strictly increasing function, we can upper-bound the first time interval \(\mathcal{T}_{\varepsilon_2}^{(i_1^\ast j_1^\ast,+)} - \mathcal{T}_{\varepsilon_1}^{(i_1^\ast j_1^\ast,+)} = \mathcal{T}_{\varepsilon_2}^{(i_1^\ast j_1^\ast,+)}\) by solving 
\[
(1-c_0) \varepsilon_N^+ \left(1-\frac{\delta}{\sqrt{N}} (1-c_0)^{p-2} (1-c_0-\varepsilon_2^2) p \lambda_{i_1^\ast } \lambda_{j_1^\ast} (\varepsilon_N^+)^{p-2} t \right)^{-\frac{1}{p-2}} = \varepsilon_2,
\]
yielding
\[
\mathcal{T}_{\varepsilon_2}^{(i_1^\ast j_1^\ast,+)} \le \frac{1 - \frac{(1-c_0)^{p-2}}{\gamma_+^{p-2}}}{\frac{\delta}{\sqrt{N}} (1-c_0)^{p-2} (1-c_0 - \varepsilon_2^2) p \lambda_{i_1^\ast} \lambda_{j_1^\ast} (\varepsilon_N^+)^{p-2}},
\]
where \(1 - \frac{(1-c_0)^{p-2}}{\gamma_+^{p-2}} > 0\) by assumption. Using the strong Markov property, we can apply Proposition~\ref{prop: inequalities SGD} for each time interval \([\mathcal{T}_{\varepsilon_k}^{(i_1^\ast j_1^\ast,+)}, \mathcal{T}_{\varepsilon_{k+1}}^{(i_1^\ast j_1^\ast,+)}]\) and the comparison inequality~\eqref{eq: lower bounding star} holds with \(\varepsilon_N^+\) replaced by \(\varepsilon_k\). Hence, for every \(k\), we have
\begin{equation} \label{eq:delta_t}
\mathcal{T}_{\varepsilon_{k+1}}^{(i_1^\ast j_1^\ast,+)}-\mathcal{T}_{\varepsilon_k}^{(i_1^\ast j_1^\ast,+)} \leq \frac{1 - \frac{(1-c_0)^{p-2}}{\gamma_+^{p-2}}}{\frac{\delta}{\sqrt{N}} (1-c_0)^{p-2} (1-c_0 - \varepsilon_{k+1}^2) p \lambda_{i_1^\ast} \lambda_{j_1^\ast} \varepsilon_k^{p-2}}.
\end{equation}
As a result, we obtain 
\[
\frac{2 \delta^2 r \alpha_2}{N} \sum_{\ell =1}^{\mathcal{T}_\varepsilon^{(i_1^\ast j_1^\ast ,+)}} m_{i_1^\ast j_1^\ast }(\ell-1) \leq \frac{2 \delta^2 r \alpha_2}{N} n_\varepsilon \max_{1 \leq k \leq n_\varepsilon} \left \{ \varepsilon_{k+1} \left(\mathcal{T}_{\varepsilon_{k+1}}^{(i_1^\ast j_1^\ast,+)}-\mathcal{T}_{\varepsilon_k}^{(i_1^\ast j_1^\ast,+)}\right ) \right \} \lesssim  \frac{\gamma_+ d_0}{2\sqrt{N}},
\]
where \(\lesssim\) hides constant that may depend on \(p, r, \sigma, \{\lambda_i\}_{i=1}^r\). Since \(d_0 \ll \log(N)\) and \(\varepsilon_N^+=\gamma_+ N^{-1/2}\), this correction term is lower order compared to the initial term \(\varepsilon_N^+\) and can therefore be absorbed into it. That is,
\[
m_{ij}(t) \leq (1+c_0+d_0) \frac{\varepsilon_N^+}{2} + \frac{\delta}{\sqrt{N}} p  \lambda_i \lambda_j\sum_{\ell =1}^t m_{ij}^{p-1}(\ell-1),
\]
and by Lemma~\ref{lem: discrete growth}, it follows 
\[
m_{ij}(t) \le (1+c_0+d_0) \frac{\gamma_+}{2\sqrt{N}} \left ( 1 -  \delta p \lambda_i \lambda_j  2^{-(p-2)}(1+c_0+d_0)^{p-2} \gamma_+^{p-2} N^{-\frac{p-1}{2}} t\right )^{-\frac{1}{p-2}},
\]
for every \( t \leq \mathcal{T}_\varepsilon^{(i_1^\ast j_1^\ast, +)} \wedge \min_{i \neq i_1^\ast, j \neq j_1^\ast} \mathcal{T}_{\varepsilon_N^{-}}^{(ij,-)} \wedge \mathcal{T}_{\textnormal{bad}} \wedge T_0\). Finally, in a similar way to~\eqref{eq:delta_t}, from~\eqref{eq: lower bounding star} we have 
\begin{equation} \label{eq:boundtime}
\mathcal{T}_\varepsilon^{(i_1^\ast j_1^\ast, +)} \le \frac{1 - \left (\frac{(1-c_0) \gamma_+}{\varepsilon \sqrt{N}} \right )^{p-2}}{C_\delta d_0 (1-c_0)^{p-2} (1-c_0-\varepsilon^2) p \lambda_{i_1^\ast} \lambda_{j_1^\ast} \gamma_+^{p-2}} \log(N) N^{p-2}.
\end{equation}
Plugging this bound into the above estimate for \(m_{ij}(t)\) gives for \(N\) sufficiently large (i.e., \(N > \frac{(1-c_0)^2 \gamma_+^2}{\varepsilon^2}\)),
\[
m_{ij}(\mathcal{T}_\varepsilon^{(i_1^\ast j_1^\ast ,+)}) \le \frac{(1+c_0+d_0)\gamma_+}{2\sqrt{N}} \left( 1 - \frac{\lambda_i \lambda_j(1+c_0+d_0)^{p-2}}{2^{p-2}\lambda_{i_1^\ast}\lambda_{j_1^\ast} (1-c_0)^{p-1}(1-c_0-\varepsilon^2)} \right)^{-\frac{1}{p-2}} \lesssim \frac{\gamma_+}{\sqrt{N}},
\]
where \(\lesssim\) hides a constant which may depend on \(\gamma_1,\gamma_2,\gamma, c_0,\{\lambda_i\}_{i=1}^r, \varepsilon\). This shows that \(\mathcal{T}_\varepsilon^{(i_1^\ast j_1^\ast, +)} \wedge \min_{i \neq i_1^\ast, j\neq j_1^\ast} \mathcal{T}_{\varepsilon_N^-}^{(ij,-)} \wedge T_0 \le \mathcal{T}_{\textnormal{bad}} \wedge \min_{(i,j) \neq (i_1^\ast, j_1^\ast)} \mathcal{T}_{\varepsilon_N^+}^{(ij,+)}\) . \\

It remains to bound all the correlations \(m_{ij}(t)\) for \(i \neq i_1^\ast\) and \(j \neq j_1^\ast\) from below, and to show that, as the correlation \(m_{i_1^\ast j_1^\ast}\) reaches \(\varepsilon\), \(m_{ij} (\mathcal{T}_\varepsilon^{(i_1^\ast j_1^\ast, +)}) \ge \gamma_{ij}' N^{-1/2}\) for some constant \(\gamma_{ij}'\). On the considered time interval,
\[
\begin{split}
- \langle \boldsymbol{v}_i, (\nabla_{\textnormal{St}} \Phi(\boldsymbol{X}_t))_j \rangle
& \ge \sqrt{N} p \lambda_i \lambda_j m_{ij}^{p-1}(t)
- \sqrt{N} p r^2 \lambda_1^2 m_{i_1^\ast j_1^\ast}^{p-1}(t) m_{i j_1^\ast}(t) m_{i_1^\ast j}(t)\\
& \ge \sqrt{N} p \lambda_i \lambda_j m_{ij}^{p-1}(t)
- \sqrt{N} p r^2 \lambda_1^2 \frac{\gamma_+^2}{N} m_{i_1^\ast j_1^\ast}^{p-1}(t),
\end{split}
\]
since from the previous step \(m_{i j_1^\ast}(t), m_{i_1^\ast j}(t) \le \gamma_+ N^{-1/2}\). Hence,
\begin{equation} \label{eq: lower bound bounding not star}
m_{ij}(t) \ge (1-c_0) \varepsilon_N^{-} + \frac{\delta}{N} \sum_{\ell=1}^t \left ( \sqrt{N} p \lambda_i \lambda_j m_{ij}^{p-1}(\ell-1) - \frac{p r^2  \lambda_1^2 \gamma_+^2}{\sqrt{N}} m_{i_1^\ast j_1^\ast}^{p-1}(\ell-1) - 2 \delta r \alpha_2 m_{i_1^\ast j_1^\ast} (\ell-1)\right ).
\end{equation} 
The proof strategy is similar to before. We partition \([0, \mathcal{T}_\varepsilon^{(i_1^\ast j_1^\ast,+)}]\) into the subintervals \([\mathcal{T}_{\varepsilon_k}^{(i_1^\ast j_1^\ast,+)},\mathcal{T}_{\varepsilon_{k+1}}^{(i_1^\ast j_1^\ast,+)}]\), where \(\varepsilon_k = \gamma_+^k N^{-1/2}\) and \(1 \le k \le n_\varepsilon\). We already know that the last term \(\frac{2 \delta r \alpha_2}{N} \sum_{\ell=1}^t m_{i_1^\ast j_1^\ast}(\ell-1)\) can be absorbed into the initial condition. It therefore remains to investigate the term \(\frac{\delta p  r^2 \lambda_1^2 \gamma_+^2}{N^{3/2}} \sum_{\ell=1}^t m_{i_1^\ast j_1^\ast}^{p-1}(\ell-1)\). Using~\eqref{eq:delta_t} we obtain
\[
\begin{split}
\frac{\delta p  r^2 \lambda_1^2 \gamma_+^2}{\sqrt{N} N} \sum_{\ell=1}^t m_{i_1^\ast j_1^\ast}^{p-1}(\ell-1) & = \frac{\delta p  r^2 \lambda_1^2 \gamma_+^2}{\sqrt{N} N} \sum_{k=1}^{n_\varepsilon} \sum_{\ell = \mathcal{T}_{\varepsilon_k}^{(i_1^\ast j_1^\ast,+)}}^{\mathcal{T}_{\varepsilon_{k+1}}^{(i_1^\ast j_1^\ast,+)}} m_{i_1^\ast j_1^\ast}^{p-1}(\ell-1) \\
& \le \frac{\delta p  r^2 \lambda_1^2 \gamma_+^2}{\sqrt{N} N} n_\varepsilon \max_{1 \le k \le n_\varepsilon} \left \{ \varepsilon_{k+1}^{p-1} \left ( \mathcal{T}_{\varepsilon_{k+1}}^{(i_1^\ast j_1^\ast,+)} - \mathcal{T}_{\varepsilon_k}^{(i_1^\ast j_1^\ast,+)} \right)\right \} \\
&  \le \frac{\log(N)}{N} \frac{\gamma_+^2 \varepsilon r^2 \lambda_1^2 \gamma_1^{p-2}}{\lambda_{i_1^\ast} \lambda_{j_1^\ast} (1-c_0)^{p-2} (1-c_0-\varepsilon^2)}  = o (N^{-1/2}).
\end{split}
\]
Therefore, for \(N\) large enough, this term is lower-bounded by \(c_0 \varepsilon_N^-\), and can be absorbed into the initial condition in~\eqref{eq: lower bound bounding not star}, yielding
\[
\begin{split}
m_{ij}(t) & \ge (1 - 2c_0 - d_0) \varepsilon_N^{-} + \frac{\delta}{\sqrt{N}} p \lambda_i \lambda_j  \sum_{\ell=1}^t m_{ij}^{p-1}(\ell-1) \\
& \ge (1 - 2c_0 - d_0) \frac{\gamma_-}{\sqrt{N}} \left ( 1 - \delta p \lambda_i \lambda_j (1-2c_0-d_0)^{p-2} \gamma_-^{p-2} N^{- \frac{p-1}{2}} t\right)^{- \frac{1}{p-2}},
\end{split}
\]
where the second inequality follows by Lemma~\ref{lem: discrete growth}. In particular, from~\eqref{eq: upper bounding star} we deduce a lower bound for \(\mathcal{T}_\varepsilon^{(i_1^\ast j_1^\ast, +)}\):
\[
\mathcal{T}_\varepsilon^{(i_1^\ast j_1^\ast, +)} \ge \frac{1 - \left (\frac{(1+c_0) \gamma_+}{\varepsilon \sqrt{N}} \right )^{p-2}}{C_\delta d_0 (1+c_0)^{p-1} p \lambda_{i_1^\ast} \lambda_{j_1^\ast} \gamma_+^{p-2}} \log(N) N^{p-2}.
\]
Plugging the hitting-time lower bound gives for \(N\) sufficiently large (i.e., \(N > \frac{(1+c_0)^2 \gamma_+^2}{\varepsilon^2}\)), 
\[
m_{ij}( \mathcal{T}_\varepsilon^{(i_1^\ast j_1^\ast, +)}) \ge (1 - 2c_0 - d_0) \frac{\gamma_-}{\sqrt{N}} \left ( 1 - \frac{(1-2c_0-d_0)^{p-2} \lambda_i \lambda_j \gamma_-^{p-2}}{(1+c_0)^{p-1} \lambda_{i_1^\ast} \lambda_{j_1^\ast} \gamma_+^{p-2}}\right)^{- \frac{1}{p-2}} \gtrsim \frac{\gamma_-}{\sqrt{N}}.
\]

Since Proposition~\ref{prop: inequalities SGD} is (re)applied on each subinterval \([\mathcal{T}_{\varepsilon_k}^{(i_1^\ast j_1^\ast,+)},\mathcal{T}_{\varepsilon_{k+1}}^{(i_1^\ast j_1^\ast,+)}]\), by the strong Markov property, with the same parameters and with the uniform horizon \(T_0\) taken as the upper bound in~\eqref{eq:boundtime}, the failure probability on each subinterval is at most \(\tilde\eta=\eta(\delta,c_0,T_0,K_N)\). As \(n_\varepsilon\asymp \log N\), a union bound yields total success probability at least \(1-\tilde\eta \log(N)\). Therefore,
\[
\mathcal{T}_{E(\varepsilon)} \le \mathcal{T}_\varepsilon^{(i_1^\ast j_1^\ast,+)}
\le \frac{1 - \left(\frac{(1-c_0)\gamma_+}{\varepsilon \sqrt{N}}\right)^{p-2}}
{C_\delta d_0 (1-c_0)^{p-2} (1-c_0 - \varepsilon^2) p \lambda_{i_1^\ast} \lambda_{j_1^\ast} \gamma_+^{p-2}}
\log(N) N^{p-2},
\]
with probability at least \(1-\tilde\eta \log(N)\), which proves \eqref{eq: step 2 p>2}.
\end{proof}

\begin{proof}[Proof of~\eqref{eq: step 3 p>2}] 
Assume \(\boldsymbol{X}_0 \in E(\varepsilon)\), so that \(m_{i_1^\ast j_1^\ast}(0)\ge \varepsilon\), and for all \(i\neq i_1^\ast\), \(j\neq j_1^\ast\), \(\varepsilon_N^- \le m_{ij}(0) \le \varepsilon_N^+\) while \(m_{i j_1^\ast}(0),m_{i_1^\ast j}(0)\le \varepsilon_N^+\). Fix a horizon \(T>0\) to be chosen later and define the time interval
\[
I=\left [0, T\wedge \mathcal T_{\varepsilon/2}^{(i_1^\ast j_1^\ast,+)}\wedge
\mathcal T_{1-\varepsilon}^{(i_1^\ast j_1^\ast,+)}\wedge \min_{(i,j)\neq(i_1^\ast,j_1^\ast)} \mathcal T_{\varepsilon_N}^{(ij,+)} \wedge \mathcal{T}_{\mathrm{low}} \right],
\]
where \(\varepsilon_N = \gamma_\dagger N^{-1/2}\) with a fixed \(\gamma_\dagger>0\) (e.g.\ \(\gamma_\dagger=2\gamma_+\)) and 
\[
\mathcal{T}_{\text{low}} \coloneqq \inf \left \{ t \in \N_0 \colon \max \left \{ \max_{i \neq i_1^\ast} m_{i j_1^\ast}(t),\max_{j \neq j_1^\ast} m_{i_1^\ast j}(t) \right \} \lesssim \log(N)^{-1/2} N^{- \frac{p-1}{4}}\right \}.
\]
Note that \(\mathcal{T}_{\text{low}} >0\). \\

We first consider the evolution of \(m_{i_1^\ast j_1^\ast}(t)\) within the above time interval. Proposition~\ref{prop: inequalities SGD} does not apply directly in this case since \(m_{i_1^\ast j_1^\ast}(0)\ge \varepsilon\). We therefore need a modification of Proposition~\ref{prop: inequalities SGD}. By Lemma~\ref{lem: expansion correlations}, for all \(t \le T\),
\[
\left | m_{i_1^\ast j_1^\ast}(t) - \left ( \varepsilon- \frac{\delta}{N} \sum_{\ell=1}^t  \langle \boldsymbol{v}_{i_1^\ast},(\nabla_{\textnormal{St}}\Phi(\boldsymbol{X}_{\ell-1}))_{j_1^\ast} \rangle - \frac{\delta}{N} \sum_{\ell=1}^t  \langle \boldsymbol{v}_{i_1^\ast},(\nabla_{\textnormal{St}}H^\ell(\boldsymbol{X}_{\ell-1}))_{j_1^\ast} \rangle \right ) \right | \leq \frac{\delta}{N} \sum_{\ell=1}^t  \vert a_{i_1^\ast}(\ell-1) \vert,
\]
with \(\mathbb{P}_{\boldsymbol{X}_0}\)-probability at least \(1 - T e^{-c_1N}\), where \(\vert a_{i_1^\ast}(\ell-1) \vert\) is bounded according to~\eqref{eq: bound A_l}. The martingale term \(\frac{\delta}{N} \lvert \sum_{\ell=1}^T \langle \boldsymbol{v}_{i_1^\ast}, (\nabla_{\textnormal{St}} H^\ell(\boldsymbol{X}_{\ell-1}))_{j_1^\ast}\rangle \rvert \) is controlled by Lemma~\ref{lem: noise martingale}, i.e., for every \(c_0 \in (0,\frac{1}{2})\),
\[
\sup_{\boldsymbol{X}_0 \in \, \textnormal{\textnormal{St}}(N,r)} \mathbb{P}_{\boldsymbol{X}_0} \left (\max_{t \leq T} \frac{\delta}{N}  \left |  \sum_{\ell=1}^T \langle \boldsymbol{v}_{i_1^\ast}, \left (\nabla_{\textnormal{St}}H^\ell(\boldsymbol{X}_{\ell-1}) \right )_{j_1^\ast} \rangle\right | \geq \frac{c_0}{2} \varepsilon \right) \leq 2\exp\left(-c_2 \frac{c_0^2 \varepsilon^2 \log(N)^2N^{p-1}}{8 \alpha_1^2 C_\delta^2 d_0^2 T}\right).
\]
To bound the term \(\frac{\delta}{N} \sum_{\ell=1}^T \vert a_{i_1^\ast}(\ell-1) \vert\), Lemma~\ref{lem: higher truncation} with a truncation level \(K_N \le \alpha_1 \frac{N^{\frac{p-1}{2}}\log(N)}{d_0}\) yields
\[
\begin{split}
& \frac{\delta}{N} \sum_{\ell=1}^T  \vert a_{i_1^\ast}(\ell-1) \vert \\
& \leq \frac{\delta^2 \alpha_2}{2N} \sum_{\ell=1}^T  \sum_{k=1}^r \left( \vert m_{i_1^\ast k}(\ell-1) \vert + \frac{\delta}{N}\vert \langle \boldsymbol{v}_{i_1^\ast}, (\nabla_{\textnormal{St}}\Phi \left(\boldsymbol{X}_{\ell-1} \right))_k \rangle \vert + \frac{\delta}{N} \vert \langle \boldsymbol{v}_{i_1^\ast}, (\nabla_{\textnormal{St}}H^\ell \left( \boldsymbol{X}_{\ell-1})_k \right) \rangle \vert \right)\\
&
\le \frac{\delta^2 \alpha_2 r}{2N}T \left (1 + \frac{\delta K_N}{N} + \frac{\delta}{\sqrt{N}} p \lambda_1^2 (1 +r^2)\right ) \\
& \quad + \frac{\delta^3 \alpha_2}{2N^2} \sum_{\ell=1}^T \sum_{k=1}^r |\langle \boldsymbol{v}_{i_1^\ast}, (\nabla_{\textnormal{St}} H^\ell(\boldsymbol{X}_\ell-1))_k\rangle| \mathbf{1}_{\{|\langle \boldsymbol{v}_i, (\nabla_{\textnormal{St}} H^\ell(\boldsymbol{X}_\ell-1))_k\rangle | > K_N\}}\\
& \le \frac{2 \delta^2 \alpha_2 r}{N}T  = \frac{2 C_\delta^2 \alpha_2 r d_0^2}{\log(N)^2 N^{p-2}} T,
\end{split}
\]
provided \(N\) is sufficiently large, with \(\mathbb{P}_{\boldsymbol{X}_0}\)-probability at least \(1 - K d_0 \log(N)^{-1} N^{-\frac{p-1}{2}}e^{-c \log(N)^2 N^{p-1}/d_0^2}\), where the constants \(K, c\) may depend only on \(p,\sigma^2,\{\lambda_i\}_{i=1}^r\). The second inequality is obtained by bounding all correlations by \(1\). Choosing \(T \ll \frac{\log(N)^2N^{p-2}}{d_0^2}\) ensures 
\[
\frac{\delta}{N} \sum_{\ell=1}^T  \vert a_{i_1^\ast}(\ell-1) \vert \le c_0 \varepsilon.
\]
Combining these estimates, for all \(t \le T\) we have
\[
m_{i_1^\ast j_1^\ast}(t) \ge (1-c_0) \varepsilon - \frac{\delta}{N} \sum_{\ell=1}^t \langle \boldsymbol{v}_{i_1^\ast}, \left(\nabla_{\textnormal{St}} \Phi (\boldsymbol{X}_{\ell-1}) \right)_{j_1^\ast} \rangle,
\]
with \(\mathbb{P}_{\boldsymbol{X}_0}\)-probability at least \(1 - T e^{-c_1 N} - 2 e^{-c_2 \frac{c_0^2 \varepsilon^2 \log(N)^2 N^{p-1}}{d_0^2T}} - K \frac{d_0}{\log(N) N^{\frac{p-1}{2}}} e^{- c_3 \log(N)^2 N^{p-1} / d_0^2}\). Recall the expression for \(\langle \boldsymbol{v}_{i_1^\ast}, \left(\nabla_{\textnormal{St}} \Phi (\boldsymbol{X}_{\ell-1}) \right)_{j_1^\ast} \rangle \) given in~\eqref{eq: correlation exact computation}. Over the considered time interval \(I\), 
\[
-\langle \boldsymbol v_{i_1^\ast},(\nabla_{\mathrm{St}}\Phi(\boldsymbol X_t))_{j_1^\ast} \rangle
\ \ge\ \sqrt{N} p \lambda_{i_1^\ast} \lambda_{j_1^\ast}  m_{i_1^\ast j_1^\ast}^{p-1}(t) (1 - m_{i_1^\ast j_1^\ast}^2(t)) - C N^{-1/2},
\]
where \(C\) is a constant depending on \(\varepsilon, \{\lambda_i\}_{i=1}^r, p\). Since \(m_{i_1^\ast j_1^\ast} (t) \le 1- \varepsilon\) for all \(t \in I\), for \(N\) large enough, we have 
\[
-\langle \boldsymbol v_{i_1^\ast},(\nabla_{\mathrm{St}}\Phi(\boldsymbol X_t))_{j_1^\ast} \rangle \ge (1-c_0) \sqrt{N} p \lambda_{i_1^\ast} \lambda_{j_1^\ast}  m_{i_1^\ast j_1^\ast}^{p-1}(t) (1-(1-\varepsilon)^2) \ge (1-c_0) \varepsilon \sqrt{N} p \lambda_{i_1^\ast} \lambda_{j_1^\ast}  m_{i_1^\ast j_1^\ast}^{p-1}(t).
\]
Therefore, on the considered time interval \(I\), 
\[
m_{i_1^\ast j_1^\ast}(t) \ge (1-c_0) \varepsilon + (1-c_0) \varepsilon p \lambda_{i_1^\ast} \lambda_{j_1^\ast}  \frac{\delta}{\sqrt{N}} \sum_{\ell=1}^t m_{i_1^\ast j_1^\ast}^{p-1}(\ell-1).
\]
This implies that \(m_{i_1^\ast j_1^\ast}\) is lower bounded by an increasing function, thus 
\[
T \wedge \mathcal{T}_{1-\varepsilon}^{(i_1^\ast j_1^\ast, +)} \wedge \min_{(i,j) \neq (i_1^\ast, j_1^\ast)} \mathcal{T}_{\varepsilon_N}^{(ij,+)} \wedge \mathcal{T}_{\mathrm{low}} \le \mathcal{T}_{\varepsilon/2}^{(i_1^\ast j_1^\ast, +)}.
\]
Applying Lemma~\ref{lem: discrete growth} then gives
\[
m_{i_1^\ast j_1^\ast}(t)  \geq (1-c_0) \varepsilon \left(1 - \frac{\delta}{\sqrt{N} }(1 -c_0)^{p-1} \varepsilon^{p-1} p  \lambda_{i_1^\ast} \lambda_{j_1^\ast}  t\right)^{-\frac{1}{p-2}} ,
\]
so that
\[
\mathcal{T}_{1-\varepsilon}^{(i_1^\ast j_1^\ast,+)} \leq \frac{1-\left(\frac{(1-c_0) \varepsilon}{1-\varepsilon}\right)^{p-2}}{C_\delta d_0 (1-c_0)^{p-1}\varepsilon^{p-1} p \lambda_{i^\ast_1} \lambda_{j^\ast_1} } \log(N) N^{\frac{p-2}{2}} \ll \frac{\log(N)^2 N^{p-2}}{d_0^2},
\]
with \(\mathbb{P}_{\boldsymbol{X}_0}\)-probability at least \(1 - T e^{-c_1 N} - 2 e^{-c_2 \frac{c_0^2 \varepsilon^2 \log(N)^2 N^{p-1}}{d_0^2T}} - K \frac{d_0}{\log(N) N^{\frac{p-1}{2}}} e^{- c_3 \log(N)^2 N^{p-1} / d_0^2}\).\\

We now turn our attention to the correlations \(m_{i_1^\ast j}\) and \(m_{i j_1^\ast}\) for \(i\neq i_1^\ast, j \neq j_1^\ast\). By Proposition~\ref{prop: inequalities SGD}, for every \(t \le T\),  
\begin{equation} \label{eq: intermediate}
\begin{split}
m_{i_1^\ast j} (t) & \le (1+c_0) m_{i_1^\ast j}(0) + \frac{\delta}{N} \sum_{\ell=1}^t \left (- \langle \boldsymbol{v}_{i_1^\ast} , (\nabla_{\textnormal{St}} \Phi (\boldsymbol{X}_{\ell - 1}))_j \rangle  + h_{i_1^\ast} (\ell-1) \right), \\
m_{i j_1^\ast} (t) & \le (1+c_0) m_{ij_1^\ast}(0) + \frac{\delta}{N} \sum_{\ell=1}^t \left ( -\langle \boldsymbol{v}_i, (\nabla_{\textnormal{St}} \Phi (\boldsymbol{X}_{\ell - 1}))_{j_1^\ast} \rangle + h_i (\ell-1) \right),
\end{split}
\end{equation}
with \(\mathbb{P}_{\boldsymbol{X}_0}\)-probability at least \(1 - T e^{-c_1 N}\). From~\eqref{eq: correlation exact computation}, for every \(t \in I\) and sufficiently large \(N\),
\[
\begin{split}
-\langle \boldsymbol{v}_{i_1^\ast}, (\nabla_{\textnormal{St}} \Phi (\boldsymbol{X}_t))_j \rangle & \le \sqrt{N} p \lambda_{i_1^\ast} \lambda_j m_{i_1^\ast j}^{p-1}(t) -  \sqrt{N} p \lambda_{i_1^\ast} \lambda_{j_1^\ast} m_{i_1^\ast j_1^\ast}^p(t) m_{i_1^\ast j}(t) \\
& \le - \kappa \sqrt{N} p \lambda_r^2 \varepsilon^p m_{i_1^\ast j} (t), 
\end{split}
\]
for some constant \(\kappa  > 0\) independent of \(N\). Here, we used that \(m_{i_1^\ast j_1^\ast}(t) \ge \varepsilon/2\) and \(m_{i_1^\ast j}(t) \le \varepsilon_N = \gamma_\dagger N^{-1/2}\). The same bound holds for \(m_{i j_1^\ast} (t)\). Moreover, on the time interval \(I\), for all \(i \in [r]\) and sufficiently large \(N\), 
\[
h_i (t) \le \delta \alpha_2 r \max_{1 \le k \le r} |m_{i k}(t)| \le \delta \alpha_2 r (1-\varepsilon),
\]
so that
\[
\frac{\delta}{N} \sum_{\ell=1}^t h_i(\ell-1) \le \frac{\delta^2 \alpha_2 r (1-\varepsilon)}{N} t.
\]
Choosing \(T \ll \frac{\log(N)^2 N^{p-3/2}}{d_0^2}\) ensures that
\[
\frac{\delta}{N} \sum_{\ell=1}^T h_i(\ell-1) \le c_0 \varepsilon_N^+ = c_0 \frac{\gamma_+}{\sqrt{N}}.
\]
Plugging the above estimates into~\eqref{eq: intermediate} yields 
\begin{equation} \label{eq: decreasing correlations}
\begin{split}
m_{i_1^\ast j} (t) & \le (1+2c_0) \varepsilon_N^+ - \kappa p \lambda_r^2 \varepsilon^p \frac{\delta}{\sqrt{N}} \sum_{\ell=1}^t m_{i_1^\ast j}(\ell-1), \\
m_{i j_1^\ast} (t) & \le (1+2c_0) \varepsilon_N^+ - \kappa p \lambda_r^2 \varepsilon^p  \frac{\delta}{\sqrt{N}} \sum_{\ell=1}^t m_{i j_1^\ast}(\ell-1),
\end{split}
\end{equation}
with \(\mathbb{P}_{\boldsymbol{X}_0}\)-probability at least \(1 - Te^{-c_1N}\). Hence, the correlations \(m_{i_1^\ast j}\) and \(m_{i j_1^\ast}\) are decreasing. Applying item (a) of Lemma~\ref{lem: discrete growth} gives 
\[
m_{i_1^\ast j} (t) \le (1+2c_0) \varepsilon_N^+ \left (1 - \kappa p \lambda_r^2 \varepsilon^p \frac{\delta}{\sqrt{N}} \right )^t  \le (1+2c_0) \varepsilon_N^+ \exp \left (- \kappa p \lambda_r^2 \varepsilon^p \frac{\delta}{\sqrt{N}}t \right ),
\]
and the same holds for \(m_{i j_1^\ast}(t)\). Since \(m_{i_1^\ast j}(t)\) and \(m_{i j_1^\ast}(t)\) decay exponentially in \(t\), the stopping time \(\mathcal{T}_{\textnormal{low}}\) is almost surely finite. Solving for the first \(t\) such that \(m_{i_1^\ast j_1^\ast} (t) \le \log(N)^{-1/2} N^{-(p-1)/4}\) yields
\begin{equation} \label{eq: T low}
\mathcal{T}_{\text{low}} \le \frac{\sqrt{N}}{\kappa p \lambda_r^2 \varepsilon^p \delta} \left ( \frac{p-3}{4} \log(N) + \frac{1}{2} \log \log(N) + \log((1+2c_0)\gamma_+)\right) \lesssim  \frac{\log(N)^2}{C_\delta d_0 p \lambda_r^2 \varepsilon^p} N^{\frac{p-2}{2}} ,
\end{equation}
which indeed satisfies \(\mathcal{T}_{\mathrm{low}} \ll \frac{\log(N)^2 N^{p - 3/2}}{d_0^2}\). \\

Finally we consider the evolution of \(m_{ij}\) for all \(i \neq i_1^\ast, j \neq j_1^\ast\). By Proposition~\ref{prop: inequalities SGD}, for every \(t \le T\), 
\begin{equation} \label{eq: intermediate 1}
\begin{split}
m_{ij}(t) & \ge (1-c_0) m_{ij}(0) + \frac{\delta}{N} \sum_{\ell=1}^t \left ( - \langle \boldsymbol{v}_i, (\nabla_{\textnormal{St}} \Phi (\boldsymbol{X}_{\ell - 1}))_j \rangle - h_i(\ell-1) \right), \\
m_{ij}(t) & \le (1+c_0) m_{ij}(0) + \frac{\delta}{N} \sum_{\ell=1}^t \left ( -\langle \boldsymbol{v}_i, (\nabla_{\textnormal{St}} \Phi (\boldsymbol{X}_{\ell - 1}))_j \rangle + h_i(\ell-1) \right),
\end{split}
\end{equation}
with \(\mathbb{P}_{\boldsymbol{X}_0}\)-probability at least \(1 - T e^{-c_1 N}\). In this case, for all \(t \in I\),
\[
\begin{split}
- \langle \boldsymbol{v}_i ,\left ( \nabla_{\textnormal{St}} \Phi(\boldsymbol{X}_t) \right )_j \rangle & \geq \sqrt{N} p \left(\lambda_i \lambda_j m_{ij}^{p-1}(t) - r^2 \lambda_1^2 m_{i_1^\ast j_1^\ast}^{p-1}(t) m_{i j_1^\ast}(t) m_{i_1^\ast j}(t)  \right),\\
- \langle \boldsymbol{v}_i ,\left ( \nabla_{\textnormal{St}} \Phi(\boldsymbol{X}_t) \right )_j \rangle & \le \sqrt{N} p \left ( \lambda_i \lambda_j m_{ij}^{p-1}(t) - \lambda_r^2 m_{i_1^\ast j_1^\ast}^{p-1}(t) m_{i j_1^\ast}(t) m_{i_1^\ast j}(t) \right).
\end{split}
\]
We observe that if \(t \ge \mathcal{T}_{\mathrm{low}}\), then the drift term \(\lambda_i \lambda_j m_{ij}^{p-1}(t)\) dominates all cross-interaction terms. Consequently, \(m_{ij} (t)\) becomes increasing beyond \(\mathcal{T}_{\textnormal{low}}\). On the time interval \(I\), however, the correlations \(m_{i_1^\ast j} (t)\) and \(m_{i j_1^\ast} (t)\) decrease from order \(\varepsilon_N^+ = \gamma_+ N^{-1/2}\), while \(m_{i_1^\ast, j_1^\ast} (t) \ge \varepsilon/2\). Hence, the cross-interaction term \(r^2 \lambda_1^2 m_{i_1^\ast j_1^\ast}^{p-1}(t) m_{i j_1^\ast}(t) m_{i_1^\ast j}(t)\) is of order \(\mathcal{O} (N^{-1})\) and can induce only a mild decrease in \(m_{ij}(t)\). Moreover, as before, the term involving \(h_i(t)\) can be absorbed in the initial condition. More precisely, \(h_i(t) \le \delta \alpha_2 r (1-\varepsilon)\) for every \(t \in I\) and \(N\) sufficiently large, so that
\[
\frac{\delta}{N} \sum_{\ell=1}^t h_i(\ell-1) \le \frac{\delta^2 \alpha_2 r (1-\varepsilon)}{N} t \le c_0 \varepsilon_N^-,
\]
where the last inequality holds by our choice of the time horizon \(T \ll d_0^{-2} \log(N)^2 N^{p-3/2}\). Therefore, using~\eqref{eq: intermediate 1} together with the drift bound above, we obtain for all \(t \in I\),
\[
m_{ij} (t) \ge (1-2c_0) \varepsilon_N^- - \frac{C_\varepsilon r^2 \lambda_1^2}{N} \frac{\delta}{\sqrt{N}} t,
\]
for some constant \(C_\varepsilon>0\) independent of \(N\). Evaluating at \(t = \mathcal{T}_{\text{low}}\) and using the previously established bound~\eqref{eq: T low} on \(\mathcal{T}_{\text{low}}\),
\[
\begin{split}
m_{ij} (\mathcal{T}_{\mathrm{low}}) & \ge (1-2c_0) \varepsilon_N^- - \frac{C_\varepsilon r^2 \lambda_1^2}{N} \frac{\delta}{\sqrt{N}} \mathcal{T}_{\mathrm{low}} \\
& \gtrsim (1-2c_0) \frac{\gamma_-}{\sqrt{N}} - \frac{C_\varepsilon r^2 \lambda_1^2 C_\delta}{p \lambda_r^2 \varepsilon^p} \frac{\log(N)}{N} \gtrsim \frac{C}{\sqrt{N}},
\end{split}
\]
for some constant \(C>0\) independent of \(N\). Therefore, while the \(m_{ij}\) correlations experience a small decrease until \(\mathcal{T}_{\mathrm{low}}\), they remain of order \(\Theta (N^{-1/2})\) throughout. \\

Finally, choose
\[
T = \frac{\log(N)^2 N^{\frac{p-2}{2}}}{C_\delta d_0 p \lambda_r^2 \varepsilon^p} ,
\]
which is an upper bound for \(\mathcal{T}_{\mathrm{low}}\), and thus also for \(\mathcal{T}_{1-\varepsilon}^{(i_1^\ast j_1^\ast,+)}\). On the event of \(\mathbb{P}_{\boldsymbol{X}_0}\)-probability at least
\[
1 - \frac{K_1}{d_0 \varepsilon^p} \log(N)^2 N^{\frac{p-2}{2}} e^{-c_1 N} - K_2 e^{-c_2 c_0^2 \varepsilon^{p+2} N^{p/2} / d_0} - K_3 \frac{d_0}{\log(N) N^{\frac{p-1}{2}}} e^{- c_3 \log(N)^2 N^{p-1} / d_0^2},
\]
we have simultaneously: (1) \(m_{i_1^\ast j_1^\ast}(t)\) has increased to \(1-\varepsilon\) before time \(T\); (2) \(m_{i_1^\ast j}(t)\) and \(m_{i j_1^\ast}(t)\) have decayed below \(\log(N)^{-1/2} N^{-(p-1)/4}\) by time \(T\); and (3) \(m_{ij}(t)\) for \(i\neq i_1^\ast, j\neq j_1^\ast\) remain at scale \(\Theta(N^{-1/2})\). Therefore, the set \(E_1(\varepsilon)\) is reached by time \(T\), which proves
\[
\mathcal T_{E_1(\varepsilon)}  \le T = \frac{\log(N)^2 N^{\frac{p-2}{2}}}{C_\delta d_0 p \lambda_r^2 \varepsilon^p} ,
\]
with the stated probability, as claimed in~\eqref{eq: step 3 p>2}.
\end{proof}
\end{proof}
%%%%%%%%%%%%%%%%%%%%%%%%%%%%%%%%%%%%%%%%%%%%%%%%%%%%%%%%%%%%%%%%%%%%%%%%%%%%%%%%%%%%%%%%%%%%%%%%%%%%%%%%%%%%%%%%%%%%%%%%%%%%%%%%%%%%%%%%%%%%%%
\subsection{Proofs for \(p =2\) and separated SNRs}

In this subsection, we focus on the case where \(p=2\) and consider SNRs \(\lambda_1 \ge \cdots \ge \lambda_r\) that are separated by constants of order \(1\). Specifically, we assume that for every \(1 \le i \le r-1\), the SNRs satisfy \(\lambda_i = \lambda_{i+1} (1+\kappa_i)\) with \(\kappa_i>0\), and we define \(\kappa = \min_{1 \le i \le r-1} \kappa_i\). Following a strategy similar to that used in the proof of Theorem~\ref{thm: strong recovery online p>2 nonasymptotic} in the previous subsection, we divide the proof of Theorem~\ref{thm: strong recovery online p=2 nonasymptotic} into several steps, which are then combined via the strong Markov property to complete the argument. \\

Fix any \(\varepsilon \in (0, \varepsilon_0)\) with \(\varepsilon_0^2 = \frac{\kappa}{1+\kappa}\), and let \(c_0 \in \left (0, \frac{1}{2} \wedge \frac{1+\kappa}{2 + \kappa} (\varepsilon_0^2-\varepsilon^2)\right)\). For every \(1 \le k \le r-1\) and \(k \le i,j \le r\), define 
\[
\delta_{ij}^{(k)} (\varepsilon, c_0) \coloneqq 1 - \frac{1+c_0}{1-c_0-\varepsilon^2} \frac{\lambda_i \lambda_j}{\lambda_k^2} \quad \textnormal{and} \quad \xi_{ij}^{(k)} (c_0) \coloneqq 1 - \frac{1-c_0}{1+c_0} \frac{\lambda_i \lambda_j}{\lambda_k^2}.
\]
By construction, \(0 < \delta_{ij}^{(k)} (\varepsilon, c_0)  < \xi_{ij}^{(k)} (c_0) < 1\), and both are symmetric in \((i,j)\). To lighten notation, we sometimes omit the dependence on \(\varepsilon\) and \(c_0\). For every \(i \in [r]\), define the weak and strong recovery regions as
\[
\begin{split}
W_i (\varepsilon,c_0) & = \left \{ \boldsymbol{X} \colon m_{ii}(\boldsymbol{X}) \ge \varepsilon \enspace \textnormal{and} \enspace m_{ij}(\boldsymbol{X}), m_{ji}(\boldsymbol{X}) \lesssim N^{-\frac{1}{2}\delta_{ij}^{(i)} }  \enspace \forall \, j >i \right \},\\
S_i (\varepsilon,c_0) & = \left \{ \boldsymbol{X} \colon m_{ii}(\boldsymbol{X}) \ge 1 - \varepsilon \enspace \textnormal{and} \enspace m_{ij}(\boldsymbol{X}), m_{ji}(\boldsymbol{X}) \lesssim N^{-\frac{1}{2} \xi_{rr}^{(1)}} \enspace \forall \,j >i \right \}.
\end{split}
\]
Finally, we define the recovering event \(E_i (\varepsilon,c_0)\) as 
\[
\begin{split}
E_1 (\varepsilon,c_0) & = W_1 (\varepsilon, c_0) \cap \left \{ \boldsymbol{X} \colon N^{- \frac{1}{2} \xi_{ij}^{(1)}} \lesssim m_{ij}(\boldsymbol{X}) \lesssim N^{- \frac{1}{2} \delta_{ij}^{(1)}} \enspace \forall \, 2 \le i,j \le r \right \},\\
E_2 (\varepsilon,c_0) & = S_1 (\varepsilon, c_0) \cap W_2 (\varepsilon, c_0) \cap \left \{ \boldsymbol{X} \colon N^{- \frac{1}{2} \xi_{ij}^{(2)}} \lesssim m_{ij}(\boldsymbol{X}) \lesssim N^{- \frac{1}{2} \delta_{ij}^{(2)}} \enspace \forall \, 3 \le i,j \le r \right \},\\
& \vdots \\
E_{r-1} (\varepsilon,c_0) & = \cap_{1 \le i \le r-2} S_i (\varepsilon, c_0)  \cap W_{r-1} (\varepsilon, c_0) \cap \left \{ \boldsymbol{X} \colon N^{- \frac{1}{2} \xi_{rr}^{(r-1)}} \lesssim m_{rr}(\boldsymbol{X}) \lesssim N^{- \frac{1}{2} \delta_{rr}^{(r-1)}} \right \},\\
E_r (\varepsilon,c_0) & = \cap_{1 \le i \le r-1} S_i (\varepsilon, c_0) \cap W_r (\varepsilon, c_0).
\end{split}
\]
We observe that the recovery set in Theorem~\ref{thm: strong recovery online p=2 nonasymptotic} corresponds to \(R(\varepsilon, c_0) = \cap_{1 \le i \le r} S_i (\varepsilon, c_0)\). \\

We begin by showing that, starting from any initialization satisfying Condition 1, online SGD reaches the first recovery region \(E_1 (\varepsilon, c_0)\) within \(\mathcal{O} \left (\log(N)^2 N^{\xi/2} \right)\) iterations, with high probability, where we recall from Section~\ref{section: main results SGD} the shorthand notation
\[
\xi = \xi (c_0) \coloneqq \xi_{rr}^{(1)}(c_0) = 1 - \frac{1-c_0}{1+c_0} \frac{\lambda_r^2}{\lambda_1^2}.
\]

\begin{lem} \label{lem: first step SGD p=2}
Assume \(\kappa_1 > \sqrt{2}-1\). Fix \(\gamma_1 > 1 > \gamma_2 >0\). Then there exists \(\bar{\varepsilon} (\kappa) \in (0,\varepsilon_0)\), with \(\varepsilon_0^2 = \frac{\kappa}{1+\kappa}\), such that for every \(\varepsilon \in (0, \bar{\varepsilon})\) there exist \(c_0 = c_0 (\kappa, \varepsilon) \in \left (0, \frac{1}{2} \wedge \frac{1+\kappa}{2 + \kappa} (\varepsilon_0^2-\varepsilon^2)\right)\) and a sequence \(d_0 = d_0(N)\) satisfying~\eqref{eq: sequence d_0}, for which, if 
\begin{equation} \label{eq: step size p=2} 
\delta = \frac{C_\delta d_0 \gamma_2^2 N^{\frac{1}{2}(1 - \xi)}}{\log \left ( \frac{2 \varepsilon \sqrt{N}}{\gamma_2} \right )},
\end{equation}
for some constant \(C_\delta >0\), then for \(N\) sufficiently large,
\[
\inf_{\boldsymbol{X}_0 \in \mathcal{C}_1(\gamma_1,\gamma_2)} \mathbb{P}_{\boldsymbol{X}_0^+} \left( \mathcal{T}_{E_1(\varepsilon, c_0)} \lesssim T_1 \right) \ge 1 - \eta,
\]
where 
\[
T_1 = \frac{N^{\xi/2} \log \left ( \frac{2 \varepsilon \sqrt{N}}{\gamma_2} \right) ^2}{C_\delta d_0 (1-c_0-\varepsilon^2) \gamma_2^2 \lambda_1^2},
\]
and 
\[
\eta = K_1 \frac{\log(N)^2 N^{\frac{1}{2} \xi}}{d_0} e^{- c_1 N} + K_2 e^{- c_2 c_0^2 N^{\frac{1}{2} \xi} / d_0} + K_3 \frac{d_0^2}{c_0 \log(N) N^{ \xi}} e^{- c_3 / d_0^2}.
\]
\end{lem}

The assumption \(\kappa_1 > \sqrt{2}-1\) guarantees a sufficiently large spectral gap between the first two SNRs, ensuring that as the dominant correlation \(m_{11}\) grows to order \(\varepsilon\), the correlations \(m_{ij}\) for \(2 \le i,j \le r\) remain stable rather than decaying. This allows the iterates to enter the first recovery region \(E_1(\varepsilon,c_0)\) with high probability. We next show that, once the first recovery event \(E_1(\varepsilon, c_0)\) has been attained, the second recovery region \(E_2 (\varepsilon, c_0)\) can also be reached with high probability. More generally, the next result establishes the inductive step: once \(E_{k-1} (\varepsilon, c_0)\) has been reached, the algorithm enters \(E_k (\varepsilon, c_0)\) within \(\mathcal{O} \left (\log(N)^2 N^{\xi/2} \right)\) iterations, with high probability. 

\begin{lem} \label{lem: kth step SGD p=2}
Assume \(\kappa > \sqrt{2}-1\). Under the same parameter choices as in Lemma~\ref{lem: first step SGD p=2}---in particular, with fixed \(\varepsilon \in (0, \bar{\varepsilon}(\kappa))\) and \(c_0 = c_0 (\kappa, \varepsilon)\), \(d_0 = d_0(N)\), and step size \(\delta = \frac{C_\delta d_0 \gamma_2^2 N^{\frac{1}{2}(1-\xi)}}{\log \left ( \frac{2 \varepsilon \sqrt{N}}{\gamma_2} \right )}\)---then, for \(N\) sufficiently large and every \(2 \le k \le r\),
\[
\inf_{\boldsymbol{X}_0 \in E_{k-1} (\varepsilon,c_0)} \mathbb{P}_{\boldsymbol{X}_0} \left( \mathcal{T}_{E_k (\varepsilon,c_0)} \lesssim T_k \right) \ge 1 - K \eta,
\]
where
\[
T_k = \frac{\xi_{kk}^{(k-1)} \log(N)^2 N^{\xi/2}}{C_\delta d_0 (1-c_0-\varepsilon^2) \lambda_k^2},
\]
and \(\eta\) is given by Lemma~\ref{lem: first step SGD p=2}.
\end{lem}

Combining Lemmas~\ref{lem: first step SGD p=2} and~\ref{lem: kth step SGD p=2} via the strong Markov property yields the full recovery result stated in Theorem \ref{thm: strong recovery online p=2 nonasymptotic}. The proofs of both lemmas follow the same strategy as in Lemma~\ref{lem: recovery first spike SGD} for \(p \ge 3\), with additional arguments to account for the quadratic interaction terms appearing when \(p=2\). In particular, the recovery regions \(E_i\) highlight a key distinction between the two regimes: for \(p=2\), the cross-correlations \(m_{ij}\) with \(i \neq j\) grow beyond their initial scale \(\Theta (N^{-1/2})\), in contrast to the higher-order case \(p \ge 3\).

\begin{proof}[\textbf{Proof of Lemma~\ref{lem: first step SGD p=2}}]
Since \(\boldsymbol{X}_0\) satisfies Condition \(1\), for every \(i,j \in [r]\), there exists \(\gamma_{ij} \in (\gamma_2,\gamma_1)\) such that \(m_{ij}(0) = \gamma_{ij} N^{-\frac{1}{2}}\). Let \(\delta\) denote the step size given by~\eqref{eq: step size p=2}, and let \(T_1 >0\) be a fixed time horizon to be chosen later. According to Proposition~\ref{prop: inequalities SGD}, for every \(t \leq T_1\), we have
\[
\ell_{ij}(t) \leq m_{ij}(t) \leq u_{ij}(t),
\]
with \(\mathbb{P}_{\boldsymbol{X}_0^+}\)-probability at least \(1 - \eta\), where \(\eta = \eta(\delta, c_0, T_1, K_N)\) is given by~\eqref{eq: eta}. The comparison functions \(\ell_{ij}\) and \(u_{ij}\) are defined by 
\[
\begin{split}
\ell_{ij}(t) & = (1 - c_0) \frac{\gamma_{ij}}{\sqrt{N}} + \frac{\delta}{N} \sum_{\ell=1}^t \left ( -  \langle \boldsymbol{v}_i ,(\nabla_{\textnormal{St}}\Phi(\boldsymbol{X}_{\ell-1}))_j \rangle  - h_i(\ell-1) \right )\\
u_{ij}(t) & = (1 + c_0) \frac{\gamma_{ij}}{\sqrt{N}} + \frac{\delta}{N} \sum_{\ell=1}^t \left ( -  \langle \boldsymbol{v}_i ,(\nabla_{\textnormal{St}}\Phi(\boldsymbol{X}_{\ell-1}))_j \rangle  + h_i(\ell-1) \right ),
\end{split}
\]
respectively, where for all \(0 \le t \le T_1\),
\[
- \langle \boldsymbol{v}_i ,(\nabla_{\textnormal{St}}\Phi(\boldsymbol{X}_t))_j \rangle = 2 \sqrt{N} \lambda_i \lambda_j m_{ij}(t) - \sqrt{N} \sum_{1 \le k, \ell \le r} \lambda_k (\lambda_j + \lambda_\ell) m_{i \ell}(t) m_{k j}(t) m_{k \ell}(t),
\]
and 
\[
h_i(t) = \frac{\delta \alpha_2}{2} \sum_{k=1}^r |m_{ik}(t)| + \frac{\delta^2 \alpha_2}{2N} \sum_{k=1}^r |\langle \boldsymbol{v}_i ,(\nabla_{\textnormal{St}}\Phi(\boldsymbol{X}_t))_k \rangle | + \frac{\delta^2 r \alpha_2 K_N}{2N}.
\]
Next, we introduce the microscopic thresholds \(\varepsilon_N^+ = \frac{\gamma_{+}}{\sqrt{N}}\) and \(\varepsilon_N^- = \frac{\gamma_{-}}{\sqrt{N}}\) for some constants \(\gamma_+ > \gamma_1\) and \(\gamma_- < \gamma_2\), and recall the definition of the hitting times \(\mathcal{T}^{(ij,+)}_{\varepsilon_N^+}\) and \(\mathcal{T}_{\varepsilon_N^-}^{(ij,-)}\) given in~\eqref{eq: hitting times}. Let \(\mathcal{T}_{E_0}\) denote the hitting time of the set
\[
E_0 = \left \{ \boldsymbol{X} \colon m_{11}(\boldsymbol{X}) \ge \varepsilon_N^+ \enspace \textnormal{and} \enspace m_{11}(\boldsymbol{X}) >2 \max_{(i,j) \neq (1,1)} m_{ij}(\boldsymbol{X}) \right \}.
\]
We establish the result via two estimates. First, we show that 
\begin{equation}  \label{eq: step 1 p=2}
\inf_{\boldsymbol{X}_0 \in \mathcal{C}_1(\gamma_1,\gamma_2)} \mathbb{P}_{\boldsymbol{X}_0^+} \left ( \mathcal{T}_{E_0} \leq \frac{ \log \left ( \frac{\gamma_+}{(1-c_0)\gamma_{11}} \right ) \log \left (\frac{ 2\varepsilon \sqrt{N}}{\gamma_2} \right )}{ C_\delta d_0 (1-c_0) \gamma_2^2 \lambda_1^2}  N^{\frac{1}{2} \xi} \right ) \geq 1 - \tilde{\eta}_1,
\end{equation}
where
\[
\tilde{\eta}_1 = \frac{C_1 \log(N) N^{\frac{1}{2} \xi}}{d_0} e^{-c_1N} + C_2 e^{-c_2 c_0^2 \log(N) N^{\frac{1}{2} \xi}/d_0} +  \frac{C_3 d_0^2}{c_0 \log(N)^2 N^{\xi}} e^{-c_3/d_0^2}.
\]
Second, we show that
\begin{equation} \label{eq: step 2 p=2}
\inf_{\boldsymbol{X} \in E_0} \mathbb{P}_{\boldsymbol{X}} \left ( \mathcal{T}_{E_1(\varepsilon,c_0)} \leq \frac{ \log \left ( \frac{\varepsilon \sqrt{N}}{(1-c_0)\gamma_+} \right ) \log \left ( \frac{ 2\varepsilon \sqrt{N}}{\gamma_2} \right) N^{\xi/2} }{C_\delta d_0 (1-c_0-\varepsilon^2) \gamma_2^2 \lambda_1^2}  \right ) \geq 1 - \tilde{\eta}_2,
\end{equation}
where
\[
\tilde{\eta}_2 = \frac{C_4 \log(N)^2 N^{\frac{1}{2} \xi}}{d_0} e^{-c_4 N} + C_5 e^{-c_5 c_0^2 N^{\frac{1}{2} \xi}/d_0} +  \frac{C_6 d_0^2 N^{ - \xi}}{c_0 \log(N)} e^{-c_6/d_0^2}.
\]
Combining~\eqref{eq: step 1 p=2} and~\eqref{eq: step 2 p=2} via the strong Markov property at stopping time \(\mathcal{T}_{E_0}\) yields the stated result. The remainder of the proof is devoted to establishing the two bounds above.

\begin{proof}[Proof of~\eqref{eq: step 1 p=2}]
We observe that for all \(t \leq \min_{1 \le i,j \le r} \mathcal{T}_{\varepsilon_N^+}^{(ij,+)} \wedge \min_{1 \le i,j \le r} \mathcal{T}_{\varepsilon_N^-}^{(ij,-)}\) and every \(i,j \in [r]\), it holds that
\[
2 \sqrt{N} \lambda_i \lambda_j m_{ij}(t) - 2 \sqrt{N} \lambda_1^2 r^2 \left(\frac{\gamma_+}{\sqrt{N}}\right)^3 \leq - \langle \boldsymbol{v}_i ,(\nabla_{\textnormal{St}}\Phi(\boldsymbol{X}_t))_j \rangle \leq 2 \sqrt{N} \lambda_i \lambda_j m_{ij}(t) - 2 \sqrt{N} \lambda_r^2 r^2 \left(\frac{\gamma_{-}}{\sqrt{N}} \right)^3.
\]
This leads to the following estimates for the comparison functions \(u_{ij}(t)\) and \(\ell_{ij}(t)\):
\begin{equation} \label{eq: upper bound u_ij}
\begin{split}
u_{ij}(t) & \leq (1 + c_0) \frac{\gamma_{ij}}{\sqrt{N}}+ \frac{\delta}{N} \sum_{\ell=1}^t \Big \{ 2 \sqrt{N} \lambda_i\lambda_j m_{ij}(\ell-1) - 2 \sqrt{N} \lambda_r^2 r^2 \left(\frac{\gamma_{-}}{\sqrt{N}}\right)^3  \\
& \quad + \frac{\delta \alpha_2r}{2} \max_{1 \le k \le r} |m_{ik}(\ell-1)| + \frac{\delta^2 \alpha_2 r}{2N} \left ( 2 \gamma_+ \lambda_i \lambda_1  + 2 \sqrt{N}  \lambda_1^2 r^2 \left ( \frac{\gamma_+}{\sqrt{N}}\right)^3 \right) + \frac{\delta^2 r \alpha_2 K_N}{2N} \Big \},
\end{split}
\end{equation}
and
\begin{equation} \label{eq: lower bound l_ij}
\begin{split}
\ell_{ij}(t) & \geq (1-c_0) \frac{\gamma_{ij}}{\sqrt{N}} + \frac{\delta}{N} \sum_{\ell=1}^t \Big \{ 2 \sqrt{N} \lambda_i\lambda_j m_{ij}(\ell-1) - 2 \sqrt{N} \lambda_1^2 r^2 \left(\frac{\gamma_{+}}{\sqrt{N}}\right)^3  \\
& \quad - \frac{\delta \alpha_2r}{2} \max_{1 \le k \le r} |m_{ik}(\ell-1)| - \frac{\delta^2 \alpha_2 r}{2 N} \left ( 2 \gamma_+ \lambda_i \lambda_1 + 2 \sqrt{N}  \lambda_1^2 r^2 \left ( \frac{\gamma_+}{\sqrt{N}}\right)^3 \right) - \frac{\delta^2 r \alpha_2 K_N}{2N} \Big \}.
\end{split}
\end{equation}
Both bounds hold for every \(t \leq \min_{1 \leq i,j \leq r} \mathcal{T}_{\varepsilon_N^+}^{(ij,+)} \wedge \min_{1 \leq i,j \leq r} \mathcal{T}_{\varepsilon_N^-}^{(ij,-)} \wedge T_1\), with \(\mathbb{P}_{\boldsymbol{X}_0^+}\)-probability at least \(1 - \eta(\delta,c_0,T_1,K_N)\). We now focus on~\eqref{eq: lower bound l_ij} and claim that
\begin{equation} \label{eq: claim l_ij p=2 SGD}
\ell_{ij}(t) \geq (1-c_0) \frac{\gamma_{ij}}{\sqrt{N}} + 2 (1 - c_0) \lambda_i \lambda_j \frac{\delta}{\sqrt{N}} \sum_{\ell=1}^t m_{ij}(\ell-1),
\end{equation}
for every \(t \leq \min_{1 \leq i,j \leq r} \mathcal{T}_{\varepsilon_N^+}^{(ij,+)} \wedge \min_{1 \leq i,j \leq r} \mathcal{T}_{\varepsilon_N^-}^{(ij,-)} \wedge T_1\). To see the claim, we first note that, on the considered time interval,
\[
\ell_{ij}(t) \ge (1-c_0) \frac{\gamma_{ij}}{\sqrt{N}} + \frac{\delta}{N} \sum_{\ell=1}^t \left \{ 2 \sqrt{N} \lambda_i\lambda_j m_{ij}(\ell-1)  -  \frac{3\delta \alpha_2r}{2} \max_{1 \le k \le r} |m_{ik}(\ell-1)| - \frac{\delta^2 r \alpha_2 K_N}{2N} \right \},
\]
where we used the prescribed value~\eqref{eq: step size p=2} for \(\delta\). Claim~\eqref{eq: claim l_ij p=2 SGD} then follows by two sufficient conditions. First, note that over the considered time interval, 
\[
c_0 \sqrt{N} \lambda_i\lambda_j m_{ij}(t) \geq c_0 \lambda_r^2 \gamma_- \ge \frac{3 \delta \alpha_2 r}{2} \frac{\gamma_+}{\sqrt{N}} \ge  \frac{3 \delta \alpha_2 r}{2}  \max_{1 \le k \le r} |m_{ik}(t)|, 
\]
where the second inequality holds, provided 
\begin{equation} \label{eq: cond d_0 p=2}
d_0 \leq \frac{2 c_0 \lambda_r^2 \gamma_-}{3 C_\delta \alpha_2 r \gamma_+ \gamma_2^2} N^{\frac{1}{2} \xi} \log \left(\frac{2\varepsilon\sqrt{N}}{\gamma_2}\right),
\end{equation}
which certainly holds since \(d_0 \ll c_0^2\). Second, we have 
\[
c_0 \sqrt{N} \lambda_i \lambda_j m_{ij}(t) \ge c_0 \lambda_r^2 \gamma_- \ge \frac{\delta^2 r \alpha_2 K_N}{2N},
\]
for every \(t \leq \min_{1 \leq i,j \leq r} \mathcal{T}_{\varepsilon_N^+}^{(ij,+)} \wedge \min_{1 \leq i,j \leq r} \mathcal{T}_{\varepsilon_N^-}^{(ij,-)}\), where the second inequality holds provided
\begin{equation}  \label{eq: cond d_0K_N p=2}
d_0^2 K_N \leq \frac{2 c_0 \lambda_r^2 \gamma_-}{C_\delta^2 \alpha_2 r \gamma_2^4} N^{\xi} \log \left(\frac{2\varepsilon\sqrt{N}}{\gamma_2}\right)^2.
\end{equation}
According to the sufficient conditions~\eqref{eq: cond d_0 p=2} and~\eqref{eq: cond d_0K_N p=2}, we can choose the truncation level \(K_N  = \frac{1}{d_0}\). Claim~\eqref{eq: claim l_ij p=2 SGD} then easily follows. A similar reasoning for~\eqref{eq: upper bound u_ij} yields  
\begin{equation} \label{eq: claim u_ij p=2 SGD}
u_{ij}(t) \leq (1+c_0) \frac{\gamma_{ij}}{\sqrt{N}} +  2 (1+c_0) \lambda_i \lambda_j \frac{\delta}{\sqrt{N}} \sum_{\ell=1}^t m_{ij}(\ell-1),
\end{equation} 
for every \(t \leq \min_{1 \leq i,j \leq r} \mathcal{T}_{\varepsilon_N^+}^{(ij,+)} \wedge \min_{1 \leq i,j \leq r} \mathcal{T}_{\varepsilon_N^-}^{(ij,-)} \wedge T_1\). According to~\eqref{eq: claim l_ij p=2 SGD} and~\eqref{eq: claim u_ij p=2 SGD} both lower- and upper-bounding functions are increasing, thus each correlation \(m_{ij}\) is increasing on the considered time interval, and \(\min_{1 \leq i,j \leq r} \mathcal{T}_{\varepsilon_N^+}^{(ij,+)}  \wedge T_1 \le \min_{1 \leq \ell \leq r} \mathcal{T}_{\varepsilon_N^-}^{(ij,-)}\). By applying item (a) of Lemma~\ref{lem: discrete growth}, we obtain the following two-sided comparison inequality:
\begin{equation} \label{eq: comparison inequality p=2 SGD}
\frac{(1-c_0)\gamma_{ij}}{\sqrt{N}}\left(1+\frac{2\delta}{\sqrt{N}}(1-c_0)\lambda_i\lambda_j\right)^t  \leq m_{ij}(t) \leq \frac{(1+c_0)\gamma_{ij}}{\sqrt{N}}\left(1+\frac{2\delta}{\sqrt{N}}(1+c_0)\lambda_i\lambda_j\right)^t ,
\end{equation}
for all \(t \leq \min_{1 \leq i,j \leq r} \mathcal{T}_{\varepsilon_N^+}^{(ij,+)} \wedge T_1\). In particular, the hitting times \(\mathcal{T}_{\varepsilon_N^+}^{(ij,+)}\) satisfy
\[
T_{\ell,\varepsilon_N^+}^{(ij)} \le \mathcal{T}_{\varepsilon_N^+}^{(ij,+)} \leq T_{u, \varepsilon_N^+}^{(ij)},
\]
where 
\[
\begin{split}
T_{\ell,\varepsilon_N^+}^{(ij)} & = \frac{\log\left(\frac{\gamma_+}{(1+c_0)\gamma_{ij}}\right)}{\log\left(1+\frac{2\delta}{\sqrt{N}}(1+c_0)\lambda_i\lambda_j\right)}\\
T_{u,\varepsilon_N^+}^{(ij)} &= \frac{\log\left(\frac{\gamma_+}{(1-c_0)\gamma_{ij}}\right)}{\log\left(1+\frac{2\delta}{\sqrt{N}}(1-c_0)\lambda_i\lambda_j\right)}.
\end{split}
\]
We claim that, for all \((i,j) \neq (1,1)\), \(T_{u, \varepsilon_N^+}^{(11)} \leq T_{\ell,\varepsilon_N^+}^{(ij)}\), and hence \(\mathcal{T}_{\varepsilon_N^+}^{(11,+)} \le \mathcal{T}_{\varepsilon_N^+}^{(ij,+)}\), provided that \(\gamma_+\) is chosen sufficiently large but still of order \(1\). Since by assumption \(c_0 < \frac{1}{2} \wedge \frac{\kappa_1}{2+\kappa_1}\), we have \((1-c_0) (1+\kappa_1) > (1+c_0)\). Define
\[
c(\kappa_1) = \frac{\log\left(1+\frac{2\delta}{\sqrt{N}}(1-c_0)(1+\kappa_1)\lambda_1\lambda_2\right)}{ \log \left(1+\frac{2\delta}{\sqrt{N}}(1+c_0)\lambda_1\lambda_2\right)} > 1.
\]
Choosing \(\gamma_+ > \gamma_1\) and such that
\begin{equation} \label{eq: gamma_+ for p=2}
\log(\gamma_+) > \frac{c(\kappa_1) \log((1+c_0) \gamma_1) - \log((1-c_0)\gamma_2)}{c(\kappa_1) -1},
\end{equation}
ensures 
\[
\begin{split}
T_{u,\varepsilon_N^+}^{(11)} \le \frac{\log \left(\frac{\gamma_+}{(1-c_0) \gamma_2}\right)}{\log\left(1+\frac{2\delta}{\sqrt{N}}(1-c_0) \lambda_1^2 \right)} \le \min_{(i,j) \neq (1,1)} \frac{\log \left(\frac{\gamma_+}{(1+c_0) \gamma_1}\right)}{\log\left(1+\frac{2\delta}{\sqrt{N}}(1+c_0)\lambda_i \lambda_j\right)} \leq \min_{(i,j) \neq (1,1)} T_{\ell,\varepsilon_N^+}^{(ij)} .
\end{split}
\]
The inequality \(T_{u, \varepsilon_N^+}^{(11)} \leq T_{\ell,\varepsilon_N^+}^{(ij)}\) guarantees that the correlation \(m_{11}\) reaches the microscopic threshold \(\varepsilon_N^+\) before any other \(m_{ij}\), i.e., \(\min_{1 \le i,j \le r} \mathcal{T}_{\varepsilon_N^+}^{(ij,+)} = \mathcal{T}_{\varepsilon_N^+}^{(11,+)}\). Finally, we can choose \(\gamma_+ > \gamma_1\) satisfying~\eqref{eq: gamma_+ for p=2} such that
\[
\max_{(i,j) \neq (1,1)} m_{ij}(\mathcal{T}_{\varepsilon_N^+}^{(11,+)}) \le \frac{1}{2} m_{11}(\mathcal{T}_{\varepsilon_N^+}^{(11,+)}) = \frac{1}{2} \varepsilon_N^+.
\]
This implies that \(\mathcal{T}_{E_0} = \mathcal{T}_{\varepsilon_N^+}^{(11, +)}\). Using the inequality \(\log(1+x) \ge x/2\) for \(x \in (0,1)\)---which is applicable here because of the prescribed choice of \(\delta\) in~\eqref{eq: step size p=2}---we obtain the explicit upper bound
\[
\mathcal{T}_{\varepsilon_N^+}^{(11, +)} \le T_{u, \varepsilon_N^+}^{(11)} \le \frac{\log\left(\frac{\gamma_+}{(1-c_0)\gamma_{11}}\right)  \log\left(\frac{2 \varepsilon \sqrt{N}}{\gamma_2}\right)}{C_\delta d_0 (1-c_0) \gamma_2^2 \lambda_1^2} N^{\frac{1}{2} \xi},
\]
with \(\mathbb{P}_{\boldsymbol{X}_0^+}\)-probability at least \(1 - \tilde{\eta}_1\), where \(\tilde{\eta}_1 = \eta (\delta, c_0, T_1, K_N)\), \(T_1=T_{u,\varepsilon_N^+}^{(11)}\), and \(K_N = 1/d_0\). This completes the proof of~\eqref{eq: step 1 p=2}.
\end{proof}

\begin{proof}[Proof of~\eqref{eq: step 2 p=2}]
We now assume that \(\boldsymbol{X}_0 \in E_0\). We let \(T_1 >0\) be a time horizon to be chosen later, and define the hitting time \(\mathcal{T}_{\textnormal{bad}}\) by
\[
\mathcal{T}_{\textnormal{bad}} = \inf \left \{ t \in \N_0 \colon \max_{1 \le i,j \le r} m_{ij}(t) > \frac{1}{2} m_{11}(t) \right \}.
\]
Since \(\boldsymbol{X}_0 \in E_0\) implies \(m_{ij}(0) \le \frac{1}{2} m_{11}(0)\) for all \((i,j) \neq (1,1)\), we have \(\mathcal{T}_{\textnormal{bad}} > 0\). In what follows, we study the evolution of the correlations \(\{m_{ij}(t)\}_{1 \le i,j \le r}\) for \( t \le \min_{1 \le i,j \le r} \mathcal{T}_\varepsilon^{(ij, +)} \wedge \mathcal{T}_{\varepsilon_N^-}^{(11,-)} \wedge \min_{2 \le i,j \le r} \mathcal{T}_{\varepsilon_N^-}^{(ij,-)} \wedge \mathcal{T}_{\textnormal{bad}} \wedge T_1\). We note that \(\min_{1 \le i,j \le r} \mathcal{T}_\varepsilon^{(ij, +)}  \wedge \mathcal{T}_{\textnormal{bad}} = \mathcal{T}_\varepsilon^{(11, +)}  \wedge \mathcal{T}_{\textnormal{bad}}\). By Proposition~\ref{prop: inequalities SGD}, with \(\mathbb{P}_{\boldsymbol{X}_0}\)-probability at least \(1-\eta (\delta, c_0, T_1, K_N)\), where \(\eta\) is given in~\eqref{eq: eta},
\[
\begin{split}
m_{ij}(t) & \geq (1 - c_0) m_{ij}(0) + \frac{\delta}{N} \sum_{\ell=1}^t  \left ( - \langle \boldsymbol{v}_i, \left (\nabla_{\textnormal{St}}\Phi(\boldsymbol{X}_{\ell-1}) \right )_j \rangle - h_i(\ell-1) \right ) ,\\
m_{ij}(t) &\leq (1 +c_0) m_{ij}(0) + \frac{\delta}{N} \sum_{\ell=1}^t  \left ( - \langle \boldsymbol{v}_i, \left (\nabla_{\textnormal{St}}\Phi(\boldsymbol{X}_{\ell-1}) \right )_j \rangle +  h_i(\ell-1)  \right ),
\end{split}
\] 
for all \( t \le  \mathcal{T}_\varepsilon^{(11, +)} \wedge \mathcal{T}_{\varepsilon_N^-}^{(11,-)} \wedge \min_{2 \le i,j \le r} \mathcal{T}_{\varepsilon_N^-}^{(ij,-)} \wedge \mathcal{T}_{\textnormal{bad}} \wedge T_1\). We continue to use the same constant \(c_0 \in (0, \frac{1}{2})\) as in the previous step, independent of \(N\), and chosen sufficiently small to satisfy all preceding constraints. On this time interval and for sufficiently large \(N\), for all \(i \in [r]\), we have
\begin{equation} \label{eq: bound h_i p=2}
h_i(t) \le \delta \alpha_2 r \max_{1 \le k \le r} |m_{ik}(t)|.
\end{equation}
This bound follows by using the prescribed value~\eqref{eq: step size p=2} for \(\delta\).\\

We first focus on the evolution of the dominant correlation \(m_{11}\) on this time interval. We have the straightforward upper bound
\[
- \langle \boldsymbol{v}_1, \left (\nabla_{\textnormal{St}}\Phi(\boldsymbol{X}_t) \right )_1 \rangle \le 2\sqrt{N} \lambda_1^2 m_{11}(t) ,
\]
and the lower bound
\[
- \langle \boldsymbol{v}_1, \left (\nabla_{\textnormal{St}}\Phi(\boldsymbol{X}_t) \right )_1 \rangle \ge 2\sqrt{N} \lambda_1^2 m_{11}(t) \left (1 - m_{11}^2(t) - C(r^2-1)N^{-1}\right) \ge \left ( 1  - \varepsilon^2 - \frac{c_0}{2}\right) 2\sqrt{N} \lambda_1^2 m_{11}(t),
\]
for \(N\) sufficiently large. Furthermore, from~\eqref{eq: bound h_i p=2} we have
\[
h_{1}(t) \le \delta \alpha_2 r m_{11}(t) \le c_0 \sqrt{N} \lambda_1^2 m_{11}(t) ,
\]
where the second inequality holds trivially by the fact that \(\delta \in \mathcal{O} \left( d_0 N^{\frac{1}{2} (1-\xi)} \log(N)^{-1} \right)\). Therefore, on the considered time interval,
\[
\begin{split}
m_{11}(t) &\ge (1-c_0) \varepsilon_N^+  + \frac{2\delta}{\sqrt{N}} (1-c_0-\varepsilon^2) \lambda_1^2 \sum_{\ell=1}^t m_{11}(\ell-1),\\
m_{11}(t) &\le (1+c_0) \varepsilon_N^+  + \frac{2\delta}{\sqrt{N}} (1+c_0) \lambda_1^2 \sum_{\ell=1}^t m_{11}(\ell-1).
\end{split}
\]
Applying Lemma~\ref{lem: discrete growth} gives
\begin{equation} \label{eq: corr m_11 p=2}
\frac{(1-c_0)\gamma_+}{\sqrt{N}} \left(1+\frac{2\delta}{\sqrt{N}}(1-c_0-\varepsilon^2)\lambda_1^2 \right)^t \le m_{11}(t) \le \frac{(1+c_0)\gamma_+}{\sqrt{N}} \left(1+\frac{2\delta}{\sqrt{N}}(1+c_0)\lambda_1^2 \right)^t .
\end{equation}
Since both upper and lower bounding functions are increasing, we deduce that \(m_{11}\) is increasing in \(t\), thus
\[
\min_{1 \le i,j \le r} \mathcal{T}_\varepsilon^{(ij, +)} \wedge \min_{2 \le i,j \le r} \mathcal{T}_{\varepsilon_N^-}^{(ij,-)} \wedge \mathcal{T}_{\textnormal{bad}} \wedge T_1 < \mathcal{T}_{\varepsilon_N^-}^{(11,-)}.
\]
Furthermore, from~\eqref{eq: corr m_11 p=2}, we obtain the following lower and upper bounds of \(\mathcal{T}_\varepsilon^{(11,+)}\):
\begin{equation} \label{eq: lower bound T^(11) p=2}
\frac{\log\left(\frac{\varepsilon\sqrt{N}}{(1+c_0)\gamma_+}\right)}{\log\left(1+\frac{2\delta}{\sqrt{N}}(1+c_0)\lambda_1^2 \right)} \le \mathcal{T}_\varepsilon^{(11,+)} \le  \frac{\log\left(\frac{\varepsilon\sqrt{N}}{(1-c_0)\gamma_+}\right)}{\log\left(1+\frac{2\delta}{\sqrt{N}}(1-c_0-\varepsilon^2)\lambda_1^2 \right)} .\\
\end{equation}

To derive the comparison inequalities for all pairs \((i,j) \neq (1,1)\), we start from the bound on \(h_i(t)\) in~\eqref{eq: bound h_i p=2} and estimate \(\delta \alpha_2 r \max_{1 \le k \le r} \vert m_{ik} (t) \vert \). For this purpose, for every \((i,j) \neq (1,1)\), we introduce a stopping time \(\mathcal{T}^{(ij)}_\alpha\) depending on a parameter \(\alpha \in (0,\frac{1}{2})\) to be determined later:  
\[
\mathcal{T}^{(ij)}_\alpha \coloneqq \inf \left\{ t \in \N_0 \colon C_{ij} N^\alpha m_{ij} (t) \leq \max_{1 \le k \le r} \{ \lvert m_{ik}(t) \rvert ,  \lvert m_{ki}(t) \rvert \}\right\},
\]
where \(C_{ij} >0\) is a constant. At time \(t=0\), \(\max_{k \in [r]} \{ |m_{ik}(0)|, |m_{ki}(0)|\} \le \gamma_+ N^{-1/2} \le C_{ij}N^\alpha m_{ij}(0)\) for every \(\alpha > 0\), hence \(\mathcal{T}^{(ij)}_\alpha > 0 \). For every \(t \leq \mathcal{T}_\alpha^{(ij)}\), it then follows that 
\[
\delta \alpha_2 r \max_{1 \le k \le r}  \lvert m_{ik} (t) \rvert  \leq \delta \alpha_2  r C_{ij} N^\alpha  m_{ij}(t) =  C_\delta d_0 \gamma_2^2  C_{ij} \alpha_2 r \log\left(\frac{2\varepsilon\sqrt{N}}{\gamma_2}\right)^{-1} N^{\frac{1}{2} (1-\xi)+ \alpha} m_{ij}(t),
\]
where the equality follows by plugging the value of \(\delta\) given in~\eqref{eq: step size p=2}. To ensure that this perturbation term remains smaller than the leading drift term, i.e., 
\begin{equation} \label{eq: bound h_i p=2 bis}
\delta \alpha_2 r \max_{1 \le k \le r} \vert m_{ik}(t) \vert \le c_0 \sqrt{N} \lambda_i \lambda_j m_{ij}(t),
\end{equation}
it suffices to impose
\[
c_0 \lambda_i \lambda_j \gamma_- \geq   C_\delta d_0 \gamma_2^2  C_{ij} \alpha_2 r \log\left(\frac{2\varepsilon\sqrt{N}}{\gamma_2}\right)^{-1} N^{\alpha - \xi/2} ,
\]
which holds whenever
\begin{equation} \label{eq: alpha critical}
\alpha \leq \frac{\xi}{2} = \frac{1}{2} \left (1  - \frac{1-c_0}{1+c_0} \frac{\lambda_r^2}{\lambda_1^2} \right ). 
\end{equation}
Consequently, for all \(t \le  \mathcal{T}_\varepsilon^{(11,+)} \wedge \min_{2 \le i,j \le r} \mathcal{T}_{\varepsilon_N^-}^{(ij,-)} \wedge \mathcal{T}_{\mathrm{bad}} \wedge \min_{1 \le i,j \le r} \mathcal{T}_\alpha^{(ij)} \wedge T_1\),
\[
- \langle \boldsymbol{v}_i, (\nabla_{\text{St}} \Phi (\boldsymbol{X}_t))_j\rangle + \delta \alpha_2 r \max_{1 \le k \le r} |m_{ik}(t)| \le 2(1+c_0) \lambda_i \lambda_j m_{ij} (t),
\]
yielding
\[
\begin{split}
m_{ij}(t)  &\leq   (1 + c_0) \frac{\varepsilon_N^+}{2} + 2 ( 1 + c_0) \lambda_i\lambda_j \frac{\delta}{\sqrt{N}} \sum_{\ell = 1}^t  m_{ij}(\ell-1)  \\
& \le\frac{(1 + c_0) \gamma_+}{2 \sqrt{N}} \left ( 1 + 2 ( 1 + c_0) \lambda_i\lambda_j \frac{\delta}{\sqrt{N}}\right)^t \\
& \le  \frac{(1 + c_0) \gamma_+}{2 \sqrt{N}} \exp \left ( 2 ( 1 + c_0) \lambda_i\lambda_j \frac{\delta}{\sqrt{N}} t \right ),
\end{split}
\]
where the second inequality follows by applying Lemma~\ref{lem: discrete growth} and the third by using \(\log(1+x) \le x\). Evaluating at \(t = \mathcal{T}_\varepsilon^{(11,+)}\) and using the upper bound on this hitting time from~\eqref{eq: lower bound T^(11) p=2}, for \(N\) sufficiently large we obtain
\begin{equation} \label{eq: upper bound m_ij p=2}
m_{ij}(\mathcal{T}_\varepsilon^{(11,+)}) \lesssim \frac{(1 + c_0) \gamma_+}{2}  \left (\frac{\varepsilon}{(1-c_0) \gamma_+} \right )^{\frac{1+c_0}{1-c_0-\varepsilon^2} \frac{\lambda_i \lambda_j}{\lambda_1^2}} N^{-\frac{1}{2} \left ( 1 - \frac{1+c_0}{1-c_0-\varepsilon^2} \frac{\lambda_i \lambda_j}{\lambda_1^2}\right )} \lesssim N^{-\frac{1}{2} \delta_{ij}^{(1)}},
\end{equation}
where we used the fact that
\[
\mathcal{T}_\varepsilon^{(11,+)} \le \frac{\log\left(\frac{\varepsilon\sqrt{N}}{(1-c_0)\gamma_+}\right)}{\log\left(1+\frac{2\delta}{\sqrt{N}}(1-c_0-\varepsilon^2)\lambda_1^2 \right)} \le \frac{\log\left(\frac{\varepsilon\sqrt{N}}{(1-c_0)\gamma_+}\right)}{ \frac{2\delta}{\sqrt{N}} (1-c_0-\varepsilon^2)  \lambda_1^2 \left( 1 - \frac{\delta}{\sqrt{N}} (1-c_0-\varepsilon^2)^2  \lambda_1^4 \right)} ,
\]
since \(\log(1+ x) \ge x - x^2/2\) for \(x \in (0,1)\), which is satisfied for large \(N\) by assumption on \(\delta\). This implies that \(\mathcal{T}_\varepsilon^{(11,+)} \wedge \min_{2 \le i,j \le r} \mathcal{T}_{\varepsilon_N^-}^{(ij,-)} \wedge \min_{1 \le i,j \le r} \mathcal{T}_\alpha^{(ij)} \wedge T_1 \le \mathcal{T}_{\textnormal{bad}}\). \\

Next, we need to derive a lower bound on \(m_{ij}(t)\) for all \(2 \le i,j \le r\). Using the bound for \(h_i(t)\) in~\eqref{eq: bound h_i p=2} and~\eqref{eq: bound h_i p=2 bis}, for all \(t \le \mathcal{T}_\varepsilon^{(11,+)} \wedge \min_{2 \le i,j \le r} \mathcal{T}_{\varepsilon_N^-}^{(ij,-)} \wedge \min_{1 \le i,j \le r} \mathcal{T}_\alpha^{(ij)} \wedge T_1\), we obtain
\[
\begin{split}
- \langle \boldsymbol{v}_i, (\nabla_{\text{St}} \Phi (\boldsymbol{X}_t))_j\rangle - \delta \alpha_2 r \max_{1 \le k \le r} |m_{ik}(t)| & \ge  2\sqrt{N} \left (1- \frac{c_0}{2} \right) \lambda_i \lambda_j m_{ij} (t) -  2 \sqrt{N} r^2 \lambda_1^2 m_{i1}(t) m_{1j} (t) m_{11} (t) \\
& \ge 2 \sqrt{N} m_{ij}(t) \left [ \left (1- \frac{c_0}{2} \right) \lambda_i \lambda_j  -  C r^2 \lambda_1^2 \varepsilon N^{2 \alpha} m_{ij}(t)\right] \\
& \gtrsim 2 \sqrt{N} m_{ij}(t) \left [ \left (1- \frac{c_0}{2} \right) \lambda_i \lambda_j  -  C r^2 \lambda_1^2 \varepsilon N^{2 \alpha - \frac{1}{2} \delta_{ij}^{(1)}} \right],
\end{split}
\]
where \(C>0\) is a constant independent of \(N\) and where we used~\eqref{eq: upper bound m_ij p=2} for the last inequality, which implies that \(m_{ij}(t) \le m_{ij}(\mathcal{T}_\varepsilon^{(11,+)}) \lesssim N^{- \frac{1}{2} \delta_{ij}^{(1)}}\). Then, 
\[
- \langle \boldsymbol{v}_i, (\nabla_{\text{St}} \Phi (\boldsymbol{X}_t))_j\rangle - \delta \alpha_2 r \max_{1 \le k \le r} |m_{ik}(t)| \ge 2\sqrt{N} \left (1- c_0 \right) \lambda_i \lambda_j m_{ij} (t) ,
\]
provided the sufficient condition
\begin{equation} \label{eq: alpha critical 2}
\alpha \le \frac{1}{4} \sup_{2 \le i,j \le r} \delta_{ij}^{(1)} = \frac{1}{4} \delta_{22}^{(1)}.
\end{equation}
Note that this condition is stronger than~\eqref{eq: alpha critical}. Thus, we obtain 
\[
m_{ij}(t) \ge (1 - c_0) \varepsilon_N^- + 2 ( 1 - c_0) \lambda_i\lambda_j \frac{\delta}{\sqrt{N}} \sum_{\ell = 1}^t  m_{ij}(\ell-1),
\]
which, by applying Lemma~\ref{lem: discrete growth}, leads to
\[
m_{ij}(t) \ge \frac{(1 - c_0) \gamma_-}{\sqrt{N}} \left ( 1 + 2 ( 1 - c_0) \lambda_i\lambda_j \frac{\delta}{\sqrt{N}}\right)^t ,
\]
for every \( t \leq \mathcal{T}_\varepsilon^{(11, +)} \wedge \min_{2 \le i,j \le r} \mathcal{T}_{\varepsilon_N^-}^{(ij,-)}  \wedge \min_{1 \le i,j \le r} \mathcal{T}_\alpha^{(ij)} \wedge T_1 \). Since \(\log(1+x) \ge x -\frac{x^2}{2}\) for all \(x \in (0,1)\), it follows that
\[
m_{ij}(t) \ge  \frac{(1 - c_0) \gamma_-}{\sqrt{N}} \exp \left ( 2 ( 1 - c_0) \lambda_i\lambda_j \frac{\delta}{\sqrt{N}} t \right ) \exp \left ( -  \frac{1}{2}( 1 - c_0)^2 \lambda_i^2 \lambda_j^2 \frac{\delta^2}{N} t \right ).
\]
Thus, for any time horizon \(T_1 \le \frac{2 \log \left ( 2 \varepsilon \sqrt{N} / \gamma_2\right)^2}{C_\delta^2 d_0 \gamma_2^4 (1-c_0)^2 \lambda_r^2}N^{\xi}\), we get
\[
m_{ij}(t) \ge  \frac{(1 - c_0) \gamma_-}{ \sqrt{N}} \exp (-d_0) \exp \left ( 2 ( 1 - c_0) \lambda_i\lambda_j \frac{\delta}{\sqrt{N}} t \right ).
\]
Since \(d_0 \ll 1\), we find
\[
m_{ij}(t) \ge  \frac{(1 - c_0) \gamma_-}{e \sqrt{N}} \exp \left ( 2 ( 1 - c_0) \lambda_i\lambda_j \frac{\delta}{\sqrt{N}} t \right ).
\]
Using the lower bound of \(\mathcal{T}_\varepsilon^{(11,+)}\) from~\eqref{eq: lower bound T^(11) p=2}, we obtain
\begin{equation} \label{eq: lower bound m_ij p=2}
m_{ij}(\mathcal{T}_\varepsilon^{(11,+)}) \ge \frac{(1 - c_0) \gamma_+}{e}  \left (\frac{\varepsilon}{(1 + c_0) \gamma_+} \right )^{\frac{1- c_0}{1+c_0} \frac{\lambda_i \lambda_j}{\lambda_1^2}} N^{-\frac{1}{2} \left ( 1 - \frac{1-c_0}{1+c_0} \frac{\lambda_i \lambda_j}{\lambda_1^2}\right )} \gtrsim N^{-\frac{1}{2} \xi_{ij}^{(1)}},
\end{equation}
The lower bound~\eqref{eq: lower bound m_ij p=2} implies that on the event up to time \(\mathcal{T}_\varepsilon^{(11, +)} \wedge \min_{1 \le i,j \le r} \mathcal{T}_\alpha^{(ij)} \wedge T_1\), we have \(m_{ij} \ge \varepsilon_N^{-}\) for all \(2 \le i,j \le r\), i.e., \(\mathcal{T}_\varepsilon^{(11, +)} \wedge \min_{1 \le i,j \le r} \mathcal{T}_\alpha^{(ij)} \wedge T_1 \le  \min_{2 \le i,j \le r} \mathcal{T}_{\varepsilon_N^-}^{(ij,-)}  \). Moreover, we see that
\[
\mathcal{T}_{E_1(\varepsilon, c_0)} = \mathcal{T}_\varepsilon^{(11, +)} \wedge \min_{1 \le i,j \le r} \mathcal{T}_\alpha^{(ij)} \wedge T_1.\\
\]

The final step shows that \(\mathcal{T}_\varepsilon^{(11,+)} \wedge T_1 \leq  \min_{1 \le i,j \le r} \mathcal{T}_\alpha^{(ij)} \) by comparing the upper and lower bounds on \(m_{ij}(\mathcal{T}_\varepsilon^{(11,+)})\). Specifically, we must verify that there exists a constant \(C_{ij} >0\) and some \(\alpha \leq \frac{1}{2} \delta_{22}^{(1)} (\varepsilon,c_0)\) such that, at time \(t = \mathcal{T}_\varepsilon^{(11,+)}\), \(\lvert m_{ik}(t) \rvert \leq C_{ij} N^\alpha m_{ij}(t)\). Using the upper and lower bounds obtained in~\eqref{eq: upper bound m_ij p=2} and~\eqref{eq: lower bound m_ij p=2}, we find for all \((i,j) \neq (1,1)\),
\[
\max_{1 \le k \le r} \{\lvert m_{ik}(\mathcal{T}_\varepsilon^{(11,+)}) \rvert,\lvert m_{ki}(\mathcal{T}_\varepsilon^{(11,+)}) \rvert\}  \leq C_i N^{-\frac{1}{2} \delta_{i1}^{(1)}}  \leq C_{ij} N^\alpha N^{-\frac{1}{2} \xi_{ir}^{(1)}}  \leq C_{ij} N^\alpha m_{ij}(\mathcal{T}_\varepsilon^{(11,+)}) ,
\]
where the second inequality holds whenever 
\[
\alpha \geq \frac{1}{2} \left ( \xi_{ir}^{(1)} - \delta_{i1}^{(1)} \right) = \frac{\lambda_i}{2 \lambda_1} \left ( \frac{1+c_0}{1-c_0-\varepsilon^2} - \frac{1-c_0}{1+c_0} \frac{\lambda_r}{\lambda_1}\right ). 
\]
In particular, a sufficient condition is
\[
\alpha \geq \alpha_{\min} \coloneqq \frac{\lambda_r}{2 \lambda_1} \left ( \frac{1+c_0}{1-c_0-\varepsilon^2} - \frac{1-c_0}{1+c_0}\frac{\lambda_r}{\lambda_1}\right ). 
\]
To guarantee the existence of such an \(\alpha\), from~\eqref{eq: alpha critical 2} we must ensure that
\[
\Delta (\varepsilon, c_0) \coloneqq \frac{1}{4} \delta_{22}^{(1)} - \alpha_{\min} > 0.
\]
Evaluating at \((\varepsilon, c_0) =(0,0)\),
\[
\Delta (0,0) = \frac{1}{4} \left (1 - \frac{\lambda_2^2}{\lambda_1^2} - 2  \frac{\lambda_r}{\lambda_1} \left (1 - \frac{\lambda_r}{\lambda_1}\right) \right) >0,
\]
which holds under the assumption \(\frac{1}{(1+\kappa_1)^2} < \frac{1}{2}\) (i.e., \(\kappa_1 > \sqrt{2}-1\)) since \(2  \frac{\lambda_r}{\lambda_1} \left (1 - \frac{\lambda_r}{\lambda_1}\right)  \le \frac{1}{2}\). Since \(\Delta (\varepsilon, c_0)\) is continuous in both variables and \(\Delta(0,0)>0\), there exist \(\bar{\varepsilon}_0 >0\) and \(\bar{c}_0>0\) such that \(\Delta (\varepsilon, c_0)  \ge \frac{1}{2} \Delta (0,0) > 0\) for all \(c_0 \in (0, \bar{c}_0)\) and \(\varepsilon_0 \in (0,\bar{\varepsilon}_0)\). Hence, for sufficiently small \((\varepsilon, c_0)\), we can find an \(\alpha \in \left (\alpha_{\min} ,\frac{1}{4} \delta_{22}^{(1)}\right)\). With this choice of \(\alpha\), we have \( \mathcal{T}_\varepsilon^{(11,+)}  \wedge T_1 \le \min_{(i,j) \neq (1,1)} \mathcal{T}_\alpha^{(ij)}\), hence \(\mathcal{T}_{E_1(\varepsilon,c_0)} = \mathcal{T}_\varepsilon^{(11,+)}  \wedge T_1\). Therefore, choosing \(T_1\) as the upper bound for \( \mathcal{T}_\varepsilon^{(11,+)} \) from~\eqref{eq: lower bound T^(11) p=2} yields
\[
\mathcal{T}_{E_1(\varepsilon,c_0)} = \mathcal{T}_\varepsilon^{(11,+)} \le  \frac{ \log \left ( \frac{\varepsilon \sqrt{N}}{(1-c_0)\gamma_+} \right ) \log \left ( \frac{ 2\varepsilon \sqrt{N}}{\gamma_2} \right) N^{\xi/2} }{C_\delta d_0 (1-c_0-\varepsilon^2) \gamma_2^2 \lambda_1^2} ,
\]
with \(\mathbb{P}_{\boldsymbol{X}_0}\)-probability at least \(1 - \tilde{\eta}_2\). 
\end{proof}
\end{proof}

We now prove Lemma~\ref{lem: kth step SGD p=2}. 

\begin{proof}[\textbf{Proof of Lemma~\ref{lem: kth step SGD p=2}}]
We provide the proof for the first step of the sequential elimination phenomenon, i.e., we consider \(k=2\) and assume that \(\boldsymbol{X}_0 \in E_1(\varepsilon, c_0)\). In particular, we have \(m_{11}(0) = \varepsilon\), and for all \((i,j) \neq (1,1)\), the initial values satisfy
\[
\gamma_-^{(ij)} N^{-\frac{1}{2} \xi_{ij}^{(1)}} \leq m_{ij}(0) \leq \gamma_+^{(ij)} N^{-\frac{1}{2} \delta_{ij}^{(1)}},
\]
for some constants \(\gamma_-^{(ij)} >0\) and \(\gamma_+^{(ij)}>0\) of order \(1\). A modification of Proposition~\ref{prop: inequalities SGD} is required, since the correlations at initialization are not in the scale \(\Theta (N^{-\frac{1}{2}})\). Let \(T_2 > 0\) be a time horizon to be chosen later. By assumption~\eqref{eq: step size p=2} on the step size \(\delta\), we can apply Lemma~\ref{lem: expansion correlations} which gives
\[
\left | m_{ij}(t) - \left ( m_{ij}(0) - \frac{\delta}{N} \sum_{\ell=1}^t  \langle \boldsymbol{v}_i,(\nabla_{\textnormal{St}}\Phi(\boldsymbol{X}_{\ell-1}))_j \rangle - \frac{\delta}{N} \sum_{\ell=1}^t  \langle \boldsymbol{v}_i,(\nabla_{\textnormal{St}}H^\ell(\boldsymbol{X}_{\ell-1}))_j \rangle \right ) \right | \leq \frac{\delta}{N} \sum_{\ell=1}^t  \vert a_i(\ell-1) \vert,
\]
for all \(t \le T_2\), with \(\mathbb{P}_{\boldsymbol{X}_0}\)-probability at least \(1 - T_2 e^{-cN}\). The noise term \(\frac{\delta}{N} \left | \sum_{\ell=1}^t   \langle \boldsymbol{v}_i,(\nabla_{\textnormal{St}}H^\ell(\boldsymbol{X}_{\ell-1}))_j \rangle  \right | \) can be upper-bounded according to Lemma~\ref{lem: noise martingale}, i.e., for every \(a > 0\),  
\[
\max_{t \le T_2} \frac{\delta}{N} \left | \sum_{\ell=1}^t   \langle \boldsymbol{v}_i,(\nabla_{\textnormal{St}}H^\ell(\boldsymbol{X}_{\ell-1}))_j \rangle  \right | \leq a,
\]
with \(\mathbb{P}_{\boldsymbol{X}_0}\)-probability at least \(1 - 2 e^{- a^2 N^2 / (\delta^2 T_2)}\). Moreover, for all \(t \leq T_2\), we have
\[
\frac{\delta}{N} \sum_{\ell=1}^t |a_i(\ell-1)| \leq \frac{\delta}{N} \sum_{\ell=1}^t h_i(\ell-1) + \frac{\delta^2 \alpha_2}{2N} \sum_{\ell=1}^t \sum_{k=1}^r  \frac{\delta}{N} | \langle \boldsymbol{v}_i,(\nabla_{\textnormal{St}}H^t(\boldsymbol{X}_{\ell-1}))_k \rangle | \mathbf{1}_{\{| \langle \boldsymbol{v}_i,(\nabla_{\textnormal{St}}H^t(\boldsymbol{X}_{\ell-1}))_k \rangle |> K_N\}},
\]
where \(K_N\) is a truncation sequence, and 
\[
h_i(t) = \frac{\delta \alpha_2}{2} \sum_{k=1}^r \left  ( |m_{ik}(t)| + \frac{\delta}{N}  | \langle \boldsymbol{v}_i,(\nabla_{\textnormal{St}}\Phi(\boldsymbol{X}_t))_k \rangle | + \frac{\delta K_N}{N}  \right).
\]
We now set the truncation value \(K_N\) so that the term \(\frac{\delta K_N}{N}\) is at most \(1\), which implies that \(K_N \lesssim \frac{\log(N)}{C_\delta d_0}N^{(1+\xi)/2}\). This ensures that for sufficiently large \(N\), \(h_i (t) \lesssim \delta \alpha_2 r\) for every \(i \in [r]\) and every \(t \leq T_2\). Therefore, we have
\[
\frac{\delta}{N} \sum_{\ell=1}^{T_2}   h_i(\ell-1) \lesssim \frac{d_0^2 N^{- \xi}}{\log(N)^2} T_2.
\]
For every time horizon \(T_2 \ll \frac{\log(N)^2 N^\xi}{d_0^2}\), the above sum can be set to any arbitrary fraction of the initial value \(\varepsilon\). Similarly, for every \(T_2 \ll \frac{\log(N)^2}{d_0^2} N^{ \frac{1}{1 + c_0}  -1 + \xi}\), the above sum can be set to any arbitrary fraction of the initial value \(N^{-\frac{1}{2} \delta_{ij}^{(1)}}\). Finally, applying Lemma~\ref{lem: higher truncation} to the second summand of \(|a_i(t)|\), we obtain for every \(a > 0\),
\[
\max_{t \le T_2} \frac{\delta^3 \alpha_2}{2N^2} \left \lvert \sum_{\ell = 1}^t \sum_{k=1}^r  \frac{\delta}{N} | \langle \boldsymbol{v}_i,(\nabla_{\textnormal{St}}H^\ell(\boldsymbol{X}_{\ell-1}))_k \rangle | \mathbf{1}_{\{| \langle \boldsymbol{v}_i,(\nabla_{\textnormal{St}}H^\ell(\boldsymbol{X}_{\ell-1}))_k \rangle |> K_N\}} \right \rvert \leq a
\]
with \(\mathbb{P}_{\boldsymbol{X}_0}\)-probability at least \(1 - \frac{\delta^3}{N^2a} T_2 e^{- \log(N)^2 N^{1 + \xi} / d_0^2}\). Combining the previous estimates with \(a = \frac{c_0 \varepsilon}{2} \), we find
\begin{equation} \label{eq: discrete population 11}
(1 - c_0) \varepsilon - \frac{\delta}{N} \sum_{\ell=1}^t  \langle \boldsymbol{v}_1, (\nabla_{\textnormal{St}} \Phi(\boldsymbol{X}_{\ell-1}))_1 \rangle  \leq m_{11}(t) \leq (1+c_0) \varepsilon - \frac{\delta}{N} \sum_{\ell=1}^t  \langle \boldsymbol{v}_1, (\nabla_{\textnormal{St}} \Phi(\boldsymbol{X}_{\ell-1}))_1 \rangle ,
\end{equation}
with \(\mathbb{P}_{\boldsymbol{X}_0}\)-probability at least \(1 - T_2 e^{-c_1 N} - 2 e^{- c_2 c_0^2 \varepsilon^2 N^2 / (\delta^2 T_2)} - \frac{ C \delta^3}{N^2 c_0 \varepsilon} T_2 e^{- c_3 \log(N)^2 N^{1 + \xi} / d_0^2}\). Similarly, for every \((i,j) \neq (1,1)\), using \(a = \frac{c_0 }{2} N^{-\frac{1}{2} \delta_{ij}^{(1)}} \), we obtain 
\begin{equation} \label{eq: discrete population ij}
(1 - c_0) N^{- \frac{1}{2} \xi_{ij}^{(1)}} - \frac{\delta}{N} \sum_{\ell=1}^t  \langle \boldsymbol{v}_i, (\nabla_{\textnormal{St}} \Phi(\boldsymbol{X}_{\ell-1}))_j \rangle  \leq m_{ij}(t) \leq (1+c_0) N^{- \frac{1}{2} \delta_{ij}^{(1)}} - \frac{\delta}{N} \sum_{\ell=1}^t  \langle \boldsymbol{v}_i, (\nabla_{\textnormal{St}} \Phi(\boldsymbol{X}_{\ell-1}))_j \rangle ,
\end{equation}
with \(\mathbb{P}_{\boldsymbol{X}_0}\)-probability at least \(1 - T_2 e^{-c_1 N} - 2 e^{- c_2 c_0^2 N^{2-\xi} /( \delta^2 T_2)}  - \frac{ K \delta^3}{c_0 N^{\frac{3}{2}+\frac{1}{2} \frac{1+c_0}{1-c_0} \frac{\lambda_2^2}{\lambda_1^2}}} T_2 e^{- c_3 \log(N)^2 N^{1 + \xi} / d_0^2}\). At this point, note that~\eqref{eq: discrete population 11} and~\eqref{eq: discrete population ij} only depend on the population term:
\[
- \langle \boldsymbol{v}_i, (\nabla_{\textnormal{St}} \Phi(\boldsymbol{X}_t))_j\rangle = 2 \sqrt{N} \lambda_i \lambda_j m_{ij}^{p-1}(t) - \sqrt{N} \sum_{1 \le k, \ell \le 2} \lambda_k (\lambda_j + \lambda_\ell) m_{i \ell}(t) m_{k j}(t) m_{k \ell}(t).
\]
The rest of the proof involves a detailed analysis of this term by distinguishing between the cases \(i=j=1\), \(i=1, j\neq 1\), \(i\neq 1, j=1\), and \(i \neq 1, j \neq 1\). For this, we refer the reader to the proof of Lemma 6.4 of~\cite{langevin}, where a detailed analysis using Langevin dynamics is carried out. We can apply the same strategy using~\eqref{eq: discrete population 11} and~\eqref{eq: discrete population ij} along with discrete difference inequalities, see e.g.\ Lemmas~\ref{lem: discrete growth} and~\ref{lem: discrete isotropic}. Following this analysis, we obtain 
\[
\mathcal{T}_{E_2(\varepsilon, c_0)} \lesssim \frac{\log \left ( \frac{\varepsilon }{1-c_0-\varepsilon^2}  N^{\xi_{22}^{(1)}/2} \right) \log \left (\varepsilon \sqrt{N} \right) N^{\xi/2}}{C_\delta d_0 (1-c_0) \lambda_2^2} ,
\]
which satisfies the prescribed condition on the time horizon. 
\end{proof}

%%%%%%%%%%%%%%%%%%%%%%%%%%%%%%%%%%%%%%%%%%%%%%%%%%%%%%%%%%%%%%%%%%%%%%%%%%%%%%%%%%%%%%%%%%%%%%%%%%%%%%%%%%%%%%%%%%%%%%%%%%%%%%%%%%%%%%%%%%%%%%%%%%%%%%%%%%
\section{Proof of subspace recovery}   \label{section: proof isotropic SGD} 
In this section, we consider \(p=2\) and assume that \(\lambda_1 = \cdots = \lambda_r \equiv \lambda >0\). The goal is to prove Theorem~\ref{thm: strong recovery isotropic SGD nonasymptotic}. We focus on the evolution of the eigenvalues of the matrix-valued function \(\boldsymbol{G}(\boldsymbol{X}) = \boldsymbol{V}^\top  \boldsymbol{X} \boldsymbol{X}^\top \boldsymbol{V}\). Let \(\theta_1(t), \ldots, \theta_r(t)\) denote the eigenvalues of \(\boldsymbol{G}_t =\boldsymbol{G} (\boldsymbol{X}_t)\), where \(\boldsymbol{X}_t\) denotes the online SGD output at time \(t\). Note that \(\boldsymbol{G}_t\) is symmetric and positive definite, thus \(\norm{\boldsymbol{G}_t} = \theta_{\max}(t)\). 

The first result of this section is about the weak recovery of all the eigenvalues. 

\begin{lem} \label{lem:weak_iso_min_p2}
For every \(\gamma_1 > 1 > \gamma_2 >0\), there exists \(\varepsilon \in (0, \frac{1}{2})\) such that, for all \(\varepsilon_0 < \varepsilon\), there exists a sequence \(d_0 = d_0(N)\) satisfying~\eqref{eq: sequence d_0} such that if
\[
\delta = C_\delta\frac{d_0\sqrt{N}}{\log(N)^2},
\]
for some constant \(C_\delta > 0\), then for \(N\) sufficiently large,
\[
\inf_{\boldsymbol{X}_0 \in \mathcal{C}_1'(\gamma_1,\gamma_2)} \mathbb{P}_{\boldsymbol{X}_0} \left(\varepsilon_0 \le \theta_{\min}\left(T \right) \le \theta_{\max}\left(T \right) \le \frac{10 \gamma_1}{\gamma_2} \varepsilon_0 \right) \geq 1 - \eta,
\]
where 
\[
T = \frac{\log(N)^2 \log\left(\frac{16 \varepsilon_0 N}{\gamma_2 }\right)}{2 C_\delta \lambda^2 d_0 },
\]
and 
\[
\eta = K_1 \frac{\log(N)^3}{d_0} e^{- c_1 N} + K_2  e^{-c_2 /d_0} + K_3 \frac{\log(N)^3}{d_0 N^{\frac{c_3}{d_0}}}.
\]
\end{lem}

\begin{proof}
Since \(\boldsymbol{X}_0 \in \mathcal{C}_1'(\gamma_1,\gamma_2)\), it holds that \(\frac{\gamma_2}{N} \leq \theta_{\min} (0) \leq \theta_{\max} (0) \leq \frac{\gamma_1}{N}\). Let \(T >0\) be a fixed time horizon to be determined later. According to Lemma~\ref{lem: iso_partial_inter} and using the fact that \( - \norm{\boldsymbol{S}_t} \boldsymbol{I}_r \preceq \boldsymbol{S}_t \preceq \norm{\boldsymbol{S}_t} \boldsymbol{I}_r\), we have the partial ordering
\begin{equation} \label{eq:useful_eq_iso}
- \norm{ \boldsymbol{S}_t} \boldsymbol{I}_r - \boldsymbol{A}_t \preceq \boldsymbol{G}_t - \left (\boldsymbol{G}_0  + \frac{4 \lambda^2 \delta}{\sqrt{N}} \sum_{\ell=1}^t    \boldsymbol{G}_{\ell-1} \left( \boldsymbol{I}_r - \boldsymbol{G}_{\ell-1} \right) \right) \preceq \norm{ \boldsymbol{S}_t} \boldsymbol{I}_r + \boldsymbol{A}_t,
\end{equation}
for every \(t \le T\), with probability at least \(1-T\exp(-c_1 N)-4 r^2 T\exp\left(-c_2\frac{K_N ^2}{\sigma^2}\right)\). Here, \(\boldsymbol{S}_t\) is given by~\eqref{eq: S_t} and \(\boldsymbol{A}_t\) by~\eqref{eq: A_t} which depends on the truncation sequence \((K_N)_{N \ge 1}\). Let \(\xi > 2\frac{\gamma_1}{\gamma_2}\) be a fixed constant. We define the hitting time \(\mathcal{T}_\xi\) as 
\[
\mathcal{T}_\xi  = \min \left \{t \in \N_0 \colon \theta_{\max}(t) \geq \xi \theta_{\min}(t) \right\}.
\]
Since the ratio \(\theta_{\max}(0) / \theta_{\min}(0) < \xi\), it follows that \(\mathcal{T}_\xi  > 0\). Furthermore, let \(\mathcal{T}_{\textnormal{bad}}\) denote the hitting time 
\[
\mathcal{T}_{\textnormal{bad}} = \inf \left \{t \in \N \colon c_0 \lambda_{\min} \left(\boldsymbol{G}_0 + \frac{4 \lambda^2 \delta}{\sqrt{N}}  \sum_{\ell=1}^t \boldsymbol{G}_{\ell-1} \left(\boldsymbol{I}_r - \boldsymbol{G}_{\ell-1} \right) \right) \leq \norm{\boldsymbol{S}_t} \right\},
\] 
where \(c_0 = c_0(N) \in (0,\frac{1}{6})\) is a control sequence to be determined later. We show that \(\mathcal{T}_{\textnormal{bad}} > 1\). Set \(K_N = C \sqrt{\frac{\log(N)}{d_0}}\) for a constant \(C\) that may depend on \(\lambda,\sigma^2,r,\gamma_1,\gamma_2\), and recall that
\[
\boldsymbol{S}_t = \frac{4\lambda \delta}{N}  \sum_{\ell=1}^t \mbox{sym}\left(\boldsymbol{V}^\top \Pi^{\textnormal{St}}_{\boldsymbol{X}_{\ell-1}} \left( \boldsymbol{W}^\ell \boldsymbol{X}_{\ell-1}\right) \boldsymbol{M}_{\ell-1}^\top \right) .
\]
At time \(t=1\), we have
\[
\norm{\mbox{sym}\left(\boldsymbol{V}^\top \Pi^{\textnormal{St}}_{\boldsymbol{X}_0} \left( \boldsymbol{W}^1 \boldsymbol{X}_0 \right) \boldsymbol{M}_0^\top \right)} \leq  2rK_N \norm{\boldsymbol{G}_0}^{1/2} \leq  2 rK_N \frac{\sqrt{\gamma}_1}{\sqrt{N}},
\]
where we used that \(\norm{\boldsymbol{V}^\top \Pi^{\textnormal{St}}_{\boldsymbol{X}_0} \left( \boldsymbol{W}^1 \boldsymbol{X}_0 \right)} \leq 2rK_N\) on the event defined by the above use of Lemma~\ref{lem: iso_partial_inter}. Using the prescribed values for \(\delta\) and \(K_N\), we obtain
\[
\norm{\boldsymbol{S}_1} = \frac{4 \lambda \delta}{N} \norm{\mbox{sym}\left(\boldsymbol{V}^\top \Pi^{\textnormal{St}}_{\boldsymbol{X}_0} \left( \boldsymbol{W}^1 \boldsymbol{X}_0 \right) \boldsymbol{M}_0^\top \right)} \leq C' \frac{\sqrt{d_0}}{N\log(N)^{3/2}} < c_0 \frac{\gamma_2}{N},
\]
with \(C' = C' (\lambda, \sigma^2, r , \gamma_1,\gamma_2)\), since \(d_0\) tends to \(0\) arbitrarily slowly. This shows that \(\mathcal{T}_{\textnormal{bad}}>1\).

We now introduce the following hitting times. For every \(\varepsilon \in (0,1)\), we define
\[
\begin{split}
\mathcal{T}_\varepsilon^{(\max, +)}  & = \inf \left\{t \in \N_0 \colon  \theta_{\max} (t) \geq \varepsilon \right\}, \\
\mathcal{T}_\varepsilon^{(\min, +)}  & = \inf \left\{t \in \N_0 \colon \theta_{\min} (t) \geq \varepsilon\right\}, \\
\mathcal{T}_\varepsilon^{(\min,-)} & = \inf \left\{t \in \N_0 \colon  \theta_{\min}(t) \leq \varepsilon \right\}.
\end{split}
\]
We introduce the threshold \(\varepsilon_N = \frac{\gamma_2}{2N}\) and consider \(\varepsilon \in (0, \frac{1}{2})\). We want to establish a comparison inequality for \(\theta_i (t)\) for every \(t \leq \mathcal{T}_{\varepsilon_N}^{(\min,-)} \wedge \mathcal{T}_\varepsilon^{(\min,+)} \wedge \mathcal{T}_{1/2}^{(\max,+)} \wedge \mathcal{T}_\xi \wedge \mathcal{T}_{\textnormal{bad}}\) starting from~\eqref{eq:useful_eq_iso}. In particular, we aim to find estimates for \(\norm{\boldsymbol{S}_t}\) and \(\boldsymbol{A}_t\). We first note that, on the considered time interval,
\[
\norm{\boldsymbol{S}_t}  \le c_0 \lambda_{\min} \left (\boldsymbol{G}_0  + \frac{4 \lambda^2 \delta}{\sqrt{N}}  \sum_{\ell=1}^t \boldsymbol{G}_{\ell-1} \left(\boldsymbol{I}_r - \boldsymbol{G}_{\ell-1} \right) \right),
\]
yielding
\begin{equation} \label{eq: norm S_t}
\norm{\boldsymbol{S}_t} \boldsymbol{I}_r \preceq c_0 \left (\boldsymbol{G}_0  + \frac{4 \lambda^2 \delta}{\sqrt{N}}  \sum_{\ell=1}^t \boldsymbol{G}_{\ell-1} \left(\boldsymbol{I}_r - \boldsymbol{G}_{\ell-1} \right) \right).
\end{equation}
Next, we provide an estimate for \(\boldsymbol{A}_t\) for every \(t \leq \mathcal{T}_{\varepsilon_N}^{(\min,-)} \wedge \mathcal{T}_\varepsilon^{(\min,+)} \wedge \mathcal{T}_{1/2}^{(\max,+)} \wedge \mathcal{T}_\xi \wedge \mathcal{T}_{\textnormal{bad}}\). According to~\eqref{eq: A_t}, \(\boldsymbol{A}_t = \sum_{\ell=1}^t a_{\ell-1} \boldsymbol{I}_r\), where
\[
a_{\ell-1} = 32 r^3 \lambda^2  \frac{\delta^2 K_N^2}{N^2} + (8 \lambda^4+\alpha_2) \frac{\delta^2}{N} \norm{\boldsymbol{G}_{\ell-1}} + 8 r \alpha_2 \lambda \frac{\delta^3 K_N}{N^2}\norm{\boldsymbol{G}_{\ell-1}}^{1/2} + 4 \alpha_2 \lambda^2 \frac{\delta^3}{N^{3/2}} \norm{\boldsymbol{G}_{\ell-1}}.
\]
Using \(\norm{\boldsymbol{G}_t} = \theta_{\max}(t)\) and since \(t \le \mathcal{T}_\xi\), we have 
\[
a_{\ell-1} \leq C_1 \left ( \frac{\delta^2 K_N^2}{N^2} + \sqrt{\xi} \frac{\delta^3 K_N}{N^2} \theta_{\min}(\ell-1)^{1/2} + \xi \frac{ \delta^2}{N} \left (1 + \frac{\delta}{\sqrt{N}}\right) \theta_{\min}(\ell-1) \right),
\]
where \(C_1 = C_1 (r, \lambda, \sigma^2)\). Using the prescribed values for \(\delta\) and \(K_N\), and the fact that \(\theta_{\min} (\ell-1) \ge \frac{\gamma_2}{2N}\) on the considered time interval, we obtain
\[ 
a_{\ell-1} \leq C_2 \frac{d_0}{\log(N)^3} \theta_{\min}(\ell-1) ,
\]
for some constant \(C_2 = C_2 (r, \lambda, \sigma^2, \gamma_2)\). Since $d_0$ is a sequence going to \(0\) arbitrarily slowly, we have
\[ 
a_{\ell-1} \leq 2 c_0 \lambda \frac{\delta}{\sqrt{N}}\theta_{\min}(\ell-1),
\]
which holds for any sequence $c_0(N) = c\log(N)^{-1}$ where $c$ may depend on the parameters of the problem (save for the dimension). This bound is useful since we note that
\begin{equation} \label{eq: lambda min cond 1}
c_0 \lambda_{\min} \left(\frac{4 \lambda \delta}{\sqrt{N}} \boldsymbol{G}_{t-1} \left(\boldsymbol{I}_r-\boldsymbol{G}_{t-1} \right)\right) = c_0 \frac{4 \lambda \delta}{\sqrt{N}} \min_{1 \le i \le r} \left \{ \theta_i(t-1) \left (1 - \theta_i (t-1) \right) \right \} \geq 2 c_0 \lambda\frac{\delta}{\sqrt{N}} \theta_{\min}(t-1),
\end{equation}
where we used the fact that \(t \le \mathcal{T}_{1/2}^{(\max,+)}\). Therefore, we get
\[
a_{\ell-1} \le c_0 \lambda_{\min} \left(\frac{4 \lambda \delta}{\sqrt{N}} \boldsymbol{G}_{t-1} \left( \boldsymbol{I}_r - \boldsymbol{G}_{t-1}\right)\right),
\]
which, in turn, implies that
\begin{equation} \label{eq: bound A_t}
\boldsymbol{A}_t = \sum_{\ell=1}^t a_{\ell-1} \boldsymbol{I}_r \preceq c_0 \sum_{\ell=1}^t \left(\frac{4 \lambda \delta}{\sqrt{N}} \boldsymbol{G}_{\ell-1}\left(\boldsymbol{I}_r-\boldsymbol{G}_{\ell-1} \right)\right).
\end{equation}
Inserting~\eqref{eq: norm S_t} and~\eqref{eq: bound A_t} into~\eqref{eq:useful_eq_iso}, we obtain the following ordering: 
\begin{equation} \label{eq:partial_ord_p2_iso}
\begin{split}
\boldsymbol{G}_t & \succeq (1-c_0 ) \boldsymbol{G}_0 + 4 \lambda^2 (1 - 2c_0 ) \frac{\delta}{\sqrt{N}} \sum_{\ell=1}^t \boldsymbol{G}_{\ell-1} \left(\boldsymbol{I}_r - \boldsymbol{G}_{\ell-1} \right),\\
\boldsymbol{G}_t &\preceq (1+c_0 ) \boldsymbol{G}_0 + 4 \lambda^2 (1+2c_0 ) \frac{\delta}{\sqrt{N}} \sum_{\ell=1}^t \boldsymbol{G}_{\ell-1} \left(\boldsymbol{I}_r - \boldsymbol{G}_{\ell-1}\right),
\end{split}
\end{equation}
for every \(t \leq \mathcal{T}_{\varepsilon_N}^{(\min,-)} \wedge \mathcal{T}_\varepsilon^{(\min,+)} \wedge \mathcal{T}_{1/2}^{(\max,+)}  \wedge \mathcal{T}_\xi \wedge \mathcal{T}_{\textnormal{bad}}\). This implies that
\begin{equation} \label{eq: partial ordering eigenvalues}
\begin{split}
\theta_{\min} (t) & \ge (1-c_0 )\frac{\gamma_2 }{N}+ 4 \lambda^2 (1-2c_0 )\frac{\delta}{\sqrt{N}}\sum_{\ell=1}^t \theta_{\min} (\ell-1) \left(1-\theta_{\min} (\ell-1)\right)  \\
\theta_{\max} (t) &\leq (1+c_0 )\frac{\gamma_1}{N}+ 4\lambda^2 (1+2c_0 ) \frac{\delta}{\sqrt{N}}\sum_{\ell=1}^t \theta_{\max} (\ell-1) \left(1-\theta_{\max} (\ell-1)\right),
\end{split}
\end{equation}
where we used the fact that \(\min_i \{\theta_i (1 - \theta_i)\} = \theta_{\min} (1 - \theta_{\min})\) and similarly \(\max_i \{\theta_i (1 - \theta_i)\} = \theta_{\max} (1 - \theta_{\max})\) on \([0, \mathcal{T}_{1/2}^{(\max, +)}]\), since the function \(x \mapsto x (1-x)\) is monotone for all \(0 \le x \le 1/2\). According to Lemma~\ref{lem: discrete isotropic} and since \(t \le \mathcal{T}_\varepsilon^{(\min, +)} \wedge \mathcal{T}_{1/2}^{(\max, +)}\), the eigenvalues \(\theta_1(t), \ldots, \theta_r(t)\) are increasing in \(t\). Consequently, we have \(\mathcal{T}_{\varepsilon_N}^{(\min,-)} > \mathcal{T}_\varepsilon^{(\min, +)} \wedge \mathcal{T}_{1/2}^{(\max, +)}\). Applying Lemma~\ref{lem: discrete isotropic} to~\eqref{eq: partial ordering eigenvalues} and performing some straightforward manipulations, we have the following two-sided comparison inequality:
\begin{equation} \label{eq:ineq_comp_p2_iso}
\frac{(1-c_0 )\gamma_2 }{4N}\exp\left(\frac{4(1- 2c_0 )\frac{\delta}{\sqrt{N}}\lambda^2t}{1+4(1-2c_0 )\frac{\delta}{\sqrt{N}}\lambda^2}\right) \leq \theta_{\min}(t) \le \theta_{\max} (t) \leq 2(1+c_0 )\frac{\gamma_1 }{N}\exp\left(4(1+ 2c_0 )\frac{\delta}{\sqrt{N}}\lambda^2t\right),
\end{equation} 
for every \(t \leq \mathcal{T}_\varepsilon^{(\min,+)} \wedge \mathcal{T}_{1/2}^{(\max,+)}  \wedge \mathcal{T}_\xi \wedge \mathcal{T}_{\textnormal{bad}}\).

To complete the proof, we need to show that \(\mathcal{T}^{(\min,+)}_\varepsilon \wedge \mathcal{T}^{(\max,+)}_{1/2} \leq  \mathcal{T}_{\textnormal{bad}} \wedge \mathcal{T}_\xi\). We first focus on \(\mathcal{T}_{\textnormal{bad}}\). The comparison inequality~\eqref{eq: partial ordering eigenvalues} implies that
\[
\frac{4 \lambda^2 \delta}{\sqrt{N}} \sum_{\ell=1}^t \theta_{\min}\left( \ell-1 \right) \leq \frac{1}{(1- 2c_0)(1-\varepsilon)} \theta_{\min} (t)  \le 3 \norm{\boldsymbol{G}_t },
\]
where we used the fact that \(\theta_{\min}(t) \le \theta_{\max}(t) = \norm{\boldsymbol{G}_t}, \varepsilon < 1/2\), and \(c_0 < \frac{1}{6}\). Then, since \(t \le \mathcal{T}_\xi\), it holds that 
\[
\frac{4 \lambda^2 \delta}{\sqrt{N}} \sum_{\ell=1}^t \norm{\boldsymbol{G}_{\ell - 1}} = \frac{4 \lambda^2 \delta}{\sqrt{N}} \sum_{\ell=1}^t \theta_{\max} (\ell - 1) \leq \frac{4 \lambda^2 \delta}{\sqrt{N}} \xi \sum_{\ell=1}^t \theta_{\min}(\ell-1) \leq 4 \xi \norm{\boldsymbol{G}_t }.
\]
Combining this estimate with~\eqref{eq:ineq_comp_p2_iso}, we obtain
\begin{equation} \label{eq:w_hat_t}
\begin{split}
\left (\frac{4 \lambda \delta}{N} \right)^2 \sum_{\ell=1}^t \norm{\boldsymbol{G}_{\ell-1}} \leq \frac{16 \xi \delta}{N^{3/2}} \norm{\boldsymbol{G}_t } \leq \frac{32 \xi \gamma_1 \delta}{N^{5/2}} (1+c_0) \exp\left(4(1+2c_0)\frac{\delta}{\sqrt{N}}\lambda^2t\right).
\end{split}
\end{equation}
Let \(\hat{w}(t)\) denote the right-hand side of~\eqref{eq:w_hat_t}. Let us now find a lower bound on the other quantity appearing in the definition of the hitting time \(\mathcal{T}_{\textnormal{bad}}\) on the same time interval. From the partial ordering~\eqref{eq:partial_ord_p2_iso}, it follows that
\[
\boldsymbol{G}_0 + \frac{4 \lambda^2 \delta}{\sqrt{N}} \sum_{\ell=1}^t \boldsymbol{G}_{\ell-1}(\boldsymbol{I}_r-\boldsymbol{G}_{\ell-1}) \succeq \frac{1}{1+ 2c_0}\boldsymbol{G}_t ,
\]
which, combined with~\eqref{eq:ineq_comp_p2_iso}, leads to 
\begin{equation} \label{eq:v_hat_t}
\begin{split}
\lambda_{\min} \left(\boldsymbol{G}_0 + \frac{4 \lambda^2 \delta}{\sqrt{N}} \sum_{\ell=1}^t \boldsymbol{G}_{\ell-1}(\boldsymbol{I}_r-\boldsymbol{G}_{\ell-1})\right) & \geq \frac{1}{1+ 2c_0} \theta_{\min}(t) \\
& \geq \frac{1}{1+2c_0 }\frac{(1-c_0 )\gamma_2}{4N}\exp\left(\frac{4(1-2c_0 )\frac{\delta}{\sqrt{N}} \lambda^2t}{1+4(1-2c_0 )\frac{\delta}{\sqrt{N}}\lambda^2}\right).
\end{split}
\end{equation}
Let \(\hat{v}(t)\) denote the function on the last line of~\eqref{eq:v_hat_t}. We then note that \(\mathcal{T}_{\textnormal{bad}}\) can be bounded below by the hitting time 
\[
\tilde{\mathcal{T}}_{\textnormal{bad}} = \min \left\{t \geq 1 \colon c_0 \hat{v}(t) \leq \norm{\boldsymbol{S}_t} \right\}.
\]
To control \(\tilde{\mathcal{T}}_{\textnormal{bad}} \), we first compare \(\sqrt{\hat{w}(t)}\) with \(c_0 \hat{v}(t)\) using~\eqref{eq:w_hat_t} and~\eqref{eq:v_hat_t}, i.e.,
\[
\begin{split}
\frac{\sqrt{\hat{w}(t)}}{c_0 \hat{v}(t)} 
&\leq \frac{\sqrt{32 \xi \gamma_1}}{\gamma_2} \frac{\sqrt{\delta}}{c_0 (1-c_0) N^{1/4}} \exp \left(\frac{2 \delta}{\sqrt{N}}\lambda^2t\left( 1 + 2c_0 - \frac{ 2 (1 - 2 c_0)}{1 + 4 (1-2 c_0) \frac{\delta}{\sqrt{N}}\lambda^2} \right)\right).
\end{split}
\]
Using the prescribed values for \(\delta\), we see that the argument of the exponential in the previous display is negative for every \(t \geq 0\), provided any \(c_0 < 1/6\), thus reaching the following bound:
\[
\frac{\sqrt{\hat{w}(t)}}{c_0 \hat{v}(t)} \leq \frac{\sqrt{32 \xi \gamma_1}}{\gamma_2} \frac{\sqrt{\delta}}{c_0 (1-c_0) N^{1/4}}.
\]
This implies that we can lower bound \(\tilde{\mathcal{T}}_{\textnormal{bad}}\) using the hitting time
\[
\tilde{\mathcal{T}}_{\textnormal{bad}}' = \inf \left\{t \geq 0 \colon \norm{\boldsymbol{S}_t } \geq a(t) \right\},
\]
where \(a(t) = \frac{\gamma_2 c_0  (1-c_0) N^{1/4}}{\sqrt{32 \xi \gamma_1 \delta}} \sqrt{\hat{w}(t)}\). Write \(b(t) = K a^2(t)\) with \(K = \frac{32 \xi \gamma_1}{\gamma_2^2 c_0^2 (1-c_0^2)} \frac{\delta}{\sqrt{N}}\). Applying Lemma~\ref{lem: matrix_martin_incr} gives 
\[
\mathbb{P} \left( \exists t \in [0,T] \colon \norm{\boldsymbol{S}_t } \geq a(t) \right) \leq 2r \exp \left( - \frac{c}{2K}\right) \le 2 r \exp\left(-\frac{c'}{d_0}\right),
\] 
where we used the prescribed values for \(c_0\) and \(\delta\), and the fact that \(\lambda_{\max}(\boldsymbol{Q}_t) = \left ( \frac{4 \lambda \delta}{N} \right)^2 \sum_{\ell=1}^t \norm{\boldsymbol{G}_{\ell-1}} \leq \hat{w}(t)\) with probability at least \(1 - T \exp(-c_1 N) - 4 r^2 T \exp\left(-c_2 \frac{K_N^2}{\sigma^2}\right)\) owing to Lemma \ref{lem: iso_partial_inter}. Since \(d_0 = d_0(N)\) tends to zero arbitrarily slowly, we have with high probability that \(\mathcal{T}_\varepsilon^{(\min,+)} \wedge \mathcal{T}_{1/2}^{(\max, +)} \wedge \mathcal{T}_\xi \le \mathcal{T}_{\textnormal{bad}}\).

Now, we need to control the hitting time \(\mathcal{T}_\xi\). To this end, we study the maximal deviation between \(\theta_{\min} (t) \) and \(\theta_{\max} (t)\) using~\eqref{eq:ineq_comp_p2_iso}. Once \(\theta_{\max}(t)\) exceeds \(\varepsilon\), it remains of order \(1\), while the remaining eigenvalues continue to rise. We can therefore upper bound the maximal deviation between the largest and smallest eigenvalue of \(\boldsymbol{G}(t)\) by computing an upper bound on the hitting time \(\mathcal{T}_\varepsilon^{(\min,+)}\) (which is also an upper bound on the hitting time \(\mathcal{T}^{(\max,+)}_\varepsilon\) since \(\mathcal{T}_\varepsilon^{(\max, +)} < \mathcal{T}_\varepsilon^{(\min,+)}\)) by solving
\[
\frac{(1-c_0 )\gamma_2 }{4N}\exp\left(\frac{4 (1 - 2c_0) \frac{\delta}{\sqrt{N}} \lambda^2t}{ 1 + 4 (1 - 2c_0 )\frac{\delta}{\sqrt{N}}\lambda^2} \right) = \varepsilon,
\]
which gives that 
\[
\mathcal{T}_\varepsilon^{(\min,+)} \le T_{u,\varepsilon}^{\min} \coloneqq \frac{\log\left(\frac{4\varepsilon N}{(1-c_0 )\gamma_2 }\right)}{\frac{4(1-2c_0 )\frac{\delta}{\sqrt{N}}\lambda^2}{1+4(1-2c_0)\frac{\delta}{\sqrt{N}}\lambda^2}}.
\]
By the comparison inequality~\eqref{eq:ineq_comp_p2_iso}, we obtain
\begin{equation} \label{eq: ratio max/min}
\begin{split}
\frac{\theta_{\max}( T_{u,\varepsilon}^{\min})}{\theta_{\min}( T_{u,\varepsilon}^{\min})} & \leq 8\frac{(1+c_0 )\gamma_1 }{(1-c_0 )\gamma_2} \exp \left( \frac{4 c_0 + 4(1-2c_0) (1+2c_0) \lambda^2 \frac{\delta}{\sqrt{N}}}{1-2c_0}\log\left(\frac{4\varepsilon N}{(1-c_0 )\gamma_2 }\right)\right) \\
& \le 8\frac{(1+c_0 )\gamma_1 }{(1-c_0 )\gamma_2} \left (\frac{4 \varepsilon N}{(1-c_0) \gamma_2} \right)^{4c_0},
\end{split}
\end{equation}
where we used the prescribed value for \(\delta\). By choosing the control sequence \(c_0(N)\) such as \(c_0 = c \log(N)^{-1}\) for some constant \(c\) which may depend on \(r, \lambda, \sigma^2, \gamma_1, \gamma_2\), we conclude that the maximal deviation between \(\theta_{\min}\) and \(\theta_{\max}\) is at most a constant given by the last line of~\eqref{eq: ratio max/min}. Choosing \(\xi\) to be this constant which is clearly larger than \(\gamma_1 / \gamma_2\), we deduce that \(\mathcal{T}_\xi > \mathcal{T}_\varepsilon^{(\min,+)} \wedge \mathcal{T}_{1/2}^{(\max,+)}\). Moreover, from~\eqref{eq: ratio max/min} we also have
\[
\begin{split}
\theta_{\max} ( T_{u,\varepsilon}^{\min}) & \leq \frac{2 (1+c_0) \gamma_1}{N} \exp \left ( \frac{1+2c_0}{1-2c_0} \log \left ( \frac{4 \varepsilon N}{(1-c_0) \gamma_2} \right ) \left(1 + 4 \lambda^2 (1-2c_0) \frac{\delta}{\sqrt{N}}\right)\right ) \\
& \lesssim  \frac{2 (1+c_0) \gamma_1}{N} \left ( \frac{4 \varepsilon N}{(1-c_0) \gamma_2} \right)^{\frac{1+2c_0}{1-2c_0}} \\
& \lesssim 8 \frac{(1+c_0) \gamma_1}{(1-c_0) \gamma_2} \varepsilon,
\end{split}
\]
where the above inequality holds for \(N\) sufficiently large. We can choose \(\varepsilon\) sufficiently small so that \(\theta_{\max} (\mathcal{T}_\varepsilon^{(\min,+)}) \le \theta_{\max} ( T_{u,\varepsilon}^{\min})  \leq 1/2\). This shows that \(\mathcal{T}_\varepsilon^{(\min, +)} \le \mathcal{T}_{1/2}^{(\max, +)}\). We complete the proof by choosing the time horizon \(T\) equal to \( T_{u,\varepsilon}^{\min}\).
\end{proof}

The following lemma shows that weak implies strong recovery. Theorem~\ref{thm: strong recovery isotropic SGD nonasymptotic} then follows straightforwardly by combining Lemmas~\ref{lem:weak_iso_min_p2} and~\ref{lem:strong_iso_p2_lmin} using the strong Markov property. 

\begin{lem} \label{lem:strong_iso_p2_lmin}
Assume that \(\delta = C_\delta d_0\sqrt{N} \log(N)^{-2}\),
where \(C_\delta >0\) is a constant and \(d_0 = d_0(N)\) is a sequence satisfying~\eqref{eq: sequence d_0}. Then, for every \( 0 < \varepsilon < \varepsilon' < \frac{1}{2}\) and \(N\) sufficiently large,
\[
\inf_{\varepsilon \le \theta_{\min}(0) \le \theta_{\max}(0)  \le \varepsilon'} \mathbb{P}_{\boldsymbol{X}_0} \left( \theta_{\min} (T) \geq 1-\varepsilon'\right) \geq 1 -\eta,
\]
where the time horizon \(T\) is given by 
\[
T = \frac{\log(N)^2 \log \left (\frac{1-\varepsilon/2}{\varepsilon'} \right)}{C_\delta d_0 \lambda^2 \varepsilon},
\]
and 
\[
\eta = K_1 \frac{\log(N)^2}{d_0} e^{-c_1 N} +K_2 \frac{\log(N)^2}{d_0 }\exp(-c_2N)+2r\exp\left(-c_3N\log(N)^2\right)
\]
\end{lem}

\begin{proof}
According to Lemma~\ref{lem: iso_partial_inter}, for a fixed time horizon \(T\) to be set later and for every \(t \leq T\), there exists a truncation sequence \((K_N)_{N \ge 1}\) and constants \(c_1, c_2>0\) such that
\begin{equation} \label{eq: partial ordering G_t}
\begin{split}
\boldsymbol{G}_t & \succeq \boldsymbol{G}_0  + \frac{4 \lambda^2 \delta}{\sqrt{N}} \sum_{\ell=1}^t    \boldsymbol{G}_{\ell-1} \left( \boldsymbol{I}_r - \boldsymbol{G}_{\ell-1} \right) + \boldsymbol{S}_t - \boldsymbol{A}_t, \\
\boldsymbol{G}_t & \preceq \boldsymbol{G}_0  + \frac{4 \lambda^2 \delta}{\sqrt{N}} \sum_{\ell=1}^t    \boldsymbol{G}_{\ell-1} \left( \boldsymbol{I}_r - \boldsymbol{G}_{\ell-1} \right) + \boldsymbol{S}_t + \boldsymbol{A}_t ,
\end{split}
\end{equation}
with probability at least \(1 - T e^{-c_1 N} - 4 r^2 T e^{-c_2 K_N^2 / \sigma^2}\). Fix \(\varepsilon' > \varepsilon > 0\) and consider the time horizon
\[
T = \frac{\log(N)^2 \log(\frac{1-\varepsilon/2}{\varepsilon'})}{C_\delta d_0 \lambda^2 \varepsilon}.
\]
Moreover, set \(K_N = C\sqrt{N}\) for a constant \(C\) that may depend on \(\lambda, \sigma^2, r, \gamma_1, \gamma_2\), and let \(c_0 \in (0,\frac{1}{2})\) independent on \(N\). We next aim to estimate \(\boldsymbol{S}_t\) and \(\boldsymbol{A}_t\), which are defined in~\eqref{eq: S_t} and~\eqref{eq: A_t}, respectively on the considered time interval. 

We begin with \(\boldsymbol{A}_t\). According to~\eqref{eq: A_t}, \(\boldsymbol{A}_t = \sum_{\ell=1}^t a_{\ell-1} \boldsymbol{I}_r\), where
\[
a_{\ell-1} = 32 r^3 \lambda^2  \frac{\delta^2 K_N^2}{N^2} + (8 \lambda^4+\alpha_2) \frac{\delta^2}{N} \norm{\boldsymbol{G}_{\ell-1}} + 8 r \alpha_2 \lambda \frac{\delta^3 K_N}{N^2}\norm{\boldsymbol{G}_{\ell-1}}^{1/2} + 4 \alpha_2 \lambda^2 \frac{\delta^3}{N^{3/2}} \norm{\boldsymbol{G}_{\ell-1}}.
\]
Since \(\norm{\boldsymbol{G}_t} \leq 1\), we have 
\[
a_{\ell-1}  \le C_1  \left( K_N^2\frac{\delta^2}{N^2}+\frac{\delta^2}{N}+K_N\frac{\delta^3}{N^2}+\frac{\delta^3}{N^{3/2}}\right) \leq  C_2 \frac{d_0^2}{\log(N)^4},
\]
where \(C_1 = C_1 (r, \lambda,\sigma^2)\) and \(C_2 = C_2 (r, \lambda, \sigma^2, \gamma_1, \gamma_2)\).
Moreover, for every \(t \le T\) it holds that, for \(N\) sufficiently large,
\[
\sum_{\ell=1}^t a_{\ell-1} \le C_2 \frac{d_0^2}{\log(N)^4} t \le C_2 \frac{d_{0)}}{\log(N)^2} \frac{\log(\frac{1-\varepsilon/2}{\varepsilon'})}{C_\delta \lambda^2 \varepsilon} \le \frac{c_0}{2} \varepsilon \le  \frac{c_0}{2}  \theta_{\min}(0),
\]
yielding
\[
\frac{c_0}{2}  \boldsymbol{G}_0 \succeq \sum_{\ell=1}^t a_{\ell-1} \boldsymbol{I}_r = \boldsymbol{A}_t.
\]
We now consider the martingale term \(\boldsymbol{S}_t\), which we upper bound using Lemma~\ref{lem: matrix_martin_incr} with time horizon \(T\). Using the notations of the aforementioned lemma, note that the partial sum process \(\boldsymbol{Q}_t\) deterministically verifies 
\[
\boldsymbol{Q}_t = \left(\frac{4\lambda \delta}{N}\right)^2\sum_{\ell=1}^{T}\norm{\boldsymbol{G}_{\ell-1}}\leq \left(\frac{4\lambda \delta}{N}\right)^2T,
\]
since $\norm{\boldsymbol{G}_{\ell-1}} \leq 1$ for all $\ell \geq 1$. Thus, choosing a constant thresholds $a(t) = \frac{c_0}{2}\varepsilon$ and $b(t) = \left(\frac{4\lambda \delta}{N}\right)^2T$, we obtain the bound 
\[
\mathbb{P}\left(\exists \, t \in [0,M] : \norm{\boldsymbol{S}}_t \geq \frac{c_0}{2}\varepsilon\right) \leq 2r\exp\left(-\frac{cc_0^2\varepsilon^2 N^2}{128\lambda^2\delta^2T}\right).
\]
Replacing $\delta$ and $T$ with their prescribed values, we find there exists a constant $c$ depending only on $r,\sigma,\lambda,\varepsilon,\varepsilon'$ such that 
\[
\mathbb{P}\left(\sup_{0 \leq t \leq T} \norm{\boldsymbol{S}}_t \geq \frac{c_0}{2}\varepsilon\right) \leq 2r\exp\left(-\frac{c c_0^2N\log(N)^2}{d_0}\right),
\]
so that, on the above event,
\[
\frac{c_0}{2}  \boldsymbol{G}_0  \succeq \sum_{\ell=1}^t s_{\ell-1} \boldsymbol{I}_r = \boldsymbol{S}_t.
\]
From~\eqref{eq: partial ordering G_t} we thus obtain
\begin{equation} \label{eq: partial ordering G_t 2}
\begin{split}
\boldsymbol{G}_t  &\succeq (1-c_0)  \boldsymbol{G}_0  + 4 \lambda^2 \frac{\delta}{\sqrt{N}}\sum_{\ell=1}^t    \boldsymbol{G}_{\ell-1} \left( \boldsymbol{I}_r - \boldsymbol{G}_{\ell-1} \right) , \\
\boldsymbol{G}_t  & \preceq (1 + c_0) \boldsymbol{G}_0  + 4 \lambda^2 \frac{\delta}{\sqrt{N}}\sum_{\ell=1}^t  \boldsymbol{G}_{\ell-1} \left( \boldsymbol{I}_r - \boldsymbol{G}_{\ell-1} \right) ,
\end{split}
\end{equation}
for every \(t \leq T\), with probability at least \(1 - T e^{-c_1 N} - 4 r^2 T e^{-c_2 K_N^2 / \sigma^2} -2r\exp\left(-\frac{cc_0^2N\log(N)^2}{d_0}\right)\). Note that, since the sequence $d_0(N)$ decreases to $0$, for any choice of $\varepsilon,\varepsilon'$, we may find \(N\) sufficiently large such that  the prefactor $cc_0^2$ is lower bounded by a constant $c_3(r,\sigma,\lambda)$.

Since \(\boldsymbol{G}_t (\boldsymbol{I}_r - \boldsymbol{G}_t) \succeq 0\) we also have the crude bound
\[
\boldsymbol{G}_t \succeq (1-c_0) \boldsymbol{G}_0,
\]
for every \(t \le T\). This implies that 
\[
\theta_{\min}(t) \ge (1-c_0) \theta_{\min}(0) > \frac{\varepsilon}{2},
\]
for every \(t \le T\), since \(\theta_{\min}(0)  > \varepsilon\) and \(c_0 < 1/2\).

Now define 
\[
\Delta_t \coloneqq r - \Tr (\boldsymbol{G}_t).
\]
According to the partial ordering~\eqref{eq: partial ordering G_t 2}, we obtain 
\[
\begin{split}
\Delta_t & \le r - (1-c_0) \Tr (\boldsymbol{G}_0) - 4 \lambda^2  \frac{\delta}{\sqrt{N}} \sum_{\ell=1}^t \sum_{i=1}^r \theta_i (\ell-1) (1 - \theta_i(\ell-1)) \\
& \le r (1 - (1-c_0) \theta_{\min}(0)) - 2 \lambda^2 \varepsilon \frac{\delta}{\sqrt{N}} \sum_{\ell=1}^t \Delta_{\ell-1} \\
& \le r (1 - (1-c_0) \varepsilon)   \left ( 1 - 2 \lambda^2  \varepsilon \frac{\delta}{\sqrt{N}} \right )^t,
\end{split}
\]
for every \(t \leq  \mathcal{T}_\Delta \wedge T\), where in the second line we used the fact that \(\theta_{\min}(\ell-1) \ge \varepsilon/2\) and in the last line we applied Lemma~\ref{lem: discrete growth}. Define the hitting time 
\[
\mathcal{T}_\Delta = \min \left \{ t  \colon \Delta_t \le r \varepsilon' \right \}.
\]
Since the right-hand side is decreasing in \(t\), we obtain
\begin{equation} \label{eq: T delta}
\mathcal{T}_\Delta \le \frac{\log \left ( \frac{\varepsilon'}{1 - (1-c_0) \varepsilon} \right)}{\log \left ( 1 - 2 \lambda^2 \varepsilon \frac{\delta}{\sqrt{N}} \right)} = \frac{\log \left ( \frac{1 - (1-c_0) \varepsilon}{\varepsilon'} \right)}{\log \left ( \frac{1}{1 - 2 \lambda^2 \varepsilon \frac{\delta}{\sqrt{N}}} \right)} \le \frac{\log \left ( \frac{1 - (1-c_0) \varepsilon}{\varepsilon'} \right)}{2 \lambda^2 \varepsilon \frac{\delta}{\sqrt{N}}},
\end{equation}
where the equality holds since \(\varepsilon' < 1 - (1-c_0) \varepsilon\) and the last inequality follows by \(\log \left ( \frac{1}{1-x} \right)\ge x\) for all \(x > 0\). By construction of \(T\), the right-hand side of~\eqref{eq: T delta} is upper bounded by \(T\). Therefore, with probability at least \(1 - T e^{-c_1 N} - 4 r^2 T e^{-c_2 K_N^2 / \sigma^2}-2r e^{-c_3 N \log(N)^2}\),
\[
\mathcal{T}_\Delta  \le T.
\]
This implies that with the same probability, 
\[
\theta_{\min}(\mathcal{T}_\Delta ) \ge 1 - \frac{\Delta_{\mathcal{T}_\Delta }}{r} \ge 1-\varepsilon',
\]
completing the proof.
\end{proof}

%%%%%%%%%%%%%%%%%%%%%%%%%%%%%%%%%%%%%%%%%%%%%%%%%%%%%%%%%%%%%%%%%%%%%%%%%%%%%%%%%%%%%%%%%
%%%%%%%%%%%%%%%%%%%%%%%%%%%%%%%%%%%%%%%%%%%%%%%%%%%%%%%%%
\appendix

\section{Comparison inequalities for discrete sequences}

In this appendix, we present comparison inequalities for discrete time dynamical systems.

\begin{lem}[\cite{arous2021online,abbe2023sgd}] \label{lem: discrete growth}
Let \(p\geq 2\) be an integer and \(a_1, a_2, b_1, b_2 > 0\) be positive constants such that \(a_1 \leq a_2\) and \(b_1 \leq b_2\). Consider a sequence \((u_t)_{t \in \N}\) satisfying for all \(t \in \N\),
\[
a_1 + b_1 \sum_{s = 0}^{t-1} u_s^{p-1} \leq u_t \leq a_2 + b_2 \sum_{s=0}^{t-1} u_s^{p-1}.
\]  
Then, 
\begin{enumerate}
\item[(a)] \textnormal{(Discrete Gronwall inequality)} if \(p=2\), for all \(t \in \N\) it holds that 
\[
a_1 \left(1 + b_1 \right)^t \leq u_t \leq a_2 \left(1 + b_2 \right)^t;
\]
\item[(b)] \textnormal{(Discrete Bihari-LaSalle inequality)} if \(p \geq 3\), for all \(t \in \N\) it holds that,  
\[
a_1 \left(1 - \frac{b_1}{(1+b_1u_{t-1}^{p-2})^{p-1}} a_1^{p-2} t \right)^{-\frac{1}{p-2}} \leq u_t \leq  a_2 \left (1 - b_2 a_2^{p-2} t \right )^{-\frac{1}{p-2}}.
\]
In particular, if \(u_t \ge 0\) for all \(t \in \N\), it holds that
\[
a_1 \left(1 - b_1 a_1^{p-2} t \right)^{-\frac{1}{p-2}} \leq u_t \leq  a_2 \left (1 - b_2 a_2^{p-2} t \right )^{-\frac{1}{p-2}}.
\]
\end{enumerate} 
\end{lem}

Lemma~\ref{lem: discrete growth} is proved in~\cite{arous2021online,abbe2023sgd}. The following lemma can be derived using analogous techniques. Since the proof mimics the arguments used for Lemma~\ref{lem: discrete growth}, we omit the details here.

\begin{lem} \label{lem: discrete isotropic}
Let \(a_1, a_2, b_1, b_2 >0\) be positive constants such that \(a_1 \leq a_2\) and \(b_1 \leq b_2\). Consider a sequence \((u_t)_{t \in \N}\) taking values in \([0,1]\) and satisfying for all \(t \in \N\),
\[
a_1 + b_1 \sum_{s = 0}^{t-1} u_s(1-u_{s}) \leq u_t \leq a_2 + b_2 \sum_{s=0}^{t-1} u_s(1-u_{s}).
\]
Then, for every \(0 \leq t \leq \min \left\{t \colon u_t \geq \frac{1}{2}\right\}\),
\[
\frac{\frac{a_1}{1 -a_1} e^{\frac{b_1}{1 + b_1} t}}{1 + \frac{a_1}{1-a_1}e^{\frac{b_1}{1+b_1}t}} \leq u_t \leq \frac{\frac{a_2}{1-a_2} e^{b_2t}}{1 + \frac{a_2}{1-a_2} e^{b_2t}}.
\]
\end{lem}

\section{Matrix concentration inequalities}

We provide a deviation inequality to control the operator norm of matrix martingale sequences. The next result is a straightforward modification of Theorem 2.3 from~\cite{tropp2011freedman}, which provides an extension of Freedman's inequality to matrix martingales. 

\begin{lem}  \label{lem: tropp time}
Let \((\Omega,\mathcal{F},\{\mathcal{F}_t\}_{t \in \N},\mathbb{P})\) be a filtered probability space. Consider an adapted sequence \(\{\boldsymbol{X}_t\}_{t \in \N}\) and a predictable sequence \(\{\boldsymbol{V}_t\}_{t \in \N}\) of \(d \times d\) self-adjoint matrices such that
\[
\log( \E \left[ \exp(\theta \boldsymbol{X}_t)\vert \mathcal{F}_{t-1}\right] \preceq g(\theta) \boldsymbol{V}_t \quad \textnormal{almost surely for each \(\theta >0\)},
\]
for some function \(g \colon (0,\infty) \to [0,\infty]\). Define the partial sum processes
\[
\boldsymbol{Y}_t = \sum_{\ell=1}^t \boldsymbol{X}_\ell \quad \textnormal{and} \quad \boldsymbol{W}_t = \sum_{\ell=1}^t \boldsymbol{V}_\ell.
\]
Let \(T\) be a fixed time horizon. Then, for every real-valued continuous functions \(a\) and \(b\) we have
\[
\mathbb{P}\left(\exists t \in [0,T] \colon \lambda_{\max} (\boldsymbol{Y}_t) \geq a(t) \enspace \textnormal{and} \enspace \lambda_{\max}(\boldsymbol{W}_t) \leq b(t)\right) \leq d \inf_{\theta >0} \sup_{t \in [0,T]} e^{-\theta a(t) + g(\theta) b(t)}.
\]
\end{lem}

\begin{proof}
Let \(\mathcal{E}\) denote the event 
\[
\mathcal{E} = \{\exists t \in [0,T] \colon \lambda_{\max}(\boldsymbol{Y}_t) \geq a(t) \enspace \textnormal{and} \enspace \lambda_{\max}(\boldsymbol{W}_t) \leq b(t)\}.
\]
Let \(\kappa\) denote the stopping time given by
\[
\kappa = \inf \left\{t \geq 0 \colon \lambda_{\max}(\boldsymbol{Y}_t) \geq a(t) \enspace \textnormal{and} \enspace \lambda_{\max}(\boldsymbol{W}_t) \leq b(t)\right\}.
\]
We observe that \(\kappa \leq T < \infty\) on the event \(\mathcal{E}\). Then, for every \(t \in [0,T]\) and every \(k \in [0,t]\), if \(\mathcal{E}_k^t\) denotes the event 
\[
\mathcal{E}_k^t = \{ \lambda_{\max}(\boldsymbol{Y}_k) \geq a(t) \enspace \textnormal{and} \enspace \lambda_{\max}(\boldsymbol{W}_k) \leq b(t)\},
\]
we have 
\[
\mathcal{E}= \bigcup_{t=0}^T \left( \bigcup_{k=0}^t \mathcal{E}_k^t \right).
\]
Now, for each fixed \(\theta >0\), consider the real-valued process \(\{S_t = S_t(\theta) \colon t \in \N\}\) given by
\[
S_t (\theta) = \textnormal{tr} \exp\left(\theta \boldsymbol{Y}_t-g(\theta)\boldsymbol{W}_t\right), 
\]
which is a real-valued positive supermartingale according to~\cite[Lemma 2.1]{tropp2011freedman}. For every \(\theta >0\), according to Lemma 2.2 from~\cite{tropp2011freedman}, we have on each event \(\bigcup_{k=0}^t \mathcal{E}_k^t\),
\[
S_t \geq e^{\theta a(t)-g(\theta)b(t)},
\]
so that on the event \(\mathcal{E}\), it holds that
\[
S_t \geq \inf_{t \in [0,T]} e^{\theta a(t)-g(\theta)b(t)}.
\]
We then conclude the proof as done for~\cite[Theorem 2.3]{tropp2011freedman} upon noting that \(\inf_{t \in [0,T]} e^{\theta a(t)-g(\theta)b(t)} = \sup_{t \in [0,T]} e^{-\theta a(t)+g(\theta)b(t)}\).
\end{proof}

\section{Stability lemmas for strong recovery} \label{app:stability_results}

The purpose of this appendix is to establish stability guarantees for the online SGD dynamics~\eqref{eq: online SGD} once the trajectories enter the well-aligned regions characterized in Theorems~\ref{thm: strong recovery online p>2 nonasymptotic}, \ref{thm: strong recovery online p=2 nonasymptotic}, and~\ref{thm: strong recovery isotropic SGD nonasymptotic}. We begin by analyzing the trajectories of the correlations \(\{m_{ij}\}\), and subsequently address subspace recovery.  

For every \(\varepsilon >0\), define
\[
\mathcal{S}_{\pi^\ast}(\varepsilon) = \left \{ \boldsymbol{X} \in \textnormal{St}(N,r) \colon  m_{i_k^\ast j_k^\ast}(\boldsymbol{X}) \geq 1 - \varepsilon  \enspace \textnormal{for all} \: k \in [r] \enspace \textnormal{and} \enspace \vert m_{ij}(\boldsymbol{X}) \vert \leq \sqrt{2\varepsilon} \enspace \textnormal{otherwise} \right \},
\]
where \(\pi^\ast \in S_r\) denotes the permutation of recovered spikes defined by \(\pi^\ast(j_k^\ast) = i_k^\ast\) for all \(k \in [r]\), and \(\{(i_k^\ast, j_k^\ast)\}_{k=1}^r\) denotes the greedy maximum selection of the initialization matrix \(\boldsymbol{I}_0\) defined in~\eqref{eq: initialization matrix}. In particular, the set of pairs \(\{(i_k^\ast,j_k^\ast)\}_{k=1}^r\) coincides (as a set) with \(\{(i,\pi^\ast(i))\}_{i=1}^r\). In what follows, we show that once the SGD iterates enter the region \(\mathcal{S}_{\pi^\ast}(\varepsilon)\), they remain within a slightly enlarged neighborhood \(\mathcal{S}_{\pi^\ast}(2\varepsilon)\) for \(M\) further iterations with overwhelming probability. This establishes that \(\mathcal{S}_{\pi^\ast}(\varepsilon)\) is a stable region for the online SGD dynamics~\eqref{eq: online SGD}.

\begin{lem} \label{lem: stability separated}
Assume that either $p\geq 3$ and $M,\delta$ satisfy the conditions of Theorem~\ref{thm: general recovery SGD p>2 asymptotic}, or that $p=2$ and $M,\delta$ satisfy the conditions of Theorem v. Then, there exists $\varepsilon_0 \in \left (0,\frac{1}{4} \right)$, depending only on $r$ and $\{\lambda_i\}_{i=1}^r$, such that for every $\varepsilon \in (0,\varepsilon_0)$ and \(N\) sufficiently large, 
\[
\inf_{\boldsymbol{X}_0 \in \mathcal{S}_{\pi^\ast}(\varepsilon)}\mathbb{P}_{\boldsymbol{X}_0}\left(\boldsymbol{X}_M \in \mathcal{S}_{\pi^\ast}(2\varepsilon)\right) \geq 1-\eta,
\]
where 
\[
\eta = M\exp(-c(p)N) + 2\exp\left(- c_1 \frac{\varepsilon^2 N\log(N)}{1152 \alpha_1^2}\right)+C\frac{6 \delta r\alpha_1\alpha_2}{N\log(N)\varepsilon}\exp\left(-c_2 \frac{N^2}{2\delta^2\alpha_1^2}\right).
\]
\end{lem}

\begin{proof}
By Lemma~\ref{lem: expansion correlations} with the time horizon $T=M$, we obtain for any $t \leq M$,
\[ 
\left | m_{ij}(t) - \left ( m_{ij}(0) - \frac{\delta}{N} \sum_{\ell=1}^t  \langle \boldsymbol{v}_i,(\nabla_{\textnormal{St}}\Phi(\boldsymbol{X}_{\ell-1}))_j \rangle - \frac{\delta}{N} \sum_{\ell=1}^t  \langle \boldsymbol{v}_i,(\nabla_{\textnormal{St}}H^\ell(\boldsymbol{X}_{\ell-1}))_j \rangle \right ) \right | \leq \frac{\delta}{N} \sum_{\ell=1}^t  \vert a_i(\ell-1) \vert,
\]
with \(\mathbb{P}_{\boldsymbol{X}_0}\)-probability at least \(1-M e^{-c(p) N}\), where
\begin{equation} \label{eq: a_i}
\vert a_i(t) \vert \leq \frac{\delta \alpha_2}{2} \sum_{k=1}^r \left( \vert m_{ik}(t) \vert + \frac{\delta}{N}\vert \langle \boldsymbol{v}_i, (\nabla_{\textnormal{St}}\Phi \left(\boldsymbol{X}_t \right))_k \rangle \vert + \frac{\delta}{N} \vert \langle \boldsymbol{v}_i, (\nabla_{\textnormal{St}}H^\ell \left( \boldsymbol{X}_t)_k \right) \rangle \vert \right).
\end{equation}
For some $c_0 \in (0,1)$ to be set later, we can apply Lemma~\ref{lem: noise martingale} with $a=c_0\varepsilon$ and obtain
\[
\sup_{\boldsymbol{X}_0 \in \, \mathcal{S}_{\pi^\ast}(\varepsilon)} \mathbb{P}_{\boldsymbol{X}_0} \left (\max_{t \leq M} \frac{\delta}{N}  \left |  \sum_{\ell=1}^t \langle \boldsymbol{v}_i, \left (\nabla_{\textnormal{St}}H^\ell(\boldsymbol{X}_{\ell-1}) \right )_j \rangle\right | \geq c_0\varepsilon \right ) \leq 2\exp\left(- c\frac{c_0^2\varepsilon^2 N^2}{32 \delta^2 \alpha_1^2  M}\right).
\]
Introducing a truncation level $K_N$ for the last term in the right-hand side of~\eqref{eq: a_i} and applying Lemma~\ref{lem: higher truncation}, we find 
\[
\max_{t \le M} \frac{\delta^3 \alpha_2}{2 N^2} \sum_{\ell=1}^t \sum_{k=1}^r \left | \langle \boldsymbol{v}_i , (\nabla_{\textnormal{St}} H^\ell(\boldsymbol{X}_{\ell-1}))_k \rangle \right | \mathbf{1}_{ \{| \langle \boldsymbol{v}_i , (\nabla_{\textnormal{St}} H^\ell(\boldsymbol{X}_{\ell-1}))_k \rangle | > K_N\}} \le c_0\varepsilon,
\]
with \(\mathbb{P}_{\boldsymbol{X}_0}\)-probability at least $1 - C \frac{\delta^3r\alpha_1\alpha_2}{N^2 c_0 \varepsilon} M \exp\left(-c\frac{K_N^2}{2\alpha_1^2}\right)$. Thus, with \(\mathbb{P}_{\boldsymbol{X}_0}\)-probability at least \(1 -  M e^{-c(p) N} - 2\exp\left(- c_1 \frac{c_0^2\varepsilon^2 N^2}{32 \delta^2 \alpha_1^2  M}\right) - C \frac{\delta^3r\alpha_1\alpha_2}{N^2 c_0 \varepsilon} M \exp\left(-c_2\frac{K_N^2}{2\alpha_1^2}\right)\), we have the following upper and lower bounds for the $\left\{m_{ij}\right\}_{1 \leq i,j \leq r}$, 
\[
\left \vert m_{ij}(t)-\left(m_{ij}(0)-\frac{\delta}{N}\sum_{\ell=1}^{t}\langle \boldsymbol{v}_i,(\nabla_{\textnormal{St}}\Phi(\boldsymbol{X}_{\ell-1}))_j \rangle\right) \right \vert \leq 2c_0\varepsilon+\frac{\delta}{N}\sum_{\ell=1}^{t}h_i(\ell-1),
\]
where 
\[
h_i(\ell-1) = \frac{\delta \alpha_2}{2} \sum_{k=1}^r \left( \vert m_{ik}(\ell-1) \vert + \frac{\delta}{N}\vert \langle \boldsymbol{v}_i, (\nabla_{\textnormal{St}}\Phi \left(\boldsymbol{X}_{\ell-1}\right))_k \rangle \vert +\frac{K_N \delta}{N} \right).
\]
Then, for any $t \leq M$,
\[
\frac{\delta}{N} \sum_{\ell=1}^t   h_i(\ell-1)  \leq \frac{\delta^2M\alpha_2r}{N}\left(1+\frac{\delta}{N}(\sqrt{N}+K_N)\right).
\]
Since, by assumption, \(\delta^2 M \ll N\log(N)^{-1}\), we may choose any sequence \(K_N\) such that \(\frac{\delta}{N}K_N = \mathcal{O}(1)\). Then, for \(N\) sufficiently large,  
\[
\frac{\delta}{N} \sum_{\ell=1}^t   h_i(\ell-1)  \leq \frac{\delta^2M\alpha_2r}{N}\left(1+\frac{\delta}{N}(\sqrt{N}+K_N)\right) < c_0 \varepsilon,
\]
so that for all \(t \le M\), 
\[
\left \vert m_{ij}(t)-\left(m_{ij}(0)-\frac{\delta}{N}\sum_{\ell=1}^{t}\langle \boldsymbol{v}_i,(\nabla_{\mathrm{St}}\Phi(\boldsymbol{X}_{\ell-1}))_j \rangle\right) \right \vert \leq 3c_0\varepsilon,
\]
with \(\mathbb{P}_{\boldsymbol{X}_0}\)-probability at least \(1 -  M e^{-c(p) N} - 2\exp\left(- c_1 \frac{c_0^2\varepsilon^2 N^2}{32 \delta^2 \alpha_1^2  M}\right) - C \frac{\delta^3r\alpha_1\alpha_2}{N^2 c_0 \varepsilon} M \exp\left(-c_2\frac{N^2}{2 \delta^2 \alpha_1^2}\right)\).
In particular, for every \(n \in [r]\), we have 
\begin{equation} \label{eq: corr}
m_{i_n^\ast j_n^\ast}(t) \ge m_{i_n^\ast j_n^\ast}(0) - \frac{\delta}{N}\sum_{\ell=1}^{t}\langle \boldsymbol{v}_i ,(\nabla_{\mathrm{St}}\Phi(\boldsymbol{X}_{\ell-1}))_j \rangle - 3 c_0 \varepsilon,
\end{equation}
with the same probability as above. We remark that for \(p \ge 3\), the assumption \(\delta^2 M \ll N \log(N)^{-1}\) is precisely the one stated in Theorem~\ref{thm: strong recovery online p=2 asymptotic}. For \(p=2\), Theorem~\ref{thm: strong recovery online p=2 asymptotic} requires the stronger condition \(\delta^2 M \ll N^{1 - \xi_0/2}\), where \(\xi_0 = 1 - \lambda_r^2 / \lambda_1^2 \in (0,1)\). Since \(N^{- \xi/2} = o (\log(N)^{-1})\), this also implies \(\delta^2 M \ll N \log(N)^{-1}\).

In what follows, we focus on the case \(p \ge 3\). The extension to \(p=2\) is obtained by straightforward adaptation. From Proposition~\ref{prop: inequalities SGD}, we have
\[
\begin{split}
- \langle \boldsymbol{v}_{i_n^\ast} ,(\nabla_{\textnormal{St}}\Phi(\boldsymbol{X}))_{j_n^\ast} \rangle & =  \sqrt{N} p \lambda_{i_n^\ast} \lambda_{j_n^\ast} m_{i_n^\ast j_n^\ast}^{p-1}(1 - m_{i_n^\ast j_n^\ast}^2) \\
& \quad -   \frac{\sqrt{N} p}{2}  \sum_{(k, \ell) \neq (i_n^\ast, j_n^\ast)} \lambda_k m_{i_n^\ast \ell} m_{k j_n^\ast} m_{k\ell} \left ( \lambda_{j_n^\ast} m_{k j_n^\ast}^{p-2} + \lambda_\ell m_{k \ell}^{p-2} \right).
\end{split}
\]
Let 
\[
\tau \coloneqq \mathcal{T}_{\mathcal{S}_{\pi^\ast} (2 \varepsilon)^\mathrm{c}} = \inf \{ t \in \N_0 \colon \boldsymbol{X}_t \notin \mathcal{S}_{\pi^\ast} (2 \varepsilon) \}.
\]
Then, for all \(t \le M \wedge \tau\), the drift can be bounded by
\[
\lambda_{i_n^\ast} \lambda_{j_n^\ast} m_{i_n^\ast j_n^\ast}^{p-1} (t) (1 - m_{i_n^\ast j_n^\ast}^2 (t)) \ge \lambda_r^2 (1 - 2 \varepsilon)^{p-1} (1 - m_{i_n^\ast j_n^\ast}^2(t) ) .
\]
We next consider the correction term and write 
\begin{equation} \label{eq: sum corrector}
\begin{split}
& \sum_{(k, \ell) \neq (i_n^\ast , j_n^\ast)} \lambda_k m_{i \ell} m_{kj} m_{k\ell} \left ( \lambda_j m_{k j}^{p-2} + \lambda_\ell m_{k \ell}^{p-2} \right) \\
& = \lambda_{j^\ast_n}  \sum_{(k, \ell) \neq (i^\ast_n,j^\ast_n)} \lambda_k m_{i^\ast_n \ell} m_{kj^\ast_n}^{p-1} m_{k\ell} + \sum_{(k, \ell) \neq (i^\ast_n,j^\ast_n)} \lambda_k \lambda_\ell m_{i^\ast_n \ell} m_{kj^\ast_n} m_{k\ell}^{p-1}.
\end{split}
\end{equation}
We first consider the first term in~\eqref{eq: sum corrector}. For any \(t \le M \wedge \tau\),
\begin{align*}
\lambda_{j^\ast_n}  \sum_{(k, \ell) \neq (i^\ast_n,j^\ast_n)} \lambda_k m_{i^\ast_n \ell}(t) m_{kj^\ast_n}^{p-1}(t) m_{k\ell} (t)&\leq r \lambda_{j^\ast_n} \max_{(k, \ell)\neq (i^\ast_n,j^\ast_n)} \{ \lambda_k | m_{i^\ast_n \ell} (t) m_{k \ell} (t) | \} \sum_{ k \neq i^\ast_n} |m_{kj^\ast_n}(t)|^{p-1} \\
& \le r \lambda_{j^\ast_n} \lambda_1 \max_{(k, \ell)\neq (i^\ast_n,j^\ast_n)} \{| m_{i^\ast_n \ell} (t) m_{k \ell}(t) |\} \sum_{k \neq i^\ast_n} m_{kj^\ast_n}^2(t) \\
& \leq 2 r \lambda_1^2 \sqrt{\varepsilon} (1-m_{i^\ast_n j^\ast_n}^2(t)),
\end{align*}
where we used \(\sum_{k=1}^r m_{k j_n^\ast}^2(t) \le1\) and \(\max_{(k, \ell) \neq (i^\ast_n,j^\ast_n)} \{ | m_{i^\ast_n \ell} (t) m_{k\ell}(t)|\} \leq 2 \sqrt{\varepsilon}\) since \(\boldsymbol{X}_t \in \mathcal{S}_{\pi^\ast}(2 \varepsilon)\). For the second term in~\eqref{eq: sum corrector}, we write
\[
\sum_{(k, \ell) \neq (i^\ast_n,j^\ast_n)} \lambda_k \lambda_\ell m_{i^\ast_n \ell} m_{k j_n^\ast} m_{k\ell}^{p-1} = \lambda_{j^\ast_n} m_{i^\ast_n j^\ast_n} \sum_{k \neq i^\ast_n}  \lambda_k m_{k j^\ast_n}^p +\sum_{k, \ell \neq j^\ast_n}\lambda_k \lambda_\ell m_{i^\ast_n \ell} m_{kj^\ast_n} m_{k \ell}^{p-1}.
\]
We note that, on the time interval \([0, M \wedge \tau]\), the first subterm can be bounded as
\[
\lambda_{j^\ast_n} m_{i^\ast_n j^\ast_n} (t) \sum_{k \neq i^\ast_n}  \lambda_k m_{k j^\ast_n}^p (t) \le 2 \lambda_1^2 \sqrt{\varepsilon} \sum_{k \neq i^\ast_n} |m_{k j^\ast_n}(t)|^{p-1} \le 2 \lambda_1^2 \sqrt{\varepsilon} \sum_{k \neq i^\ast_n} m_{k j^\ast_n}^2 (t) \le 2 \lambda_1^2 \sqrt{\varepsilon} (1-m^2_{i^\ast_n j^\ast_n}(t)).
\]
For the second subterm, using \(|m|^{p-1} \le |m|\) and Cauchy-Schwarz,
\begin{align*}
\sum_{k, \ell \neq j^\ast_n} \lambda_k \lambda_\ell m_{i^\ast_n \ell} (t) m_{kj^\ast_n} (t) m_{k\ell}^{p-1} (t) & \leq (r-1)\lambda_1^2 \max_{\ell \neq j^\ast_n} \left \{ |m_{i^\ast_n \ell}(t)| \left |  \sum_{k=1}^r m_{k j_n^\ast} (t) m_{k \ell} (t) \right | \right\} \\
& \le 2 (r-1) \lambda_1^2 \sqrt{\varepsilon} \max_{\ell \neq j^\ast_n} \left | \langle \boldsymbol{P}_{\boldsymbol{V}} \boldsymbol{x}_{j_n^\ast}, \boldsymbol{P}_{\boldsymbol{V}} \boldsymbol{x}_\ell \rangle \right |  \\
& \le 2(r-1) \lambda_1^2 \sqrt{\varepsilon}  \Vert \boldsymbol{P}^{\perp}_{\boldsymbol{V}}\boldsymbol{x}_{j^\ast_n}\Vert \max_{\ell \neq j^\ast_n} \Vert \boldsymbol{P}^{\perp}_{\boldsymbol{V}}\boldsymbol{x}_\ell \Vert,
\end{align*}
where \(\boldsymbol{P}_{\boldsymbol{V}}\) denotes the orthogonal projector onto the subspace spanned by the hidden directions $\boldsymbol{v}_1, \ldots,\boldsymbol{v}_r$, i.e., \(\boldsymbol{P}_{\boldsymbol{V}} \boldsymbol{x}_j = \sum_{k=1}^r m_{kj} \boldsymbol{v}_k\), and we used the fact that \(\langle \boldsymbol{P}_{\boldsymbol{V}} \boldsymbol{x}_{j_n^\ast}, \boldsymbol{P}_{\boldsymbol{V}} \boldsymbol{x}_\ell \rangle  = - \langle \boldsymbol{P}_{\boldsymbol{V}}^\perp \boldsymbol{x}_{j_n^\ast}, \boldsymbol{P}_{\boldsymbol{V}}^\perp \boldsymbol{x}_\ell \rangle \) since \(\langle \boldsymbol{x}_{j_n^\ast}, \boldsymbol{x}_\ell \rangle = 0\) for \(\ell \neq j_n^\ast\). Inside \(\mathcal{S}_{\pi^\ast}(2\varepsilon)\), $\Vert \boldsymbol{P}^{\perp}_{\boldsymbol{V}}\boldsymbol{x}_{j^\ast_n}\Vert  \leq \sqrt{1-m_{i^\ast_n j^\ast_n}^2}$ and $\Vert \boldsymbol{P}^{\perp}_{\boldsymbol{V}}\boldsymbol{x}_\ell \Vert \leq 2 \sqrt{\varepsilon}$. Hence, for any $t \leq M \wedge \tau$,
\[
\sum_{k, \ell \neq j^\ast_n} \lambda_k \lambda_\ell m_{i^\ast_n \ell} (t) m_{kj^\ast_n}(t) m_{k\ell}^{p-1} (t) \le 4 \varepsilon (r-1) \lambda_1^2\sqrt{1-m_{i^\ast_n j^\ast_n}^2 (t)}.
\]
Combining the above estimates into~\eqref{eq: sum corrector} gives
\[
\sum_{(k, \ell) \neq (i_n^\ast , j_n^\ast)} \lambda_k m_{i \ell} m_{kj} m_{k\ell} \left ( \lambda_j m_{k j}^{p-2} + \lambda_\ell m_{k \ell}^{p-2} \right) \le  2\sqrt{\varepsilon} (r+1) \lambda_1^2 (1-m_{i^\ast_n j^\ast_n}^2) + 4 \varepsilon (r-1) \lambda_1^2\sqrt{1-m_{i^\ast_n j^\ast_n}^2},
\]
so that for all \(t \le M \wedge \tau\),
\[
\begin{split}
& - \langle \boldsymbol{v}_{i^\ast_n},(\nabla_{\textnormal{St}} \Phi(\boldsymbol{X}_t))_{j^\ast_n} \rangle \\
& \geq \sqrt{N}p \left [ \left( \lambda_r^2(1-2\varepsilon)^{p-1}  - 2 \sqrt{\varepsilon}(r +1) \lambda_1^2\right) (1 - m_{i^\ast_n j^\ast_n}^2(t)) -  4 \varepsilon (r-1) \lambda_1^2\sqrt{1-m_{i^\ast_n j^\ast_n}^2(t)} \right ].
\end{split}
\]
To ensure that \( - \langle \boldsymbol{v}_{i^\ast_n},(\nabla_{\textnormal{St}} \Phi(\boldsymbol{X}_t))_{j^\ast_n} \rangle  > 0\) for all \(\boldsymbol{X} \in \mathcal{S}_{\pi^\ast}(2 \varepsilon)\), it suffices that
\[
\lambda_r^2 (1-2 \varepsilon)^{p-1} > 4 r \lambda_1^2 \sqrt{\varepsilon},
\]
since this implies \(\lambda_r^2(1-2\varepsilon)^{p-1}  - 2 \sqrt{\varepsilon}(r +1) \lambda_1^2 > 0\) and 
\[
\frac{4 \varepsilon (r-1) \lambda_1^2}{\lambda_r^2(1-2\varepsilon)^{p-1}  - 2 \sqrt{\varepsilon}(r +1) \lambda_1^2}  < 2 \sqrt{\varepsilon} = \max_{t \le \tau} \sqrt{1 - m_{i_n^\ast j_n^\ast}^2(t)}. 
\]
For \(\varepsilon < \frac{1}{4}\), we have \((1-2 \varepsilon)^{p-1}  \ge 2^{-(p-1)}\), so a convenient sufficient condition for \(\varepsilon\) is given by
\[
\varepsilon \le \left ( \frac{\lambda_r^2}{2^{p+1} r \lambda_1^2 }\right)^2.
\]
Under the above condition, it follows from~\eqref{eq: corr},
\[
m_{i^\ast_n j^\ast_n}(t)  \geq m_{i_n^\ast j_n^\ast} (0) - 3 c_0\varepsilon \ge 1- \varepsilon -  3 c_0 \varepsilon.
\]
Choosing $c_0 = \frac{1}{6}$, we obtain for all $0 \leq t \leq M \wedge \tau$,
\[
m_{i^\ast_n j^\ast_n}(t) \ge 1 - \frac{3}{2} \varepsilon.
\]
This bound holds for all \(n \in [r]\) with \(\mathbb{P}_{\boldsymbol{X}_0}\)-probability at least 
\[
1 -  M e^{-c(p) N} - 2\exp\left(- c_1 \frac{\varepsilon^2 N^2}{1152 \delta^2 \alpha_1^2  M}\right) - C \frac{6 \delta^3r\alpha_1\alpha_2}{N^2 \varepsilon} M \exp\left(-c_2\frac{N^2}{2 \delta^2\alpha_1^2}\right). 
\]
Consequently, the claimed lower bounds for the correlations $\{m_{i^\ast_k j^\ast_k}\}_{k \in [r]}$ hold with the above probability. The expression for the probability in the statement is written under the scaling assumption \(\delta^2 M \ll N log(N)^{-1}\). By orthogonality of the columns, this immediately implies the corresponding upper bounds for all remaining correlations. Finally, the same reasoning applies to the case \(p=2\), with only minor modifications.
\end{proof}

For the subspace recovery problem, we establish the stability at the level of the eigenvalues of the Gram matrix $\boldsymbol{G} = \boldsymbol{M}\boldsymbol{M}^{\top}$.

\begin{lem}
\label{lem:stability_isotropic}
Assume that $p=2$, \(\lambda_1 = \cdots = \lambda_r\), and that $M,\delta$ satisfy the conditions of Theorem~\ref{thm: strong recovery online p=2 asymptotic}. Then, for every $\varepsilon \in \left (0,\frac{1}{2}\right)$ and \(N\) sufficiently large,
\[
\inf_{\boldsymbol{X}_0 \colon \theta_{\min}(\boldsymbol{G}) \geq 1-\varepsilon} \mathbb{P}\left(\theta_{\min}(\boldsymbol{G}(M)) \geq 1-2\varepsilon \right) \geq 1-\eta,
\]
where $\eta$ is given by 
\[
\eta = 1-M\exp(-c_1N)-4rM\exp\left(-c_2\frac{N}{\sigma^2}\right)-2r\exp\left(-c_3N\log(N)^2\right).
\]
\end{lem}

\begin{proof}
The argument parallels that of Lemma~\ref{lem:strong_iso_p2_lmin}. We fix the time horizon \(M\) and choose the truncation sequence \(K_N = \sqrt{N}\). Since by assumption $\frac{\delta^2M}{N} \ll \frac{1}{\log(N)}$, for \(N\) sufficiently large we may guarantee that with probability at least $1-\eta$, the partial ordering given in~\eqref{eq: partial ordering G_t 2} holds uniformly for $0 \leq t \leq M \wedge \mathcal{T}^{(\min,-)}_{1-2\varepsilon}$. The conclusion follows upon noting that the resulting decreasing upper bound on $\Delta_t = r - \Tr (\boldsymbol{G}_t)$ is unconditional on any other hitting time, so that $\Delta_t$ is non-increasing.
\end{proof}
%%%%%%%%%%%%%%%%%%%%%%%%%%%%%%%%%%%%%%%%%%%%%%%%%%%%%%%%%%%%%%%%%%%%%%%%%%%%%%%%%%%%%%%%%%%%%%%%%
\printbibliography

@article {arous2020algorithmic,
    AUTHOR = {Ben Arous, G\'{e}rard and Gheissari, Reza and Jagannath, Aukosh},
     TITLE = {Algorithmic thresholds for tensor {PCA}},
    SORTTITLE = {B},
   JOURNAL = {Ann. Probab.},
  FJOURNAL = {The Annals of Probability},
    VOLUME = {48},
      YEAR = {2020},
    NUMBER = {4},
     PAGES = {2052--2087},
      ISSN = {0091-1798},
   MRCLASS = {62F10 (60H30 62F30 62H25 62M05 65C05 82C44 82D30)},
  MRNUMBER = {4124533},
       DOI = {10.1214/19-AOP1415},
       URL = {https://doi.org/10.1214/19-AOP1415},
}

@article {ben2020bounding,
    AUTHOR = {Ben Arous, G\'{e}rard and Gheissari, Reza and Jagannath, Aukosh},
     TITLE = {Bounding flows for spherical spin glass dynamics},
    SORTTITLE = {A},
   JOURNAL = {Comm. Math. Phys.},
  FJOURNAL = {Communications in Mathematical Physics},
    VOLUME = {373},
      YEAR = {2020},
    NUMBER = {3},
     PAGES = {1011--1048},
      ISSN = {0010-3616},
   MRCLASS = {82D30 (60H10)},
  MRNUMBER = {4061404},
MRREVIEWER = {Qiang Zeng},
       DOI = {10.1007/s00220-019-03649-4},
       URL = {https://doi.org/10.1007/s00220-019-03649-4},
}

@article {arous2021online,
    AUTHOR = {Ben Arous, G\'{e}rard and Gheissari, Reza and Jagannath, Aukosh},
     TITLE = {Online stochastic gradient descent on non-convex losses from high-dimensional inference},
   JOURNAL = {J. Mach. Learn. Res.},
  FJOURNAL = {Journal of Machine Learning Research},
    VOLUME = {22},
    NUMBER = {106},
      YEAR = {2021},
     PAGES = {1--51},
      ISSN = {1532-4435},
   MRCLASS = {62L20 (62H25 62J12 90C15 90C52)},
  MRNUMBER = {4279757},
    URL     = {http://jmlr.org/papers/v22/20-1288.html}
}

@book {horn2012matrix,
    AUTHOR = {Horn, Roger A. and Johnson, Charles R.},
     TITLE = {Matrix analysis},
   EDITION = {Second},
 PUBLISHER = {Cambridge University Press, Cambridge},
      YEAR = {2013},
     PAGES = {xviii+643},
      ISBN = {978-0-521-54823-6},
   MRCLASS = {15-01},
  MRNUMBER = {2978290},
MRREVIEWER = {Mohammad Sal Moslehian},
}

@book {chikuse2012statistics,
    AUTHOR = {Chikuse, Yasuko},
     TITLE = {Statistics on special manifolds},
    SERIES = {Lecture Notes in Statistics},
    VOLUME = {174},
 PUBLISHER = {Springer-Verlag, New York},
      YEAR = {2003},
     PAGES = {xxvi+399},
      ISBN = {0-387-00160-3},
   MRCLASS = {62-02 (62H11 62H12 62H15)},
  MRNUMBER = {1960435},
MRREVIEWER = {M. Hu\v{s}kov\'{a}},
       DOI = {10.1007/978-0-387-21540-2},
       URL = {https://doi.org/10.1007/978-0-387-21540-2},
}

@book {vershynin2018high,
    AUTHOR = {Vershynin, Roman},
     TITLE = {High-dimensional probability},
    SERIES = {Cambridge Series in Statistical and Probabilistic Mathematics},
    VOLUME = {47},
      NOTE = {An introduction with applications in data science,
              With a foreword by Sara van de Geer},
 PUBLISHER = {Cambridge University Press, Cambridge},
      YEAR = {2018},
     PAGES = {xiv+284},
      ISBN = {978-1-108-41519-4},
   MRCLASS = {60-01 (60B05 60B20 60E15 60Fxx 62H25)},
  MRNUMBER = {3837109},
MRREVIEWER = {Sasha Sodin},
       DOI = {10.1017/9781108231596},
       URL = {https://doi.org/10.1017/9781108231596},
}

@book {wainwright2019high,
    AUTHOR = {Wainwright, Martin J.},
     TITLE = {High-dimensional statistics},
    SERIES = {Cambridge Series in Statistical and Probabilistic Mathematics},
    VOLUME = {48},
      NOTE = {A non-asymptotic viewpoint},
 PUBLISHER = {Cambridge University Press, Cambridge},
      YEAR = {2019},
     PAGES = {xvii+552},
      ISBN = {978-1-108-49802-9},
   MRCLASS = {62-01 (60B20 60E15 60Fxx 62Gxx 62Hxx 62Jxx)},
  MRNUMBER = {3967104},
MRREVIEWER = {Pierre Alquier},
       DOI = {10.1017/9781108627771},
       URL = {https://doi.org/10.1017/9781108627771},
}

@inbook{mcdiarmid1998concentration,
author="McDiarmid, Colin",
editor="Habib, Michel and McDiarmid, Colin and Ramirez-Alfonsin, Jorge and Reed, Bruce",
title="Concentration",
bookTitle="Probabilistic Methods for Algorithmic Discrete Mathematics",
year="1998",
publisher="Springer Berlin Heidelberg",
address="Berlin, Heidelberg",
pages="195--248",
isbn="978-3-662-12788-9",
doi="10.1007/978-3-662-12788-9_6",
url="https://doi.org/10.1007/978-3-662-12788-9_6"
}

@inproceedings{abbe2023sgd,
  title = 	 {SGD learning on neural networks: leap complexity and saddle-to-saddle dynamics},
  author =       {Abbe, Emmanuel and Boix-Adser{\`a}, Enric and Misiakiewicz, Theodor},
  booktitle = 	 {Proceedings of Thirty Sixth Conference on Learning Theory},
  pages = 	 {2552--2623},
  year = 	 {2023},
  editor = 	 {Neu, Gergely and Rosasco, Lorenzo},
  volume = 	 {195},
  series = 	 {Proceedings of Machine Learning Research},
  publisher =    {PMLR},
  url = {https://proceedings.mlr.press/v195/abbe23a.html},
}

@inproceedings{MontanariRichard,
 author = {Richard, Emile and Montanari, Andrea},
 booktitle = {Advances in Neural Information Processing Systems},
 editor = {Z. Ghahramani and M. Welling and C. Cortes and N. Lawrence and K.Q. Weinberger},
 pages = {2897--2905},
 publisher = {Curran Associates, Inc.},
 title = {A statistical model for tensor {PCA}},
 url = {https://proceedings.neurips.cc/paper/2014/file/b5488aeff42889188d03c9895255cecc-Paper.pdf},
 volume = {27},
 year = {2014}
}

@article {Johnstone,
    AUTHOR = {Johnstone, Iain M.},
     TITLE = {On the distribution of the largest eigenvalue in principal components analysis},
   JOURNAL = {Ann. Statist.},
  FJOURNAL = {The Annals of Statistics},
    VOLUME = {29},
      YEAR = {2001},
    NUMBER = {2},
     PAGES = {295--327},
      ISSN = {0090-5364},
   MRCLASS = {62H25 (15A52 33C45 33E17 60F05)},
  MRNUMBER = {1863961},
       DOI = {10.1214/aos/1009210544},
       URL = {https://doi.org/10.1214/aos/1009210544},
}

@article {HuangPCA,
    AUTHOR = {Huang, Jiaoyang and Huang, Daniel Z. and Yang, Qing and Cheng, Guang},
     TITLE = {Power iteration for tensor {PCA}},
   JOURNAL = {J. Mach. Learn. Res.},
  FJOURNAL = {Journal of Machine Learning Research},
    VOLUME = {23},
    NUMBER = {128},
      YEAR = {2022},
     PAGES = {1--47},
      ISSN = {1532-4435},
   MRCLASS = {65F10 (62H25)},
  MRNUMBER = {4577080},
URL     = {http://jmlr.org/papers/v23/21-1290.html}
}

@article {BenArousUnfolding,
    AUTHOR = {Ben Arous, G\'{e}rard and Huang, Daniel Zhengyu and Huang, Jiaoyang},
     TITLE = {Long random matrices and tensor unfolding},
   JOURNAL = {Ann. Appl. Probab.},
  FJOURNAL = {The Annals of Applied Probability},
    VOLUME = {33},
      YEAR = {2023},
    NUMBER = {6B},
     PAGES = {5753--5780},
      ISSN = {1050-5164},
   MRCLASS = {60B20 (62H25)},
  MRNUMBER = {4677744},
       DOI = {10.1214/23-aap1958},
       URL = {https://doi.org/10.1214/23-aap1958},
}

@inproceedings{Hopkins15,
  title = 	 {Tensor principal component analysis via sum-of-square proofs},
  author = 	 {Hopkins, Samuel B. and Shi, Jonathan and Steurer, David},
  booktitle = 	 {Proceedings of The 28th Conference on Learning Theory},
  pages = 	 {956--1006},
  year = 	 {2015},
  editor = 	 {Grünwald, Peter and Hazan, Elad and Kale, Satyen},
  volume = 	 {40},
  series = 	 {Proceedings of Machine Learning Research},
  address = 	 {Paris, France},
  publisher =    {PMLR},
  url = 	 {https://proceedings.mlr.press/v40/Hopkins15.html},
}

@inproceedings{Hopkins16, 
author = {Hopkins, Samuel B. and Schramm, Tselil and Shi, Jonathan and Steurer, David}, 
title = {Fast spectral algorithms from sum-of-squares proofs: tensor decomposition and planted sparse vectors}, 
year = {2016}, 
isbn = {9781450341325}, 
publisher = {Association for Computing Machinery}, 
url = {https://doi.org/10.1145/2897518.2897529},
doi = {10.1145/2897518.2897529},
booktitle = {Proceedings of the Forty-Eighth Annual ACM Symposium on Theory of Computing}, pages = {178–191}, 
numpages = {14}, 
series = {STOC '16}
}

@inproceedings{Bandeira2017,
  author={Kim, Chiheon and Bandeira, Afonso S. and Goemans, Michel X.},
  booktitle={2017 International Conference on Sampling Theory and Applications (SampTA)}, 
  title={Community detection in hypergraphs, spiked tensor models, and Sum-of-Squares}, 
  year={2017},
  volume={},
  number={},
  pages={124-128},
  doi={10.1109/SAMPTA.2017.8024470}
}

@article {ben2022effective,
    AUTHOR = {Ben Arous, G\'{e}rard and Gheissari, Reza and Jagannath, Aukosh},
     TITLE = {High-dimensional limit theorems for {SGD}: effective dynamics
              and critical scaling},
   JOURNAL = {Comm. Pure Appl. Math.},
  FJOURNAL = {Communications on Pure and Applied Mathematics},
    VOLUME = {77},
      YEAR = {2024},
    NUMBER = {3},
     PAGES = {2030--2080},
      ISSN = {0010-3640},
   MRCLASS = {60F05 (62H12 62H30)},
  MRNUMBER = {4692886},
       DOI = {10.1002/cpa.22169},
       URL = {https://doi.org/10.1002/cpa.22169},
}

@article{arous2023eigenspace,
      title={{High-dimensional SGD aligns with emerging outlier eigenspaces}}, 
      author={Ben Arous, G\'{e}rard and Gheissari, Reza and Huang, Jiaoyang and Jagannath, Aukosh},
      year={2023},
      eprint={2310.03010},
      archivePrefix={arXiv},
      primaryClass={cs.LG}, 
}

@book {boumal2023,
    AUTHOR = {Boumal, Nicolas},
     TITLE = {An introduction to optimization on smooth manifolds},
 PUBLISHER = {Cambridge University Press, Cambridge},
      YEAR = {2023},
     PAGES = {xviii+338},
      ISBN = {978-1-009-16617-1},
   MRCLASS = {90-01 (49J27 53B21 58Exx 90Cxx)},
  MRNUMBER = {4533407},
}

@article{dandi2023learning,
      title={How Two-Layer Neural Networks Learn, One (Giant) Step at a Time}, 
      author={Yatin Dandi and Florent Krzakala and Bruno Loureiro and Luca Pesce and Ludovic Stephan},
      year={2023},
      eprint={2305.18270},
      archivePrefix={arXiv},
      primaryClass={stat.ML}, 
}

@article{bietti2023learning,
  title={On learning gaussian multi-index models with gradient flow},
  author={Bietti, Alberto and Bruna, Joan and Pillaud-Vivien, Loucas},
  year={2023},
eprint={2310.19793},
      archivePrefix={arXiv},
      primaryClass={stat.ML},
}

@inproceedings{bietti2022learning,
 author = {Bietti, Alberto and Bruna, Joan and Sanford, Clayton and Song, Min Jae},
 booktitle = {Advances in Neural Information Processing Systems},
 editor = {S. Koyejo and S. Mohamed and A. Agarwal and D. Belgrave and K. Cho and A. Oh},
 pages = {9768--9783},
 publisher = {Curran Associates, Inc.},
 title = {Learning single-index models with shallow neural networks},
 url = {https://proceedings.neurips.cc/paper_files/paper/2022/file/3fb6c52aeb11e09053c16eabee74dd7b-Paper-Conference.pdf},
 volume = {35},
 year = {2022}
}

@incollection {Weinkikuchi,
    AUTHOR = {Wein, Alexander S. and El Alaoui, Ahmed and Moore, Cristopher},
     TITLE = {The {K}ikuchi hierarchy and tensor {PCA}},
 BOOKTITLE = {2019 {IEEE} 60th {A}nnual {S}ymposium on {F}oundations of
              {C}omputer {S}cience},
     PAGES = {1446--1468},
 PUBLISHER = {IEEE Comput. Soc. Press, Los Alamitos, CA},
      YEAR = {2019},
   MRCLASS = {62H25 (15A23 15A69 68Q11 90C22)},
  MRNUMBER = {4228236},
}

@article {perry2018optimality,
    AUTHOR = {Perry, Amelia and Wein, Alexander S. and Bandeira, Afonso S.
              and Moitra, Ankur},
     TITLE = {Optimality and sub-optimality of {PCA} {I}: {S}piked random
              matrix models},
   JOURNAL = {Ann. Statist.},
  FJOURNAL = {The Annals of Statistics},
    VOLUME = {46},
      YEAR = {2018},
    NUMBER = {5},
     PAGES = {2416--2451},
      ISSN = {0090-5364},
   MRCLASS = {62H15 (60B20 62B15 62F15 62H25)},
  MRNUMBER = {3845022},
       DOI = {10.1214/17-AOS1625},
       URL = {https://doi.org/10.1214/17-AOS1625},
}

@article {bandeira2020,
    AUTHOR = {Perry, Amelia and Wein, Alexander S. and Bandeira, Afonso S.},
     TITLE = {Statistical limits of spiked tensor models},
   JOURNAL = {Ann. Inst. Henri Poincar\'{e} Probab. Stat.},
  FJOURNAL = {Annales de l'Institut Henri Poincar\'{e} Probabilit\'{e}s et
              Statistiques},
    VOLUME = {56},
      YEAR = {2020},
    NUMBER = {1},
     PAGES = {230--264},
      ISSN = {0246-0203},
   MRCLASS = {62F03 (62H25 68Q87)},
  MRNUMBER = {4058987},
       DOI = {10.1214/19-AIHP960},
       URL = {https://doi.org/10.1214/19-AIHP960},
}

@inproceedings{ba2022high,
 author = {Ba, Jimmy and Erdogdu, Murat A and Suzuki, Taiji and Wang, Zhichao and Wu, Denny and Yang, Greg},
 booktitle = {Advances in Neural Information Processing Systems},
 editor = {S. Koyejo and S. Mohamed and A. Agarwal and D. Belgrave and K. Cho and A. Oh},
 pages = {37932--37946},
 publisher = {Curran Associates, Inc.},
 title = {High-dimensional Asymptotics of Feature Learning: How One Gradient Step Improves the Representation},
 url = {https://proceedings.neurips.cc/paper_files/paper/2022/file/f7e7fabd73b3df96c54a320862afcb78-Paper-Conference.pdf},
 volume = {35},
 year = {2022}
}

@inproceedings{damian2022neural,
  title = 	 {Neural Networks can Learn Representations with Gradient Descent},
  author =       {Damian, Alexandru and Lee, Jason D. and Soltanolkotabi, Mahdi},
  booktitle = 	 {Proceedings of Thirty Fifth Conference on Learning Theory},
  pages = 	 {5413--5452},
  year = 	 {2022},
  editor = 	 {Loh, Po-Ling and Raginsky, Maxim},
  volume = 	 {178},
  series = 	 {Proceedings of Machine Learning Research},
  publisher =    {PMLR},
  url = 	 {https://proceedings.mlr.press/v178/damian22a.html},
}

@inproceedings{sarao2019,
 author = {Sarao Mannelli, Stefano and Biroli, Giulio and Cammarota, Chiara and Krzakala, Florent and Zdeborov\'{a}, Lenka},
 booktitle = {Advances in Neural Information Processing Systems},
 editor = {H. Wallach and H. Larochelle and A. Beygelzimer and F. d\textquotesingle Alch\'{e}-Buc and E. Fox and R. Garnett},
 pages = {},
 publisher = {Curran Associates, Inc.},
 title = {Who is Afraid of Big Bad Minima? Analysis of gradient-flow in spiked matrix-tensor models},
 url = {https://proceedings.neurips.cc/paper/2019/file/fbad540b2f3b5638a9be9aa6a4d8e450-Paper.pdf},
 volume = {32},
 year = {2019}
}

@inproceedings{sarao2019bis,
  title = 	 {Passed \& Spurious: Descent Algorithms and Local Minima in Spiked Matrix-Tensor Models},
  author =       {Sarao Mannelli, Stefano and Krzakala, Florent and Urbani, Pierfrancesco and Zdeborov\'{a}, Lenka},
  booktitle = 	 {Proceedings of the 36th International Conference on Machine Learning},
  pages = 	 {4333--4342},
  year = 	 {2019},
  editor = 	 {Chaudhuri, Kamalika and Salakhutdinov, Ruslan},
  volume = 	 {97},
  series = 	 {Proceedings of Machine Learning Research},
  publisher =    {PMLR},
  url = 	 {https://proceedings.mlr.press/v97/mannelli19a.html},
}

@article{saad1995line,
  title = {On-line learning in soft committee machines},
  author = {Saad, David and Solla, Sara A.},
  journal = {Phys. Rev. E},
  volume = {52},
  issue = {4},
  pages = {4225--4243},
  numpages = {0},
  year = {1995},
  publisher = {American Physical Society},
  doi = {10.1103/PhysRevE.52.4225},
  url = {https://link.aps.org/doi/10.1103/PhysRevE.52.4225}
}

@article{celentano2021high,
  title={The high-dimensional asymptotics of first order methods with random data},
  author={Celentano, Michael and Cheng, Chen and Montanari, Andrea},
  eprint={2112.07572},
 archivePrefix={arXiv},
  year={2021}
}

@article {gerbelot2024rigorous,
    AUTHOR = {Gerbelot, C\'{e}dric and Troiani, Emanuele and Mignacco, Francesca and Krzakala, Florent and Zdeborov\'{a}, Lenka},
     TITLE = {Rigorous {D}ynamical {M}ean-{F}ield {T}heory for {S}tochastic
              {G}radient {D}escent {M}ethods},
   JOURNAL = {SIAM J. Math. Data Sci.},
  FJOURNAL = {SIAM Journal on Mathematics of Data Science},
    VOLUME = {6},
      YEAR = {2024},
    NUMBER = {2},
     PAGES = {400--427},
   MRCLASS = {90 (60 62 68T05)},
  MRNUMBER = {4741502},
       DOI = {10.1137/23M1594388},
       URL = {https://doi.org/10.1137/23M1594388},
}

@book {engel2001statistical,
    AUTHOR = {Engel, A. and Van den Broeck, C.},
     TITLE = {Statistical mechanics of learning},
 PUBLISHER = {Cambridge University Press, Cambridge},
      YEAR = {2001},
     PAGES = {xii+329},
      ISBN = {0-521-77307-5},
   MRCLASS = {82C32 (68T05 92B20)},
  MRNUMBER = {1827275},
MRREVIEWER = {Min Ping Qian},
       DOI = {10.1017/CBO9781139164542},
       URL = {https://doi.org/10.1017/CBO9781139164542},
}

@article {tan2023online,
 AUTHOR = {Tan, Yan Shuo and Vershynin, Roman},
TITLE = {Online stochastic gradient descent with arbitrary              initialization solves non-smooth, non-convex phase retrieval},
   JOURNAL = {J. Mach. Learn. Res.},
  FJOURNAL = {Journal of Machine Learning Research},
    VOLUME = {24},
    NUMBER = {58},
      YEAR = {2023},
     PAGES = {1--47},
      ISSN = {1532-4435},
   MRCLASS = {65K10 (90C20 90C90)},
  MRNUMBER = {4582480},
    URL = {http://jmlr.org/papers/v24/20-902.html}
}

@article {goldt2019dynamics,
    AUTHOR = {Goldt, Sebastian and Advani, Madhu S. and Saxe, Andrew M. and Krzakala, Florent and Zdeborov\'{a}, Lenka},
     TITLE = {Dynamics of stochastic gradient descent for two-layer neural networks in the teacher-student setup},
   JOURNAL = {J. Stat. Mech. Theory Exp.},
  FJOURNAL = {Journal of Statistical Mechanics: Theory and Experiment},
      YEAR = {2020},
    NUMBER = {12},
     PAGES = {124010, 15},
   MRCLASS = {62M45},
  MRNUMBER = {4241363},
       DOI = {10.1088/1742-5468/abc61e},
       URL = {https://doi.org/10.1088/1742-5468/abc61e},
}

@article{mousavi2022neural,
      title={Neural Networks Efficiently Learn Low-Dimensional Representations with SGD}, 
      author={Alireza Mousavi-Hosseini and Sejun Park and Manuela Girotti and Ioannis Mitliagkas and Murat A. Erdogdu},
      year={2023},
      eprint={2209.14863},
      archivePrefix={arXiv},
}

@article {tropp2011freedman,
    AUTHOR = {Tropp, Joel A.},
     TITLE = {Freedman's inequality for matrix martingales},
   JOURNAL = {Electron. Commun. Probab.},
  FJOURNAL = {Electronic Communications in Probability},
    VOLUME = {16},
      YEAR = {2011},
     PAGES = {262--270},
   MRCLASS = {60B20 (60F10 60G42)},
  MRNUMBER = {2802042},
MRREVIEWER = {Michael Stolz},
       DOI = {10.1214/ECP.v16-1624},
       URL = {https://doi.org/10.1214/ECP.v16-1624},
}

@article {tropp2012user,
    AUTHOR = {Tropp, Joel A.},
     TITLE = {User-friendly tail bounds for sums of random matrices},
   JOURNAL = {Found. Comput. Math.},
  FJOURNAL = {Foundations of Computational Mathematics. The Journal of the
              Society for the Foundations of Computational Mathematics},
    VOLUME = {12},
      YEAR = {2012},
    NUMBER = {4},
     PAGES = {389--434},
      ISSN = {1615-3375},
   MRCLASS = {60B20 (60F10 60G42 60G50)},
  MRNUMBER = {2946459},
       DOI = {10.1007/s10208-011-9099-z},
       URL = {https://doi.org/10.1007/s10208-011-9099-z},
}

@inproceedings {bottou2010large,
    AUTHOR = {Bottou, L\'{e}on},
     TITLE = {Large-scale machine learning with stochastic gradient descent},
 BOOKTITLE = {Proceedings of {COMPSTAT}'2010},
     PAGES = {177--186},
 PUBLISHER = {Physica-Verlag/Springer, Heidelberg},
      YEAR = {2010},
   MRCLASS = {68T05 (62L20)},
  MRNUMBER = {3362066},
}

@article{lecun1998gradient,
  author={Le Cun, Y. and Bottou, L. and Bengio, Y. and Haffner, P.},
  journal={Proceedings of the IEEE}, 
  title={Gradient-based learning applied to document recognition}, 
  year={1998},
  volume={86},
  number={11},
  pages={2278-2324},
  doi={10.1109/5.726791}
}

@article {dieuleveut2020bridging,
    AUTHOR = {Dieuleveut, Aymeric and Durmus, Alain and Bach, Francis},
     TITLE = {Bridging the gap between constant step size stochastic
              gradient descent and {M}arkov chains},
   JOURNAL = {Ann. Statist.},
  FJOURNAL = {The Annals of Statistics},
    VOLUME = {48},
      YEAR = {2020},
    NUMBER = {3},
     PAGES = {1348--1382},
      ISSN = {0090-5364},
   MRCLASS = {62L20 (90C15 90C25 93E35)},
  MRNUMBER = {4124326},
       DOI = {10.1214/19-AOS1850},
       URL = {https://doi.org/10.1214/19-AOS1850},
}

@article{needell2014stochastic,
    title={{Stochastic gradient descent, weighted sampling, and the randomized Kaczmarz algorithm}},
    author={Needell, Deanna and Srebro, Nathan and Ward, Rachel},
    journal={Mathematical Programming},
    volume={155},
    year={2016},
    number = {1},
    pages = {549--573},
    doi = {10.1007/s10107-015-0864-7},
    url = {https://doi.org/10.1007/s10107-015-0864-7},
}

@inproceedings{dudeja2018learning,
  title = 	 {Learning Single-Index Models in Gaussian Space},
  author =       {Dudeja, Rishabh and Hsu, Daniel},
  booktitle = 	 {Proceedings of the 31st  Conference On Learning Theory},
  pages = 	 {1887--1930},
  year = 	 {2018},
  editor = 	 {Bubeck, Sébastien and Perchet, Vianney and Rigollet, Philippe},
  volume = 	 {75},
  series = 	 {Proceedings of Machine Learning Research},
  publisher =    {PMLR},
  url = 	 {https://proceedings.mlr.press/v75/dudeja18a.html},

}

@article{berthier2023learning,
  title={Learning time-scales in two-layers neural networks},
  author={Berthier, Rapha{\"e}l and Montanari, Andrea and Zhou, Kangjie},
  eprint={2303.00055},
archivePrefix={arXiv},
  year={2023}
}

@article{zdeborova2016statistical,
author = {Zdeborov\'{a}, Lenka and Krzakala, Florent},
title = {Statistical physics of inference: thresholds and algorithms},
journal = {Advances in Physics},
volume = {65},
number = {5},
pages = {453--552},
year = {2016},
publisher = {Taylor \& Francis},
doi = {10.1080/00018732.2016.1211393},
url = {https://doi.org/10.1080/00018732.2016.1211393},
}

@article {veiga2022phase,
    AUTHOR = {Veiga, Rodrigo and Stephan, Ludovic and Loureiro, Bruno and Krzakala, Florent and Zdeborov\'{a}, Lenka},
     TITLE = {Phase diagram of stochastic gradient descent in
              high-dimensional two-layer neural networks},
   JOURNAL = {J. Stat. Mech. Theory Exp.},
  FJOURNAL = {Journal of Statistical Mechanics: Theory and Experiment},
      YEAR = {2023},
    NUMBER = {11},
     PAGES = {Paper No. 114008, 28},
   MRCLASS = {82C32},
  MRNUMBER = {4725421},
       DOI = {10.1088/1742-5468/ad01b1},
       URL = {https://doi.org/10.1088/1742-5468/ad01b1},
}

@book {nesterov2018lectures,
    AUTHOR = {Nesterov, Yurii},
     TITLE = {Lectures on convex optimization},
    SERIES = {Springer Optimization and Its Applications},
    VOLUME = {137},
 PUBLISHER = {Springer, Cham},
      YEAR = {2018},
     PAGES = {xxiii+589},
      ISBN = {978-3-319-91577-7},
   MRCLASS = {90-01 (90C25)},
  MRNUMBER = {3839649},
MRREVIEWER = {Giorgio Giorgi},
       DOI = {10.1007/978-3-319-91578-4},
       URL = {https://doi.org/10.1007/978-3-319-91578-4},
}

@article {mignacco2020dynamical,
    AUTHOR = {Mignacco, Francesca and Krzakala, Florent and Urbani, Pierfrancesco and Zdeborov\'{a}, Lenka},
     TITLE = {Dynamical mean-field theory for stochastic gradient descent in
              {G}aussian mixture classification},
   JOURNAL = {J. Stat. Mech. Theory Exp.},
  FJOURNAL = {Journal of Statistical Mechanics: Theory and Experiment},
      YEAR = {2021},
    NUMBER = {12},
     PAGES = {Paper No. 124008, 23},
   MRCLASS = {62M45 (49N80 68T07)},
  MRNUMBER = {4412836},
       DOI = {10.1088/1742-5468/ac3a80},
       URL = {https://doi.org/10.1088/1742-5468/ac3a80},
}

@article {benarouscomplexity2019,
    AUTHOR = {Ben Arous, G\'{e}rard and Mei, Song and Montanari, Andrea and
              Nica, Mihai},
     TITLE = {The landscape of the spiked tensor model},
   JOURNAL = {Comm. Pure Appl. Math.},
  FJOURNAL = {Communications on Pure and Applied Mathematics},
    VOLUME = {72},
      YEAR = {2019},
    NUMBER = {11},
     PAGES = {2282--2330},
      ISSN = {0010-3640},
   MRCLASS = {62H12 (60B20)},
  MRNUMBER = {4011861},
       DOI = {10.1002/cpa.21861},
       URL = {https://doi.org/10.1002/cpa.21861},
}

@article{ros2019,
  title = {Complex Energy Landscapes in Spiked-Tensor and Simple Glassy Models: Ruggedness, Arrangements of Local Minima, and Phase Transitions},
  author = {Ros, Valentina and Ben Arous, Gerard and Biroli, Giulio and Cammarota, Chiara},
  journal = {Phys. Rev. X},
  volume = {9},
  issue = {1},
  pages = {011003},
  numpages = {42},
  year = {2019},
  publisher = {American Physical Society},
  doi = {10.1103/PhysRevX.9.011003},
  url = {https://link.aps.org/doi/10.1103/PhysRevX.9.011003}
}

@inproceedings{lesieur2017,
  author={Lesieur, Thibault and Miolane, Léo and Lelarge, Marc and Krzakala, Florent and Zdeborov\'{a}, Lenka},
  booktitle={2017 IEEE International Symposium on Information Theory (ISIT)}, 
  title={Statistical and computational phase transitions in spiked tensor estimation}, 
  year={2017},
  volume={},
  number={},
  pages={511-515},
  doi={10.1109/ISIT.2017.8006580}
}

@article {jagannath2020,
    AUTHOR = {Jagannath, Aukosh and Lopatto, Patrick and Miolane, L\'{e}o},
     TITLE = {Statistical thresholds for tensor {PCA}},
   JOURNAL = {Ann. Appl. Probab.},
  FJOURNAL = {The Annals of Applied Probability},
    VOLUME = {30},
      YEAR = {2020},
    NUMBER = {4},
     PAGES = {1910--1933},
      ISSN = {1050-5164},
   MRCLASS = {62F05 (62F10 62H25 82B26 82D30)},
  MRNUMBER = {4132641},
       DOI = {10.1214/19-AAP1547},
       URL = {https://doi.org/10.1214/19-AAP1547},
}

@article {chen2019,
    AUTHOR = {Chen, Wei-Kuo},
     TITLE = {Phase transition in the spiked random tensor with {R}ademacher
              prior},
   JOURNAL = {Ann. Statist.},
  FJOURNAL = {The Annals of Statistics},
    VOLUME = {47},
      YEAR = {2019},
    NUMBER = {5},
     PAGES = {2734--2756},
      ISSN = {0090-5364},
   MRCLASS = {60G35 (60K35 62F15 82B44 93E10)},
  MRNUMBER = {3988771},
       DOI = {10.1214/18-AOS1763},
       URL = {https://doi.org/10.1214/18-AOS1763},
}

@article {chen2021,
    AUTHOR = {Chen, Wei-Kuo and Handschy, Madeline and Lerman, Gilad},
     TITLE = {Phase transition in random tensors with multiple independent
              spikes},
   JOURNAL = {Ann. Appl. Probab.},
  FJOURNAL = {The Annals of Applied Probability},
    VOLUME = {31},
      YEAR = {2021},
    NUMBER = {4},
     PAGES = {1868--1913},
      ISSN = {1050-5164},
   MRCLASS = {62F05 (62F10 62H25 82B26 82D30)},
  MRNUMBER = {4312849},
       DOI = {10.1214/20-aap1636},
       URL = {https://doi.org/10.1214/20-aap1636},
}

@article{dominik2024,
      title={Matrix Concentration Inequalities and Free Probability II. Two-sided Bounds and Applications}, 
      author={Afonso S. Bandeira and Giorgio Cipolloni and Dominik Schröder and Ramon van Handel},
      year={2024},
      eprint={2406.11453},
      archivePrefix={arXiv},
      primaryClass={math.PR}, 
}

@article {robbins1951stochastic,
    AUTHOR = {Robbins, Herbert and Monro, Sutton},
     TITLE = {A stochastic approximation method},
   JOURNAL = {Ann. Math. Statistics},
  FJOURNAL = {Annals of Mathematical Statistics},
    VOLUME = {22},
      YEAR = {1951},
     PAGES = {400--407},
      ISSN = {0003-4851},
   MRCLASS = {62.0X},
  MRNUMBER = {42668},
MRREVIEWER = {R. P. Peterson},
       DOI = {10.1214/aoms/1177729586},
       URL = {https://doi.org/10.1214/aoms/1177729586},
}

@inproceedings{bottou2003large,
 author = {Bottou, L\'{e}on and Le Cun, Yann},
 booktitle = {Advances in Neural Information Processing Systems},
 editor = {S. Thrun and L. Saul and B. Sch\"{o}lkopf},
 pages = {},
 publisher = {MIT Press},
 title = {Large Scale Online Learning},
 url = {https://proceedings.neurips.cc/paper_files/paper/2003/file/9fb7b048c96d44a0337f049e0a61ff06-Paper.pdf},
 volume = {16},
 year = {2003}
}

@inproceedings{yehudai2020learning,
  title = 	 {Learning a Single Neuron with Gradient Methods},
  author =       {Yehudai, Gilad and Ohad, Shamir},
  booktitle = 	 {Proceedings of Thirty Third Conference on Learning Theory},
  pages = 	 {3756--3786},
  year = 	 {2020},
  editor = 	 {Abernethy, Jacob and Agarwal, Shivani},
  volume = 	 {125},
  series = 	 {Proceedings of Machine Learning Research},
  publisher =    {PMLR},
  url = 	 {https://proceedings.mlr.press/v125/yehudai20a.html},
}

@inproceedings{soltanolkotabi2017learning,
 author = {Soltanolkotabi, Mahdi},
 booktitle = {Advances in Neural Information Processing Systems},
 editor = {I. Guyon and U. Von Luxburg and S. Bengio and H. Wallach and R. Fergus and S. Vishwanathan and R. Garnett},
 pages = {},
 publisher = {Curran Associates, Inc.},
 title = {Learning ReLUs via Gradient Descent},
 url = {https://proceedings.neurips.cc/paper_files/paper/2017/file/e034fb6b66aacc1d48f445ddfb08da98-Paper.pdf},
 volume = {30},
 year = {2017}
}

@inproceedings{frei2020agnostic,
 author = {Frei, Spencer and Cao, Yuan and Gu, Quanquan},
 booktitle = {Advances in Neural Information Processing Systems},
 editor = {H. Larochelle and M. Ranzato and R. Hadsell and M.F. Balcan and H. Lin},
 pages = {5417--5428},
 publisher = {Curran Associates, Inc.},
 title = {Agnostic Learning of a Single Neuron with Gradient Descent},
 url = {https://proceedings.neurips.cc/paper_files/paper/2020/file/3a37abdeefe1dab1b30f7c5c7e581b93-Paper.pdf},
 volume = {33},
 year = {2020}
}

@inproceedings{abbe2022merged,
  title = 	 {The merged-staircase property: a necessary and nearly sufficient condition for SGD learning of sparse functions on two-layer neural networks},
  author =       {Abbe, Emmanuel and Boix-Adser{\`a}, Enric and Misiakiewicz, Theodor},
  booktitle = 	 {Proceedings of Thirty Fifth Conference on Learning Theory},
  pages = 	 {4782--4887},
  year = 	 {2022},
  editor = 	 {Loh, Po-Ling and Raginsky, Maxim},
  volume = 	 {178},
  series = 	 {Proceedings of Machine Learning Research},
  publisher =    {PMLR},
  url = 	 {https://proceedings.mlr.press/v178/abbe22a.html},
}

@article {lelarge2019fundamental,
    AUTHOR = {Lelarge, Marc and Miolane, L\'{e}o},
     TITLE = {Fundamental limits of symmetric low-rank matrix estimation},
   JOURNAL = {Probab. Theory Related Fields},
  FJOURNAL = {Probability Theory and Related Fields},
    VOLUME = {173},
      YEAR = {2019},
    NUMBER = {3-4},
     PAGES = {859--929},
      ISSN = {0178-8051},
   MRCLASS = {60B20 (62B10 62H12)},
  MRNUMBER = {3936148},
       DOI = {10.1007/s00440-018-0845-x},
       URL = {https://doi.org/10.1007/s00440-018-0845-x},
}

@article{gradientflow,
      title={Permutation recovery of spikes in noisy high-dimensional tensor estimation}, 
      author={Ben Arous, Gérard and Gerbelot, Cédric and Piccolo, Vanessa},
      year={2024},
      eprint={2412.14650},
      archivePrefix={arXiv}, 
}

@article{langevin,
      title={{Langevin dynamics for high-dimensional optimization: the case of multi-spiked tensor PCA}}, 
      author={Ben Arous, Gérard and Gerbelot, Cédric and Piccolo, Vanessa},
      year={2024},
      eprint={2408.06401},
      archivePrefix={arXiv}, 
}

@article{Wu24,
  author  = {Yuchen Wu and Kangjie Zhou},
  title   = {{Sharp analysis of power iteration for tensor PCA}},
  journal = {Journal of Machine Learning Research},
  year    = {2024},
  volume  = {25},
  number  = {195},
  pages   = {1--42},
  url     = {http://jmlr.org/papers/v25/24-0006.html}
}

@inproceedings{benarousneurips,
 author = {Ben Arous, Gérard and Gheissari, Reza and Jagannath, Aukosh},
 booktitle = {Advances in Neural Information Processing Systems},
 editor = {S. Koyejo and S. Mohamed and A. Agarwal and D. Belgrave and K. Cho and A. Oh},
 pages = {25349--25362},
 publisher = {Curran Associates, Inc.},
 title = {{High-dimensional limit theorems for SGD: Effective dynamics and critical scaling}},
 url = {https://proceedings.neurips.cc/paper_files/paper/2022/file/a224ff18cc99a71751aa2b79118604da-Paper-Conference.pdf},
 volume = {35},
 year = {2022}
}

@article {papyan2020,
    AUTHOR = {Papyan, Vardan and Han, X. Y. and Donoho, David L.},
     TITLE = {Prevalence of neural collapse during the terminal phase of
              deep learning training},
   JOURNAL = {Proc. Natl. Acad. Sci. USA},
  FJOURNAL = {Proceedings of the National Academy of Sciences of the United
              States of America},
    VOLUME = {117},
      YEAR = {2020},
    NUMBER = {40},
     PAGES = {24652--24663},
      ISSN = {0027-8424},
   MRCLASS = {68T07 (92B20 94A08)},
  MRNUMBER = {4250189},
       DOI = {10.1073/pnas.2015509117},
       URL = {https://doi.org/10.1073/pnas.2015509117},
}

@article {papyan2020traces,
    AUTHOR = {Papyan, Vardan},
     TITLE = {Traces of class/cross-class structure pervade deep learning spectra},
   JOURNAL = {J. Mach. Learn. Res.},
  FJOURNAL = {Journal of Machine Learning Research},
    VOLUME = {21},
      YEAR = {2020},
     PAGES = {1--64},
      ISSN = {1532-4435},
   MRCLASS = {62H30 (62J12 62M45 68T07)},
  MRNUMBER = {4213433},
    URL = {http://jmlr.org/papers/v21/20-933.html}
}

@inproceedings{Mahankali2023,
 author = {Mahankali, Arvind and Zhang, Haochen and Dong, Kefan and Glasgow, Margalit and Ma, Tengyu},
 booktitle = {Advances in Neural Information Processing Systems},
 editor = {A. Oh and T. Naumann and A. Globerson and K. Saenko and M. Hardt and S. Levine},
 pages = {57367--57480},
 publisher = {Curran Associates, Inc.},
 title = {Beyond NTK with Vanilla Gradient Descent: A Mean-Field Analysis of Neural Networks with Polynomial Width, Samples, and Time},
 url = {https://proceedings.neurips.cc/paper_files/paper/2023/file/b3748cdac932d91f0a51a37db90dec50-Paper-Conference.pdf},
 volume = {36},
 year = {2023}
}
\end{document}